\documentclass[twoside,11pt]{article}
\usepackage[abbrvbib, preprint]{jmlr2e}
\ShortHeadings{Covariate Shift in High-Dimensional Random Feature Regression}{Tripuraneni, Adlam, Pennington}
\firstpageno{1}

\usepackage[utf8]{inputenc} 
\usepackage[T1]{fontenc}    
\usepackage{url}            
\usepackage{booktabs}       
\usepackage{nicefrac}       
\usepackage{microtype}      
\usepackage[dvipsnames]{xcolor}     
\usepackage{titletoc}

\usepackage{graphicx}
\usepackage{subcaption}

\usepackage{amsmath,amsfonts}
\usepackage{bbm}
\usepackage{color}
\usepackage{mathrsfs}
\usepackage{enumitem}
\usepackage{bm}
\usepackage{multirow}
\usepackage{makecell}
\usepackage{caption}
\usepackage{thmtools}
\usepackage{times}
\usepackage{url}

\usepackage{hyperref}       
\usepackage[capitalize]{cleveref}

\makeatletter
\def\cleartheorem#1{%
    \expandafter\let\csname#1\endcsname\relax
    \expandafter\let\csname c@#1\endcsname\relax
}
\def\clearthms#1{ \@for\tname:=#1\do{\cleartheorem\tname} }
\makeatother

\crefname{section}{Sec.}{Secs.}
\crefname{lemma}{Lem.}{Lems.}
\crefname{proposition}{Prop.}{Props.}
\crefname{corollary}{Cor.}{Cors.}
\crefname{theorem}{Thm.}{Thms.}
\crefname{example}{Ex.}{Exs.}
\crefname{assumption}{Assump.}{Assumps.}
\crefname{equation}{Eq.}{Eqs.}
\crefname{definition}{Def.}{Defs.}
\crefname{appendix}{App.}{Apps.}

\newcommand{\tgamma}{\tilde{\gamma}}
\newcommand{\z}{\mathbf{z}}

\newcommand{\tlam}{\tilde{\lambda}}
\newcommand{\vect}[1]{\ensuremath{\mathbf{#1}}}
\newcommand{\x}{\vect{x}}
\newcommand{\vv}{\vect{v}}
\newcommand{\ww}{\vect{w}}
\newcommand{\vvs}{\vect{v}^*}
\newcommand{\mR}{\mathbb{R}}
\newcommand{\mE}{\mathbb{E}}
\newcommand{\tr}{\mathrm{tr}}
\newcommand{\ntr}{\bar{\mathrm{tr}}}

\newcommand{\Sigmatr}{\Sigma}
\newcommand{\Sigmate}{\Sigma^{*}}
\newcommand{\sigmaeps}{\sigma_{\epsilon}^2}
\newcommand{\I}{\cal{I}}
\newcommand{\Ishift}{\cal{I}^{*}}

\newcommand{\bias}[1]{B_{#1}}
\newcommand{\var}[1]{V_{#1}}

\newcommand{\corrs}{r}
\newcommand{\tcorrs}{\tilde{r}}

\newcommand{\err}[1]{E_{#1}}
\newcommand{\Err}[1]{E_{#1}} 
\newcommand{\slope}{\textsc{slope}}

\newcommand{\Etest}{\Err{\Sigma^*}}
\newcommand{\mutr}{\mu_\text{train}}
\newcommand{\munull}{\mu_\emptyset}

\newcommand{\fs}{\sigma}

\renewcommand{\cal}{\mathcal}

\DeclareMathOperator*{\argmin}{arg\,min}

\newcommand{\QE}[1]{Q^{\scriptstyle{E_{#1}}}}
\newcommand{\GE}[1]{G^{\scriptstyle{E_{#1}}}}
\newcommand{\RE}[1]{R^{\scriptstyle{E_{#1}}}}
\newcommand{\bQE}[1]{\bar{Q}^{\scriptstyle{E_{#1}}}}
\newcommand{\bGE}[1]{\bar{G}^{\scriptstyle{E_{#1}}}}
\newcommand{\bRE}[1]{\bar{R}^{\scriptstyle{E_{#1}}}}
\newcommand{\QK}{Q^{K^{-1}}}
\newcommand{\GK}{G^{K^{-1}}}
\newcommand{\RK}{R^{K^{-1}}}
\newcommand{\bGK}{\bar{G}^{K^{-1}}}
\newcommand{\bRK}{\bar{R}^{K^{-1}}}
\newcommand{\bQK}{\bar{Q}^{K^{-1}}}
\newcommand{\hK}{\hat{K}}
\newcommand{\id}{\mathrm{id}}

\newcommand{\etas}{\eta_*}
\newcommand{\zetas}{\zeta_*}
\newcommand{\rhos}{\rho_*}
\newcommand{\omegas}{\omega_*}
\newcommand{\gammaeff}{\gamma_{\text{eff}}}
\newcommand{\gammaopt}{\gamma^{\text{opt}}}

\newcommand{\tauone}{\tau}
\newcommand{\taub}{\bar{\tau}}

\newcommand{\muspike}{\mu^{\text{diatomic}}}

\newcommand{\cN}{\mathcal{N}}

\renewcommand{\cal}{\mathcal}

\newcommand{\e}{\varepsilon}
\newcommand{\la}{\lambda}
\renewcommand{\P}{\mathbb{P}}
\newcommand{\E}{\mathbb{E}}

\newcommand{\R}{\mathbb{R}}

\newcommand{\V}{\mathbb{V}}
\newcommand{\deq}{\mathrel{\mathop:}=}

\newcommand{\bfx}{\mathbf{x}}

\newcommand{\bfv}{\mathbf{v}}
\newcommand{\bfu}{\mathbf{u}}

\DeclareMathOperator{\diag}{diag}

\newcommand{\eq}[1]{\begin{equation}#1\end{equation}}

\newcommand{\al}[1]{\begin{align}#1\end{align}}

\newcommand{\p}[1]{({#1})}

\newcommand{\pa}[1]{\left({#1}\right)}

\newcommand{\qa}[1]{\left[{#1}\right]}

\newcommand{\h}[1]{\{{#1}\}}

\newcommand{\ha}[1]{\left\{{#1}\right\}}

\newcommand{\abs}[1]{\lvert #1 \rvert}

\newcommand{\absa}[1]{\left\lvert #1 \right\rvert}

\newcommand{\norm}[1]{\lVert #1 \rVert}

\newcommand{\norma}[1]{\left\lVert #1 \right\rVert}

\clearthms{theorem,claim,corollary,lemma,proposition,conjecture,definition,example,remark,assumption}
\newtheorem{theorem}{Theorem}[section]

\newtheorem{corollary}{Corollary}[section]
\newtheorem{lemma}{Lemma}[section]
\newtheorem{proposition}{Proposition}[section]

\newtheorem{definition}{Definition}[section]
\newtheorem{example}{Example}[section]
\newtheorem{remark}{Remark}[section]
\newtheorem{assumption}{Assumption}

\begin{document}
\title{Covariate Shift in High-Dimensional Random Feature Regression}

\author{\name Nilesh Tripuraneni\thanks{Equal contribution.} \email nilesh\_tripuraneni@berkeley.edu\\
        \addr Department of EECS, University of California, Berkeley
        \AND
        \name Ben Adlam\footnotemark[1] \email adlam@google.com\\
        \addr Brain Team, Google Research
        \AND
        \name Jeffrey Pennington\footnotemark[1]\email jpennin@google.com\\
        \addr Brain Team, Google Research}

\editor{Kevin Murphy and Bernhard Sch{\"o}lkopf}

\maketitle

\begin{abstract}
A significant obstacle in the development of robust machine learning models is \emph{covariate shift}, a form of distribution shift that occurs when the input distributions of the training and test sets differ while the conditional label distributions remain the same. Despite the prevalence of covariate shift in real-world applications, a theoretical understanding in the context of modern machine learning has remained lacking. In this work, we examine the exact high-dimensional asymptotics of random feature regression under covariate shift and present a precise characterization of the limiting test error, bias, and variance in this setting. Our results motivate a natural partial order over covariate shifts that provides a sufficient condition for determining when the shift will harm (or even help) test performance. We find that overparameterized models exhibit enhanced robustness to covariate shift, providing one of the first theoretical explanations for this intriguing phenomenon. Additionally, our analysis reveals an exact linear relationship between in-distribution and out-of-distribution generalization performance, offering an explanation for this surprising recent empirical observation.
\end{abstract}

\setcounter{tocdepth}{0}
\setcounter{footnote}{0}

\section{Introduction}
\label{sec:intro}

Theoretical justification for almost all machine learning methods relies upon the equality of the distributions from which the training and test data are drawn.
Nevertheless, in many real-world applications, this equality is  violated. Naturally-occurring distribution shift between the training data and the data encountered during deployment is the rule, not the exception \citep{koh2020wilds}. Even \emph{non-adversarial} changes in distributions can reveal the surprising fragility of modern machine learning models \citep{recht2018cifar, recht2019imagenet, pmlr-v119-miller20a,hendrycks2020faces,d2020underspecification,ovadia2019can}. 
Such shifts are distinct from adversarial examples, which require explicit poisoning attacks \citep{goodfellow2014explaining}; rather, they can result from mild corruptions, ranging from changes of camera angle or blur \citep{hendrycks2020faces}, to subtle, unintended changes in data acquisition procedures \citep{recht2019imagenet}. 
Moreover, this fragility limits the application of deep learning in certain safety-critical areas \citep{koh2020wilds}. Empirical studies of distribution shift have observed several intriguing phenomena, including linear trends between performance on shifted and unshifted test distributions \citep{recht2019imagenet, hendrycks2020faces, koh2020wilds}, dramatic degradation in calibration \citep{ovadia2019can}, and surprising spurious inductive biases \citep{d2020underspecification}. Theoretical understanding of why such patterns occur across a variety of real-world domains is scant. 
Even basic questions such as what makes a certain distribution shift likely to hurt (or help) a model's performance, and by how much, are not understood. One reason that these phenomena have eluded theoretical understanding is that even the qualitative behavior of different shifts can vary widely. In many cases, there is a strong coupling between the model and distribution shift; for example, in the context of character recognition, spatial translations typically hurt the performance of fully-connected networks but have minimal effect on convolutional models. On the other hand, there do exist certain types of shifts that affect most models in roughly the same way---a shift that reduces the frequency of ambiguous or contradictory examples would be of this type. Simultaneously capturing these model-dependent and model-independent characteristics of distribution shift is a significant obstacle in developing a realistic theoretical model.

Another reason that many phenomena related to distribution shift continue to lack satisfactory theoretical explanations is that the go-to formalism for studying generalization in classical models, namely uniform convergence theory \citep[see e.g.][]{wainwright2019high}, may be insufficient to explain the behavior of modern, deep learning methods even in the absence of distribution shift \citep{nagarajan2019uniform, yang2021exact}. Indeed, classical measures of model complexity, such as various norms of the parameters, have been found to lead to ambiguous conclusions \citep{neyshabur2017exploring}. To overcome these limitations, it can be fruitful to shift the focus from worst-case bounds for generic distributions to the average-case behavior for narrowly specified distributions. While this change in perspective sacrifices generality, it can produce more precise predictions, which we believe are necessary to fully capture the relevant phenomenology of distribution shift.

In this paper, we study a specific type of distribution shift called covariate shift, in which the distributions of the training and test covariates differ, while the conditional distribution of the labels given the covariates remains fixed. In particular, we focus on the setting in which the covariates are multivariate Gaussian and the targets are generated by a simple linear signal-plus-noise model. This narrow specification of the data distribution enables an asymptotically exact computation of the generalization error of random feature regression under covariate shift, which we perform using tools from random matrix theory. The results naturally lead to a  model-independent notion of shift strength and facilitate a nuanced analysis of the model-dependent behavior of covariate shift within the context of random feature regression.
\subsection{Summary of contributions}
\label{sec:contribs}

Here we provide a brief overview of the paper and highlight our main results. After discussing related work in \cref{sec:related}, we discuss the general random feature regression setup in \cref{sec:setup} and build some intuition for our approach in the simpler context of linear regression in \cref{sec:lin_reg}. In the remainder of the paper, we:

\begin{enumerate}[leftmargin=0.5cm]
    \item Compute the test error, bias, and variance of random feature regression for general multivariate Gaussian covariates under covariate shift in the high-dimensional limit (see \cref{sec:main_thms}). The results generalize prior work in two important ways: first, they provide one of the first asymptotically exact results for random feature regression with anisotropic covariates, even in the absence of covariate shift; second, they provide one of the first asymptotically exact results for covariate shift, even in the simple special case of linear regression;
    \item Provide a \emph{model-independent} partial order over covariate shifts that can indicate how a given shift will affect the performance of random feature regression (see \cref{def:hard}). When the shift does not alter the covariates' overall scale, the partial order is sufficient to determine whether the bias and test error will increase or decrease in response to the shift; otherwise, we characterize the \emph{model-dependent} conditions that are necessary to understand the effect of the shift (see \cref{sec:test_error_rf});
    \item Prove that overparameterization enhances robustness to covariate shift, and that the error, bias, and variance are nonincreasing functions of the number of excess parameters (see \cref{sec:b_v_rf}). Such benefits of overparameterization have been previously observed empirically for practical models and real-world datasets, but our results show that the behavior persists in simpler settings and may have a prosaic explanation stemming from the high-dimensionality of the problem;
    \item Deduce a linear relationship between in-distribution and out-of-distribution generalization performance, offering an explanation for this surprising recent empirical observation (see \cref{sec:linear_trends}). The relationship holds exactly in the ridgeless limit and is given parameterically in the degree of overparameterization, i.e. between models with varying numbers of excess parameters. The linearity approximately persists under non-zero regularization, but it is clearly violated in the underparameterized regime (see \cref{sec:nonlinear_trend});
    \item Prove that, in the limit of infinite overparameterization, the optimal regularization constant is independent of the covariate shift, and compute the associated optimal test error (see \cref{sec:optimal_gamma}). Owing to the implicit regularization effect of the nonlinear feature map, the optimal regularization constant can be negative. When the constant is optimized on the shifted test distribution, we observe the total error to be a nonincreasing function of the number of parameters, i.e. the characteristic double descent peak is eliminated; however, when it is tuned on the unshifted training distribution, a peak persists and the resulting error curve remains nonmonotonic;
    \item Examine how shifts in covariate scale can induce unexpected effects, such as nonmonotonicity of the bias with respect to the number of parameters and the amount of ridge regularization (see \cref{sec:scale}). By investigating this behavior in a special coordinate system, we develop an interpretation in terms of a coordinate-wise rescaling of the model's expected estimator for the target coefficients. Notably, this nonmonotonicity in the bias is only uncovered when the activation function is nonlinear and when the shift alters the scale of the covariates.
\end{enumerate}

\subsection{Related work}
\label{sec:related}
There is extensive literature focusing on the empirical analysis of distribution shift in all of its myriad forms, ranging from domain adaptation~\citep{sugiyama2007covariate,glorot2011domain,becker2013non, zhao2018adversarial,zhao2019learning, pmlr-v70-long17a, long2016unsupervised} to defenses against adversarial attacks~\citep{madry2017towards, schmidt2018adversarially} to distributionally robust optimization~\citep{sagawa2020distributionally, duchi2020distributionally, duchi2020learning}, among many others. Interestingly, for naturally occurring distribution shifts~\citep{koh2020wilds, hendrycks2020faces}, standard robustness interventions provide little to no protection \citep{recht2019imagenet, taori2020measuring}. Indeed, empirical risk minimization on clean, unshifted training data often performs better on out-of-distribution benchmarks than more sophisticated methods such as invariant risk minimization~\citep{koh2020wilds}. One of the most striking observations in the context of natural distribution shifts is that models become more robust as their in-distribution accuracy increases~\citep{recht2019imagenet, taori2020measuring, hendrycks2020faces, pmlr-v119-miller20a}. For example, \citet{recht2019imagenet} found that if a classifier's accuracy increases by 1.0\% on the unshifted CIFAR-10 test set, this tends to increase its accuracy by 1.7\% on the CIFAR-10.1 dataset (a dataset subject to natural distribution shift). Moreover, such linear trends between the unshifted and shifted measures of error have now been observed in a variety of contexts~\citep{recht2019imagenet, taori2020measuring, pmlr-v119-miller20a, pmlr-v139-miller21b, mania2021classifier}.

The number of theoretical works studying the impact of distribution shift on generalization is far smaller. The pioneering work of \citet{ben2007analysis} provides VC-dimension-based error bounds for classification that are augmented by a discrepancy measure between shifted source and target domains, while \citet{cortes2014domain} demonstrate a similar class of uniform convergence-based results in the setting of kernel regression.
\citet{zhao2019learning} show that learning domain-invariant features is insufficient to guarantee generalization when the class-conditional distributions of features may shift; while \citet{kumar2020understanding} provide non-vacuous margin-based bounds for self-training when the source domain gradually shifts toward the target domain. \citet{lei21a} compute the minimax risk for  linear models under distribution shift given access to labeled data from a source domain and unlabeled data from the target domain. \citet{mania2021classifier} introduce assumptions based on model similarity to help explain why classifiers exhibit linear trends between their accuracies on shifted and unshifted test sets.

Our technical tools build on a series of works that have studied the high-dimensional limit of the test error for a growing class of model families and data distributions. In the context of linear models, \citet{pmlr-v130-richards21b} and \citet{wu2020optimal} analyze ridge regression for general covariances and a general non-isotropic source condition on the parameters that generate the targets, extending earlier work studying minimum-norm interpolated least squares and ridge regression in the random design setting~\citep{belkin2019reconciling, dobriban2018high, hastie2019surprises}. The non-isotropy of source parameter effectively induces a shift on the bias term of this model, but the phenomenon is distinct from the covariate shifts we study here. 

Beyond linear regression, random feature models provide a rich but tractable class of models to gain further insight into generalization phenomena~\citep{adlam2020neural, adlam2020understanding, mei2019generalization, liao2020random}. These methods are of particular interest because of their connection to neural networks, with the number of random features corresponding to the network's width or model complexity \citep{neal1996priors, lee2017deep, jacot2018neural}, and because they serve as a practical method for data analysis in their own right~\citep{rahimi2007random, pmlr-v119-shankar20a}. In this context, \citet{yang2021exact} provide a precise characterization of the gaps between uniform convergence and the exact (asymptotic) test error as a function of the sample size and number of random features.

From the technical perspective, our results rely on the concept of Gaussian equivalents, specifically a linear-signal-plus-noise representation of the random feature matrix, an approach whose origin stems from \citet{Karoui2010TheSO} in the context of kernel random matrices and from \citet{pennington2019nonlinear,adlam2019random}; and \citet{peche2019note} for the Gram matrices relevant here. This linearization technique was further developed with applications to high-dimensional learning problems by a number of authors, including \citet{adlam2019random,mei2019generalization,goldt2020gaussian,Chang2021ProvableBO}, among many others. Our analytic techniques build upon these works and a series of recent results stemming from the literature on random matrix theory and free probability~\citep{pennington2018spectrum, pennington2019nonlinear, adlam2019random, adlam2020neural, louart2018random, peche2019note, far2006spectra, mingo2017free}.

An abbreviated version of this article focusing on the restricted setting in which the covariate shift preserves the overall scale may be found in \cite{tripuraneni2021over}.

\section{Preliminaries}
\label{sec:setup}
Here we introduce our problem setup, some important assumptions we make throughout the paper, and a simple example to illustrate our setup and assumptions.

\subsection{Problem setup and notation} As in prior work studying random feature regression~\citep{hastie2019surprises,mei2019generalization,adlam2020neural,adlam2020understanding}, we compute the test error in the high-dimensional, proportional asymptotics where the dataset size $m$, input feature dimension $n_0$, and hidden layer size $n_1$ all tend to infinity at the same rate, with the constants $\phi$ and $\psi$ defined as the limits of the ratios $n_0/m$ and $n_0/n_1$ respectively. We may think of $n_1$ and $m$ as functions of $n_0$, to write
\eq{
    \phi\deq\lim_{n_0\to\infty} \frac{n_0}{m(n_0)} \quad\text{and}\quad \psi\deq\lim_{n_0\to\infty} \frac{n_0}{n_1(n_0)}.
}
We further define the \emph{overparameterization ratio} as $\phi/\psi$, which is the limit of ${n_1}/{m}$ and characterizes the normalized complexity of the random feature model.

Interestingly, in this high-dimensional limit only linear functions of the data are learnable. A wide class of nonlinear teacher functions give rise to the same generalization behavior as linear labeling functions asymptotically \citep[for more details, see][]{mei2019generalization, adlam2020neural}. With this in mind, we consider the task of learning an unknown function from $m$ i.i.d. samples $(\x_i, y_i) \in \mathbb{R}^{n_0} \times \mathbb{R}$ for $i\in\h{1,\ldots, m}$, where the covariates are Gaussian, $\x_i \sim \cN(0, \Sigmatr)$ with positive definite covariance matrix $\Sigmatr$, and the labels are generated by a linear function parameterized by $\beta \in \mathbb{R}^{n_0}$, drawn from $\cN\pa{0,I_{n_0}}$. In particular,
\eq{\label{eq_data_dist}
    y(\x_i) = \beta^\top \x_i / \sqrt{n_0} + \e_i ,
}
where $\e_i \sim \cN(0, \sigma_{\e}^2)$ is additive label noise on the training points. We note that assuming the target vector is random and isotropic is common in the high-dimensional regression literature (see e.g. \citet{dobriban2018high}) as it decouples the geometry of $\Sigma$ and $\beta$, which would otherwise introduce additional complications (though see \citet{hastie2019surprises,wu2020optimal}; and \citet{mel2021theory} for an analysis of these complications in the context of linear regression without distribution shift). 

We study the class of prediction models defined by kernel ridge regression using unstructured random feature maps \citep{rahimi2007random}. The random features are given by a single-layer neural network with independent random weights. Given training data $X = [\x_1,\ldots,\x_m]$ and a prospective test point $\x$, the random features embeddings of the training and test data are given by
\eq{
\label{eqn:rf_embedding}
  F \deq \fs(W X/\sqrt{n_0})\,\quad\text{and}\quad f \deq \fs(W \x/\sqrt{n_0})\,,
}
for a random weight matrix $W \in \mathbb{R}^{n_1\times n_0}$ with i.i.d.  standard Gaussian entries and an activation function $\sigma : \mR \to \mR$ applied elementwise. The induced kernel is
\eq{
\label{eq_K2}
    K(\x_1,\x_2) \deq \frac{1}{n_1} \fs(W \x_1/\sqrt{n_0})^\top\fs(W \x_2/\sqrt{n_0})\,,
}
and the model's predictions are given by $\hat{y}(\x) = YK^{-1}K_\x$, where $Y\deq[y(\x_1),\ldots,y(\x_m)]$, $K \deq K(X,X) + \gamma I_m$, $K_\x \deq K(X, \x)$, and $\gamma$ is a ridge regularization constant.\footnote{We overload the definition of $K$ to include the additive regularization whenever no arguments are present. Also, if $\gamma = 0$ and $K$ is not full-rank, $K^{-1}$ should be understood as the Moore–Penrose pseudoinverse.} Owing to the implicit regularization effect of the nonlinear feature maps~\citep{bartlett2021deep}, in low noise settings the optimal value of $\gamma$ can sometimes be negative~\citep{kobak2020optimal}, as we discuss in~\cref{sec:optimal_gamma}. Otherwise, for simplicity, throughout the rest of the paper we make the standard assumption that $\gamma\ge 0$.

Our central object of study is the expected test error for a datapoint $\x \sim \cN(0, \Sigmate)$ where $\Sigmate$ may be different from the training covariance $\Sigmatr$. The expected test error (without label noise on the test point) is
\al{\label{eq:eq_test_error}
    \err{\Sigmate} &= \E[(\beta^\top \x/\sqrt{n_0} - Y K^{-1}K_{\x})^2]\\
    &= \underbrace{\E_{\x,\beta} [(\mE[\hat{y}(\x)]-y(\x))^2]}_{\bias{\Sigmate}} + \underbrace{\mE_{\x,\beta}[\V[\hat{y}(\x)]]}_{\var{\Sigmate}}\,,\label{eq_bias_variance_def}
}
where the inner expectations defining the bias and variance are computed over $W$, $X$, and all $\e_i$.\footnote{In fact, the outer expectation over $\beta$ can be removed, since the expectation concentrates around its mean.}

Note that this definition of the bias and variance is somewhat nonstandard in the statistics literature, where typically only the randomness from the label noise is considered. While this is appropriate for analyses of linear regression that condition on the design matrix, in our setting there is additional randomness in the model coming from $W$, and so it would be inappropriate to compute the bias conditional on $W$ as well. Instead, we treat all random variables ($X$, $W$, and $\varepsilon$) equivalently, which leads to unambiguous interpretations of the bias-variance decomposition~\citep{adlam2020understanding,Lin2020WhatCT}. Indeed, failing to incorporate the additional random variables $X$ and $W$ into the inner expectation of \cref{eq_bias_variance_def} has the effect of moving a component of the variance into the bias and can lead to counterintuitive conclusions (see \citet{adlam2020understanding} for additional discussion).

For finite values of $n_0$, the covariate shift is entirely specified by the training and test covariance matrices $\Sigmatr$ and $\Sigmate$, but to characterize the shift asymptotically we have to do so with their limiting spectral properties. Denote the eigenvalues of $\Sigmatr$ in nondecreasing order by ${0<\lambda_1 \le \lambda_2 \le \ldots \le \lambda_{n_0}}$ and denote a choice for the corresponding eigenvectors by ${\bfv_1 \le \bfv_2 \le \ldots \le \bfv_{n_0}}$. Similarly, denote a choice for $\Sigmate$ by ${0<\lambda^*_1 \le \lambda^*_2 \le \ldots \le \lambda^*_{n_0}}$ and ${\bfv_1^* \le \bfv_2^* \le \ldots \le \bfv_{n_0}^*}$.\footnote{Though when $\Sigmatr$ or $\Sigmate$ have repeated eigenvalues their eigendecompositions are not unique, the choice of eigendecomposition will not influence later conclusions (see \cref{app:ambiguity}).} We define the \emph{overlap coefficients} of $\Sigmate$ with respect to $\Sigmatr$ by
\begin{align}
\corrs_{i} \deq \mathbf{v}_i^\top \Sigmate \vv_i = \sum_{j=1}^{n_0} (\vvs_j\cdot \vv_i)^2 \lambda^*_j \label{eq:overlap_coeffs}
\end{align}
to measure the alignment of $\Sigmate$ with the $i$th eigendirection of $\Sigmatr$. In particular, $r_i$ is the induced norm of $\vv_i$ with respect to $\Sigmate$.

We use $\ntr$ to denote the dimension-normalized trace: for a matrix $A \in \mathbb{R}^{n \times n}$, $\ntr(A) = \tr(A) / n$. We also use $\norm{A}_{\infty}$ and $\norm{A}_{F}$ to denote the operator norm and Frobenius norm of a matrix $A$ respectively. Finally, we use $\delta_{\bfx}$ to denote the Dirac delta function centered at $\bfx$.

\subsection{Assumptions}
Regularity assumptions on the spectra of $\Sigmatr$ and $\Sigmate$ are necessary to state the limiting behavior of this system. As in \citet{wu2020optimal}, it is not sufficient to consider the spectra of these matrices individually; they must be considered jointly. A natural choice is to do this in an eigenbasis of $\Sigma$.
\begin{assumption}
    \label{assump:esd}
    We define the empirical joint spectral distribution (EJSD) as 
    \eq{
       \mu_{n_0} \deq \frac{1}{n_0} \sum_{i=1}^{n_0} \delta_{(\lambda_i, r_i)}
    }
    and assume it converges in distribution to some $\mu$, a distribution on $\mathbb{R}_{+}^2$ as $n_0 \to \infty$. We refer to $\mu$ as the limiting joint spectral distribution (LJSD), and emphasize that this defines the relevant limiting properties of the training and test distributions.\footnote{The EJSD depends not only on $\Sigmatr$ and $\Sigmate$ but also on a choice of eigendecomposition for $\Sigmatr$ when it has repeated eigenvalues. However, all possible choices for the EJSD form an equivalence class, and the ambiguity does not affect later definitions and conclusions. Again see \cref{app:ambiguity}.} Additionally, we require that $\limsup_{n_0} \max(\norm{\Sigmatr}_{\infty}, \norm{\Sigmate}_{\infty}) \leq C$ for a constant $C$. 
\end{assumption}
Often we use $(\lambda, r)$ for random variables sampled jointly from $\mu$ and denote the marginal of $\la$ under $\mu$ with $\mutr$: for measurable $\mathcal{E}\subseteq\R_+$, define
\eq{
    \mutr(\mathcal{E}) \deq \mu\p{ \mathcal{E} \times \R_+ }.
}
Since the conditional expectation $\mE[r|\lambda]$ is an important object in our study, we assume the following for simplicity.
\begin{assumption}
\label{assump:abs_cts}
    $\mu$ is either absolutely continuous or a finite sum of delta masses. Moreover, the expectations of $\lambda$ and $r$ are finite. 
\end{assumption}
When the eigenspaces of $\Sigmatr$ and $\Sigmate$ are aligned and $r_i=\la_i^*= \Phi(\lambda_i)$ for some smooth function $\Phi$, the support of the LJSD degenerates. Here, \cref{assump:esd} is essentially equivalent to assuming the empirical spectral distribution of $\Sigmatr$ converges in distribution to some $\mu_\text{train}$, which is a standard assumption in the high-dimensional regression literature \citep[e.g.][]{dobriban2018high, mei2019generalization}; moreover, \citet{wu2020optimal} use an identical specification in the context of linear regression. One special case of note is when there is no shift, i.e. $\Phi$ is the identity, in which case the LJSD degenerates to $\munull$ defined by
\eq{
    \munull(\mathcal{E}) \deq \mutr(\h{x: (x,x)\in\mathcal{E}})\,,\quad \text{i.e.} \quad (\la, \la) \sim \munull \;\;\text{for}\;\;\la \sim \mutr\,
}
for measurable $\mathcal{E}\subseteq\R_+^2$ and any choice of $\mutr$ a distribution over $\R_+$.

As our analysis takes place in the high-dimensional limit, we further define the limiting scales of the training and test covariances as
\eq{
    s\deq\lim_{n_0 \to \infty} \ntr(\Sigmatr) = \mE_{\mu}[\lambda] \quad\text{and}\quad s_* \deq \lim_{n_0 \to \infty} \ntr(\Sigmate) = \mE_{\mu}[r]
}
under the limiting behavior specified in \cref{assump:esd}.

Throughout this paper, we also enforce the following standard regularity assumptions on the activation functions to ensure the existence of the moments and derivatives we compute.
\begin{assumption}
\label{assump:sigma}
    The activation function $\sigma :\mR \to \mR$ is assumed to be differentiable almost everywhere. We assume that, $| \sigma(x) |,| \sigma'(x) | \leq c_0 \exp(c_1 x)$
    for constants $c_0$ and $c_1$.
\end{assumption}

\subsection{A simple family of diatomic distributions}
\label{sec:diatomic}
Because the above assumptions allow for such a general class of covariance structures, it will prove useful for examples and illustrations to consider our results in the context of a simple family of distributions that readily admits straightforward interpretation. 

\begin{definition}
\label{def:diatomic}
For $\alpha \geq 1$ and $\theta \in \R$, we define the family of \textit{$(\alpha, \theta)$-diatomic} LJSDs as 
    \eq{\label{eq:eq_alpha_diatomic}
        \muspike_{\alpha,\theta} \deq \frac{1}{\alpha+1}\delta_{(\alpha, C\alpha^\theta)} + \frac{\alpha}{\alpha+1}\delta_{(\alpha^{-1}, C\alpha^{-\theta})},
    }
    where $C$ is a normalization constant chosen so that $\E_{\muspike_{\alpha,\theta}}[r]=1$. 
    Note that $\muspike_{\alpha,\theta}$ is the limit of 
    \eq{
    \label{eq:eq_alpha_diatomic_finite}
        \Sigmatr_{ij}\deq\begin{cases} \alpha & \text{if } i=j \text{ and } i\leq \lfloor \frac{n_0}{1+\alpha} \rfloor \\ \alpha^{-1} & \text{if } i=j \text{ and } i> \lfloor \frac{n_0}{1+\alpha} \rfloor \\ 0 & \text{if } i\neq j \end{cases}\qquad\text{and}\quad \Sigmate \deq  \frac{1}{\ntr (\Sigmatr^\theta)}\Sigmatr^\theta\,.
    }
\end{definition}
This simple two-parameter family of distributions captures the fast eigenvalue decay observed in many datasets in machine learning, for which the covariance spectra are often dominated by several large eigenvalues and exhibit a long tail of many small eigenvalues~\citep{lee2020finite}. Note that the trivial case of $\alpha=1$ yields an identity covariance with no shift. For the nontrivial setting $\alpha>1$, the exponent $\theta$ parameterizes the strength of the shift in an intuitive way. When $\theta=1$, there is no shift. When $\theta<1$, $\alpha^{\theta}<\alpha$, so the large eigendirections of the training distribution are suppressed in the test distribution, suggesting that the shift makes learning harder. When $\theta>1$, $\alpha^{\theta}>\alpha$, so the large eigendirections of the training distribution are further emphasized in the test distribution, suggesting that the shift makes learning easier. We will return to the notion of shift strength in \cref{sec:lin_reg,,sec:shift_strength}.

\section{Motivating example: linear regression}
\label{sec:lin_reg}
We first consider the relatively simple case of ridgeless linear regression (LR), which will help build some intuition for the more general analysis of random feature regression in \cref{sec:main_thms}. Assuming the labels are generated by the linear model defined above, i.e. $y_i=\beta^\top \x_i /\sqrt{n_0}+\e_i$, the estimator is given by $\hat{\beta} = (X X^\top)^{-1} XY$, and the test risk (see \cref{eq:eq_test_error}) has the following simple form.

\begin{proposition}
    \label{prop:lin_reg}
    For fixed dimension $n_0$ and sample size $m\!>\!n_0\!+\!1$, the test error of LR is given by
    \begin{align}
    \label{eqn:err_lr}
       \err{\Sigmate}^{\text{LR}} = \sigma_{\e}^2 \frac{n_0}{m-n_0-1} \ntr(\Sigmate \Sigmatr^{-1}) = \sigma_{\e}^2 \frac{n_0}{m-n_0-1} \frac{1}{n_0} \sum_{i=1}^{n_0} \frac{\corrs_i}{\lambda_i}.
    \end{align}
    Under \cref{assump:esd}, as $n_0,m \to \infty$ such that $n_0/m\to\phi$, $\err{\Sigmate}^{\text{LR}} \to \err{\mu}^{\text{LR}} = \sigma_{\e}^2\phi/(1-\phi)\E_\mu[r/\la]$.
\end{proposition}

One immediate question is whether a given shift will increase or decrease the test error relative to $E^{\text{LR}}_\Sigmatr$. While the precise answer is of course determined by the value of $\ntr(\Sigmate \Sigmatr^{-1})$, it useful for the subsequent analysis to develop an understanding of the individual contributions to this term. Similar decompositions of the test error into eigenspaces have proved useful in a variety of other contexts \citep{mel2021theory, advani2020high}. We begin with a specific example in the setting of the finite-dimensional analog of \cref{def:diatomic}. 
\begin{example}
\label{examp:alpha_diatomic}
For the finite form of the $(\alpha,\theta)$-diatomic LJSD defined in Eq.~\eqref{eq:eq_alpha_diatomic_finite}, the test error of LR is given by
\al{
\err{\Sigmate}^{\text{LR}} &= \sigma_{\e}^2 \frac{n_0}{m-n_0-1} \frac{\alpha + w(\alpha^{2\theta-1}-\alpha)}{1 + w(\alpha^{2\theta}-1)}\,\quad\text{for}\quad w = \frac{1}{n_0} \left\lfloor \frac{n_0}{1+\alpha} \right\rfloor\,,
}
and so $\frac{\partial}{\partial \theta} \err{\Sigmate}^{\text{LR}} = \sigma_{\e}^2 \frac{n_0}{m-n_0-1} \frac{2 (1-w) w \alpha^{2 \theta -1}(1-\alpha^2) \log (\alpha )}{(1+w(\alpha^{2\theta}-1))^2}\le 0$, which implies $\err{\Sigmate_1}^{\text{LR}} \le \err{\Sigmate_2}^{\text{LR}}$ whenever $\theta_1 \ge \theta_2$, in accordance with the discussion in \cref{sec:diatomic}. It follows from \cref{eq:eq_alpha_diatomic_finite} that the condition $\theta_1 \ge \theta_2$ not only implies $\ntr(\Sigmate_1 \Sigmatr^{-1}) \le \ntr(\Sigmate_2 \Sigmatr^{-1})$, which is necessary for an ordering of the test error, but also that the ratios of overlap coefficients $r_{i,1}/r_{i,2}$ form a nondecreasing sequence. It is this condition involving all the eigendirections that will generalize to the nonlinear random feature setting in \cref{sec:shift_strength}.
\end{example}
The following proposition captures the essence of these considerations in the context of linear regression.
\begin{proposition}
\label{prop:ordering_lr}
Let $\corrs_{i,1}$ and $\corrs_{i,2}$ denote the overlap coefficients\footnote{Recall the definition of the overlap coefficients in Eq.~\eqref{eq:overlap_coeffs}.} of $\Sigmate_{1}$ and $\Sigmate_{2}$ relative to $\Sigmatr$. If $\tr(\Sigmate_2) \ge \tr(\Sigmate_1)$ and the ratios ${\corrs_{i,1}}/{\corrs_{i,2}}$ form a nondecreasing sequence, then in the setting of \cref{prop:lin_reg}, $E_{\Sigmate_2}^{\text{LR}} \geq E_{\Sigmate_1}^{\text{LR}}$.
\end{proposition}
\begin{proof}
    Using the Harris inequality, we see
    \eq{
        \err{\Sigmate_1}^{\text{LR}} =  \frac{\sigma_{\e}^2n_0}{m-n_0-1} \frac{\ntr(\Sigmate_2)}{n_0}  \sum_{i=1}^{n_0} \frac{\corrs_{i,2}}{\ntr(\Sigmate_2)}\frac{\corrs_{i,1}}{\corrs_{i,2}} \frac{1}{\la_i} 
        \leq \frac{\sigma_{\e}^2n_0}{m-n_0-1}  \frac{\ntr(\Sigmate_1)}{\ntr(\Sigmate_2)}  \frac{1}{n_0}   \sum_{i=1}^{n_0}  \frac{\corrs_{i,2}}{\la_i} \leq \err{\Sigmate_2}^{\text{LR}},
    }
    since $1/\la_i$ and ${\corrs_{i,1}}/{\corrs_{i,2}}$ are nonincreasing and nondecreasing in $i$ respectively and $r_{i,2} / \ntr(\Sigmate_2)$ is a discrete distribution over $\ha{1,\ldots,n_0}$.\footnote{For more discussion and a generalization of this result and proof technique, see \cref{lem:harris_Ifunc}.}
\end{proof}
Whereas the $(\alpha,\theta)$-diatomic LJSDs of \cref{examp:alpha_diatomic} explicitly enforce the trace normalizations that ${\tr(\Sigmatr) = \tr(\Sigmate_1) = \tr(\Sigmate_2) = 1}$, \cref{prop:ordering_lr} provides sufficient conditions for the ordering of test errors for non-unit traces. The fact that $\err{\Sigmate}^{\text{LR}}$ scales linearly with the overall scale of $\Sigmate$ is a unique feature of linear regression and does not generalize to the nonlinear random feature setting. We return to this issue in \cref{sec:shift_strength}.

\section{Definition of shift strength}
\label{sec:shift_strength}

Deriving conditions governing whether a shift will hurt or help a model's performance is crucial to building an understanding of covariate shift. Motivated in part by the above results for linear regression, and in part by the results for random feature regression that we present in \cref{sec:b_v_rf}, we introduce the following definition of shift strength, which is a direct generalization of the conditions of \cref{prop:ordering_lr}.

\begin{definition}
    \label{def:hard}
    Let $\mu_1$ and $\mu_2$ be LJSDs with the same marginal distribution over $\lambda$. If the asymptotic overlap coefficients are such that $\E_{\mu_1}[r|\la]/\E_{\mu_2}[r|\la]$ is nondecreasing as a function of $\lambda$ on the support of $\mutr$ and ${\E_{\mu_1}[r]\le\E_{\mu_2}[r]}$, we say $\mu_1$ is \emph{easier} than $\mu_2$ (or $\mu_2$ is \emph{harder} than $\mu_1$), and write $\mu_1 \le \mu_2$. Comparing against the case of no shift, $\munull$, we say $\mu_1$ is \emph{easy} when $\mu_1\leq \munull$ and \emph{hard} when $\mu_1\geq\munull$.
\end{definition}

A priori, there is little reason to hope that such a \emph{model-independent} definition of shift strength would adequately characterize a shift's impact on the total error, bias, or variance of a given model. Even for the relatively simple case of random feature regression, the nonlinear feature maps of \cref{eqn:rf_embedding} would seem to inextricably couple the covariance distribution to the model. Indeed, in general, it is impossible to completely disentangle the distribution from the model (see \cref{sec:counterexamples}).

Nevertheless, as we show in \cref{sec:main_thms}, the coupling between model and shift simplifies considerably in the high-dimensional limit. It is characterized by a handful of constants that depend solely on the overall covariance scale $\E_{\mu}[r]$ and a collection of functionals of $\mu$ that can be controlled. Specifically, when comparing against the case of no shift, i.e. $\mu_2=\munull$, we have $\E_{\munull}[r|\la]=\la$. \cref{def:hard} then places conditions on $\E_{\mu}[r]$ and $\E_{\mu}[r|\la]/\la$, which, when combined with simple constraints on the model, bound how the total error and bias will respond to a shift of a given strength. We develop this analysis in~\cref{sec:b_v_rf}.

Remarkably, it turns out that nearly all of the model dependence can be eliminated by merely normalizing the scales of the covariate distributions, i.e. enforcing $s=s_*$, in which case~\cref{def:hard} provides an essentially \emph{model-independent} definition of shift strength that determines how random feature models respond to shifts of different strength. This observation motivates the following assumption:

\begin{assumption}
\label{assump:scale}
The training and test covariance scales are equal, $\E_{\mu}[\lambda] = \E_{\mu}[r]$, i.e. $s=s_*$.
\end{assumption}

We emphasize that \cref{assump:scale} reflects common practice for many models and data modalities, as preprocessing techniques such as standardization are ubiquitous and many architectural components such as layer- or batch-normalization achieve a similar effect~\citep{goodfellow2016deep,ioffe2015batch,nado2020evaluating}. While we state all of our main results without this assumption, we will at times invoke~\cref{assump:scale} to illustrate how the results simplify when $s=s_*$.

In the expression for the test error for linear regression, \cref{eqn:err_lr}, we saw how the effects of the scale in the test distribution decoupled from the more nuanced relative restructuring of the covariates. As discussed above, this decoupling does not happen for random feature regression. Nevertheless, it is useful to keep in mind these four possibilities for the covariate shift: (1)~$\E_\mu[r|\la] = \la$, which occurs when there is no covariate shift; (2)~$\E_\mu[r|\la] = (s^*/s) \la$ with $s^* \neq s$, which occurs when there is a pure-scale shift; (3)~$\E_\mu[r|\la] \not \propto \la$ with $s^* = s$, which occurs when the covariates are restructured without changing the scale, i.e. in the setting of \cref{assump:scale}; and (4) $\E_\mu[r|\la] \not \propto \la$ with $s^* \neq s$, which occurs for a generic covariate shift. Throughout the paper, we consider these various categories of shift, with certain results specializing to certain cases.

\section{Covariate shift in random feature kernel regression}
\label{sec:main_sec}
In this section, we first present our main results characterizing the bias, variance, and total test error of random feature regression. We then use these results to prove that harder shifts (in the sense of \cref{def:hard}) lead to increased bias and total error, that overparameterization improves performance and robustness under covariate shift, and that linear trends exist between the generalization error on unshifted and shifted distributions.

\subsection{Main results}
\label{sec:main_thms}

Our main result characterizes the high-dimensional limits of the test error, bias, and variance of the nonlinear random feature model of \cref{sec:setup}. Before presenting these results, we must first introduce some additional constants that capture the effect of the nonlinearity $\sigma$. These constants depend only on the overall scale of the covariance. For the training covariance, let $z\sim \cN(0,s)$ and define
\begin{alignat}{4}
\label{eq:constants_train}
\eta &\deq \mathbb{V}[\sigma(z)]\,,\;\; \rho &\deq (\tfrac{1}{s}\mE[z \sigma(z)])^2\,,\;\; \zeta &\deq s\rho\,,\;\;\;\;\text{and}\;\;\;\;\omega\deq s(\eta/\zeta-1)\,.
\end{alignat}
Viewing the constants in \cref{eq:constants_train} as function of $s$, i.e. $\Omega_\sigma:s\mapsto (\eta,\rho,\zeta,\omega)$, we write
\eq{\label{eq:constants_test}
    (\etas, \rhos,\zetas,\omegas) \deq \Omega_\sigma(s_*)
}
for the test covariance scale $s_*$. Similarly, sometimes it will be convenient to compare two different LJSDs $\mu_1$ and $\mu_2$, so we define $(\eta_i, \rho_i,\zeta_i,\omega_i) \deq \Omega_\sigma(s_i)$, where $s_i$ is the test covariance scale for $\mu_i$ for $i\in\h{1,2}$.

We also define $\xi \deq \sqrt{\rhos/\rho}$, which measures the overall scale mismatch between the training and test covariance matrices.
\begin{remark}
    \label{rem:equal_constants}
    Many of our results will simplify considerably under~\cref{assump:scale}, which implies $\etas = \eta$, $\zetas = \zeta$, $\rhos = \rho$, $\omegas = \omega$, and $\xi = 1$.
\end{remark}
\begin{remark}
    Many of our results will also simplify if the activation function is affine. In this case, if $\sigma(z) = az+b$ for $a,b\in\R$, we have $\rho=\rhos=a^2$, $\zeta=\eta=sa^2$, and $\zetas=\etas=s_*a^2$, which imply $\xi=1$ and $\omega=\omegas=0$.
\end{remark}

Finally, the bias, variance, and total error depend on the covariance spectra through two sets of functionals of $\mu$, which we define as
\eq{
\label{eqn:Idefs}
    \I_{a,b}(x) \deq \phi \ \E_\mu\qa{\la^{a} \pa{ \phi + x \la }^{-b} } \quad\text{and}\quad \Ishift_{a,b}(x) \deq \phi \ \E_\mu\qa{ r \la^{a-1} \pa{ \phi + x \la }^{-b} }.
}
\begin{theorem}
\label{thm:main_b_v}
Under \cref{assump:esd,,assump:abs_cts,,assump:sigma}, as $n_0,n_1,m\to\infty$ the test error $\err{\Sigmate}$ converges to ${\err{\mu} = \bias{\mu} + \var{\mu}}$, with the bias $  \bias{\mu}$ and variance $\var{\mu}$ given by
\al{
\label{eqn:main_bias}
    \bias{\mu} &= (1-\xi)^2 s_{*} + 2(1-\xi)\xi \Ishift_{1,1} + \xi^2 \phi \Ishift_{1,2} \qquad\text{and} \\
\label{eqn:main_var}
    \var{\mu} &= -\rho\xi^2 \frac{\psi}{\phi}\frac{\partial x}{\partial \gamma} \bigg(\I_{1,1}(\omega + \phi \I_{1,2})(\omegas+\Ishift_{1,1}) +\frac{\phi^2}{\psi}\gamma\taub \I_{1,2}\Ishift_{2,2} \nonumber\\
&\quad+\gamma\tauone\I_{2,2}(\omegas+\phi\Ishift_{1,2}) + \sigma_\e^2\Big((\omega + \phi \I_{1,2})(\omegas+\Ishift_{1,1}) +\frac{\phi}{\psi}\gamma\taub \Ishift_{2,2}\Big)\bigg),
}
where $x$ is the unique nonnegative real root of $x = \p{1-\gamma\tauone}/\p{\omega + \I_{1,1}}$,
\al{
    \frac{\partial x}{\partial \gamma} &= -\frac{x}{\gamma+\rho \gamma (\tauone \psi/\phi + \taub)(\omega + \phi \I_{1,2})},\label{eq:xprime}\\
    \tauone &= \frac{\sqrt{(\psi -\phi )^2+ 4 x \psi\phi  \gamma/\rho}+\psi -\phi }{2 \psi \gamma}, \\
    \text{and}\quad\taub &= \frac{\phi-\psi}{\phi}\frac{1}{\gamma} + \frac{\psi}{\phi}\tauone\,.
}
\end{theorem}
We note that $\tau$ is limiting Stieltjes transform of the random feature kernel matrix $\tfrac{1}{n_1}F^\top F$ and $\taub$ is the companion transform, i.e. the limiting Stieltjes transform of $\tfrac{1}{n_1}F F^\top$. Numerical predictions from \cref{thm:main_b_v} can be obtained by first solving the self-consistent equation for $x$ by fixed-point iteration, $x\mapsto \p{1-\gamma\tauone}/\p{\omega + \I_{1,1}}$, and then plugging the result into the remaining terms. \cref{fig:predictions} shows excellent agreement between these asymptotic predictions and finite-size simulations.

\begin{figure}[t]
\centering
\includegraphics[width=\linewidth]{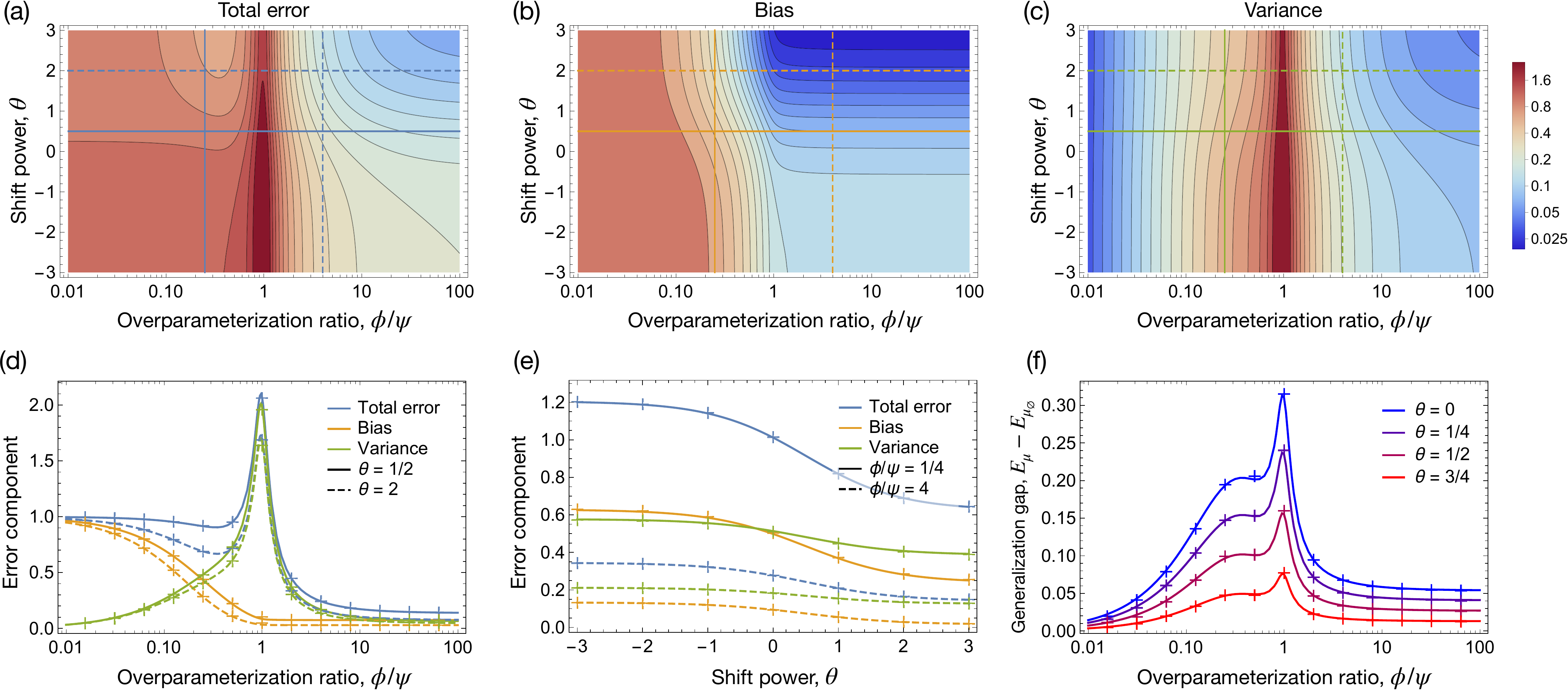}
\caption{Predictions of \cref{thm:main_b_v} as a function of the overparameterization ratio ($\phi/\psi = n_1/m$) and the shift power ($\theta$) for the $(2,\theta)$-diatomic LJSD (\cref{eq:eq_alpha_diatomic}) with ${\phi=n_0/m=0.5}$, $\sigma=\text{ReLU}$, $\gamma=0.001$, and $\sigma_{\e}^2=0.1$. (\textbf{a}) The test error exhibits the characteristic double descent behavior for all shift powers. (\textbf{b}) The bias is a nonincreasing function of $\phi/\psi$ for all shift powers, as in \cref{prop:b}. (\textbf{c}) The variance is the source of the double-descent peak, and is a nonincreasing function of $\phi/\psi$ for all shift powers in the overparameterized regime, as in \cref{prop:v}. In (\textbf{a},\textbf{b}), the total error and bias are nonincreasing functions of $\theta$, as in \cref{prop:error_order}. (\textbf{d}) 1D horizontal slices of (\textbf{a},\textbf{b},\textbf{c}) demonstrate the monotonicity in $\phi/\psi$ predicted by \cref{prop:b,,prop:v}. (\textbf{e}) 1D vertical slices of (\textbf{a},\textbf{b},\textbf{c}) demonstrate the monotonicity in $\theta$ predicted by \cref{prop:error_order} (the variance also appears monotonic, but it need not be in general). (\textbf{f}) The generalization gap between the error on shifted and unshifted distributions is a nonincreasing function of $\phi/\psi$ in the overparameterized regime, as in \cref{prop:gen_gap}. Markers in (\textbf{d},\textbf{e},\textbf{f}) show simulations for $n_0=512$ and agree well with predictions.} 
\label{fig:predictions}
\end{figure}
At times we will find it convenient to consider various limits of \cref{thm:main_b_v}. By carefully expanding $x$ and $\tauone$ for small $\gamma$, it is straightforward to obtain the following corollary, which characterizes the ridgeless limit of~\cref{thm:main_b_v}.
\begin{corollary}
\label{cor:ridgeless_lim}
In the setting of \cref{thm:main_b_v}, as the ridge regularization constant $\gamma \to 0$, $\err{\mu} = \bias{\mu} + \var{\mu}$ with $\bias{\mu}$ given in \cref{eqn:main_bias} and $\var{\mu}$ given by
\begin{align}
\label{eqn:v_ridgeless}
\var{\mu} = \xi^2 \frac{\psi}{|\phi - \psi|}x(\sigma_{\e}^2+\I_{1,1})(\omegas + \Ishift_{1,1}) + \begin{cases}
\xi^2 x\Big(1-\frac{x (\omega-\sigma_{\e}^2)}{1-x^2 \I_{2,2}}\Big)\Ishift_{2,2} & \phi \ge \psi \\
\xi^2 \frac{x^2 \psi \I_{2,2}}{\phi - x^2 \psi \I_{2,2}}\big(\omegas + \phi \Ishift_{1,2}\big) & \phi < \psi \\
\end{cases}\,,
\end{align}
where $x$ is the unique positive real root of $x = {\min(1,\phi/\psi)}/\p{\omega+\I_{1,1}}$.
\end{corollary}

As a consistency check, taking $\sigma(x) = x$ and $\psi\to 0$ in \cref{cor:ridgeless_lim} yields an expression for the test risk of ridgeless linear regression that agrees with the results of \citet{hastie2019surprises} and with the asymptotic form of \cref{prop:lin_reg} (see \cref{app:limits}).

Similarly, the limit of infinite overparameterization can be obtained by expanding~\cref{thm:main_b_v} for small $\psi$ (or large $\phi/\psi$ with fixed $\phi$):
\begin{corollary}
\label{cor:infinite_overparameterization}
In the setting of \cref{thm:main_b_v}, as the overparameterization ratio $\phi/\psi\to \infty$, ${\err{\mu} = \bias{\mu} + \var{\mu}}$ with $\bias{\mu}$ given in \cref{eqn:main_bias} and $\var{\mu}$ given by
\al{
\label{eq:var_infinite_overparameterization}
\var{\mu} = \xi^2 \frac{\sigma_{\e}^2 + \phi \I_{1,2}}{\gammaeff + \phi \I_{1,2}}x \Ishift_{2,2}\,,
}
where $x$ is the unique positive real root of $x=1/(\gammaeff + \I_{1,1})$, and $\gammaeff = \gamma/\rho + \omega$.
\end{corollary}
\begin{remark}
\label{rem:ridge}
Under~\cref{assump:scale}, the expressions in \cref{cor:infinite_overparameterization} are identical to those of linear ridge regression with ridge constant equal to $\gammaeff$.
\end{remark}
\cref{cor:infinite_overparameterization} highlights how the nonlinearity $\sigma$ acts as an implicit regularizer through the constant $\omega$, an effect described previously for isotropic data without covariate shift in \citet{bartlett2021deep}. To the best of our knowledge, for nontrivial covariate shift, even the linear ridge regression limit ($\omega\to 0$) of \cref{cor:infinite_overparameterization} has not appeared in previous work. In the absence of covariate shift, \cref{cor:infinite_overparameterization} does yield expressions that agree with existing results for ridge regression, as we show in \cref{app:ridge_regression_connect}.
\subsection{Harder shifts increase the bias and test error}
\label{sec:test_error_rf}
In~\cref{sec:lin_reg}, we found that the test error of linear regression increases in response to harder shifts, as defined by~\cref{def:hard}. The same is true for random feature regression, provided that the label noise $\sigma_{\e}^2$ is not too large and that the model-dependent constants in~\cref{eq:constants_train,eq:constants_test} respect the ordering of scales. As a matter of notation, recall that we use the subscripts ``$_1$'' to denote parameters (i.e. $\eta_1, \zeta_1, \xi_1)$ associated with $\mu_1$ and the subscript ``$_2$'' to denote parameters (i.e. $\eta_2, \zeta_2, \xi_2)$ associated with $\mu_2$.

\begin{proposition}
    \label{prop:error_order}
    Consider two LJSDs $\mu_1$ and $\mu_2$ such that $\mu_1 \le \mu_2$. Then, in the setting of \cref{thm:main_b_v},
    \begin{align}
        \bias{\mu_1} \le \bias{\mu_2} &\quad\text{if}\quad \xi_2 \le \xi_1 \le 1\label{eqn:bias_xi}\\
        \err{\mu_1} \le \err{\mu_2} &\quad\text{if}\quad \xi_2 \le \xi_1 \le 1\le\zeta_2/\zeta_1 \le \eta_2/\eta_1 \quad \text{and}\quad \sigma_{\e}^2 \le \omega\label{eqn:error_xi}\,.
    \end{align}
\end{proposition}
In~\cref{sec:counterexamples}, we thoroughly investigate the consequences of the above conditions on the model-dependent constants. Here we simply note that, when the training and test covariance scales are the same, \cref{rem:equal_constants} implies that the conditions on the constants in~\cref{prop:error_order} are automatically satisfied and we obtain the following corollary.
\begin{corollary}
    \label{cor:error_order}
    Consider two LJSDs satisfying \cref{assump:scale} such that $\mu_1 \le \mu_2$. Then, in the setting of \cref{thm:main_b_v}, $\bias{\mu_1} \leq \bias{\mu_2}$ and, if $\sigma_{\e}^2\le \omega$, $\err{\mu_1} \leq \err{\mu_2}$.
\end{corollary}
When the scale of the training and test distributions are equal, \cref{cor:error_order} shows that \cref{def:hard} provides an essentially model-independent condition to determine the impact of covariate shift on the test error. Interestingly, both the bias (which arises from both regularization and model misspecification) and the total error (which has additional contributions from the randomness in $W$, $X$, and $\e$) respond in tandem to a covariate shift. This behavior is illustrated in~\cref{fig:predictions} for the $(\alpha,\theta)$-diatomic LJSD. Following the vertical lines upward in (a), (b), or the x-axis rightward in (e) yields easier shifts and a corresponding decrease in the bias and total error.

As we discuss in~\cref{sec:counterexamples}, even when \cref{assump:scale} and $\sigma_{\e}^2 \le \omega$ hold, the variance need not be ordered. The underlying obstacle is a lack of ordering on the $\Ishift_{2,2}$ functionals in general; however, in the special case of a \emph{pure-scale} shift, this ordering is restored, and we can derive conditions that guarantee that the variance is also ordered.
\begin{proposition}
\label{prop:pure_scale}
Consider two LJSDs $\mu_1$ and $\mu_2$ such that $\E_{\mu_i}[r|\la] = s_{i}\la$, i.e. a pure-scale shift by $s_{i}$, with $\mu_1\leq\mu_2$. Then $s_1\le s_2$ and, in the setting of \cref{thm:main_b_v},
    \begin{align}
        \bias{\mu_1} \le \bias{\mu_2} &\quad\text{if}\quad \xi_2 \le \xi_1 \le 1\label{eq:bias_pure_scale}\\
        \var{\mu_1} \le \var{\mu_2} &\quad\text{if}\quad 1\le\zeta_2/\zeta_1 \le \eta_2/\eta_1\label{eq:variance_pure_scale}\\
        \err{\mu_1} \le \err{\mu_2} &\quad\text{if}\quad \xi_2 \le \xi_1 \le 1\le\zeta_2/\zeta_1 \le \eta_2/\eta_1\,.
    \end{align}
\end{proposition}
\cref{prop:pure_scale} provides a complement to many of our results, which otherwise emphasize nontrivial spectral differences that persist even when the training and test scales are equal. By focusing on pure-scale shifts, \cref{prop:pure_scale} isolates the model-dependent conditions under which an increase in scale results in an increase in the bias, variance, and total error, generalizing the analogous results from \cref{sec:lin_reg} on linear regression.

\subsection{The benefit of overparameterization}
\label{sec:b_v_rf}

While \cref{prop:error_order} shows that harder shifts increase the error, it is natural to wonder whether this increase can be mitigated by judicious model selection. In practice, empirical investigations have shown that the performance of large, overparameterized models tends to deteriorate less under distribution shift than their smaller counterparts~\citep{hendrycks2020faces}. We obtain a number of theoretical results that formally prove the benefit of overparameterization in our random feature setting.

First, we show that the bias decreases (or stays constant) when additional random features are added, which increases the model capacity and accords with the intuition of the bias as a measure of the model's ability to fit the data.

\begin{proposition}
    \label{prop:b}
     In the setting of \cref{thm:main_b_v}, if $\xi \leq 1$, the bias $\bias{\mu}$ is a nonincreasing function of the overparameterization ratio $\phi/\psi$.
\end{proposition}

The condition $\xi\le 1$ is necessary to \cref{prop:b}, as without it the bias can behave nonmonotonically, demonstrating one counterintuitive consequence of covariate shift (see \cref{sec:counterexamples} for more discussion). We emphasize that if we additionally invoke~\cref{assump:scale}, then this condition is automatically satisfied. The following proposition shows that, in contrast to the bias, the variance is strictly monotonic in the overparameterized regime, irrespective of the value of $\xi$. Note that our proof requires the setting of ridgeless regression ($\gamma=0$), but numerical investigation suggests this condition may not be necessary (see \cref{fig:predictions}).

\begin{proposition} 
    \label{prop:v}
    In the setting of \cref{cor:ridgeless_lim} and in the overparameterized regime (i.e. $\psi < \phi$), the variance $\var{\mu}$ is a nonincreasing function of the overparameterization ratio $\phi/\psi$.
\end{proposition}

The explosion of variance at the interpolation threshold and its subsequent decay have been demonstrated in previous studies of the high-dimensional limit of random feature regression in the absence of covariate shift~\citep{adlam2020neural, mei2019generalization}, in stark contrast to what classical theory would suggest. \cref{prop:v} confirms the existence of analogous behavior under covariate shift.

Taken together, \cref{prop:b,prop:v} imply that some of the benefits of overparameterzation extend to models evaluated out-of-distribution. An additional benefit is that overparameterized models are more robust: the difference in error between unshifted and shifted test distributions is smaller for larger models. A formal statement of this enhanced robustness is given in the following proposition.

\begin{proposition} 
    \label{prop:gen_gap}
     Consider two LJSDs $\mu_1$ and $\mu_2$ such that $\mu_1 \leq \mu_2$ and ${1 \le \zeta_2/\zeta_1 \le \eta_2/\eta_1}$. Then, in the setting of \cref{cor:ridgeless_lim} and in the overparameterized regime (i.e. $\psi < \phi$), the generalization gap $E_{\mu_2}-E_{\mu_1}$ is a nonincreasing function of the overparameterization ratio $\phi/\psi$.   
\end{proposition}
Recalling~\cref{rem:equal_constants}, we again note that if we additionally invoke~\cref{assump:scale}, then the condition on the model-dependent parameters $1 \le \zeta_2/\zeta_1 \le \eta_2/\eta_1$ is automatically satisfied.

In \cref{fig:predictions}, \cref{prop:b,,prop:v,,prop:gen_gap} are illustrated. Following the horizontal lines rightward in (a), (b), and (c) or the x-axis rightward in (d) and (f) leads to models with more parameters. The monotonicity of the bias across the whole range of parameterization is evident, as is the necessity of considering the monotonicity of the variance and generalization gap only when $\phi>\psi$.

\subsection{Linear trends between in-distribution and out-of-distribution generalization}
\label{sec:linear_trends}

We have discussed how overparameterization yields improved generalization performance for both unshifted and shifted distributions, which hints that these two quantities are positively correlated. Indeed, recently \citet{hendrycks2020faces} have suggested increasing model size as a path to increased robustness. Additionally, the empirical studies of \citet{recht2019imagenet,taori2020measuring, pmlr-v139-miller21b} have further refined this observation by discovering a linear relationship between the performance of models of varying complexity on unshifted and shifted test data. In the context of ridgeless random feature regression, we provide a formal proof of this linear relationship, which holds in the overparameterized regime. In the setting of random feature regression, the overparameterization ratio $\phi/\psi$ provides a surrogate for the complexity of the model class.

\begin{figure}[t]
\centering
\includegraphics[width=0.8\linewidth]{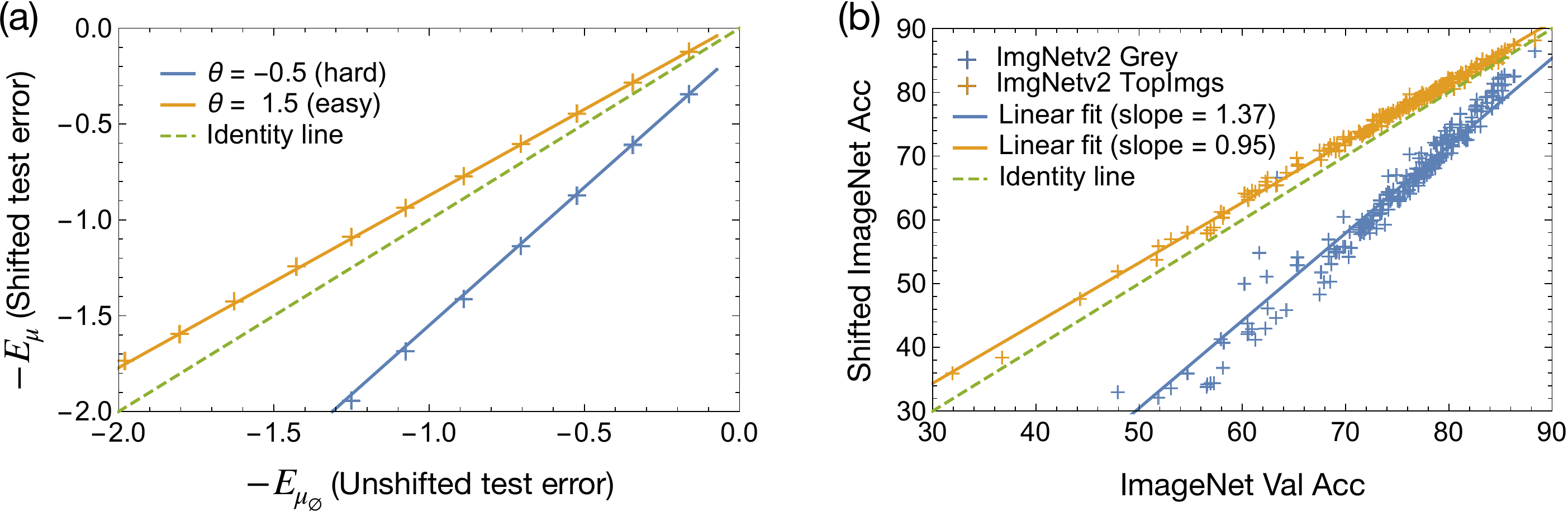}
\caption{Linear relationship between in-distribution and out-of-distribution generalization error. \textbf{(a)}~Theoretical predictions for shifted versus unshifted error for models with varying degrees of overparameterization $\phi/\psi > 1$, obtained via \cref{cor:ridgeless_lim} for the $(3,\theta)$-diatomic LJSD (\cref{eq:eq_alpha_diatomic}) with $\phi=n_0/m=0.5$, $\sigma=\text{ReLU}$, $\sigma_{\e}^2 = 0.01$ and two different values of the shift-power $\theta$. Markers represent simulations for $n_0 = 512$. The negated errors are plotted so that performance improves left to right and bottom to top, in order to match the behavior of the accuracy metric. \textbf{(b)}~Reproduction of the empirical results of \citet{recht2019imagenet, taori2020measuring}, showing the relationship between the classification accuracy of various models on the original ImageNet test set and two shifted ImageNet datasets: a ``hard" dataset with greyscale corruptions, Grey, and an ``easy" dataset with high inter-annotator agreement, TopImgs. In both (\textbf{a}) and (\textbf{b}), the slope is greater than one for the hard shift and less than one for the easy shift, in accordance with~\cref{prop:slope}.}
\label{fig:linear_slope_recht}
\end{figure}

\begin{proposition}
\label{prop:slope}
Consider two LJSDs $\mu_1$ and $\mu_2$. In the setting of \cref{cor:ridgeless_lim} and in the overparameterized regime (i.e. $\psi < \phi$), $E_{\mu_2}$ is linear in $E_{\mu_1}$, parametrically in $\phi/\psi$:
 \begin{equation}
\label{eqn:slope_modified}
    \err{\mu_2} = E_0 + \underbrace{\frac{\xi_2^2}{\xi_1^2} \pa{\frac{\omega_2 + (\Ishift_{1,1})_2}{\omega_1 + (\Ishift_{1,1})_1}}}_{\slope} \err{\mu_1}\,,
\end{equation}
where $E_0$ and $\slope \ge 0$ are constants independent of the overparameterization ratio $\phi/\psi$. Moreover, if $\mu_1 \leq \mu_2$ and $1 \le \zeta_2/\zeta_1 \le \eta_2/\eta_1$, then $\slope \ge 1$.
\end{proposition}
As above,~\cref{rem:equal_constants} implies that the condition $1 \le \zeta_2/\zeta_1 \le \eta_2/\eta_1$ is automatically satisfied if we additionally invoke~\cref{assump:scale}. Specializing~\cref{prop:slope} to the case $\mu_1 = \munull$ yields the claimed linear relationship between the generalization performance with shifted and unshifted test data as the overparameterization ratio varies. Finally, we note that $\slope\ge0$ implies that improved performance on the unshifted distribution translates into improved performance on the shifted distribution as well.
\begin{remark}
\label{cor:slope_simple}
In the setting of \cref{prop:slope}, if \cref{assump:scale} holds and $E_{\munull}$ denotes the error on the unshifted distribution, $\err{\mu} = E_0 + \slope \cdot \err{\munull}$, parametrically in $\phi/\psi$, where $E_0$ and $\slope\ge 0$ are constants independent of the overparameterization ratio $\phi/\psi$. Moreover, $\slope \geq 1$ when $\mu$ is hard and $\slope \leq 1$ when $\mu$ is easy.
\end{remark}

\cref{cor:slope_simple} makes the additional nontrivial prediction that an improvement on the unshifted distribution leads to a relatively greater improvement on the shifted distribution when the shift is hard, and to a relatively smaller improvement when the shift is easy. These predictions are corroborated qualitatively in the data from \citet{recht2018cifar, recht2019imagenet, pmlr-v119-miller20a, pmlr-v139-miller21b}. We plot this linear behavior in \cref{fig:linear_slope_recht}, where (a) shows the random feature model and (b) shows an example of data from \citet{recht2019imagenet,taori2020measuring}. The striking similarity in these plots is evident.

\section{Optimal regularization}
\label{sec:optimal_gamma}

Several prior works studying double descent in the random feature model with identity covariance have found a divergence of the total error and the variance at the interpolation threshold ($\phi = \psi$) in the ridgeless limit \citep[e.g.][]{mei2019generalization,adlam2020neural,adlam2020understanding}. For non-zero regularization, the divergence is tamed, and when the regularizer is tuned optimally for each overparameterization scale, \citet{mei2019generalization} observed that this effect is removed and the test error appears to be strictly decreasing in the number of random features.

The story is more nuanced in the presence of covariate shift because the optimal regularizer depends on the LJSD $\mu$, i.e. $\gammaopt_\mu \deq \argmin \err{\mu}(\gamma)$. Moreover, there are two natural ways to tune the regularizer: it can be tuned to an optimal value $\gammaopt \deq \gammaopt_{\munull}$ that minimizes the test error on the unshifted training distribution, or it can be tuned to an optimal value $\gammaopt_* \deq \gammaopt_{\mu}$ that minimizes the test error on the shifted test distribution. These settings can be approximately realized when one has access to a large validation set drawn from either the training or test distributions, respectively. 

We investigate the behavior of the test error under these two choices of optimal regularization in~\cref{fig:opt_gamma}. The curves in this figure were obtained by minimizing the shifted and unshifted test error with a zeroth-order optimizer over a nonnegative support.  Note that, counterintuitively, there is no guarantee in this setting that the optimal value of $\gamma$ will be nonnegative \citep[see e.g.][]{kobak2020optimal, wu2020optimal, bartlett2021deep}. While negative regularization is certainly an interesting scenario worthy of further investigation, it does not arise for the low signal-to-noise ratios (or equivalently large values of $\sigma_{\e}^2$) chosen for these simulations.

\begin{figure}[t!]
\centering
\includegraphics[width=\linewidth]{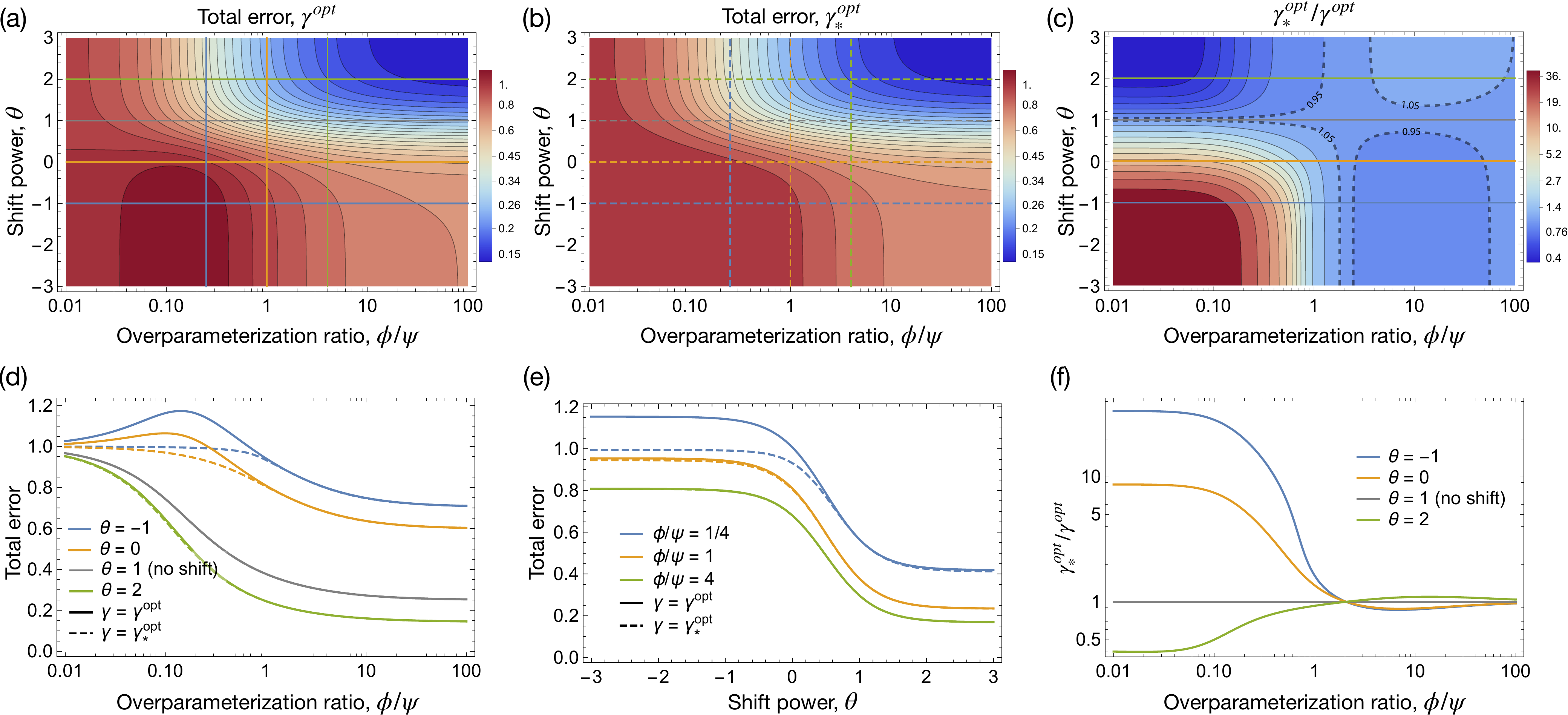}
\caption{The limiting test error and relative regularization strength under two methods of optimal regularization as a function of the overparameterization ratio ($\phi/\psi = n_1/m$) and the shift power ($\theta$) for the $(4,\theta)$-diatomic LJSD (\cref{eq:eq_alpha_diatomic}) with ${\phi=n_0/m=0.5}$, $\sigma=\text{ReLU}$, and $\sigma_{\e}^2=1$. (\textbf{a}) With regularization tuned on the unshifted training distribution, $\gamma = \gammaopt$, the test error exhibits nonmonotonicity in the overparameterization ratio when the shift is sufficiently hard. (\textbf{b}) With regularization tuned on the shifted test distribution, $\gamma = \gammaopt_*$, the test error is a nonincreasing function of the overparameterization ratio. (\textbf{c}) The ratio of optimal regularization values, $\gammaopt_*/\gammaopt$ is large for hard shifts in the underparameterized regime, but approaches one as $\phi/\psi$ becomes large, as predicted by \cref{prop:opt_gamma}. Dashed lines denote $5\%$ deviations from unity, and highlight the existence of a finite $\phi/\psi \approx 2$ for which the optimal regularization value is independent of shift strength, which is also evident in (\textbf{f}). (\textbf{d}) 1D horizontal slices of (\textbf{a},\textbf{b}) demonstrate the monotonicity in $\phi/\psi$ for $\gamma = \gammaopt_*$ (dashed lines) and nonmonotonicity for $\gamma=\gammaopt$ (solid lines) for hard shifts. Consistent with \cref{prop:opt_gamma}, for large overparameterization ratios, the solid and dashed lines overlap.  (\textbf{e}) 1D vertical slices of (\textbf{a},\textbf{b}) demonstrate monotonicity in $\theta$ and that the distinction between $\gammaopt$ and $\gammaopt_*$ diminishes for large overparameterization ratios. (\textbf{f}) 1D horizontal slices of (\textbf{c}) show that $\gammaopt \approx \gammaopt_*$ for large overparameterization ratios, consistent with \cref{prop:opt_gamma}.}
 \label{fig:opt_gamma}
\end{figure}

We observe that the qualitative features of optimal regularization with and without covariate shift are similar if the regularizer is tuned on the shifted distribution. Specifically, the peak at the interpolation threshold vanishes and the test error becomes nonincreasing as a function of the overparameterization ratio. Indeed, horizontal slices of \cref{fig:opt_gamma}(b) (shown as dashed lines in (d)) exhibit the monotonicity of the total error in this setting.

On the other hand, if $\gamma$ is tuned to minimize the test error on the unshifted training distribution, horizontal slices of~\cref{fig:opt_gamma}(a) (shown as solid lines in (d)) show residual nonmonotonicity in the test error. Interestingly, when a peak in the test error does persist under optimal regularization, it no longer appears at the interpolation threshold, suggesting that the residual peak may be of a different origin.

The behavior of the test error under decreasing shift strength can be seen by following vertical slices upward in~\cref{fig:opt_gamma}(a,b), or by following the x-axis rightward in (e), and, as might be expected, the distinction between $\gammaopt$ and $\gammaopt_*$ is less apparent for easier shifts. Indeed, the nonmonotonicity of the test error in (d) appears only for hard shifts, and the curves for $\gammaopt$ and $\gammaopt_*$ in (e) overlap for hard shifts.

The relative strength of the regularizer when optimized on the shifted versus unshifted distributions is shown in \cref{fig:opt_gamma}(c), with horizontal slices for several values of the shift strength shown in (f). In the underparameterized regime, $\gammaopt_*/\gammaopt$ depends strongly on the shift strength, with $\gammaopt_* \gg \gammaopt$ for hard shifts. In the overparameterized regime, however, the optimal $\gamma$ values are not significantly different. For all values of the the shift strength $\theta$, it appears that $\gammaopt_* = \gammaopt$ for some intermediate overparameterization ratio (around $\phi/\psi \approx 2$ in (c) and (f)). While we have no immediate explanation for this curious observation, the equality $\gammaopt_* = \gammaopt$ does occur when $\phi/\psi \to \infty$, as following proposition shows.

\begin{proposition}
\label{prop:opt_gamma}
In the setting of~\cref{cor:infinite_overparameterization} and under~\cref{assump:scale}, the optimal regularization $\gammaopt_\mu$ is independent of $\mu$ and satisfies,
\al{
\gammaopt_\mu \deq \argmin_{\gamma\ge-\rho\omega} \err{\mu}(\gamma) = \rho(\sigma_{\e}^2 - \omega)\qquad\text{and}\qquad \err{\mu}^{\text{opt}} \deq \err{\mu}(\gammaopt_\mu) = \Ishift_{1,1}\label{eqn:gammaopt}\,.
}
Moreover, the optimal error is ordered with respect to shift strength, i.e. for two LJSDs $\mu_1 \le \mu_2$, $\err{\mu_1}^{\text{opt}}\le \err{\mu_2}^{\text{opt}}$.
\end{proposition}
As also observed by \citet{kobak2020optimal, wu2020optimal, bartlett2021deep}, \cref{prop:opt_gamma} confirms that the optimal regularization coefficient $\gammaopt_\mu$ can actually be negative. Nevertheless, owing to the implicit regularization of the nonlinear feature map described in \cref{cor:infinite_overparameterization}, the effective regularizer $\gammaeff$ remains nonnegative, which can be seen by rewriting the expression for $\gammaopt_\mu$ in terms of $\gammaeff = \gamma/\rho + \omega$ to find $\gammaeff^{\text{opt}} = \sigma_{\e}^2$. The latter relationship is the same as the well-known result for optimal regularization in linear ridge regression with orthonormal design \citep[see e.g.][]{Wieringen2015LectureNO}, and might have been anticipated given the connection to ridge regression described in \cref{rem:ridge}. Even without strict orthonormality, \citet[Theorem 2.1]{dobriban2018high} established that the optimal regularizer in high-dimensional ridge regression also equals $\sigma_{\e}^2$ in the absence of covariate shift and under an isotropic source condition on the coefficient of the linear data-generating process.\footnote{Their optimal regularizer is given as $\phi\sigma_{\e}^2/\mE[\beta^\top \beta ]$ but is equivalent to our $\sigma_{\e}^2$ owing to normalization differences.} Perhaps surprisingly, \cref{prop:opt_gamma} shows that their result extends to the setting of random feature regression in the overparameterized limit, and, moreover, that the optimal regularizer is actually independent of the covariate shift, highlighting an additional benefit of overparameterization. This behavior can be seen in \cref{fig:opt_gamma}(f) where the optimal regularizers under all shifts approach each other for large $\phi/\psi$.

\section{Analysis of shift hardness, covariate scale, and assumptions}
\label{sec:counterexamples}

In the previous sections, many concepts and assumptions have emerged as prerequisites for the stated results. The goal of this section is to interpret those assumptions and explain why they are a natural consequence of the model and our analysis. By giving concrete counterexamples, we exhibit the various ways our results can break down when certain assumptions are violated. In some cases, assumptions were used in the proof of a result but the conclusion seems robust to relaxing the assumption---this kind of robustness speaks to the generality of our qualitative observations. This section also allows for us to introduce some refinements of our previous results. Interestingly, we find that optimal regularization can preserve the monotonicity of the total error with respect to the shift strength, even when the assumptions of \cref{prop:error_order} are violated.

In \cref{app:nec_monotonicity}, we look in detail at the monotonicity condition in the partial order over covariate shifts given in \cref{def:hard}. By constructing a simple example of several LJSDs, we demonstrate that when the monotonicity condition is violated between a pair of LJSDs, i.e. the LJSDs are incomparable, the ordering of their total errors becomes model-dependent and can be changed by varying the overparameterization ratio or the activation function. Nevertheless, we introduce a quantitative measure of a shift's deviation from monotonicity that can be leveraged to deduce model-independent upper and lower bounds on its total error.

In \cref{app:large_label_noise}, we investigate the impact of label noise, i.e. the signal-to-noise ratio for \cref{eq_data_dist}. More specifically, we show what happens when the condition $\sigma_\e^2\leq\omega$ is violated. We find this condition is necessary to guarantee the error ordering in \cref{prop:error_order} and \cref{cor:error_order}, since making the label noise sufficiently large can cause the variance to greatly increase and violate the asserted inequalities. We find that optimal regularization can prevent this behavior and maintain the error ordering, even in the presence of large label noise.

\cref{sec:scale} contains a detailed discussion of the effect of the training and test covariance scales, which are denoted by $s$ and $s_*$. In \cref{sec:nonmonotonic_bias}, we begin by developing an intuition for how a mismatch between $s$ and $s_*$ affects the random feature model in high dimensions and can lead to nonmonotonicity in the bias as a function of the overparameterization ratio $\phi/\psi$ and the regularization constant $\gamma$. Following that, in \cref{sec:model_dependent}, we study how $s$ and $s_*$ are related to the model-dependent constants of \cref{eq:constants_train} and the model-dependent conditions of several propositions, such as the condition $\xi_2 \le \xi_1 \le 1\le\zeta_2/\zeta_1 \le \eta_2/\eta_1$ in \cref{eqn:error_xi} of \cref{prop:error_order}, and show that the monotonicity of the bias and total error can fail when these conditions are violated. Finally, in \cref{sec:pure_scale}, we focus on the case of pure-scale shifts, refine \cref{prop:pure_scale}, and show how the conditions for the monotonicity of the bias and the variance can be mutually exclusive.

To conclude this section, \cref{sec:nonlinear_trend} probes the prerequisite conditions in \cref{prop:slope}. We find that the overparameterization condition $\phi>\psi$ is critical, with highly nonlinear trends when it is violated. In contrast, we observe the ridgeless assumption can be relaxed without significant deviations from linearity.

\subsection{Monotonicity condition in partial order of shifts}
\label{app:nec_monotonicity}

The partial order we introduced on covariate shifts in \cref{def:hard} has been central to much of our discussion. The core idea arises from asking what assumption is necessary to make conclusions about the ordering of the functionals $\I$ and $\Ishift$ using the Harris inequality, but a natural first-principles interpretation can be developed based on how well the target function can be learned in different eigendirections from the training data, and then how much the learned function will be tested in each eigendirection based on its representation in the test data. 

Unfortunately, the partial order condition is quite strict, as it requires all eigendirections of the shift to act in concert, rendering many shifts incomparable under the partial order. Such incomparable shifts do not unanimously hurt or help in all eigendirections, and thus their combined effect is a complicated average of the effects in each eigendirection. While \cref{thm:main_b_v} accounts for this averaging in a precise way, it nevertheless remains challenging to draw simple conclusions about the behavior of incomparable shifts.

\subsubsection{Monotonicity is necessary to guarantee ordering of the total error}
The monotonicity of overlap coefficients in~\cref{def:hard} is necessary in order to guarantee the ordering of errors in \cref{prop:error_order}. For a simple counterexample, consider the following four LJSDs,
\begin{align}
    \mu_1 &= \frac{1}{4}(\delta_{(\lambda_1,\lambda_4)} + \delta_{(\lambda_2,\lambda_3)} + \delta_{(\lambda_3,\lambda_2)} + \delta_{(\lambda_4,\lambda_1)})\label{eqn:mu1}\\   
    \mu_2 &= \frac{1}{4}(\delta_{(\lambda_1,\lambda_4)} + \delta_{(\lambda_2,\lambda_3)} + \delta_{(\lambda_3,\lambda_1)} + \delta_{(\lambda_4,\lambda_2)})\\
    \mu_3 &= \frac{1}{4}(\delta_{(\lambda_1,\lambda_2)} + \delta_{(\lambda_2,\lambda_4)} + \delta_{(\lambda_3,\lambda_3)} + \delta_{(\lambda_4,\lambda_1)})\\
    \text{and}\quad\mu_4 &= \frac{1}{4}(\delta_{(\lambda_1,\lambda_1)} + \delta_{(\lambda_2,\lambda_2)} + \delta_{(\lambda_3,\lambda_3)} + \delta_{(\lambda_4,\lambda_4)})\label{eqn:mu4}\,, 
\end{align}
where $\lambda_1 = 0.6$, $\lambda_2 = 0.24$, $\lambda_3 = 0.12$ and $\lambda_4 = 0.04$. Note that $\mu_4 = \munull$ and these LJSDs have equal training and test covariance scales, i.e. $s=s_1=s_2=s_3=s_4=1$. Moreover, the partial order in~\cref{def:hard} gives $\mu_1 \ge \mu_4$ and $\mu_2 \ge \mu_4$, with all other pairs of LJSDs incomparable. In particular, focusing on $\mu_2$ and $\mu_3$, the ratios of overlap coefficients are
\begin{equation}
    \frac{r_{2,1}}{r_{3,1}} = \frac{1}{6}\,,\quad  \frac{r_{2,2}}{r_{3,2}} = 3\,,\quad  \frac{r_{2,3}}{r_{3,3}} = 5\,,\quad\text{and}\quad  \frac{r_{2,4}}{r_{3,4}} = \frac{2}{5}\,,
\end{equation}
so the sequence $r_{2,i}/r_{3,i}$ is nonmonotonic in $i$. The violation of monotonicity is minimal, in the sense that only a single ratio ($r_{2,4}/r_{3,4}$) is out of order; nevertheless, the strict model-independent ordering of test errors is broken because of this nonmonotonicity, as can be seen in \cref{fig:incomparable_LJSDs}. The model-dependence in the ordering is evident for models of varying sizes in \cref{fig:incomparable_LJSDs}(a) and models with different nonlinearities in \cref{fig:incomparable_LJSDs}(b). On the other hand, for the pairs of LJSDs for which the monotonicity of overlap coefficients is satisfied, \cref{fig:incomparable_LJSDs} shows that the ordering of test errors is satisfied for all models, as predicted by \cref{prop:error_order}.

\begin{figure}[t!]
\centering
\includegraphics[width=0.8\linewidth]{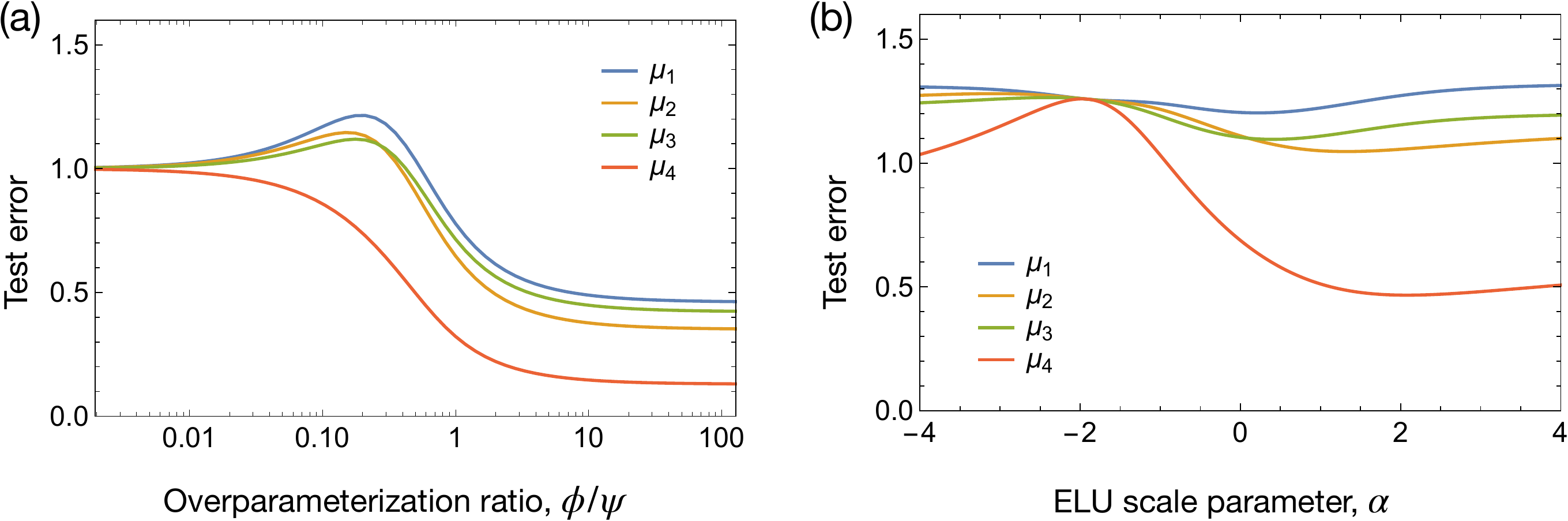}
 \caption{Predictions of \cref{thm:main_b_v} for the total test error with $\phi=0.5$, $\gamma=0.1$, $\sigma_{\e}^2=0.01$, $\phi=0.5$, and $\sigma(x) = \alpha(e^x - 1)1_{x\le0} + \text{ReLU}(x)$ \citep[i.e. the ELU function of][]{clevert2015fast} as a function of (\textbf{a}) the overparameterization ratio $\phi/\psi$ for $\alpha=0$ and (\textbf{b}) the ELU scale parameter $\alpha$ for $\phi/\psi=0.25$, for four different LJSDs $\mu_1,\ldots,\mu_4$. The $\mu_i$ are described in Eqs.~(\ref{eqn:mu1})-(\ref{eqn:mu4}). The only pairs of LJSDs that are comparable via the partial order in~\cref{def:hard} are $\mu_1 \ge \mu_4$ and $\mu_2\ge\mu_4$, and the strict ordering of the error for those pairs is seen for all values of $\phi/\psi$ and $\alpha$. The orange ($\mu_2$) and green ($\mu_3$) curves cross one another in both (\textbf{a}) and (\textbf{b}), illustrating two ways the ordering of the error exhibits model dependence induced by nonmontonicity of overlap coefficients in~\cref{def:hard}, thereby showing that the monotonicity condition is necessary for~\cref{prop:error_order}.}
 \label{fig:incomparable_LJSDs}
\end{figure}

\subsubsection{Predictable behavior for incomparable shifts}
One natural question is whether \cref{prop:error_order} has anything at all to say about shifts for which the overlap ratios are not strictly monotonic. First, note that small violations of monotonicity for finite-size systems caused by $o(1)$ overlap coefficients do not survive in the high-dimensional limit. To make this point concrete, consider the training and test covariance matrices given by $\Sigmatr =\diag(\la_1, \la_2,\ldots, \la_{n_0})= \diag(1/ n_0, 2/ n_0,\ldots, 1)$ and
\eq{
    \Sigmate_{ij}=\begin{cases}
        1/ n_0 & \text{if } i=j=2 \\
        \Sigmatr_{ij} & \text{otherwise}
    \end{cases},
}
resulting in overlap coefficients $r_2=1/n_0$ and $r_i=\la_i$ otherwise. Evidently, the monotonicity condition on $r_i/\la_i$ akin to \cref{prop:ordering_lr} does not hold; however, the limit of this covariate shift has LJSD $\mu = \munull=\text{Unif}(\h{(x,x): x\in[0,1]})$, meaning that the finite-size nonmonotonicity vanishes in the high-dimensional limit. This observation suggests that the limiting predictions might still be useful in practice despite small deviations from monotonicity at finite-size. 

In fact, even if nonmonotonicity in the overlap ratios does persist in the high-dimensional limit, \cref{def:hard} and~\cref{prop:error_order} can still provide useful information. To illustrate this point, suppose the LJSD $\mu$ is incomparable with $\munull$, implying that the function $\E[r|\la]/\la$ is nonmonotonic, and introduce the functions
\eq{
    L(x) \deq \inf_{0\leq\la\leq x} \frac{\E_\mu[r|\la]}{\la} \quad \text{and} \quad U(x) \deq \sup_{0\leq\la\leq x} \frac{\E_\mu[r|\la]}{\la}\,,
}
which satisfy $L(\la)\leq {\E_\mu[r|\la]}/{\la} \leq U(\la)$ for all $\la\geq0$, with $L(\la)$ nonincreasing and $U(\la)$ nondecreasing in $\lambda$. The functions $L$ and $U$ are a quantitative way to characterize the extent to which $\mu$ violates the monotonicity condition. Indeed, if ${\E_\mu[r|\la]}/{\la}$ is nonincreasing or nondecreasing, then ${\E_\mu[r|\la]}/{\la}=L(\la)$ or ${\E_\mu[r|\la]}/{\la}=U(\la)$ respectively. In general, $L$ and $U$ create the smallest envelope around the function $\E[r|\la]/\la$ possible with nonincreasing and nondecreasing functions (see \cref{fig_incomparable}(a,b) for an example).

In \cref{prop_incomparable}, we translate this quantitative characterization of nonmonotonicity into a bound on $B_\mu$ and $E_\mu$ relative to $B_{\munull}$ and $E_{\munull}$.
\begin{figure}[t]
    \centering
    \includegraphics[width=0.8\linewidth]{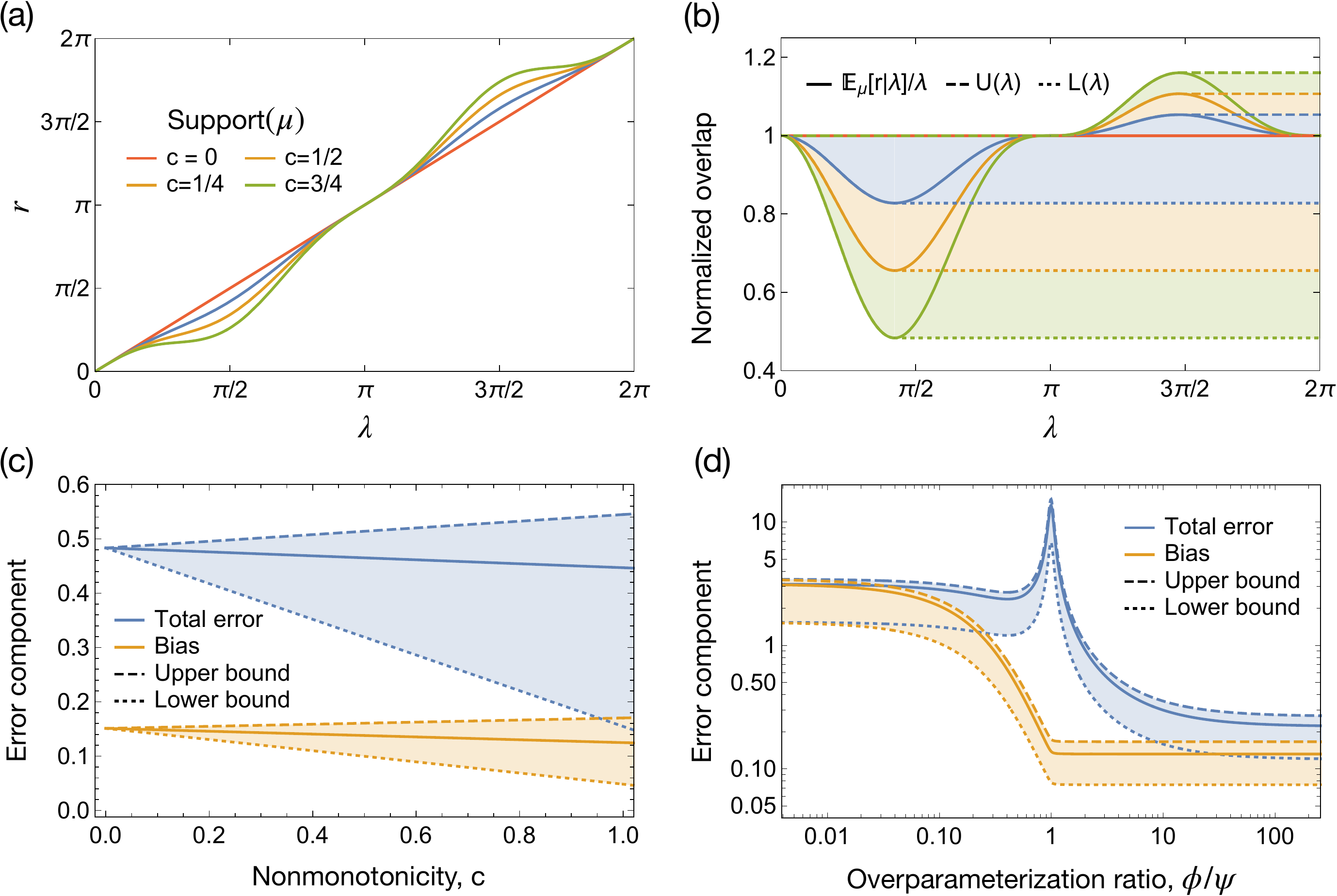}    
    \caption{We define the parametric family of LJSD $\mu$ whose distribution is given by the random variable $\pa{\la, \la - c \sin(\la)^3 }\in\R_+^2$ for $\la\sim \text{Unif}[0,2\pi]$. 
    \textbf{(a)} The support of $\mu$ in $\R_+^2$ is plotted for different values of $c\in\h{0,0.3,0.6,0.9}$. As expected $\munull$ is recovered by $c=0$ and $\mu$ becomes increasingly nonmonotonic as $c$ increases.
    \textbf{(b)} The ratio of overlaps $\E_\mu[r|\la]/\la$, as in \cref{def:hard}, is plotted as a function of $\la$ (solid line) for the same values of $c$. Clearly this function is neither nonincreasing nor nondecreasing in $\la$ when $c>0$, and moreover, the size of its oscillations are controlled by $c$. This is reflected in $U$ (dashed line) and $L$ (dotted lined), which progressively differ from $1$ as $c$ increases, thus increasing the shaded area between these curves. \textbf{(c)} The total error (blue) and bias (orange) are plotted for $\mu$ as a function of $c$, along with their upper (dashed) and lower (dotted) bounds from \cref{prop_incomparable}. For large, $c$ the gap between the upper and lower bounds are large, but as $c\to0$ they tightly bound the total error and bias. \textbf{(d)} $B_\mu$ and $E_\mu$ are tracked and enveloped by their bounds for all parameterization ratios.}
    \label{fig_incomparable}
\end{figure}
\begin{proposition}\label{prop_incomparable}
    Let $\mu$ be a LJSD such that $\E_\mu[\la]=s=s_*=\E_\mu[r]$. Then, in the setting of~\cref{thm:main_b_v},
    \eq{\label{eq_bias_bound}
        \frac{\E_{\mutr}\qa{ \la L(\la) }}{\E_{\mutr} \qa{\la} } B_{\munull} \leq B_\mu \leq \frac{\E_{\mutr} \qa{\la U(\la)} }{\E_{\mutr}\qa{ \la} } B_{\munull},
    }
    and, if $\sigma_{\e}^2 \leq \omega$,
    \eq{\label{eq_total_error_bound}
        \frac{\E_{\mutr}\qa{ \la L(\la) }}{\E_{\mutr} \qa{\la} } E_{\munull} \leq E_\mu \leq \frac{\E_{\mutr} \qa{\la U(\la)} }{\E_{\mutr}\qa{ \la} } E_{\munull}.
    }
\end{proposition}

In the special case where $\mu=\munull$, \cref{prop_incomparable} yields a tight bound since $L(\la)=U(\la)=1$. We exemplify \cref{prop_incomparable} for a specific family of LJSDs $\mu$ given by the random variable
\eq{    
    \pa{\la, \la - c \sin(\la)^3 }\in\R_+^2 \text{ for } \ \la \ \sim \text{Unif}[0,2\pi]
}
parameterized by some constant $c\in[0,1]$ (see \cref{fig_incomparable}(a) for the LJSD's support). When $c=0$, we recover the case of no shift $\munull$. By increasing $c$, the magnitude of the nonmonotonicity increases, since $c$ is the coefficient of the oscillatory term $\sin(\la)^3/\la$. Note that for all $c$, we have
\eq{
    \E_\mu[r] = \int_0^{2\pi} \frac{x - c\sin(x)^3 }{2\pi} dx = \pi,
}
implying $\mu$ is always an equal-scale shift. Calculating further, we find that
\al{\label{eq_muc_Erla}
    \E_\mu[r|\la]/\la &= 1-c\sin(\la)^3/\la,\\
    \label{eq_muc_L}L(\la) &= \begin{cases} 1 - c \sin(\la)/\la & \text{if } \la \leq x_L \\ 1 - c \sin(x_L)/x_L & \text{if } \la>x_L \end{cases}, \quad \text{and} \\ 
    \label{eq_muc_U} U(\la) &= \begin{cases} 1 & \text{if } \la\leq\pi \\ 1 - c \sin(\la)/\la & \text{if } \pi<\la \leq x_U \\ 1 - c \sin(x_U)/x_U & \text{if } x_U<\la \end{cases},
}
where $x_L\approx 1.324$ and $x_U\approx4.641$ are solutions to $3x=\tan(x)$. In \cref{fig_incomparable}(b), we can see the effect of $c$ on the functions $L$ and $U$: as $c$ increases, the envelopes around $\E_\mu[r|\la]/\la$ must also increase. Indeed, many possible LJSDs could lie within each envelope, which should make intuitive the requirement for a bound, as well as for this bound to loosen as the envelope increases in size.

Finally, the coefficients $\E_\mu[\la L(\la)]$ and $\E_\mu[\la U(\la)]$ can be calculated by integrating \cref{eq_muc_L,eq_muc_U} with respect to the uniform density over $[0,2\pi]$. Clearly, the difference in these coefficients (and hence the bounds on the bias and test error) depend on how close $L$ and $U$ are, i.e. the size of the envelope. In \cref{fig_incomparable}(c), we can observe this behavior. As $c$ increases, so does the envelope around $\E_\mu[r|\la]/\la$ and thus also the looseness of the bound on the bias and total error of the shift. Note that for the simple LJSDs considered here, we found that $L$, $U$, $B_\mu$, and $E_\mu$ are linear in $c$, which occurs because the nonmonotonic component of $\E_\mu[r|\la]$, i.e. $\sin(\la)^3/\la$, is scaled by $c$. Of course, this linear behavior is specific to this example and need not occur in general. Finally, in \cref{fig_incomparable}(d), we observe how the respective bounds follow the bias and total error as a function of the overparameterization ratio, $\phi/\psi$.

\subsection{Label noise}
\label{app:large_label_noise}

\cref{prop:error_order} and \cref{cor:error_order} provide conditions for the bias and total error to decrease in response to easier shifts, but they place no explicit constraints on the variance. Indeed, the variance is not guaranteed to decrease, and while it happened to do so in the setting of \cref{fig:predictions}, it may instead increase, as can be seen in \cref{fig:label_noise}. Still, for the total error to decrease in response to easier shifts, any increase in the variance must be offset by a larger decrease in the bias. The small label-noise condition $\sigma_{\e}^2 \leq \omega$ ensures that this is the case. Note that since the label noise is simply an additive effect on the labels, its magnitude does not affect the bias, but recall from \cref{eq_bias_variance_def} that this is not the only source of variance in the random feature model \citep{adlam2020understanding}.

\cref{fig:label_noise}(a) shows the bias, variance, and total error as a function of the label noise for an easy shift ($\theta = 2$) and a hard shift ($\theta = -2$). As expected, the bias is independent of the label noise, and the variance and total error increase with increasing $\sigma_{\e}^2$. For small label noise ($\sigma_{\e}^2 < \omega$), the total error for the hard shift is larger than that of the easy shift, consistent with \cref{prop:error_order}. For $\sigma_{\e}^2 > \omega$, however, as the label noise increases, the variance for the easy shift increases more rapidly than that of the hard shift, leading to a crossing of the error curves so that the total error for the easy shift becomes larger than that of the hard shift.

\begin{figure}[t!]
\centering
\includegraphics[width=\linewidth]{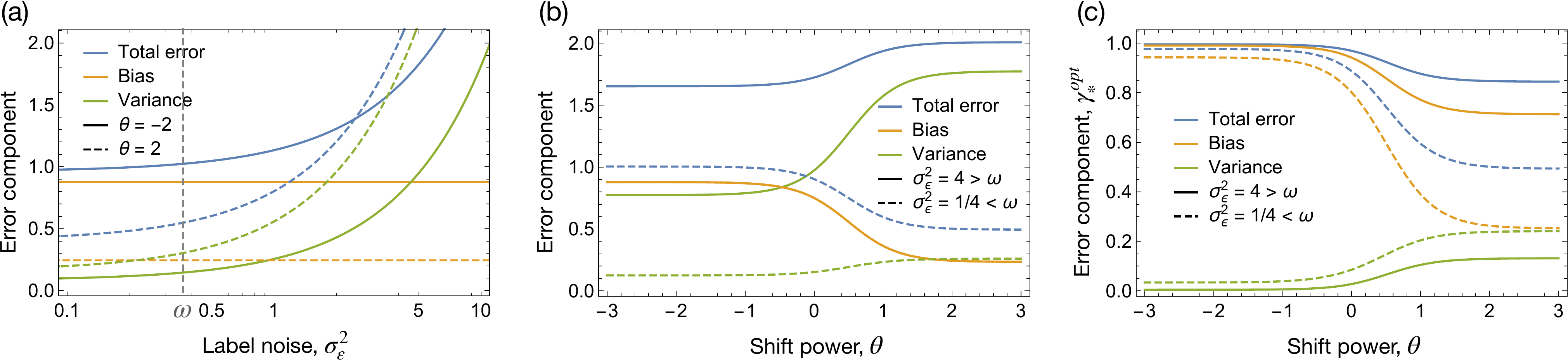}
 \caption{Predictions from \cref{thm:main_b_v} for the total test error, bias, and variance with $\alpha=4$, $\phi=4$, $\psi=1/4$, $\sigma = \text{ReLU}$, $\gamma=10^{-4}$, and $\omega = 1-\tfrac{2}{\pi} \approx 0.36$ as a function of \textbf{(a)} increasing label noise and \textbf{(b,c)} increasing shift power $\theta$ (corresponding to progressively easier shifts in the sense of \cref{def:hard}). \textbf{(a)} For sufficiently large label noise, the variance for easy shift increases more rapidly than that of the hard shift, leading to a crossing of the error curves so that the total error for the easy shift becomes larger than that of the hard shift. \textbf{(b)} The predictions of \cref{prop:error_order} are corroborated when $\sigma_{\e}^2 < \omega$, but when $\sigma_{\e}^2 > \omega$ the total error and variance can increase with increasing $\theta$. \textbf{(c)} Under the same configuration in (\textbf{b}), but with optimal regularization, the total error is seen to decrease for easier shifts, irrespective of the label noise.}
 \label{fig:label_noise}
\end{figure}

\cref{fig:label_noise}(b) shows a complementary perspective on the same phenomenon. In this panel, the bias, variance, and total error are plotted as a function of the shift power for large ($\sigma_{\e}^2 = 4 > \omega$) and small ($\sigma_{\e}^2 = 1/4 < \omega$) label noise. As predicted by \cref{prop:error_order}, the bias is the same for both values of label noise, and it decreases for easier shifts. On the other hand, the variance is seen to increase for easier shifts, and for large label noise, it increases sufficiently rapidly so as to induce an increase in the total error, whereas for small label noise it does not. 

Together, \cref{fig:label_noise}(a) and \cref{fig:label_noise}(b) show that the small label-noise condition in \cref{prop:error_order} and \cref{cor:error_order} is necessary in order to prevent a large increase in the variance from breaking the ordering of the total error. Of course, in practice it might not be possible to avoid large label noise, and regularization is an effective remedy to prevent the variance from becoming too large in such cases. It is therefore natural to wonder whether tuning the regularizer optimally is sufficient to guarantee that the total error decreases in response to easier shifts, irrespective of the label noise. \cref{prop:opt_gamma} proves this is indeed the case for infinite overparameterization and \cref{fig:label_noise}(c) also shows it to be the case for finite overparameterization ratios in the setting of \cref{fig:label_noise}(b). 

\subsection{Discrepancies in scale between the training and test distributions}
\label{sec:scale}

In the high-dimensional limit, the behavior of the random feature matrix can be modeled by a projection of the data into a linear component and noise component, whose scales are determined by the activation function $\sigma$. This notion is made mathematically precise by the linearization technique of \citet{adlam2019random}, reviewed in \cref{sec_gaussian}. Recalling the definition of the constants in \cref{eq:constants_train}, the linear component has scale $\sqrt{\rho}$ and the noise component has scale $\sqrt{\eta-\zeta}$. Note that the bias of the predictive function is only affected by its linear component, whereas the variance depends on the noise component as well. This observation helps explain why results concerning the bias, such as \cref{prop:error_order} in \cref{eqn:bias_xi}, only make assumptions on $\rho$, while results concerning the variance or total error, such as \cref{prop:error_order} in \cref{eqn:error_xi}, also make assumptions on $\eta$. 

In the presence of covariate shift, the training and test distributions may have different scales, leading to two separate sets of constants and an expanded set of possible relationships between them. As a result, for unequal scales, many of our results require specifying substantially more intricate conditions on the constants (e.g. compare \cref{prop:error_order} to \cref{cor:error_order}). In this section, we provide a detailed analysis of these conditions and the effect of unequal scales.

\subsubsection{Unequal scales and nonmonotonicity of the bias}
\label{sec:nonmonotonic_bias}
Among the various constants induced by the activation function $\sigma$, of particular interest is $\xi = \sqrt{\rhos/\rho}$, which is the ratio of the scales of the linear component of $\sigma$ on the training and test distributions. \cref{thm:main_b_v} shows that the bias and variance have markedly different dependence on $\xi$---while the variance scales multiplicatively with $\xi^2$, the bias has nontrivial quadratic dependence.

In fact, the behavior of the bias can be radically different depending on whether $\xi$ is greater or less than one. Indeed, \cref{prop:error_order,prop:b} both require the assumption that $\xi\leq 1$, since the bias can be highly nonmonotonic when $\xi>1$. To elucidate this dichotomy, it is fruitful to understand how the bias behaves as a function of $x$. Because $x$ is monotonic with respect to various other parameters, such as $\gamma$ (see \cref{eq:dxdgamma_sm1}) and $\phi/\psi$ (see \cref{eq:dxdpsi}), any nonmonotonicity with respect to $x$ implies nonmonotonicity with respect to those parameters as well.

Recall that the bias compares the average prediction at a test point $\bfx$ against its label in squared error, and then averages over the test point and $\beta$ (see \cref{eq_bias_variance_def}). Using the linearization technique and splitting the random features on the test point into their linear and noise components, we find that the average predictor satisfies
\eq{
    \mE[\hat{y}(\x)] = \beta^\top\frac{1}{\sqrt{n_0}} \E \qa{X K^{-1} F^\top f / n_1} \approx \hat{\beta}^\top{\bfx}/{\sqrt{n_0}},\label{eq_ex_est}
}
where $\hat{\beta}\deq\sqrt{\rhos} \,\E \qa{{W^\top F K^{-1} X^\top }/\p{{n_1\sqrt{n_0}}}}\beta$. Writing $\beta$ and $\hat{\beta}$ in the eigenbasis of $\Sigmatr$, where their coefficents in eigendirection $\bfv_\la$ are $\beta_\la = \bfv_\la^\top \beta$ and $\hat{\beta}_\la = \bfv_\la^\top \hat{\beta}$, the bias can be expressed as,
\al{
    \bias{\Sigmate} &= \E_{\x,\beta} [(\E[\hat{y}(\x)]-y(\x))^2] \\
    &= \E_{\x,\beta}[(\hat{\beta}-\beta)^\top\bfx/\sqrt{n_0})^2] \\
    &= \frac{1}{n_0}\E_{\beta}[(\hat{\beta}-\beta)^\top \Sigmate (\hat{\beta}-\beta)] \\
    &= \frac{1}{n_0} \sum_{\la,\la'} \E_{\beta}[(\hat{\beta}_\la-\beta_\la)^\top \bfv^\top_\la \Sigmate \bfv_{\la'} (\hat{\beta}_{\la'}-\beta_{\la'})] \\
    &\approx \frac{1}{n_0}\sum_\la r_\la \E_{\beta} (\hat{\beta}_\la-\beta_\la)^2\label{eq_bias_finite_decomp}\,,
}
where final line follows from the (nontrivial) observation that $\mathbb{E}_\beta[(\hat{\beta}-\beta)(\hat{\beta}-\beta)^\top]$ is approximately diagonal in the eigenbasis of $\Sigmatr$. \cref{eq_bias_finite_decomp} shows that the bias can be understood as a weighted measurement of how well the expected estimator $\hat{\beta}$ approximates $\beta$ in each of the eigendirections.
\begin{figure}[t]
\centering
\includegraphics[width=\linewidth]{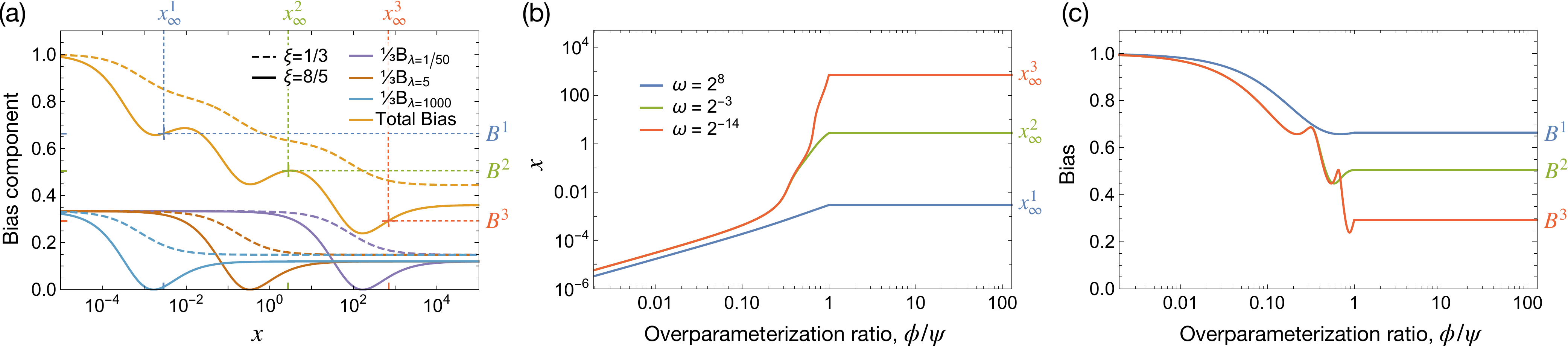}
\caption{Predictions for the bias under LJSD $\mu=\p{\delta_{1/50,1}+\delta_{5,1}+\delta_{1000,1}}/3$ with $\gamma = 0$ and $\phi = 1$ for two different values of $\xi$, $1/3$ and $8/5$. \textbf{(a)} The components of the bias in each eigendirection, \cref{eq_eigendecomp_bias}, normalized by the probability of the eigenvalue (i.e. $1/3$ under $\mu$) are plotted as a function of $x$ for each $\la$ and both $\xi$. When $\xi=1/3<1$ (dashed), all components start at $1/3$, then decrease monotonically to $(1-\xi)^2/3$; when $\xi=8/5>1$ (solid), each component starts at $1/3$, decreases, has a root and minimum at $x=\phi/(\la(\xi-1))$, and then increases again to $(\xi-1)^2/3$. This difference between $\xi<1$ and $\xi>1$ translates into the total bias (orange curves), which is obtained by summing the components from each $\la$. \textbf{(b)} $x$ as a function of $\phi/\psi$ for three values of $\omega$. While $x$ is monotonic to $\phi/\psi$, not all values of $x$ are achievable, with $x$ plateauing to a constant for $\phi>\psi$. These limiting values of $x$ and their corresponding biases are denoted with dotted lines in (a). \textbf{(c)} Combining (a) and (b), the bias is also nonmonotonic in $\phi/\psi$. Smaller values of $\omega$ allow for larger values of $x$, which uncover greater portions of the bias curve in (a), leading to additional nonmonotonicity.
}
\label{fig:xi}
\end{figure}

In fact, it is straightforward to extend this interpretation to the high-dimensional limit by rearranging the expression in \cref{eqn:main_bias} to find $\bias{\mu} = \int B_\lambda(x;\xi) d\mutr(\lambda)$, where 
\eq{\label{eq_eigendecomp_bias}
	B_\lambda(x;\xi) \deq \E[r|\la] w_\xi(x)\,,\qquad\text{and}\qquad w_\xi(x)\deq \pa{1-\xi + \xi \frac{\phi}{\phi + x \la} }^2\,.
}

Comparing \cref{eq_bias_finite_decomp} to \cref{eq_eigendecomp_bias}, we can immediately see that $\E_{\beta} (\hat{\beta}_\la-\beta_\la)^2 \to w_\xi(x)$ in the high-dimensional limit. The factor $\E[r|\la]$ reweights the contribution of each eigendirection by its relative importance in the test distribution. Note that $w_\xi$ depends on the scale ratio $\xi$ but is otherwise independent of the shift. For the simple uniform setting in which $\mathbb{E}[r|\la] = 1$, \cref{fig:xi}(a) illustrates $B_\lambda(x;\xi) = w_\xi(x)$ for $\la\in\h{10^{-2}, 5, 10^3}$ and $\xi\in\h{1/3,8/5}$. It can be observed that the overall shape depends strongly on the value of $\xi$, with monotonicity for $\xi = 1/3 < 1$ and nonmonotonicity for $\xi = 8/5> 1$.

To build intuition for this behavior, we first note that $x\to 0$ yields the zero predictor, i.e. ${\hat{\beta} = 0}$, which follows from the fact that $w_\xi(0)$ is independent of $\la$ and satisfies ${w_\xi(0) = 1 = \mathbb{E}[\beta^\top \beta]/n_0}$.
Similarly, $\lim_{x\to \infty} w_\xi(x) = (1-\xi)^2$, which is also independent of $\la$, and is consistent with a mean predictor that scales each eigendirection uniformly, i.e. $\hat{\beta} = \xi \beta $. Piecing these two limits together paints a picture of how the bias depends on $x$, but the details depend on the value of $\xi$.

When $\xi\le 1$, $w_\xi$ is a decreasing function of $x$ and asymptotes for large $x$ to $(1-\xi)^2$, implying that $x\to\infty$ is optimal from the perspective of bias reduction for all $\la$. Still, unless $\xi = 1$, the predictor remains biased and \emph{underestimates} all $\beta_\la$ by a factor of $\xi$. This underestimation is due to the fact that $\xi < 1$ implies $\rhos < \rho$, i.e. the linear component of the random feature projection of the data is \emph{smaller} on the test distribution than it is on the training distribution, using the terminology of the linearization technique described above.

When $\xi>1$, $w_\xi$ also asymptotes to $(1-\xi)^2$ for large $x$, but this behavior is due to the opposite effect, namely an \emph{overestimation} of all $\beta_\la$ by a factor of $\xi$. This overestimation occurs because $\xi > 1$ implies $\rhos > \rho$, i.e. the linear component of the random feature projection of the data is \emph{larger} on the test distribution than it is on the training distribution. Because $x=0$ corresponds to the zero predictor (which necessarily underestimates the coefficients in each eigendirection), there must be a transition between underestimation and overestimation, leading to a local minimum and nonmonotonicity in $x$. It follows from \cref{eq_eigendecomp_bias} that the local minimum is in fact a zero and occurs at $x=\phi/(\la(\xi-1))$, implying that for every $\la$ there is a unique $x$ that renders the predictor unbiased in the corresponding eigendirection. However, because this optimal $x$ depends on $\la$, it is not possible to optimize all eigendirections simultaneously, creating a tradeoff that can result in highly nonmonotonic behavior of the bias.

This behavior is depicted in \cref{fig:xi}(a), which shows that when $\xi=1/3<1$ (dashed lines), all curves are strictly decreasing, but when $\xi=8/5>1$ (solid lines), the minima of the curves for each $\la$ occur at different values of $x$, resulting in several local minima. Note also that $x\to\infty$ is not optimal.

In \cref{fig:xi}(b), we show how varying $\phi/\psi$ and $\omega$ adjusts the value of $x$. Observe that $x$ is monotonic in $\phi/\psi$ and $\omega$ but not all values of $x>0$ are achievable. Specifically, when $\phi>\psi$, $x$ is constant in $\phi/\psi$ and is bounded by $1/\omega$. Finally, \cref{fig:xi}(c) demonstrates that because $x$ is monotonic in $\phi/\psi$, the nonmonotonic behavior of the bias as a function of $x$ when $\xi>1$ causes nonmonotonicity of the bias as a function of $\phi/\psi$ as well.

\subsubsection{Scale-induced model dependence}
\label{sec:model_dependent}

We now examine~\cref{prop:b} and the consequences of violating the condition $\xi\le 1$. In~\cref{fig:scale_violation}(a), we show the total error, bias, and variance as a function of the overparameterization ratio for a model and LJSD for which $\xi > 1$. While \cref{prop:v} does still guarantee that the variance is nonincreasing in the overparameterized regime, \cref{prop:b} no longer constrains the bias, and indeed we observe nonmonotonicity of the bias, with it actually increasing for large overparameterization ratios. Consistent with the discussion in the previous section, the underlying reason for this behavior can be traced to the $\xi$-dependence of \cref{eqn:main_bias}, which allows for nontrivial interference between the various summands for $\xi > 1$. This tradeoff between terms can be seen as a function of the overparameterization ratio $\phi/\psi$ in \cref{fig:scale_violation}(a), which shows that the bias can be nonmonotonic and increase sufficiently rapidly so as to induce an increase in the total error, even in the overparameterized regime. While it may seem counterintuitive that the bias could increase when the overparameterization ratio is increased, the fact that $\partial x /\partial \psi \le 0$ (see \cref{eq:dxdpsi} and \cref{fig:xi}(b)) means that the discussion in the previous section about $x$ provides the explanation.

It is natural to wonder if the observed nonmonotonicity in \cref{fig:scale_violation}(a) might be a consequence of sub-optimal regularization. To examine this question, in the same figure we also show the behavior of the model under optimal regularization, and indeed we observe that the nonmonotonicity disappears and the bias and total error become nonincreasing functions of the overparameterization ratio.
\begin{figure}[t]
\centering
\includegraphics[width=0.8\linewidth]{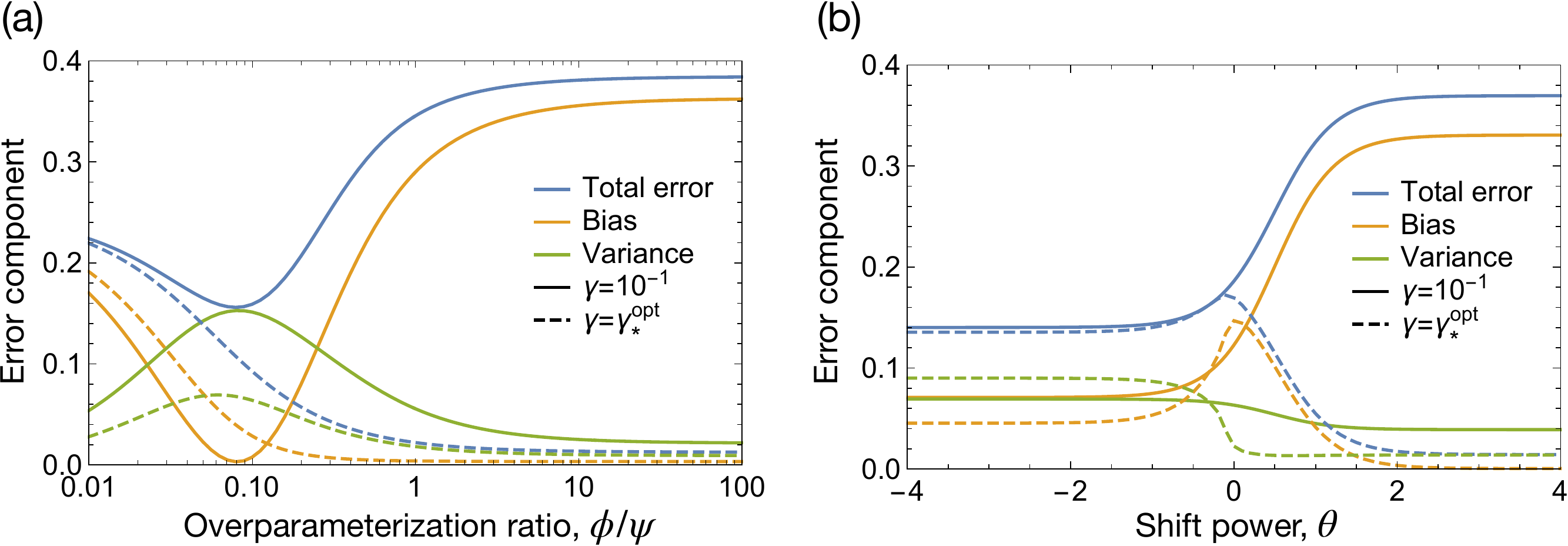}
\caption{Predictions for the total error, bias, and variance for the random feature model with $\sigma(x) = x \exp(-x^2)$, $\phi=0.25$, $\sigma_{\e}^2 = 0.1$, and $\theta = 2$ in (\textbf{a}) and $\phi/\psi = 2$ in (\textbf{b}), for a scaled version of the $(4,\theta)$-diatomic LJSD with $s=1$ and $s_*=0.25$, violating the model-dependent conditions in ~\cref{prop:error_order} and ~\cref{prop:b}. For fixed regularization, $\gamma = 0.1$, the bias and total error are no longer nonincreasing functions of (\textbf{a}) the overparameterization ratio when $\phi/\psi > 1$, nor (\textbf{b}) the shift power~$\theta$. For optimal regularization, $\gamma = \gammaopt_*$, the bias and total error are nonincreasing in $\phi/\psi$, but they are still nonmonontic as a function of $\theta$.}
\label{fig:scale_violation}
\end{figure}

Next, we turn to \cref{prop:error_order} and a comparison of the bias and total error of two different shifts. In this setting, there are two separate sets of model-dependent constants to consider, leading to the somewhat complicated set of conditions given in \cref{eqn:bias_xi} and \cref{eqn:error_xi}. We first remark that these inequalities are sufficient to establish the results, but they are not necessarily optimal. In particular, the condition $\xi_2\le\xi_1\le 1$ in \cref{eqn:bias_xi} forces $\xi_1 = 1$ whenever $\mu_2 = \munull$, preventing any conclusions about \emph{easy} shifts in the presence of unequal scales. While the bounds on $\xi_1$ in~\cref{eqn:bias_xi} can be loosened to allow for such conclusions, the improvements require introducing additional nontrivial spectral information, thereby coupling the model and the shift in an intricate way. We defer this more detailed analysis to future work.

Instead, we focus on the setting in which the two test scales are equal but differ from the training scale in such a manner as to violate the conditions of~\cref{prop:error_order}, i.e. $1 < \xi = \xi_1 = \xi_2$. In this case, the bias and total error can actually increase as the shifts become easier, as illustrated in \cref{fig:scale_violation}(b). Again, the origin of this behavior can be traced to the nontrivial interference between the various summands in \cref{eqn:main_bias} for $\xi > 1$. Interestingly, \cref{fig:scale_violation}(b) also shows that, unlike the behavior as a function of the overparameterization ratio, optimal regularization does not seem to guarantee monotonicity of the bias and error with respect to shift strength.

\subsubsection{Pure-scale shifts}
\label{sec:pure_scale}
Much of the complexity in the statements and analysis of Props.~\ref{prop:error_order}-\ref{prop:slope} stems from the model's dependence on the training and test covariance scales, $s$ and $s_*$. To examine the effect of scale alone, we consider the simple pair of LJSDs $\mu_1 = \delta_{(s,s_1)}$ and $\mu_2=\delta_{(s,s_2)}$, which are the asymptotic limits of  $\Sigmatr = s I_{n_0}$, $\Sigmate_1 = s_1 I_{n_0}$, and $\Sigmate_2 = s_2 I_{n_0}$. We assume $s_1\le s_2$ so that $\mu_1\le \mu_2$, as in the statement of \cref{prop:error_order}. As an example, we focus on the activation function $\sigma(x) = x \exp(-x^2)$, which was introduced in \cref{fig:scale_violation}. By straightforward Gaussian integration, one can easily compute the scale-dependence of the quantities appearing in \cref{prop:error_order}:
\begin{align}
    \eta_1 &= \eta(s_1) = \frac{s_1}{(1+4s_1)^{3/2}}\,,\quad \zeta_1 = \zeta(s_1) = \frac{s_1}{(1+2s_1)^{3}}\,, \quad \xi_1=\xi_1(s_1) = \frac{(1+2s)^{3/2}}{(1+2s_1)^{3/2}}\,,\\
    \eta_2 &= \eta(s_2) = \frac{s_2}{(1+4s_2)^{3/2}}\,,\quad \zeta_2 = \zeta(s_2) = \frac{s_2}{(1+2s_2)^{3}}\,,\quad \xi_2=\xi_2(s_2) = \frac{(1+2s)^{3/2}}{(1+2s_2)^{3/2}}\,.
\end{align}

\begin{figure}[t!]
\centering
\includegraphics[width=0.8\linewidth]{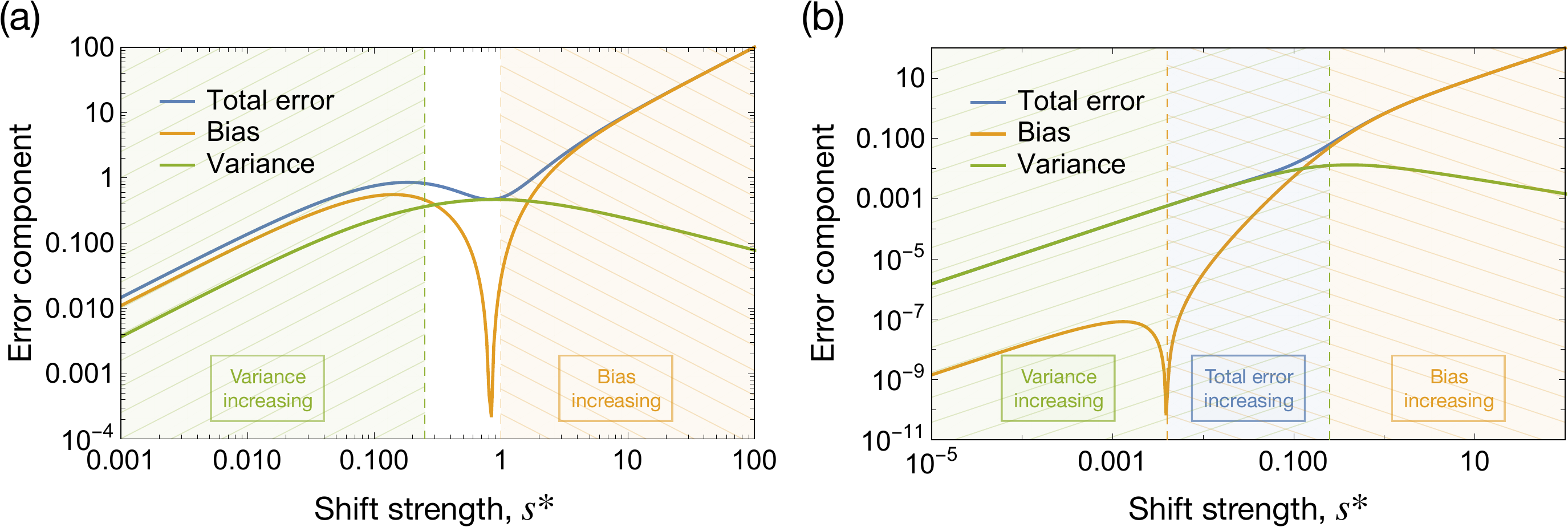}
 \caption{Predictions for the effect of pure-scale shifts on the total error, bias, and variance, where the shifted distribution is simply a scaled version of the identity, i.e. $\mu = \delta_{(s,s_*)}$, with (\textbf{a}) $s=1$ and (\textbf{b}) $s=1/250$, for $\gamma=10^{-3}$, $\phi=1/8$, $\phi/\psi = 4$, $\sigma_{\e}^2 = 1$, and $\sigma(x) = x \exp(-x^2)$. The model-dependent conditions in \cref{prop:pure_scale} predict monotonicity of the bias, variance, and total error in the indicated shaded regions. Nonmonotonicity can be observed outside these regions, highlighting the importance of these conditions in \cref{prop:pure_scale}. (\textbf{a}) For $s=1$, the regions of increasing bias and increasing variance do not overlap, and, as such, the region in between need not have increasing total error. (\textbf{b}) For $s=1/250$, the overlap of the regions of increasing bias and increasing variance extends from $s_*=1/250$ to $s_*=1$, and the total error is observed to be increasing in this region.}
 \label{fig:model_dependent_conditions}
\end{figure}
We now examine how the scale-dependence of these quantities affects the bounds in \cref{prop:pure_scale}. For the bias, the condition in \cref{eq:bias_pure_scale}, $\xi_2\le\xi_1\le 1$, can be rewritten in this setting as
\begin{align}
\label{eqn:bias_bound}
    s \le s_1 \le s_2\,.
\end{align}
All pairs of shift scales $s_1$ and $s_2$ satisfying \cref{eqn:bias_bound} will obey $\bias{\mu_1}\le \bias{\mu_2}$, which implies that the bias corresponding to $\mu=\delta_{(s,s_*)}$ should be a nondecreasing function of the shift strength $s_*$, so long as $s\le s_*$. The earlier discussion in \cref{sec:scale} explains this behavior. See \cref{fig:model_dependent_conditions} for an illustration, which demonstrates how the bias may generally be nonmonotonic as a function of shift strength, but so long as $s_*\ge s$ (to the right of the dashed orange lines), the bias is nondecreasing. 

For the variance, the condition in \cref{eq:variance_pure_scale}, $1\le\zeta_2/\zeta_1 \le \eta_2/\eta_1$, can be written in this setting as
\begin{align}
\label{eqn:s2_bound}
    s_1 \le s_2 \le \frac{1}{4}\left(\sqrt{1+\frac{2}{s_1}}(1+2s_1) - (3+2s_1)\right)\,.
\end{align}
The maximum value of $s_*$ for which all pairs of shift scales $s_1\le s_2 \le s_*$ satisfy \cref{eqn:s2_bound} is given by $s_*=1/4$, which implies that the variance corresponding to $\mu=\delta_{(s,s_*)}$ should be a nondecreasing function of the shift strength $s_*$, so long as $s_*\le 1/4$. This behavior can be seen in \cref{fig:model_dependent_conditions}, which illustrates how the variance may generally be nonmonotonic as a function of shift strength, but so long as $s_*\le 1/4$ (to the left of the dashed green lines), the variance is nondecreasing. 

Finally, by combining the above bounds for the bias and the variance, we see that the total error corresponding to $\mu=\delta_{(s,s_*)}$ should be a nondecreasing function of the shift strength $s_*$, so long as $s\le s_*\le 1/4$. If $s > 1/4$, these conditions cannot be satisfied, which is the case in \cref{fig:model_dependent_conditions}(a), which shows that the regimes for which the bias and the variance are nondecreasing do not overlap. Indeed, when $1/4<s_*<s=1$, the total error is seen to decrease as a function of $s_*$. Note that it is not strictly necessary for the bias and the variance to both be nondecreasing for the total error itself to be nondecreasing, since one term can dominate the other, and indeed \cref{fig:model_dependent_conditions} shows this to be the case for large (and small) $s_*$. \cref{fig:model_dependent_conditions}(b) illustrates the case $s=1/250$, which allows for a finite range of $s_*$ (shaded in blue) satisfying $s\le s_*\le 1/4$. In this regime, both the bias and the variance are nondecreasing, and so the error must be as well, which is indeed the case in the figure.

\subsection{Nonlinear trends with nonzero ridge constant or underparameterization}
\label{sec:nonlinear_trend}
\begin{figure}[t!]
\centering
\includegraphics[width=\linewidth]{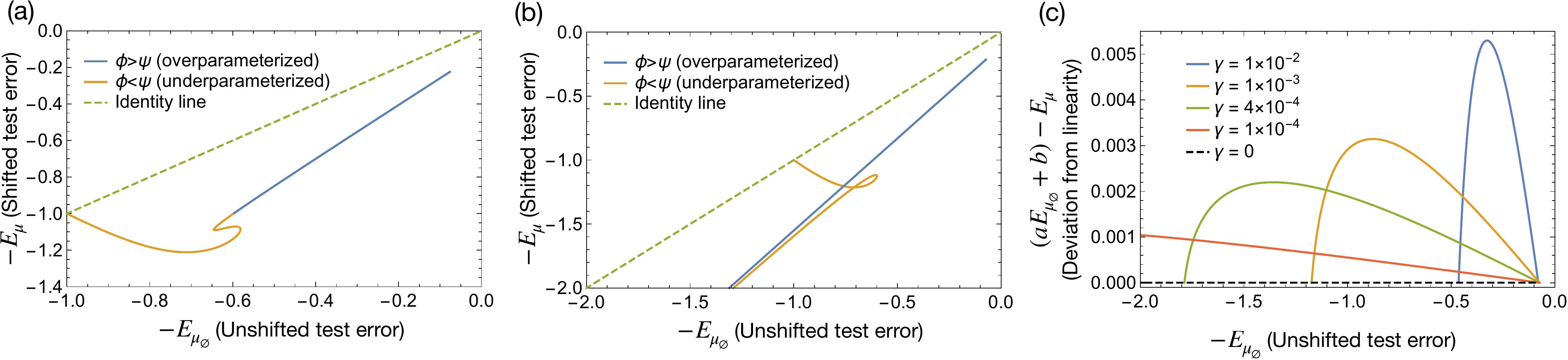}
 \caption{Predictions for shifted versus unshifted error for models with varying values of the overparameterization ratio $\phi/\psi$, obtained via \cref{thm:main_b_v} and \cref{cor:ridgeless_lim} for the $(3,-1/2)$-diatomic LJSD with $\phi=0.5$, $\sigma=\text{ReLU}$, and $\sigma_{\e}^2 = 0.01$. (\textbf{a}) For small but finite regularization, ${\gamma=0.005}$, the relationship is close to linear in the overparameterized regime, but it is markedly nonlinear in the underparameterized regime, highlighting the importance of the condition $\phi>\psi$ in \cref{prop:slope}. (\textbf{b}) The nonlinear relationship in the underparameterized persists in the ridgeless ($\gamma=0$) setting. (\textbf{c}) For nonzero ridge constant, the relationship is nonlinear in the overparameterized regime, but because the deviations from linearity tend to be small, they are not easily visible in (a), so here we plot the deviations $(a \err{\munull}+b)-\err{\mu}$ directly for several values of $\gamma$. The constants $a$ and $b$ are chosen to fit a line through the two points corresponding to $\phi/\psi = 1 + 2^{-16}$ and $\phi/\psi = 1 + 2^{16}$, which correspond to the points where the various curves cross the $x$-axis. For $\gamma = 0$, there is no deviation from linearity, as predicted by \cref{prop:slope}, but small deviations are evident for nonzero $\gamma$.}
 \label{fig:deviations_from_linear_trend}
\end{figure}
The exact linear relationship between in-distribution and out-of-distribution generalization characterized by \cref{prop:slope} in \cref{eqn:slope_modified} relies crucially on the overparameterization condition, $\psi<\phi$, as evidenced in \cref{fig:deviations_from_linear_trend}(a,b), which show marked nonlinearity in the underparameterized regime. This observation is perhaps unsurprising, as severely underparameterized models tend towards chance predictions, which produce comparable errors on shifted and unshifted data. Indeed, similar nonlinear behavior is seen in the low-accuracy regime for realistic models~\citep[Figure 17]{recht2019imagenet}. \cref{fig:deviations_from_linear_trend}(c) illustrates that the ridgeless condition in \cref{prop:slope} is necessary in order to achieve an exactly linear relationship. In this case, the deviations from linearity are quite small, so we plot the deviations $(a \err{\munull}+b)-\err{\mu}$, where the constants $a$ and $b$ are chosen to fit a line through the two points corresponding to $\phi/\psi = 1 + 2^{-16}$ and $\phi/\psi = 1 + 2^{16}$. If the function were perfectly linear, there would be no deviations from this fit line; instead, we observe small deviations for $\gamma > 0$. Interestingly, these deviations are quite small, suggesting an approximately linear relationship might hold more broadly than the narrow confines of \cref{prop:slope}.

\section{Conclusion}
\label{sec:conclusion}
Random feature models are a powerful and practically useful class of machine learning models that are analytically tractable in the high-dimensional limit. Because these models can exhibit rich phenomenology, they have emerged as fruitful null models to separate phenomena that are merely consequences of high-dimensional statistics from those that are unique to deep learning. In this work, we extended this line of research significantly by presenting exact results for the test error, bias, and variance for random feature regression in the high-dimensional limit under covariate shift. Our results enabled the exploration of many interesting phenomena that have been observed empirically for distribution shift.

We defined a partial order over covariate shifts (motivated by the setting of linear regression) to categorize the relative hardness of shifts. We found that when the shift does not change the covariate scale, the partial order provides a model-independent condition that determines when a shift will help or hurt performance. In contrast, when the shift does adjust the scale, its impact on performance is coupled to the model through a set of constants obtained from Gaussian integrals of the activation function. We saw that when there is a discrepancy between the scales of the training and test distributions, counterintuitive behavior can occur. For example, the bias can be highly nonmonotonic as a function of the number of random features or the amount of ridge regularization. Guided by our mathematical analysis, we developed a statistical explanation for many of these unexpected behaviors, including how they may be remedied or whether they remain when the regularization is tuned optimally.

In the simple context of random feature regression, our results capture many empirical phenomena that have been widely observed for realistic models and deep neural networks, suggesting that their occurrence is likely unrelated to deep learning per se, but rather a consequence of more prosaic factors such as high-dimensionality. One such important observation is that the conventional benefits of overparameterization, which exist in the absence of covariate shift, extend to settings with covariate shift in random feature regression. Additionally, our results predict a linear relationship (and the magnitude of its slope) between the generalization error on shifted and unshifted data in the ridgeless, overparameterized random feature model, matching the empirical phenomenology observed in broad generality in \citet{recht2018cifar, recht2019imagenet, pmlr-v119-miller20a, pmlr-v139-miller21b}. 

There exist several opportunities for future work. Exploring covariate shift in models which incorporate more richly structured kernels such as the neural tangent kernel (see \citet{jacot2018neural} and \citet{adlam2020neural}) and accommodating models with feature learning are interesting directions. Investigating the impact of covariate shift (and the relevance of our definition of shift strength) for other generative models or loss functions---such as logistic regression---is important as well. While we focused in this work on the generalization performance of random feature models on a \emph{single fixed} test distribution, understanding the simultaneous generalization performance of a fixed predictor across several distinct groups of test data (i.e. the worst-case loss over several test distributions) is also of interest in practice \citep{sagawa2020distributionally}.

\section*{Acknowledgements}
The authors would like to thank Rodolphe Jenatton, Horia Mania, Ludwig Schmidt, D. Sculley, Vaishaal Shankar, Lechao Xiao, and Steve Yadlowsky for valuable discussions. This work was performed at and funded by Google. No third party funding was used.

\newpage
\appendix

\setcounter{equation}{0}
\setcounter{figure}{0}
\setcounter{table}{0}
\setcounter{section}{0}
\setcounter{theorem}{0}

\onecolumn
\begin{center}

\textbf{{\large Appendix: Covariate Shift in High-Dimensional Random Feature Regression} \\
}
\end{center}
\renewcommand{\contentsname}{Table of Contents: Appendix}
\tableofcontents
\addtocontents{toc}{\protect\setcounter{tocdepth}{2}}
\newpage
\makeatletter
\renewcommand{\theequation}{S\arabic{equation}}
\renewcommand{\thefigure}{S\arabic{figure}}
\renewcommand{\bibnumfmt}[1]{[S#1]}
\section{Useful inequalities}
Here we include the statements and proofs of several auxiliary inequalities that we use throughout the Appendix.
\subsection{Basic properties of the self-consistent equation for \texorpdfstring{$x$}{}}
We begin by establishing several basic inequalities. The definitions of the following quantities can be found in \cref{thm:main_b_v}.
\begin{lemma}
\label{lem:positivity}
We have the following bounds: $\omega, \tauone, \taub, x, \I_{a,b}, \Ishift_{a,b} \ge 0$ and $\tfrac{\partial x}{\partial \gamma} \le 0$.
\end{lemma}
\begin{proof}
As shown in \citet{pennington2018spectrum} for the unit-variance case, a simple Hermite expansion argument establishes the relation $\eta\ge \zeta$, which implies $\omega = s(\eta/\zeta -1) \ge 0$. From \cref{eqn:tau1,eqn:tau1b}, $\tauone$ and $\taub$ are traces of positive semi-definite matrices and are therefore nonnegative. From the same equations, it follows that $x = \gamma \rho \tauone \taub \ge 0$. Nonnegativity of $x$ implies $\I_{a,b}\ge 0$ and $\Ishift_{a,b}\ge 0$ from their definitions in \cref{eqn:Idefs}. Finally, using the nonnegativity of $\omega$, $\tauone$, $\taub$, $x$, and $\I_{a,b}$, the expression for $\tfrac{\partial x}{\partial \gamma}$ in \cref{thm:main_b_v} immediately gives,
\begin{equation}
    \frac{\partial x}{\partial \gamma} = -\frac{x}{\gamma+\rho \gamma (\tfrac{\psi}{\phi} \tauone + \taub)(\omega + \phi \I_{1,2})} \le 0 \label{eq:dxdgamma_sm1}.
\end{equation}
\end{proof}

Next we show that the self-consistent equation $x=\tfrac{1-\gamma\tauone}{\omega + \I_{1,1}}$ appearing in \cref{thm:main_b_v} and defined in \cref{eqn:x} admits a unique positive real solution for $x$.

\begin{lemma}
\label{lem:unique_x}
There is a unique real $x\ge 0$ satisfying $x=\tfrac{1-\gamma\tauone}{\omega + \I_{1,1}}$.
\end{lemma}
\begin{proof}
Let $t=1/x\ge0$ and define
\begin{equation}
    h(t) = t\pa{\frac{\rho (\psi-\phi)+ \sqrt{\rho^2(\psi-\phi)^2+4\gamma\rho\phi\psi/t}}{2\rho\psi}-1} + \omega + \I_{1,1}(1/t)\,,
\end{equation}
which is a rewriting of \cref{eqn:x}. It suffices to show that $h$ admits a unique real positive root. To that end, first observe that $\lim_{t\to 0}\I_{1,1}(1/t) = 0$ and $\lim_{t\to \infty}\I_{1,1}(1/t) = s$, so that
\begin{align}
    h(0) =\omega > 0\,\quad\text{and}\quad \lim_{t\to\infty} h(t)/t = -\min\{1,\phi/\psi\} < 0\,,
\end{align}
which together imply that $h$ has an odd number of positive real roots. Next, we show that $h$ is concave for $t\ge 0$:
\eq{
    h''(t) = -\frac{2 \phi}{t^3}\pa{ \frac{ \gamma^2\rho\phi \psi}{(\rho^2(\psi-\phi)^2+4\gamma\rho\phi\psi/t)^{3/2}} + \I_{2,3}(1/t)} \le 0\,,
}
which implies that $h$ has at most two positive real roots. Therefore, we conclude that $h$ has exactly one positive real root.
\end{proof}

\subsection{\texorpdfstring{$\I$}{} and \texorpdfstring{$\Ishift$}{} inequalities} 
\label{app:Iineqs}
We now establish some useful properties of the $\I$ and $\Ishift$ functionals defined in \cref{eqn:Idefs}. To begin, we note that simple algebraic manipulations establish the following raising and lowering identities:
\eq{\label{eq:raise_lower}
    \I_{a-1, b-1} = \phi \I_{a-1, b} + x \I_{a, b} \quad\text{and}\quad \Ishift_{a-1, b-1} = \phi \Ishift_{a-1, b} + x \Ishift_{a, b}\,.
}
Next, we consider how the partial order of LJSDs given in~\cref{def:hard} leads to inequalities on the $\Ishift$ functionals. Letting $(\Ishift_{a,b})_{1}$ and $(\Ishift_{a,b})_{2}$ to denote the corresponding functionals with the LJSDs $\mu_1$ and $\mu_2$ respectively, we can establish the following useful lemma.
\begin{lemma}
\label{lem:harris_Ifunc}
Let $\mu_1 \leq \mu_2$, so $\mu_2$ is harder then $\mu_1$. Suppose the functions $f,g,h:\R\to\R$ are such that $f(\la)=g(\la)h(\la)$, $g(\la)$ is nonnegative, and $h(\la)$ is nonincreasing for all $\la>0$, then
\al{
    \frac{\E_{\mu_1}[rf(\la)]}{\E_{\mu_2}[rf(\la)]} \leq \frac{\E_{\mu_1}[rg(\la)]}{\E_{\mu_2}[rg(\la)]}.
}
If instead $h(\la)$ is nondecreasing for all $\la>0$, then
\al{
    \frac{\E_{\mu_1}[rf(\la)]}{\E_{\mu_2}[rf(\la)]} \geq \frac{\E_{\mu_1}[rg(\la)]}{\E_{\mu_2}[rg(\la)]}.
}
\end{lemma}
\begin{proof}
By the law of iterated expectation, we have
\eq{\label{eq_I_iterated_ex}
   \E_{\mu_1}[rf(\la)] = \mE_{\mu_2}[rg(\la)] \mE_\la \qa{ \frac{\mE_{\mu_{2}}[rg(\la)|\lambda]}{\mE_{\mu_2}[rg(\la)]} \frac{\mE_{\mu_{1}}[r|\lambda]}{\mE_{\mu_2}[r| \lambda]} h(\la) }.
}
Note that the expectation $\E_\la$ in \cref{eq_I_iterated_ex} over $\la$ is the same under $\mu_1$ and $\mu_2$ by assumption. Moreover, the function $h(\la)$ is nonincreasing in $\la>0$ by assumption. Finally, observe that the factor ${\mE_{\mu_{2}}[rg(\la)|\lambda]}/{\mE_{\mu_2}[rg(\la)]}$ defines a change in distribution for the random variable $\la$, since taking its expectation over $\la$ yields 1. Denote a new random variable with this distribution by $\tilde{\la}$. Then, we may apply the Harris inequality\footnote{See \citet[Section 2.2]{grimmett1999percolation} for example.} to \cref{eq_I_iterated_ex} to see
\al{
    \E_{\mu_1}[rf(\la)] &= \mE_{\mu_2}[rg(\la)] \mE_{\tilde{\la}} \qa{ \frac{\mE_{\mu_{1}}[r|\tilde{\la}]}{\mE_{\mu_2}[r| \tilde{\la}]} h(\tilde{\la}) } \\
    & \leq \mE_{\mu_2}[rg(\la)] \mE_{\tilde{\la}} \qa{ \frac{\mE_{\mu_{1}}[r|\tilde{\la}]}{\mE_{\mu_2}[r| \tilde{\la}]} } \mE_{\tilde{\la}} \qa{  h(\tilde{\la}) }\\
    & = \mE_{\mu_2}[rg(\la)] \mE_{{\la}} \qa{\frac{\mE_{\mu_{2}}[rg(\la)|\lambda]}{\mE_{\mu_2}[rg(\la)]} \frac{\mE_{\mu_{1}}[r|{\la}]}{\mE_{\mu_2}[r| {\la}]} } \mE_{{\la}} \qa{ \frac{\mE_{\mu_{2}}[rg(\la)|\lambda]}{\mE_{\mu_2}[rg(\la)]} h({\la}) }\\
    & = \frac{\mE_{\mu_1}[rg(\la)]}{\mE_{\mu_2}[rg(\la)]} \mE_{\mu_2} \qa{rf(\la)}\,.
    \label{eq_gen_harris2}
}
\end{proof}

The following corollary is an immediate consequence of \cref{lem:harris_Ifunc}.
\begin{corollary}
\label{cor:harris_Ifunc}
Let $\mu_1 \leq \mu_2$. Then, for $a\ge 0$,
\begin{align}
     \frac{(\Ishift_{1,a})_2}{s_2}-\frac{(\Ishift_{1,a})_1}{s_1} \geq 0
\end{align}
\end{corollary}
\begin{proof}
    This result follows from~\cref{lem:harris_Ifunc} by choosing $g:\lambda \mapsto 1$ and $h:\lambda \mapsto \phi (\phi + x \lambda)^{-a}$ and recalling by definitions $s_1 = \mE_{\mu_1}[r]$ and $s_2 = \mE_{\mu_2}[r]$.
\end{proof}

The argument from \cref{lem:harris_Ifunc} can be extended to cases where the shift is incomparable, i.e. when neither $\mu\leq\munull$ not $\mu\geq\munull$, by introducing some monotonic functions that in some sense bound the shift-strength.
\begin{lemma}\label{lem_harris_incomparable}
    Let $\mu$ be a LJSD and define the functions
    \eq{
        L(x) \deq \inf_{0\leq\la\leq x} \frac{\E_\mu[r|\la]}{\la} \quad \text{and} \quad U(x) \deq \sup_{0\leq\la\leq x} \frac{\E_\mu[r|\la]}{\la}.
    }
    Then if $f$ is a nonincreasing function, we have
    \eq{\label{eq_bounds_on_exp}
        \frac{\E_{\mutr}[\la L(\la)]}{\E_{\mutr}[\la]} \leq \frac{\E_\mu[rf(\la)]}{ \E_{\mutr}[\la f(\la)] } = \frac{\E_\mu[rf(\la)]}{ \E_{\munull}[r f(\la)] } \leq \frac{\E_{\mutr}[\la U(\la)]}{\E_{\mutr}[\la]}.
    }
\end{lemma}
\begin{proof}
    Recall that the marginal of $\la$ under both $\mu$ and $\munull$ is $\mutr$, and that $(r|\la) = \la$ almost surely under $\munull$, explaining the equality in \cref{eq_bounds_on_exp}. Clearly, $L(\la)\leq {\E_\mu[r|\la]}/{\la} \leq U(\la)$ for all $\la\geq0$, and $L(\la)$ and $U(\la)$ are nonincreasing and nondecreasing respectively. Next, we see
    \eq{
        \E_\mu[rf(\la)] = \E_{\mutr}[\la] \E_{\mutr}\qa{ \frac{\la}{\E_{\mutr}[\la]} \frac{\E_\mu[r|\la]}{\la}f(\la) },
    }
    so
    \eq{
        \E_{\mutr}\qa{  \frac{\la}{\E_{\mutr}[\la]}  L(\la) f(\la) } \leq \frac{\E_{\mu}[rf(\la)]}{\E_{\mutr}[\la]} \leq  \qa{ \frac{\la}{\E_{\mutr}[\la]} U(\la)f(\la) }.
    }
    Now we can apply the Harris inequality, as in the proof of \cref{lem:harris_Ifunc}, to see
    \eq{
        \frac{\E_{\mutr}\qa{ \la L(\la) }}{\pa{\E_{\mutr}[\la]}^2} \E_{\mutr}\qa{ {\la} f(\la) }\leq \frac{\E_\mu[rf(\la)]}{\E_{\mutr}[\la]} \leq \frac{\E_{\mutr}\qa{ \la U(\la) }}{\pa{\E_{\mutr}[\la]}^2} \E_{\mutr}\qa{ {\la} f(\la) },
    }
    which can be rearranged to give the result.
\end{proof}
\section{Hardness is a partial order}
\label{sec:general_definition}
We restate~\cref{def:hard} for clarity.
\begin{definition}[Restatement of~\cref{def:hard}]
    \label{def:hard_SM}
      Let $\mu_1$ and $\mu_2$ be LJSDs with the same marginal distribution of $\lambda$. If the asymptotic overlap coefficients are such that $\E_{\mu_1}[r|\la]/\E_{\mu_2}[r|\la]$ is nondecreasing as a function of $\lambda$ and ${\E_{\mu_1}[r]\le\E_{\mu_2}[r]}$, we say $\mu_1$ is \emph{easier} than $\mu_2$ (or $\mu_2$ is \emph{harder} than $\mu_1$), and write $\mu_1 \le \mu_2$. Comparing against the case of no shift $\munull$, we say $\mu_1$ is \emph{easy} when $\mu_1\leq \munull$ and \emph{hard} when $\mu_1\geq\munull$.
\end{definition}

\begin{proposition}
\cref{def:hard} (and \cref{def:hard_SM}) form a partial order over covariate shifts $\mu$.
\end{proposition}

\begin{proof}
    Reflexivity is clearly satisfied as $\E_{\mu} [r] \leq \E_{\mu} [r]$ and $\E_{\mu} [r|\la]/\E_{\mu} [r|\la] = 1$ is nondecreasing for all $\mu$. 
    
    For antisymmetry, we see $\mu_1\leq\mu_2$ and $\mu_2\leq\mu_1$ imply $\E_{\mu_1}[r]=\E_{\mu_2}[r]$ and that $\E_{\mu_1}[r|\la] / \E_{\mu_2}[r|\la]$ is constant in $\la$ as it is nonincreasing and nondecreasing. However, setting $\E_{\mu_1}[r|\la]=c \E_{\mu_2}[r|\la]$ and taking expectation over $\la$ and rearranging yields $1=\E_{\mu_1}[r]/\E_{\mu_2}[r]=c$, so in fact ${\E_{\mu_1}[r|\la] = \E_{\mu_2}[r|\la]}$. Assuming that $\mu_1$ and $\mu_2$ are absolutely continuous (the case where they are a sum of point masses is similar), we can write their densities as $p_{i}(\la,r) = p_{i}(\la)p_{i}(r|\la)$. By assumption $p_{1}(\la)=p_{2}(\la)$, so it suffices to show $p_{1}(r|\la)=p_{2}(r|\la)$ almost everywhere. Next note
    \eq{
        0 = \E_{\mu_1}[r|\la] - \E_{\mu_2}[r|\la] = \int_{\R^{+}} r\pa{p_{1}(r|\la) - p_{2}(r|\la)} \text{d}r,
    }
    but since $r>0$ over the domain of the integral, we have $p_{1}(r|\la) - p_{2}(r|\la) = 0$ almost everywhere.
    
    Finally, for transitivity assume $\mu_1\leq\mu_2$ and $\mu_2\leq\mu_3$, then clearly $\E_{\mu_1}[r]\leq\E_{\mu_2}[r]\leq\E_{\mu_3}[r]$. Next note 
    \eq{
        \frac{\E_{\mu_1}[r|\la]}{\E_{\mu_3}[r|\la]} = \frac{\E_{\mu_1}[r|\la]}{\E_{\mu_2}[r|\la]} \cdot \frac{\E_{\mu_2}[r|\la]}{\E_{\mu_3}[r|\la]},
    }
    so ${\E_{\mu_1}[r|\la]}/{\E_{\mu_3}[r|\la]}$ is the product of two nondecreasing, positive functions and is thus also nondecreasing.
\end{proof}
\section{Repeated eigenvalues of \texorpdfstring{$\Sigmatr$}{}}
\label{app:ambiguity}

When $\Sigmatr$ has repeated eigenvalues (denote one such by $\la$), its eigendecomposition is not unique, since the eigenvectors associated to $\la$ need only span the eigenspace $\{\vv:\Sigmatr \vv = \la \vv\}$. Specifically, if the eigenspace of $\la$ has dimension $n_\la$ then the eigenvectors $\vv^\la_1, \ldots, \vv^\la_{n_\la}$ are orthonormal but not unique. However, for any other choice of orthonormal vectors $\ww^\la_1, \ldots, \ww^\la_{n_\la}$ that span the eigenspace of $\la$, there exists some orthogonal matrix $O$ such that $W = V O$, where $V$ and $W$ contain the two bases as their columns.

Let $\vvs_1,\ldots,\vvs_{n_0}$ and $\la_1^*,\ldots,\la_{n_0}^*$ denote a choice for the eigenvectors and eigenvalues of $\Sigmate$. The nonuniqueness of the eigenvectors implies that the corresponding overlaps to $\Sigmate$, defined as
\eq{
    r^\la_i = \sum_{j=1}^{n_0} (\vvs_j \cdot  {\vvs_i}^\la)^2 \la^*_j = (\vv_i^\la)^\top \Sigmate \vv_i^\la
}
are also not unique (but do not depend on the choice of eigendecomposition for $\Sigmate$). However, we note that the conditional expectation $\E[r|\la]$ for the EJSD is invariant to the choice of eigendecomposition for $\Sigmatr$. Indeed, for $V$, we have
\eq{
    \E[r|\la]=\frac{1}{n_\lambda}\sum_{i=1}^{n_\la} r^\la_i = \frac{1}{n_\la}\sum_{i=1}^{n_\la} (\vv_i^\la)^\top \Sigmate \vv_i^\la = \ntr(V^\top \Sigmate V),
    \label{eq:finite_cond_e}
}
but this is the same as $\ntr(W^\top \Sigmatr W)$ using $W=VO$ and the cyclic property of the trace.

Since $\E[r|\la]$ under the EJSD is invariant to the choice of eigendecomposition for $\Sigmatr$, this will also be true of the LJSD. Said differently, while the choice of eigendecomposition affects both the EJSD and its corresponding LJSD, all possible choices of eigendecomposition lead to JSD in the same equivalence class, where $\mu_1$ and $\mu_2$ are equivalent when $\E_{\mu_1}[r|\la]=\E_{\mu_2}[r|\la]$. 

Finally, since all potential EJSDs associated to $\Sigmatr$ and $\Sigmate$ are in the same equivalence class, the particular choice has no affect on their downstream use. Specifically, \cref{def:hard} is not changed as it depends only on $\E[r|\la]$, and none of the functionals of $\mu$ (e.g. $\I_{a,b}$ and $\Ishift_{a,b}$) are changed due to the law of iterated expectation.
\section{Test error for linear regression} 
\label{app:lin_reg}
\subsection{Asymptotic and nonasymptotic results for \texorpdfstring{$m > n_0 + 1$}{}}
We present a short proof of the nonasymptotic test error for linear regression.

Recall data is generated via the model for $y_i=\beta^\top \x_i/\sqrt{n_0}+\e_i$, which is fit with the ridgeless linear regression estimator $\hat{\beta} = (X X^\top)^{-1} XY$. Although in the main text we have assumed the same isotropic prior on $\beta$ that we use for the random feature model, no generative assumption is needed on $\beta$ for this result.
\begin{proof}[Proof of \cref{prop:lin_reg}]
Under the condition $m > n_0+1$, the sample covariance is almost surely invertible so the test error can be written as
\begin{align}
    \err{\Sigmate} &= \E\qa{ \pa{ \beta^\top \bfx /\sqrt{n_0} - \hat{\beta}^\top \bfx }^2 } \\
    &=\sigma_{\e}^2\tr \left(\Sigmate \mE[(XX^\top)^{-1}] \right)\\
    &= \sigmaeps \frac{n_0}{m-n_0-1} \tr(\Sigmate \Sigmatr^{-1}),
\end{align}
where we used the cyclicity and linearity of the trace, as well as a formula for the expectation of the inverse sample covariance of a Gaussian matrix (which applies when $m > n_0+1$) from \citet[Theorem 3.1]{von1988moments}. The asymptotic form of the result follows from the limit in the proportional asymptotics: 
\begin{equation}
 \err{\mu} = \lim_{n_0,m \to \infty} \err{\Sigmate} = \frac{\sigmaeps \phi}{1-\phi} \mathbb{E}_\mu[r/\lambda]\,.
\end{equation}
\end{proof}

Next we prove \cref{prop:ordering_lr} using the Harris inequality before proving a slightly more general result that properly handles the case where $\Sigmatr$ may have repeated eigenvalues.

\begin{proof}[Proof of \cref{prop:ordering_lr}]
    Using the Harris inequality,
    \eq{
        \err{\Sigmate_1}^{\text{LR}} =  \frac{\sigma_{\e}^2n_0}{m-n_0-1} \frac{\ntr(\Sigmate_2)}{n_0}  \sum_{i=1}^{n_0} \frac{\corrs_{i,2}}{\ntr(\Sigmate_2)}\frac{\corrs_{i,1}}{\corrs_{i,2}} \frac{1}{\la_i} 
        \leq \frac{\sigma_{\e}^2n_0}{m-n_0-1}  \frac{\ntr(\Sigmate_1)}{\ntr(\Sigmate_2)}  \frac{1}{n_0}   \sum_{i=1}^{n_0}  \frac{\corrs_{i,2}}{\la_i} \leq \err{\Sigmate_2}^{\text{LR}},
    }
    since $1/\la_i$ and ${\corrs_{i,1}}/{\corrs_{i,2}}$ are nonincreasing and nondecreasing in $i$ respectively. 
\end{proof}

\cref{prop:ordering_lr} can be strengthened, and doing so motivates the occurrence of the conditional expectation in \cref{def:hard}. Note that we assume that the eigenvalues of $\Sigmatr$ are in nondecreasing order, and each eigenvalue has associated to it two overlap coefficients, $r_{i,1}$ and $r_{i,2}$. \cref{prop:ordering_lr} assumes that $r_{i,1}/r_{i,2}$ form a nondecreasing sequence. However, what happens when $\Sigmatr$ has repeated eigenvalues? In this case, the ordering of the $\lambda_i$ can be changed, which in turn changes the associated $r_{i,1}$ and $r_{i,2}$. Therefore, the assumption on $r_{i,1}/r_{i,2}$ is too strong---reordering the repeated eigenvalues might be sufficient to satisfy the condition even if it is violated for the original ordering. Instead, we can introduce the conditional expectation to handle this more gracefully.

In the following, we use $\tlam_1, \hdots, \tlam_{k}$ to denote the $k$ non-repeated eigenvalues of the training covariance $\Sigmatr$ and the sets $S_j$ for $j \in \{1, \hdots, k \}$ to denote the indices of eigenvalues in $\{1, \hdots, n_0 \}$ associated to the $j$th repeated eigenvalue of $\Sigmatr$. Analogously, we define the corresponding non-repeated overlap coefficients as,
\begin{align}
    \tcorrs_{j} = \frac{1}{\abs{S_j}}\sum_{i \in S_j} r_i, \quad j \in \{ 1, \hdots, k \}
    \label{eq:non_repeated_overlap_coeffs}
\end{align}
This is equivalent to the conditional expectation discussed in \cref{eq:finite_cond_e}. In this case, the measure-theoretic definition of \emph{hardness} in \cref{def:hard} becomes equivalent to stating that the sequence $\tfrac{\tcorrs_{j,2}}{\tcorrs_{j,1}}$ is nondecreasing as $j$ ranges from $\{1, \hdots, k \}$ when $\Sigmate_2$ is \emph{harder} then $\Sigmate_1$ (here $\tcorrs_{j,1}$ and $\tcorrs_{j,2}$ denote the non-repeated overlap coefficients of $\Sigmate_1$ and $\Sigmate_2$ respectively).

\begin{proposition}
    \label{prop:err_order_lin_reg}
    Let $\tcorrs_{j,1}$ and $\tcorrs_{j,2}$ denote the non-repeated overlap coefficients\footnote{Recall the definition in \cref{eq:non_repeated_overlap_coeffs}.} of $\Sigmate_{1}$ and $\Sigmate_{2}$ relative to $\Sigmatr$. If $\tr(\Sigmate_2) \ge \tr(\Sigmate_1)$ and the ratios ${\tcorrs_{j,1}}/{\tcorrs_{j,2}}$ form a nondecreasing sequence, then $E_{\Sigmate_2}^{\text{LR}} \geq E_{\Sigmate_1}^{\text{LR}}$
\end{proposition}

\begin{proof}
From \cref{prop:lin_reg} it is enough to show $\tr(\Sigmate_2 \Sigmatr^{-1}) \ge \tr(\Sigmate_1 \Sigmatr^{-1})$. First, note that $\sum_{i=1}^{n_0} r_{i,1} = \sum_{j=1}^{k} \abs{S_j} \tilde{r}_{j,1} =  \tr(\Sigmate_1)$ and $\sum_{i=1}^{n_0} r_{i,2} = \sum_{j=1}^{k} \abs{S_j} \tilde{r}_{j,2} = \tr(\Sigmate_2)$. Then,
\begin{align}
    \tr(\Sigmate_1 \Sigmatr^{-1}) & =  \sum_{i=1}^{n_0} \frac{\corrs_{i,1}}{\lambda_i} =  \sum_{j=1}^{k} \abs{S_j} \frac{\tcorrs_{j,1}}{\tlam_j} = \tr(\Sigmate_2) \sum_{j=1}^{k} \abs{S_j} \frac{\tcorrs_{j,2} }{\tr(\Sigmate_2)} \frac{\tcorrs_{j,1}}{\tcorrs_{j,2}} \frac{1}{\tlam_j}\\ 
    &\leq \tr(\Sigmate_2) \left(\sum_{j=1}^{k} \abs{S_j} \frac{\tcorrs_{j,2} }{\tr(\Sigmate_2)} \frac{\tcorrs_{j,1}}{\tcorrs_{j,2}} \right) \left(\sum_{j=1}^{k} \abs{S_j} \frac{\tcorrs_{j,2} }{\tr(\Sigmate_2)} \frac{1}{\tlam_j} \right)\\
    &=\frac{\tr(\Sigmate_{1})}{\tr(\Sigmate_{2})}  \tr(\Sigmate_{2} \Sigmatr^{-1})\\
    &\leq \tr(\Sigmate_{2} \Sigmatr^{-1}),
\end{align}
where the inequality is a consequence of the Harris inequality \citep[Section 2.2]{grimmett1999percolation}: since $\abs{S_j} {\tcorrs_{j,1}}/{\tlam_j}$ is a normalized measured with respect to $j$, the sequence ${1}/{\tlam_j}$ is nonincreasing in $j$, while the sequence ${\tcorrs_{j,1}}/{\tcorrs_{j,2}}$ is nondecreasing in $j$.
\end{proof}

\subsection{Linear regression limit of random feature regression}
\label{app:limits}

In this section, we show that taking appropriate limits of~\cref{cor:ridgeless_lim} and \cref{cor:infinite_overparameterization} recovers existing results for linear regression in high-dimensions. We first examine the ridgeless limit, for which the results fall into two cases: $\phi<1$, which we discuss in ~\cref{sec_lr_lim1} and connect to \cref{prop:lin_reg} from the main text; and $\phi > 1$, which we discuss in \cref{sec_lr_lim2} and connect to \citet[Cor. 2]{hastie2019surprises}. We then consider non-zero regularization, and in \cref{sec_lr_lim3} we connect \cref{cor:infinite_overparameterization} to \citet[Theorem 1]{wu2020optimal}.

To recover these results, we let $\sigma$ approach the identity activation function and $\psi \to 0$\footnote{The order of these limits does not change the result, but for concreteness we take take the limit as $\psi\to0$ first.}, since as more random features are added the kernel concentrate around its limit. Moreover, since we are taking the limit $\psi\to0$ and we assume $\phi > 0$, we may assume $\phi>\psi$ in all calculations for simplicity. We also note that as $\sigma$ approaches the identity, the constants associated to $\sigma$ approach the following limits:
\eq{\label{eq_sigma_to_linear}
    \eta, \zeta, \eta_*, \zeta_*,\xi \to 1 \quad\text{and}\quad \omega,\omegas\to0.
}

\subsubsection{Recovering asymptotic form of \texorpdfstring{\cref{prop:lin_reg}}{} for \texorpdfstring{$\phi < 1$}{}}\label{sec_lr_lim1}
From~\cref{cor:ridgeless_lim}, the expression for the total error is
\al{
    \err{\mu} &= (1-\xi)^2 s_{*} + 2(1-\xi)\xi \Ishift_{1,1} + \xi^2 \phi \Ishift_{1,2} + \xi^2 \frac{\psi}{\phi - \psi}x(\sigma_{\e}^2+\I_{1,1})(\omegas + \Ishift_{1,1}) \\
    &\qquad+ \xi^2 x\Big(1-\frac{x (\omega-\sigma_{\e}^2)}{1-x^2 \I_{2,2}}\Big)\Ishift_{2,2},
}
where $x=1/ (\omega+\I_{1,1})$ since we are assuming $\psi<\phi$. Taking the limit $\psi\to0$, the total error $E_\mu$ converges to
\eq{\label{eq_total_error_psi0}
    (1-\xi)^2 s_{*} + 2(1-\xi)\xi \Ishift_{1,1} + \xi^2 \phi \Ishift_{1,2} + \xi^2 x\Big(1-\frac{x (\omega-\sigma_{\e}^2)}{1-x^2 \I_{2,2}}\Big)\Ishift_{2,2}.
}
Recall that \cref{prop:lin_reg} assumes that $\phi < 1$. We begin by analyzing the solution to the self-consistent equation for $x$ in the ridgeless limit when the activation function $\sigma$ becomes linear.
\begin{lemma}
    \label{lem:x_ridgeless_linear}
    Suppose $0<\phi<1$ and $\psi < \phi$. In the ridgeless limit, the solution $x$ to the self-consistent equation in \cref{cor:ridgeless_lim},
    \begin{equation}\label{eq_sce2}
        x = \frac{1}{\omega+\I_{1,1}}\,,
    \end{equation}
    satisfies $\lim_{\omega \to 0}x =\infty$.
\end{lemma}
\begin{proof}
From the definition of $\I_{1,1}$, we see
\al{
    \I_{1,1} = \phi \mathbb{E}_\mu\left[\frac{\lambda}{\phi + \lambda x}\right] = \frac{\phi}{x} \mathbb{E}_\mu\left[\frac{\lambda}{\phi/x + \lambda}\right] \le \frac{\phi}{x}.
}
Using \cref{eq_sce2}, we find $x\geq \p{1-\phi}/{\omega}$. Taking $\omega\to0$ completes the proof.
\end{proof}

From \cref{lem:x_ridgeless_linear}, the solution to the self-consistent equation for $x$ diverges, i.e. $x\to\infty$ as $\sigma$ becomes linear. In this case, by the dominated convergence theorem, we have that
\al{
    \lim_{\omega\to 0} x \Ishift_{1,1} &= \lim_{x\to \infty} \E_{\mu}\qa{\frac{x r }{\phi+x \lambda}} = \phi \E_{\mu}\qa{r/\lambda},\\
    \lim_{\omega\to 0} x^2 \I_{2,2} &= \lim_{x\to \infty} \phi \E_{\mu}\qa{\frac{x^2 \lambda^2}{(\phi+x \lambda)^2}} = \phi,\\
    \lim_{\omega\to 0} x^2 \Ishift_{2,2} &= \lim_{x\to \infty} \phi \E_{\mu}\qa{\frac{x^2 r \lambda}{(\phi+x \lambda)^2}} = \phi \E_{\mu}\qa{r/\lambda},\\
    \lim_{\omega\to 0} \Ishift_{1,1} &= \lim_{x\to \infty} \phi \E_{\mu}\qa{\frac{r}{(\phi+x \lambda)^2}} = 0,\\
    \text{and}\quad \lim_{\omega\to 0} \Ishift_{1,2} &= \lim_{x\to \infty} \phi \E_{\mu}\qa{\frac{r}{(\phi+x \lambda)^2}} = 0\,.
}
As such, the total error in~\cref{eq_total_error_psi0} converges to
\begin{align}
     \sigma_{\e}^2 \frac{\phi}{1-\phi} \E_{\mu}[r/\lambda]
\end{align}
as $\omega\to0$ and $\psi\to0$, as desired.

\subsubsection{Recovering results from \texorpdfstring{\cite{hastie2019surprises}}{}\label{sec_lr_lim2} when $\Sigmatr=\Sigmate$ and $\phi>1$}
The minimum-norm solution for under-determined linear regression is the same as the limiting ridge-regularized solution as the ridge constant converges to 0. As such, studying the ridgeless limit as in~\cref{cor:ridgeless_lim} allows for a comparison to prior results on minimum-norm interpolants~\citep{hastie2019surprises}. To do so, we again take the limit as $\sigma$ becomes linear; however, in contrast to the previous section, we now assume $\phi>1$ in order to compare with \citet[Cor. 2]{hastie2019surprises}.

As \citet{hastie2019surprises} examines the setting in which the training and test covariate distributions are equal, $\Sigmatr = \Sigmate$, we have that $\Ishift_{a,b} = \I_{a,b}$. Using this relation, the total error~\cref{eq_total_error_psi0} becomes
\eq{
    \err{\mu} = (1-\xi)^2 s + 2(1-\xi)\xi \I_{1,1} + \xi^2 \phi \I_{1,2} + \xi^2 x\Big(1-\frac{x (\omega-\sigma_{\e}^2)}{1-x^2 \I_{2,2}}\Big)\I_{2,2}.
}
Next, letting $\sigma$ become linear as in~\cref{eq_sigma_to_linear}, we obtain $x = {1}/{\I_{1,1}}$ and
\al{
    \label{eqn:ridgeless_linear_error}
    \err{\mu} \,\stackrel{\omega \to 0}{\longrightarrow}\, \I_{1,1} + \sigma_{\e}^2 \frac{x^2 \I_{2,2}}{1-x^2 \I_{2,2}},
}
where we used the identity \cref{eq:raise_lower}.

We now simplify and relate the result for the limiting error from \citet[Cor. 2]{hastie2019surprises} in the case of isotropic $\beta$ satisfying $\norm{\beta}_2=1$, where the total error is written as\footnote{\citet[Cor. 2]{hastie2019surprises} provides expressions for a bias and variance term separately, but the decomposition is defined slightly differently than the one we use, so we compare directly to the total error. Note that~\cref{eq:hastie_error} corrects a typo in~\citet{hastie2019surprises}.}
\begin{align}
\label{eq:hastie_error}
    \frac{1}{\phi^2 c_0(H, \phi)} + \sigma_{\e}^2 c_0 \phi \frac{\int{\frac{s^2}{(1+c_0 \phi s)^2 dH(s)}}}{\int{\frac{s}{(1+c_0 \phi s)^2 dH(s)}}}\,.
\end{align}
The measure $H$ is the limiting empirical spectral density of the covariance, which is equivalent to the marginal distribution of $\lambda$, and $c_0$ satisfies the equation
\eq{
    1-\frac{1}{\phi} = \int{\frac{1}{1+c_0 \phi s} dH(s)}.
}
Substituting $c_0 = x/\phi^2$ and rearranging terms, we see
\begin{align}
    0 &=  1-\frac{1}{\phi} - \int{\frac{1}{1+c_0 \phi s} dH(s)}\\
    &=  -\frac{1}{\phi} + 1- \int{\frac{1}{1+(x/\phi) s} dH(s)}\\
    &=  -\frac{1}{\phi} + x \int{\frac{ s}{\phi+x  s} dH(s)}\\
    &= -\frac{x}{\phi}\big(\frac{1}{x} - \I_{1,1}\big),
\end{align}
which is satisfied by $x=1/\I_{1,1}$, implying the two self-consistent equations are equivalent and validating the identification $c_0 = x/\phi^2$. Using the same substitution $c_0 = x/\phi^2$ in~\cref{eq:hastie_error} gives
\al{
    \I_{1,1} + \sigma_{\e}^2 \frac{1}{x \phi} \frac{x^2 \I_{2,2}}{\I_{1,2}} = \I_{1,1} + \sigma_{\e}^2 \frac{x^2 \I_{2,2}}{1-x^2 \I_{2,2}},
}
which matches our expression in~\cref{eqn:ridgeless_linear_error}.
\subsubsection{Recovering results from \texorpdfstring{\citet{wu2020optimal}}{}\label{sec_lr_lim3} when $\Sigmatr=\Sigmate$ and $\gamma > 0$}
\label{app:ridge_regression_connect}
Viewing the expressions in \cref{cor:infinite_overparameterization} as functions of $\gamma$ and $\omega$, the total error satisfies the self-consistency relation $\err{\mu}(\gamma,\omega) = \err{\mu}(\gammaeff,0)$, i.e. the error with regularizer $\gamma$ and nonlinearity coefficient $\omega$ (from \cref{eq:constants_train}) is identical to the error with regularizer $\gammaeff$ and nonlinearity constant $\omega = 0$ (achieved by taking a linear activation function). Given this relationship, we seek to connect \cref{cor:infinite_overparameterization} with the results of \citet[Theorem 1]{wu2020optimal} for linear ridge regression, in the case where $\Sigmatr = \Sigmate$ and their underlying $\beta$ has an isotropic prior and their ridge-weighting matrix is the identity. This connection establishes that the total test error in \cref{cor:infinite_overparameterization} with regularizer $\gamma$ and nonlinearity constant $\omega$ matches the total test error of linear ridge regression with regularizer $\gammaeff$.

To begin, recall in the case that $\Sigma=\Sigmate$, $\Ishift_{a,b} = \I_{a,b}$ and $\xi=1$. In this case, our total test error can be written as:
\begin{align}
    E_{\mu} = \phi \I_{1,2} + \frac{\sigma_{\e}^2 + \phi \I_{1,2}}{\gammaeff + \phi \I_{1,2}} \cdot x \I_{2,2}. \label{eq:ridge_test_error}
\end{align}
with $x=1/(\gammaeff + \I_{1,1})$.  

Simplified to the aforementioned setting, the results of \citet[Theorem 1]{wu2020optimal} can be written as\footnote{Note this expression subtracts the contribution of the irreducible label noise on the test point to match our definition of the test error. Additionally, the generative model in \citet{wu2020optimal} takes the form $y(\x_i) = \beta^\top \x_i / \sqrt{m} + \e_i$ instead of the convention $y(\x_i) = \beta^\top \x_i / \sqrt{n_0} + \e_i$ we take. Hence the term $\int{\frac{\lambda}{(\lambda m(-\tgamma)+1)^2} dH(\lambda)}$ is rescaled by a relative factor of $1/\phi$ relative to the comparable term in \citet[Theorem 1]{wu2020optimal}.},
\begin{align}
    E^{\text{LR}} = \frac{m'(-\tgamma)}{m(-\tgamma)^2} \cdot \left( \int{\frac{\lambda}{(\lambda m(-\tgamma)+1)^2} dH(\lambda)} + \sigma_{\e}^2 \right) - \sigma_{\e}^2
\end{align}
where measure $H$ is the limiting empirical spectral density of the covariance, which is equivalent to the marginal distribution of $\lambda$. Additionally, $m(z)$ (the Stieltjes transform of the limiting spectral density corresponding to the Gram matrix $X^\top X/m$) and its derivative are given by \citet[Theorem 1]{wu2020optimal} as,
\begin{align}
    \tgamma &= \frac{1}{m(-\tgamma)}-\phi \int{\frac{\lambda}{1+\lambda m(-\tgamma)} dH(\lambda)} \\
    1 &= \left(\frac{1}{m^2(-\tgamma)}-\phi \int{  \frac{\lambda^2}{(\lambda \cdot m(-\tgamma)+1)^2} dH(\lambda) } \right) m'(-\tgamma)\label{eq:test_error_wu_xu}\,,
\end{align}
which can be rearranged to read,
\begin{align}
    \phi m(-\tgamma) &= \frac{1}{\tgamma/\phi +\phi \int{\frac{\lambda}{\phi+\lambda \cdot \phi m(-\tgamma)} dH(\lambda)}}\label{eq:mtgamma}\\
    \phi^2 m'(-\tgamma) &= \frac{1}{\frac{1}{\phi^2 m^2(-\tgamma)}-\phi \int{  \frac{\lambda^2}{(\lambda \cdot \phi m(-\tgamma)+\phi)^2} dH(\lambda)}}\label{eq:mtgammaprime}\,.
\end{align}
Whereas we define the Gram matrix as $X^\top X/n_0$, it is defined as $X^\top X/m$ in \citet{wu2020optimal}, leading to some normalization differences. To resolve these discrepancies, it is necessary to identify $m(-\tgamma) = x/\phi$ and $\tgamma = \phi \gammaeff$ in order for the regression vector $\hat{\beta}$ to be consistent between the two conventions. Under these identifications, \cref{eq:mtgamma} can be written as,
\begin{align}
    x &= \frac{1}{\gammaeff+\I_{1,1}}\label{eq:x_wuxu}\,,
\end{align}
which agrees with the expression in \cref{cor:infinite_overparameterization}, and \cref{eq:mtgammaprime} can be written as,
\begin{align}
    \frac{\partial x}{\partial \gammaeff} &= \frac{1}{1/x^2-\I_{2,2}}\\
    &= - \frac{x}{1/x -\I_{1,1} + \phi \I_{1,2}}\\
    & = - \frac{x}{\gammaeff + \phi \I_{1,2}}\,,
\end{align}
which, using $\tfrac{\partial x}{\partial \gammaeff} = \rho \tfrac{\partial x}{\partial \gamma}$ agrees with the expression in \cref{eq:xprime_inf_overparam}. Next, we use these identifications to simplify the expression for the test error in \cref{eq:test_error_wu_xu},
\begin{align}
    E^{\text{LR}} &= \frac{\phi^2 m'(-\tgamma)}{(\phi m(-\tgamma))^2} \cdot \left( \phi^2 \int{\frac{\lambda}{(\lambda \cdot \phi m(-\tgamma)+\phi)^2} dH(\lambda)} +\sigma_{\e}^2 \right) - \sigma_{\e}^2\\
    & = \frac{1}{1-x^2 \I_{2,2}} \cdot \left( \phi \I_{1,2}+ \sigma_{\e}^2 \right) - \sigma_{\e}^2\,.
\end{align}
Invoking the identity $\frac{1}{1-x^2 \I_{2,2}}=\frac{1/x}{1/x-x \I_{2,2}} = \frac{\gammaeff+\I_{1,1}}{\gammaeff+\phi \I_{1,2}}$, which follows from \cref{eq:raise_lower} and \cref{eq:x_wuxu}, we can rewrite the error as,
\begin{align}
    E^{\text{LR}} = \frac{\phi \I_{1,2} \cdot(\gammaeff+\I_{1,1})}{\gammaeff+\phi \I_{1,2}} + \frac{\sigma_{\e}^2 \cdot(\I_{1,1}-\phi \I_{1,2})}{\gammaeff+\phi \I_{1,2}}\,.
\end{align}
Finally, we invoke \cref{eq:raise_lower} and \cref{eq:x_wuxu} once more to obtain,
\al{
E^{\text{LR}} = \phi \I_{1,2} + \phi \I_{1,2} \cdot \frac{x \I_{2,2}}{\gammaeff+\phi\I_{1,2}}  +  \frac{\sigma_{\e}^2 x \I_{2,2}}{\gammaeff+\phi \I_{1,2}}\,,
}
which matches the expression in \cref{eq:ridge_test_error}, confirming the equivalence of the two expressions.

\section{Harder shifts increase the bias and the total error}
\label{sec:ordering}
To begin, we state \cref{prop:error_order} in the general setting where the scales in $\mu_1$ and $\mu_2$ need not match. In this setting, the marginals of the two test distribution may have two scales $s_1$ and $s_2$ distinct from the scale of the training covariance $s$, and the parameters in \cref{thm:main_b_v} are subscripted by a $1$ or $2$ to indicate it whether they are induced by $\mu_1$ or $\mu_2$. As discussed in \cref{sec:shift_strength} and \cref{sec:counterexamples}, additional constraints on these parameters are necessary to guarantee an ordering of the bias and error. 

\begin{proposition}[Restatement of \cref{prop:error_order}]
    \label{prop:error_order_sm}
    Consider two LJSDs $\mu_1$ and $\mu_2$ such that $\mu_1 \le \mu_2$. Then, in the setting of \cref{thm:main_b_v},
    \begin{align}
        \bias{\mu_1} \le \bias{\mu_2} &\quad\text{if}\quad \xi_2 \le \xi_1 \le 1\\
        \err{\mu_1} \le \err{\mu_2} &\quad\text{if}\quad \xi_2 \le \xi_1 \le 1\le\zeta_2/\zeta_1 \le \eta_2/\eta_1 \quad \text{and}\quad \sigma_{\e}^2 \le \omega\,.
    \end{align}
\end{proposition}
\begin{proof}[Proof of \cref{prop:error_order}] Consider two LJSDs $\mu_1$ and $\mu_2$ such that $\mu_1 \le \mu_2$, and assume $\xi_2 \le \xi_1 \le 1$. Then,
\begin{align}
     \bias{\mu_2}-\bias{\mu_1} &= (1-\xi_2)^2 s_2 + 2(1-\xi_2)\xi_2 (\Ishift_{1,1})_2 + \xi_2^2 \phi (\Ishift_{1,2})_2 \\ 
    &\qquad - (1-\xi_1)^2 s_1 - 2(1-\xi_1)\xi_1 (\Ishift_{1,1})_1 - \xi_1^2 \phi (\Ishift_{1,2})_1\\
    &\ge (1-\xi_2)^2 s_2 + 2(1-\xi_2)\xi_2 (\Ishift_{1,1})_2 + \xi_2^2 \phi (\Ishift_{1,2})_2 \\ 
    &\qquad - (1-\xi_1)^2 s_2 - 2(1-\xi_1)\xi_1 (\Ishift_{1,1})_2 - \xi_1^2 \phi (\Ishift_{1,2})_2\\
    &= \mathbb{E}_{\mu_2} \left[ r\bigg((1-\xi_2) + \xi_2 \frac{\phi}{\phi + x \lambda}\bigg)^2\right]-\mathbb{E}_{\mu_2} \left[ r\bigg((1-\xi_1) + \xi_1 \frac{\phi}{\phi + x \lambda}\bigg)^2\right]\\
    &= \mathbb{E}_{\mu_2} \left[r\left(\bigg((1-\xi_2) + \xi_2 \frac{\phi}{\phi + x \lambda}\bigg)^2 - \bigg((1-\xi_1) + \xi_1 \frac{\phi}{\phi + x \lambda}\bigg)^2\right)\right]\\
    &= \mathbb{E}_{\mu_2} \left[r\left((\xi_1 - \xi_2)\Big(2-(\xi_1 + \xi_2)\frac{x \lambda}{\phi + x \lambda}\Big)\frac{x \lambda}{\phi + x \lambda}\right)\right]\\
    &\ge 0\,,
\end{align}
where we have again used \cref{lem:positivity} and the inequalities in \cref{eqn:Iineq} which follow from \cref{cor:harris_Ifunc}.

Now, additionally assume that $\sigma_{\e}^2 \le \omega$ and $1 \le \zeta_2/\zeta_1 \le \eta_2/\eta_1$. Note that the latter condition implies $s_2 \xi_2^2 \ge s_1 \xi_1^2$ and ${\omega}_2/s_2 \ge {\omega}_1/s_1$. Reorganizing the terms specifying the limiting variance, we can rewrite the total error as,
\al{
\err{\mu} &= \bias{\mu} + \var{\mu}\\
&= (1-\xi)^2 s_* + 2(1-\xi)\xi (\Ishift_{1,1}) + \xi^2 \phi (\Ishift_{1,2})\nonumber\\
&\quad - \rho\xi^2 \frac{\psi}{\phi}\frac{\partial x}{\partial \gamma} \bigg((\sigma_\e^2+\I_{1,1})(\omega + \phi \I_{1,2})(\omegas+\Ishift_{1,1})+\gamma\tauone\I_{2,2}(\omegas+\phi\Ishift_{1,2}) \nonumber\\
&\qquad\qquad\qquad\quad+\frac{\phi}{\psi x}\gamma\taub(\sigma_{\e}^2+\phi \I_{1,2})(\Ishift_{1,1} -\phi\Ishift_{1,2}) \bigg)\\
& \equiv (1-\xi)^2 s_* + 2(1-\xi)\xi (\Ishift_{1,1}) + \xi^2\big(C_1 \omegas + C_2 \Ishift_{1,1} + C_3 \Ishift_{1,2}\big)\,,
}
where the $C_i \ge 0$ and depend on $\mu$ only through the marginal $\lambda$ (i.e. they only depend on the training distribution):
\al{\label{eq_Cis1}
    C_1 & = -\rho \frac{\psi}{\phi}\frac{\partial x}{\partial \gamma}\big( (\sigma_\e^2+\I_{1,1})(\omega + \phi \I_{1,2}) +  \gamma\tauone\I_{2,2} \big) \ge 0 \\
    C_2 & = -\rho\xi^2\frac{\partial x}{\partial \gamma}\left(\frac{\psi}{\phi} (\sigma_\e^2+\I_{1,1})(\omega + \phi \I_{1,2}) + \frac{\gamma\taub}{x}(\sigma_{\e}^2+\phi \I_{1,2})\right) \ge 0 \\
    C_3 & = -\rho \frac{\partial x}{\partial \gamma} \left(\psi \gamma\tauone\I_{2,2}-\frac{\phi \gamma\taub}{x}(\sigma_{\e}^2+\phi \I_{1,2})-\frac{\phi}{\rho \tfrac{\partial x}{\partial \gamma}}\right)\\
    & = -\rho \frac{\partial x}{\partial \gamma} \left(\psi \gamma\tauone\I_{2,2}-\frac{\phi \gamma\taub}{x}(\sigma_{\e}^2+\phi \I_{1,2})+\frac{\phi}{\rho x}(\gamma+\rho \gamma (\tauone \psi/\phi + \taub)(\omega + \phi \I_{1,2}))\right)\\
    & =  -\rho \gamma \frac{\partial x}{\partial \gamma} \left(\psi \tauone\I_{2,2}+\frac{\phi \taub}{x}(\omega - \sigma_{\e}^2)+\frac{\phi}{\rho x}\big(1+\rho  \tauone \frac{\psi}{\phi}(\omega + \phi \I_{1,2})\big)\right)\\
    & \ge 0\,.\label{eq_Cis2}
}
where the inequalities follow from \cref{lem:positivity} and from the assumption $\sigma_{\e}^2 \le \omega$. It is now straightforward to see that,
\al{
\err{\mu_2} - \err{\mu_1} &= (1-\xi_2)^2 s_2 + 2(1-\xi_2)\xi_2 (\Ishift_{1,1})_2 - (1-\xi_1)^2 s_1 - 2(1-\xi_1)\xi_1 (\Ishift_{1,1})_1\\
 &\quad + \xi_2^2 \big(C_1 \omega_2 + C_2 (\Ishift_{1,1})_2 + C_3 (\Ishift_{1,2})_2\big) - \xi_1^2 \big(C_1 \omega_1 + C_2 (\Ishift_{1,1})_1 + C_3 (\Ishift_{1,2})_1\big)\nonumber\\
& \ge (1-\xi_2)^2 s_2 + 2(1-\xi_2)\xi_2 (\Ishift_{1,1})_2 - (1-\xi_1)^2 s_2 - 2(1-\xi_1)\xi_1 (\Ishift_{1,1})_2\\
 &\quad + \xi_2^2 \big(C_1 \omega_2 + C_2 (\Ishift_{1,1})_2 + C_3 (\Ishift_{1,2})_2\big) - \xi_1^2 \big(C_1 \omega_1 + C_2 (\Ishift_{1,1})_1 + C_3 (\Ishift_{1,2})_1\big)\nonumber\\
& = (\xi_1 -\xi_2)\big((2-\xi_1-\xi_2)s_2 - 2(1-\xi_1-\xi_2)(\Ishift_{1,1})_2\big)\\
&\quad + s_2 \xi_2^2 \big(C_1 \frac{\omega_2}{s_2} + C_2 \frac{(\Ishift_{1,1})_2}{s_2} + C_3 \frac{(\Ishift_{1,2})_2}{s_2}\big) - s_1 \xi_1^2 \big(C_1 \frac{\omega_1}{s_1} + C_2 \frac{(\Ishift_{1,1})_1}{s_1} + C_3 \frac{(\Ishift_{1,2})_1}{s_1}\big)\nonumber\\
& \ge (\xi_1-\xi_2)\big((2-\xi_1-\xi_2)(\Ishift_{1,1})_2 - 2(1-\xi_1-\xi_2)(\Ishift_{1,1})_2\big)\\
&\quad + s_2 \xi_2^2 \big(C_1 \frac{\omega_2}{s_2} + C_2 \frac{(\Ishift_{1,1})_2}{s_2} + C_3 \frac{(\Ishift_{1,2})_2}{s_2}\big) - s_2 \xi_2^2 \big(C_1 \frac{\omega_1}{s_1} + C_2 \frac{(\Ishift_{1,1})_1}{s_1} + C_3 \frac{(\Ishift_{1,2})_1}{s_1}\big)\nonumber\\
& \ge (\xi_1-\xi_2)(\xi_1+\xi_2)(\Ishift_{1,1})_2\\
&\quad + s_2 \xi_2^2 \big(C_1 \frac{\omega_2}{s_2} + C_2 \frac{(\Ishift_{1,1})_2}{s_2} + C_3 \frac{(\Ishift_{1,2})_2}{s_2}\big) - s_2 \xi_2^2 \big(C_1 \frac{\omega_2}{s_2} + C_2 \frac{(\Ishift_{1,1})_2}{s_2} + C_3 \frac{(\Ishift_{1,2})_2}{s_2}\big)\nonumber\\
& \ge 0\,,
}
where we have used the inequalities
\begin{align}
    \frac{(\Ishift_{1,1})_2}{s_2}-\frac{(\Ishift_{1,1})_1}{s_1} \geq 0\,, \quad \frac{(\Ishift_{1,2})_2}{s_2}-\frac{(\Ishift_{1,2})_1}{s_1} \geq 0\,, \quad\text{and}\quad (\Ishift_{1,1})_2 \le s_2\,.
\end{align}
The former two inequalities again follow from \cref{cor:harris_Ifunc}.
\end{proof}

\subsection{Pure-scale shifts}
\label{sec:order_variance}
\begin{proposition}[Restatement of~\cref{prop:pure_scale}]
\label{prop:pure_scale_sm}
Consider two LJSDs $\mu_1$ and $\mu_2$ such that $\mu_1 \le \mu_2$ and $\mathbb{E}_{\mu_1}[r|\lambda]/\mathbb{E}_{\mu_1}[r|\lambda] = \mathbb{E}_{\mu_1}[r]/\mathbb{E}_{\mu_2}[r]$, i.e. a \emph{pure-scale shift}. Then, in the setting of \cref{thm:main_b_v},
    \begin{align}
        \bias{\mu_1} \le \bias{\mu_2} &\quad\text{if}\quad \xi_2 \le \xi_1 \le 1\label{eq:bias_pure_scale_sm}\\
        \var{\mu_1} \le \var{\mu_2} &\quad\text{if}\quad 1\le\zeta_2/\zeta_1 \le \eta_2/\eta_1\label{eq:variance_pure_scale_sm}\\
        \err{\mu_1} \le \err{\mu_2} &\quad\text{if}\quad \xi_2 \le \xi_1 \le 1\le\zeta_2/\zeta_1 \le \eta_2/\eta_1 \quad \text{and}\quad \sigma_{\e}^2 \le \omega\label{eq:error_pure_scale_sm}\,.
    \end{align}
\end{proposition}
The result for the bias and total error follows immediately from~\cref{prop:error_order}. For the variance, we can reorganize the terms specifying the limiting variance as,
\begin{proof}
\al{
\var{\mu} &= -\rho\xi^2 \frac{\psi}{\phi}\frac{\partial x}{\partial \gamma} \bigg((\sigma_\e^2+\I_{1,1})(\omega + \phi \I_{1,2})(\omegas+\Ishift_{1,1}) +\frac{\phi}{\psi}\gamma\taub(\sigma_{\e}^2+\phi \I_{1,2})\Ishift_{2,2}\nonumber\\
&\qquad\qquad\qquad\quad+\gamma\tauone\I_{2,2}(\omegas+\phi\Ishift_{1,2})\bigg)\\
& \equiv C_1 (\omegas + \Ishift_{1,1}) + C_2 \Ishift_{2,2} + C_3 (\omegas + \phi \Ishift_{1,2})\,,
}
where the $C_i \ge 0$ and depend on $\mu$ only through the marginal $\lambda$ (i.e. they only depend on the training distribution):
\al{
C_1 & = -\rho\xi^2 \frac{\psi}{\phi}\frac{\partial x}{\partial \gamma} (\sigma_\e^2+\I_{1,1})(\omega + \phi \I_{1,2}) \ge 0 \\
C_2 & = -\rho\xi^2 \frac{\partial x}{\partial \gamma} \gamma\taub(\sigma_{\e}^2+\phi \I_{1,2}) \ge 0 \\
C_3 & = -\rho\xi^2 \frac{\psi}{\phi}\frac{\partial x}{\partial \gamma} \gamma\tauone\I_{2,2}  \ge 0\,,
}
where the inequalities follow directly from \cref{lem:positivity}. From this expression, it is now straightforward to see that,
\begin{align}
    V_{\mu_2} - V_{\mu_1} &= \xi_2^2\left(C_1 (\omega_2 + (\Ishift_{1,1})_2) + C_2 (\Ishift_{2,2})_2 + C_2 (\omega_2 + \phi\Ishift_{1,2})_2\right)\nonumber\\
    &\qquad - \xi_1^2\left(C_1 (\omega_1 + (\Ishift_{1,1})_1) + C_2 (\Ishift_{2,2})_1 + C_3 (\omega_1 + \phi\Ishift_{1,2})_1\right)\\
    &= s_2\xi_2^2\left(C_1 (\frac{\omega_2}{s_2} + \frac{(\Ishift_{1,1})_2}{s_2}) + C_2 \frac{(\Ishift_{2,2})_2}{s_2} + C_3(\frac{\omega_2}{s_2}+ \phi\frac{(\Ishift_{1,2})_2}{s_2})\right)\nonumber\\
    &\qquad - s_1\xi_1^2\left(C_1 (\frac{\omega_1}{s_1} + \frac{(\Ishift_{1,1})_1}{s_1}) + C_2 \frac{(\Ishift_{2,2})_1}{s_1} + C_3(\frac{\omega_1}{s_1}+ \phi\frac{(\Ishift_{1,2})_1}{s_1})\right)\\
    &\ge s_2\xi_2^2\left(C_1 (\frac{\omega_2}{s_2} + \frac{(\Ishift_{1,1})_2}{s_2}) + C_2 \frac{(\Ishift_{2,2})_2}{s_2} + C_3(\frac{\omega_2}{s_2}+ \phi\frac{(\Ishift_{1,2})_2}{s_2})\right)\nonumber\\
    &\qquad - s_2\xi_2^2\left(C_1 (\frac{\omega_1}{s_1} + \frac{(\Ishift_{1,1})_1}{s_1}) + C_2 \frac{(\Ishift_{2,2})_1}{s_1} + C_3(\frac{\omega_1}{s_1}+ \phi\frac{(\Ishift_{1,2})_1}{s_1})\right)\\
    &\ge s_2\xi_2^2\left(C_1 (\frac{\omega_2}{s_2} + \frac{(\Ishift_{1,1})_2}{s_2}) + C_2 \frac{(\Ishift_{2,2})_2}{s_2} + C_3(\frac{\omega_2}{s_2}+ \phi\frac{(\Ishift_{1,2})_2}{s_2})\right)\nonumber\\
    &\qquad - s_2\xi_2^2\left(C_1 (\frac{\omega_2}{s_2} + \frac{(\Ishift_{1,1})_2}{s_2}) + C_2 \frac{(\Ishift_{2,2})_2}{s_2} + C_3(\frac{\omega_2}{s_2}+ \phi\frac{(\Ishift_{1,2})_2}{s_2})\right)\\
    &= 0\,,
\end{align}
where we have used the inequalities 
\begin{align}
\label{eqn:Iineq}
    \frac{(\Ishift_{1,1})_2}{s_2}-\frac{(\Ishift_{1,1})_1}{s_1} \geq 0\,, \qquad \frac{(\Ishift_{2,2})_2}{s_2}-\frac{(\Ishift_{2,2})_1}{s_1} \geq 0\,, \quad\text{and}\quad
    \frac{(\Ishift_{1,2})_2}{s_2}-\frac{(\Ishift_{1,2})_1}{s_1} \geq 0\,,
\end{align}
which are immediate consequences of \cref{cor:harris_Ifunc} and the assumption $\mathbb{E}_{\mu_1}[r|\lambda]/\mathbb{E}_{\mu_1}[r|\lambda] = \mathbb{E}_{\mu_1}[r]/\mathbb{E}_{\mu_2}[r] - s_1/s_2$.
\end{proof}

\subsection{Incomparable shifts of equal scale}
\label{sec_incomparable}

In this section, we state a bound on the bias and total error of shifts that are incomparable to the null shift $\munull$ in terms of the in-distribution bias and error.

\begin{proposition}\label{prop_incomparable_sm}
    Let $\mu$ be a LJSD such that $\E_\mu[\la]=s=s_*=\E_\mu[r]$, and define the functions
    \eq{
        L(x) \deq \inf_{0\leq\la\leq x} \frac{\E_\mu[r|\la]}{\la} \quad \text{and} \quad U(x) \deq \sup_{0\leq\la\leq x} \frac{\E_\mu[r|\la]}{\la}.
    }
    Denote the bias and total error of $\mu$ by $B_\mu$ and $E_\mu$ respectively, then
    \eq{\label{eq_bias_bound_sm}
        \frac{\E_{\mutr}\qa{ \la L(\la) }}{\E_{\mutr} \qa{\la} } B_{\munull} \leq B_\mu \leq \frac{\E_{\mutr} \qa{\la U(\la)} }{\E_{\mutr}\qa{ \la} } B_{\munull},
    }
    and further if $\sigma_{\e}^2 \leq \omega$, then
    \eq{\label{eq_total_error_bound_sm}
        \frac{\E_{\mutr}\qa{ \la L(\la) }}{\E_{\mutr} \qa{\la} } E_{\munull} \leq E_\mu \leq \frac{\E_{\mutr} \qa{\la U(\la)} }{\E_{\mutr}\qa{ \la} } E_{\munull}.
    }
\end{proposition}
\begin{proof}
    Since $s=s_*$, $\xi=1$. Thus, recall $B_\mu = \phi \Ishift_{1,2}$. Note that the function $\la\mapsto (\phi+x\la)^{-2}$ is nonincreasing, thus \cref{lem_harris_incomparable} implies \eqref{eq_bias_bound_sm}, since $B_{\munull} = \phi \I_{1,2}$. 
    
    For the total error, recall from Eqs.~\eqref{eq_Cis1}-\eqref{eq_Cis2} that $E_\mu = C_1 \omega + C_2 \Ishift_{1,1} + C_3 \Ishift_{1,2}$, where the coefficients $C_1$, $C_2$, and $C_3$ are nonnegative. Moreover, these coefficients are the same for $\munull$ as $s=s_*$. Finally, we can again apply \cref{lem_harris_incomparable} as the functions $\la\mapsto (\phi+x\la)^{-1}$ and $\la\mapsto (\phi+x\la)^{-2}$ are nonincreasing, to find \cref{eq_total_error_bound_sm}.
\end{proof}
\section{The benefit of overparameterization}
\label{app:b_v_covshift}
\subsection{The bias is nonincreasing}
We begin by examining the behavior of the bias as a function of the overparameterization ratio $\phi/\psi$ in a more general setting than discussed in the main text. In particular, we relax \cref{assump:scale} and allow for unequal scales, $s\neq s_*$. As discussed in \cref{sec:counterexamples}, monotonicity of the bias is not guaranteed if $\xi > 1$. As such, we adopt the additional condition $\xi \le 1$. 

\begin{proposition}[Restatement of \cref{prop:b}]
  In the setting of \cref{thm:main_b_v}, if $\xi \leq 1$, the bias $\bias{\mu}$ is a nonincreasing function of the overparameterization ratio $\phi/\psi$. \label{prop:b_modified}
\end{proposition}

\begin{proof}[Proof of \cref{prop:b_modified}]
Recall from \cref{thm:main_b_v} that the bias is given by
\begin{equation}
\label{eqn:biassm}
\bias{\mu} = (1-\xi)^2 s_{*} + 2(1-\xi)\xi \Ishift_{1,1} + \xi^2 \phi \Ishift_{1,2}\,,
\end{equation}
where $x$ is the unique positive real root of the self-consistent equation,
\begin{equation}
\label{eqn:xsm}
    x = \frac{1-\gamma\tauone}{\omega + \I_{1,1}}\,.
\end{equation}
Differentiating~\cref{eqn:biassm} with respect to $\phi/\psi$ gives,
\begin{align}
    \frac{\partial B_{\mu}}{\partial{(\phi/\psi)}} = -\frac{\psi^2}{\phi} \frac{\partial B_{\mu}}{\partial{\psi}} = 2\frac{\psi^2}{\phi} \frac{\partial x}{\partial \psi} \cdot \left( (1-\xi)\xi \Ishift_{2,2} + \xi^2 \phi \Ishift_{2,3} \right)\,.
\end{align}
Since by assumption $\xi \leq 1$, and \cref{lem:positivity} gives $\Ishift_{a,b}\ge 0$, it is sufficient to show $\frac{\partial x}{\partial{\psi}} \le 0$, which immediately follows by implicitly differentiating \cref{eqn:xsm} and simplifying the expression:
\eq{
\frac{\partial x}{\partial \psi} = - \frac{\rho x \tauone (\omega + \I_{1,1})}{\phi\big(1 + \rho(\taub + \frac{\psi}{\phi}\tauone)(\omega + \phi \I_{1,2})\big)} \le 0 \label{eq:dxdpsi}\,.
}
Therefore we conclude that $\frac{\partial B_{\mu}}{\partial{(\phi/\psi)}} \le 0$.
\end{proof}

\subsection{The variance is nonincreasing}
Next, we turn our attention to the variance. We note that \cref{prop:v} does not rely on \cref{assump:scale}: the variance is a nonincreasing function of the overparameterization ratio $\phi/\psi$ whenever $\psi < \phi$, regardless of the magnitude of the overall covariance scales $s$ and $s_*$. However, the proposition only focuses on the ridgeless limit, whereas
\cref{fig:predictions}(d) and \cref{fig:scale_violation}(a) show that result may in fact hold for nonzero ridge constant. As the proof is considerably simpler in the ridgeless limit, we defer the analysis of the nonzero ridge setting to future work.

\begin{proposition}[Restatement of \cref{prop:v}]
    In the setting of \cref{cor:ridgeless_lim} and in the overparameterized regime (i.e. $\psi < \phi$), the variance $\var{\mu}$ is a nonincreasing function of the overparameterization ratio $\phi/\psi$.
    \label{prop:v_modified}
\end{proposition}

\begin{proof}[Proof of \cref{prop:v}]
Using the chain rule we have that, $\frac{\partial \var{\mu}}{\partial{(\phi/\psi)}} = \frac{\partial \var{\mu}}{\partial{\psi}} \left[\frac{\partial (\phi/\psi)}{\partial \psi} \right]^{-1} = -\frac{\partial \var{\mu}}{\partial{\psi}} \frac{\psi^2}{\phi}$, so it is sufficient to show that $\frac{\partial \var{\mu}}{\partial{\psi}} \ge 0$.

From \cref{cor:ridgeless_lim} in the overparameterized regime, the self-consistent equation for $x$ reads $x = \tfrac{1}{\omega+\I_{1,1}}$ and is independent of $\psi$. Therefore, $\partial x/\partial \psi = 0$ and the expression for $\partial V_\mu/\partial \psi$ follows directly from \cref{eqn:v_ridgeless},
\begin{align}
    \frac{\partial \var{\mu}}{\partial \psi} = \xi ^2 \frac{\phi}{(\phi-\psi)^2} x(\sigma_{\e}^2+\I_{1,1})(\omegas + \Ishift_{1,1}) \ge 0\,,
\end{align}
where the inequality follows from \cref{lem:positivity}.
\end{proof}

\subsection{The generalization gap is nonincreasing}
Finally, we prove that overparameterization also confers enhanced robustness, specifically that the generalization gap between shifted and unshifted test error is a nonincreasing function of the overparameterization ratio in the overparameterized regime.

\begin{proposition}[Restatement of \cref{prop:gen_gap}]
    \label{prop:gen_gap_modified}
    Consider two LJSDs $\mu_1$ and $\mu_2$ such that $\mu_1 \leq \mu_2$ and $1 \le \zeta_2/\zeta_1 \le \eta_2/\eta_1$. Then, in the setting of \cref{cor:ridgeless_lim} and in the overparameterized regime (i.e. $\psi < \phi$), the generalization gap $E_{\mu_2}-E_{\mu_1}$ is a nonincreasing function of the overparameterization ratio $\phi/\psi$.    
\end{proposition}
\begin{proof}[Proof of \cref{prop:gen_gap}]
The chain rule gives $\frac{\partial(E_{\mu_2}-E_{\mu_1})}{\partial (\phi/\psi)} =  -\frac{\psi^2}{\phi}\frac{\partial(E_{\mu_2}-E_{\mu_1})}{\partial \psi}$ so it is sufficient to show ${\frac{\partial(E_{\mu_2}-E_{\mu_1})}{\partial \psi} \ge 0}$. To that end, recall from \cref{cor:ridgeless_lim} in the overparameterized regime that the self-consistent equation for $x$ reads $x = \tfrac{1}{\omega+\I_{1,1}}$ and is independent of $\psi$. Therefore, $\partial x/\partial \psi = 0$, which implies $\partial \bias{\mu_1}/\partial \psi = 0$ and $\partial \bias{\mu_2}/\partial \psi = 0$, so that $\frac{\partial(E_{\mu_2}-E_{\mu_1})}{\partial \psi} = \frac{\partial(V_{\mu_2}-V_{\mu_1})}{\partial \psi}$, the expression for which follows directly from \cref{eqn:v_ridgeless}. Finally, note that, as in \cref{sec:ordering}, the conditions $\mu_1\le\mu_2$ and $1 \le \zeta_2/\zeta_1 \le \eta_2/\eta_1$ imply $s_2 \xi_2^2 \ge s_1 \xi_1^2$ and ${\omega}_2/s_2 \ge {\omega}_1/s_1$, so we have
\begin{align}
    \frac{\partial(E_{\mu_2}-E_{\mu_1})}{\partial \psi} &= \xi_2^2\left(C_E(\omega_2 + (\Ishift_{1,1})_2) \right)- \xi_1^2\left(C_E(\omega_1 + (\Ishift_{1,1})_1 \right) \\
    &= s_2 \xi_2^2\left(C_E\pa{\frac{\omega_2}{s_2} + \frac{(\Ishift_{1,1})_2}{s_2}} \right) - s_1 \xi_1^2\left(C_E\pa{\frac{\omega_1}{s_1} + \frac{(\Ishift_{1,1})_1}{s_1}} \right) \\
    &\ge s_2 \xi_2^2\left(C_E\pa{\frac{\omega_2}{s_2} + \frac{(\Ishift_{1,1})_2}{s_2}} \right)- s_2 \xi_2^2\left(C_E\pa{\frac{\omega_1}{s_1} + \frac{(\Ishift_{1,1})_1}{s_1}} \right) \\
    &\ge s_2 \xi_2^2\left(C_E\pa{\frac{\omega_2}{s_2} + \frac{(\Ishift_{1,1})_2}{s_2}} \right) - s_2 \xi_2^2\left(C_E\pa{\frac{\omega_2}{s_2} + \frac{(\Ishift_{1,1})_2}{s_2}} \right)\\
    &= 0\,,
\end{align}
where we have introduced the shorthand $C_E = \frac{\phi}{(\phi - \psi)^2}x(\sigma_{\e}^2+\I_{1,1})\ge0$ and in the last inequality we have used \cref{cor:harris_Ifunc} to argue $\frac{(\Ishift_{1,1})_2}{s_2} \geq \frac{(\Ishift_{1,1})_1}{s_1}$.
\end{proof}
\section{Linear trends between in-distribution and out-of-distribution generalization}

In \cref{sec:linear_trends}, we investigated the linear relationship between the shifted and unshifted test error in the ridgeless, overparameterized regime, arriving at \cref{prop:slope}. Here we provide a short proof of that proposition, which we restate for clarity. 

\begin{proposition}[Restatement of \cref{prop:slope}]
Consider two LJSDs $\mu_1$ and $\mu_2$. In the setting of \cref{cor:ridgeless_lim} and in the overparameterized regime (i.e. $\psi < \phi$), $E_{\mu_2}$ is linear in $E_{\mu_1}$, parametrically in $\phi/\psi$:
 \begin{equation}
    \err{\mu_2} = E_0 + \underbrace{\frac{\xi_2^2}{\xi_1^2} \pa{\frac{\omega_2 + (\Ishift_{1,1})_2}{\omega_1 + (\Ishift_{1,1})_1}}}_{\slope} \err{\mu_1}\,,
\end{equation}
where $E_0$ and $\slope \ge 0$ are constants independent of the overparameterization ratio $\phi/\psi$. Moreover, if $\mu_1 \leq \mu_2$ and $1 \le \zeta_2/\zeta_1 \le \eta_2/\eta_1$, then $\slope \ge 1$.
\end{proposition}

\begin{proof}[Proof of \cref{prop:slope}]
In the overparameterized regime (i.e. $\phi > \psi$) and in the ridgeless limit, the limiting test error for LJSD $\mu_1$ is given by \cref{cor:ridgeless_lim} as
\begin{equation}
\label{eqn:prop4_1}
\err{\mu_1} = \bias{\mu_1} + \xi_1^2 \frac{1}{\phi/\psi - 1}x(\sigma_{\e}^2+\I_{1,1})(\omega_1 + (\Ishift_{1,1})_1) + \xi_1^2 x\Big(1-\frac{x (\omega-\sigma_{\e}^2)}{1-x^2 \I_{2,2}}\Big)(\Ishift_{2,2})_1\,.
\end{equation}
Similarly, the limiting test error for LJSD $\mu_2$ is given by \cref{cor:ridgeless_lim} as
\begin{equation}
\label{eqn:prop4_2}
\err{\mu_2} = \bias{\mu_2} + \xi_2^2 \frac{1}{\phi/\psi - 1}x(\sigma_{\e}^2+\I_{1,1})(\omega_2 + (\Ishift_{1,1})_2) + \xi_2^2 x\Big(1-\frac{x (\omega-\sigma_{\e}^2)}{1-x^2 \I_{2,2}}\Big)(\Ishift_{2,2})_2\,.
\end{equation}
Note that in both expressions, the self-consistent equation for $x$ reads $x = \frac{1}{\omega+\I_{1,1}}$, which has no dependence on $\psi$, and, as such, $x$, $\I_{a,b}$, $(\Ishift_{a,b})_1$, $(\Ishift_{a,b})_2$, $\bias{\mu_1}$, and $\bias{\mu_2}$ do not depend on the overparameterization ratio $\phi/\psi$. Hence we can simply eliminate the quantity $\tfrac{1}{\phi/\psi-1}$ from \cref{eqn:prop4_1,,eqn:prop4_2} to obtain
\begin{equation}
    \err{\mu_2} = E_0 +\frac{\xi_2^2}{\xi_1^2} \Big(\frac{\omega_2 + (\Ishift_{1,1})_2}{\omega_1 + (\Ishift_{1,1})_1}\Big)\err{\mu_1}\,,
\end{equation}
where $E_0$ does not depend on the overparameterization ratio $\phi/\psi$ and is given by
\begin{equation}
\begin{split}
    E_0 &\deq \bias{\mu_2} - \frac{\xi_2^2}{\xi_1^2}\frac{\omega_2 + (\Ishift_{1,1})_2}{\omega_1 + (\Ishift_{1,1})_1} \bias{\mu_1} +  \xi_2^2 x\Big(1-\frac{x (\omega-\sigma_{\e}^2)}{1-x^2 \I_{2,2}}\Big)\Big((\Ishift_{2,2})_2-\frac{\omega_2 + (\Ishift_{1,1})_2}{\omega_1 + (\Ishift_{1,1})_1} (\Ishift_{2,2})_1\Big)\,.
\end{split}
\end{equation}
The fact that $\slope\ge0$ follows immediately from the expression and the inequalities in \cref{lem:positivity}. To establish the final conclusion, recall from \cref{sec:ordering} that the conditions $\mu_1\le\mu_2$ and ${1 \le \zeta_2/\zeta_1 \le \eta_2/\eta_1}$ imply $s_2 \xi_2^2 \ge s_1 \xi_1^2$ and ${\omega}_2/s_2 \ge {\omega}_1/s_1$, so we have
\begin{align}
    \slope &=  \frac{\xi_2^2}{\xi_1^2}\frac{\omega_2 + (\Ishift_{1,1})_2}{\omega_1 + (\Ishift_{1,1})_1}\\
    &=  \frac{s_2 \xi_2^2}{s_1 \xi_1^2}\frac{\omega_2/s_2 + (\Ishift_{1,1})_2/s_2}{\omega_1/s_1 + (\Ishift_{1,1})_1/s_1}\\
    &\ge  \frac{s_2 \xi_2^2}{s_2 \xi_2^2}\frac{\omega_2/s_2 + (\Ishift_{1,1})_2/s_2}{\omega_2/s_2 + (\Ishift_{1,1})_1/s_1}\\
    &\ge  \frac{s_2 \xi_2^2}{s_2 \xi_2^2}\frac{\omega_2/s_2 + (\Ishift_{1,1})_2/s_2}{\omega_2/s_2 + (\Ishift_{1,1})_2/s_2}\\
    &= 1\,,
\end{align}
where in the last inequality we have used \cref{eqn:Iineq}.
\end{proof}

\section{Optimal regularization}
As the overparameterization ratio $\phi/\psi$ goes to infinity, the expression for the test error simplifies and is given by the expressions in~\cref{cor:infinite_overparameterization}. In this case, it is straightforward to derive the value of the regularization constant that minimizes the test error.
\begin{proposition}[Restatement of~\cref{prop:opt_gamma}]
\label{prop:opt_gamma_sm}
In the setting of~\cref{cor:infinite_overparameterization} and under~\cref{assump:scale}, the optimal regularization $\gammaopt_\mu$ is independent of $\mu$ and satisfies,
\al{
\gammaopt_\mu \deq \argmin_{\gamma\ge-\rho\omega} \err{\mu}(\gamma) = \rho(\sigma_{\e}^2 - \omega)\qquad\text{and}\qquad \err{\mu}^{\text{opt}} \deq \err{\mu}(\gammaopt_\mu) = \Ishift_{1,1}\label{eq:gamma_opt_sm}\,.
}
Moreover, the optimal error is ordered with respect to shift strength, i.e. for two LJSDs $\mu_1 \le \mu_2$, $\err{\mu_1}^{\text{opt}}\le \err{\mu_2}^{\text{opt}}$.
\end{proposition}
\begin{proof}[Proof of \cref{prop:opt_gamma}]
From the expressions in~\cref{sec:cor_proofs} when $\psi/\phi = 0$, we have for $\gammaeff = \gamma/\rho + \omega$,
\al{
x = \frac{1}{\gammaeff+\I_{1,1}} >0\,, \qquad \frac{\partial x}{\partial \gamma} = -\frac{\tau}{\gammaeff + \phi \I_{1,2}} < 0\,,\qquad \tauone = \frac{x}{\rho}\,,\qquad\text{and}\qquad \taub = \frac{1}{\gamma} \,,
}
where the inequalities follow from the constraint $\gamma \ge -\rho\omega$, which implies that $\gammaeff \ge 0$.

Differentiation shows that,
\al{
\frac{\partial^2 x}{\partial \gamma^2} &= \frac{2}{x}\left( \left(\frac{\partial x}{\partial \gamma} \right)^2- \rho \frac{\partial x}{\partial \gamma} \I_{2,3}\right) > 0\,.
}
Under~\cref{assump:scale}, these expressions and those in \cref{cor:infinite_overparameterization} and~\cref{eq:raise_lower} give,
\al{
\label{eq:err_infinite_overparameterization}
\err{\mu} &= \Ishift_{1,1} +  \frac{\partial x}{\partial \gamma} \rho (\gammaeff - \sigma_{\e}^2)\Ishift_{2,2}\,,
}
so that,
\al{
\frac{\partial \err{\mu}}{\partial \gamma} &= -\frac{\partial x}{\partial \gamma} \Ishift_{2,2} + \frac{\partial^2 x}{\partial \gamma^2}\rho (\gammaeff - \sigma_{\e}^2)\Ishift_{2,2} + \frac{\partial x}{\partial \gamma} \Ishift_{2,2} - 2\left(\frac{\partial x}{\partial \gamma}\right)^2 \rho (\gammaeff - \sigma_{\e}^2)\Ishift_{3,3}\\
&= \rho (\gammaeff - \sigma_{\e}^2)\left(\frac{\partial^2 x}{\partial \gamma^2} \Ishift_{2,2} - 2 \left(\frac{\partial x}{\partial \gamma}\right)^2\Ishift_{3,3}\right)\\
&= \rho (\gammaeff - \sigma_{\e}^2)\left(\frac{\partial^2 x}{\partial \gamma^2} \left(\phi \Ishift_{2,3} + x \Ishift_{3,3}\right) - 2 \left(\frac{\partial x}{\partial \gamma}\right)^2\Ishift_{3,3}\right)\\
&= \rho (\gammaeff - \sigma_{\e}^2)\left(\frac{\partial^2 x}{\partial \gamma^2} \phi \Ishift_{2,3}  - 2 \rho \frac{\partial x}{\partial \gamma}\I_{2,3}\Ishift_{3,3}\right)\,.
}

Noting that the term in parentheses on the last line is positive, we find that the unique minimizer\footnote{Note that $\partial E_{\mu}/\partial \gamma <0$ for $\gammaeff < \sigma^2$ and $\partial E_{\mu}/\partial \gamma > 0$ for $\gammaeff > \sigma^2$.} of the error is given by $\gammaeff = \sigma_{\e}^2$. \cref{eq:gamma_opt_sm} then follows immediately from the definition of $\gammaeff = \gamma/\rho + \omega$ and from~\cref{eq:err_infinite_overparameterization}.

Finally, to establish the last statement, namely that $\err{\mu_1}^{\text{opt}}\le \err{\mu_2}^{\text{opt}}$ for $\mu_1 \le \mu_2$, simply note that $\err{\mu}(\gammaopt_\mu) = \Ishift_{1,1}$ and apply \cref{cor:harris_Ifunc}.
\end{proof}

\section{Proof of \texorpdfstring{\cref{thm:main_b_v}}{}}
As discussed in ~\cref{sec:setup}, we consider predictive functions $\hat{y}$ defined by random feature kernel ridge regression,
\eq{\label{eq_yhat_ntk}
\hat{y}(\bfx) \deq  YK^{-1}K_\bfx\,
}
for $K\deq K(X,X) + \gamma I_m$, $K_\bfx \deq K(X, \bfx)$, with $\gamma$ is a ridge regularization constant. Here
\begin{align}
\label{eq_K}
K &= \frac{1}{n_1}F^\top F + \gamma I_m \,,
\end{align}
and, as in the main text, we have introduced the abbreviations $F\deq\fs(W X/\sqrt{n_0})$. The labels are generated by a linear function parameterized by $\beta \in \mathbb{R}^{n_0}$, whose entries are drawn independently from $\cN(0,1)$, i.e. $Y = \beta^\top X / \sqrt{n_0} + \e$. In this section, we develop the techniques and detailed calculations needed to determine the high-dimensional limit of the test error,
\begin{equation}
\err{\Sigmate} = \mE_{\x, \beta} \mE[(y(\x) - \hat{y}(\x))^2] -\sigma_{\e}^2 = \mE_{\x, \beta} \mE[(\beta^\top \x/\sqrt{n_0} - Y K^{-1}K_{\x})^2]\,,
\end{equation}
as well as its decomposition into bias and variance terms,
\begin{equation}
    \err{\Sigmate} = \underbrace{\mE_{\x, \beta} \mE[(\mE[\hat{y}(\x)]-y(\x))^2]}_{\bias{\Sigmate}} + \underbrace{\mE_{\x, \beta}[\V[\hat{y}(\x)]]}_{\var{\Sigmate}}\,.
\end{equation}
We may also write the test loss as,
\eq{\label{eq:Etest}
\Etest = \mathbb{E}_{(\bfx,y)} (y - \hat{y}(\bfx))^2  = E_1 + E_2 + E_3
}
with
\begin{align}
E_1 &= \mathbb{E}_{(\bfx, \beta, \e)}y(\bfx)^2 =\mathbb{E}_{(\bfx, \beta, \e)}\tr(y(\bfx)y(\bfx)^\top) \label{eq_E1}\\
E_2 &= -2\mathbb{E}_{(\bfx,\beta,\e)}(K_\bfx^\top K^{-1}Y) y(\bfx) = -2\mathbb{E}_{(\bfx,\e)}\tr(K_\bfx^\top K^{-1}Y^\top y(\bfx)) \label{eq_E2}\\
E_3 &= \mathbb{E}_{(\bfx,\beta,\e)}(K_\bfx^\top K^{-1}Y^\top )^2= \mathbb{E}_{(\bfx,\e)}\tr(K_\bfx^\top K^{-1}Y^\top Y K^{-1} K_\bfx) \label{eq_E3}\,.
\end{align}

\subsection{Reducing to the mean-zero case}
\label{sec:mean_zero}
In this section, we aim to show that we may assume without loss of generality that the activation function is centered. While many of the ideas in the proof are similar to those of \citet{mei2019generalization}, at the technical level many distinctions arise owing to our differing assumptions on the data matrix $X$ and weight matrix $W$. In particular, their approach utilizes a decomposition into Gegenbauer polynomials and exact orthogonality among these orthogonal basis functions. Since this approach is not feasible in our setting of iid Gaussian matrices, we develop an alternative proof that we present below. 

To proceed, we define centered versions of the the random feature matrices $F$ and $f$ as
\eq{\label{eq_centering_def}
    \bar{F}_{ij} \deq F_{ij} - \E_{Z} \sigma\pa{ \sqrt{\ntr(\Sigmatr)}Z }  \quad\text{and}\quad \bar{f}_{i} \deq f_{i} - \E_{Z} \sigma\pa{ \sqrt{\ntr(\Sigmate)}Z }
}
for all $i$ and $j$ and $Z\sim\mathcal{N}(0, 1)$. Note that this centering operation is $n_0$-dependent, but in the limit $\ntr(\Sigmatr), \ntr(\Sigmate)\to s, s_*$, which appear in the limiting self-consistent equation. To show replacing $F$ and $f$ by $\bar{F}$ and $\bar{f}$ does not alter the test error in the limit, we have to show that the $E_1$, $E_2$, and $E_3$ terms of Equations~\eqref{eq_E1}-\eqref{eq_E3} are not changed in the limit. Clearly this is true for $E_1$ as it contains neither $F$ nor $f$. For $E_2$ we must show 
\eq{\label{eq_E2_to_bound}
    \E\qa{ (K_\bfx^\top K^{-1}Y^\top - \bar{K}_\bfx^\top \bar{K}^{-1}Y^\top) y(\bfx)} \to0
}
and for $E_3$ we must show
\eq{\label{eq_E3_to_bound}
    \E \qa{ (K_\bfx^\top K^{-1}Y^\top)^2-(\bar{K}_\bfx^\top \bar{K}^{-1}Y^\top)^2 } \to0,
}
where
\eq{
    \bar{K}\deq \frac{1}{n_1}\bar{F}^\top \bar{F} + \gamma I \quad\text{and}\quad \bar{K}_\bfx \deq \frac{1}{n_1} \bar{F}^\top \bar{f}.
}

Define the random variable
\eq{
    \Delta \deq (K_\bfx^\top K^{-1}  - \bar{K}_\bfx\bar{K}^{-1})Y^\top.
}
To control the typical behavior of $\Delta$, we will define an event $\mathcal{E}$ (see \cref{def_E}), and let $\mathbf{1}_\mathcal{E}$ be an indicator for $\mathcal{E}$. In \cref{lem_E_is_hp} we show $\P[\mathcal{E}^\mathsf{c}] \to 0$.

We can argue
\eq{\label{eq_E2_to_bound_2}
    \absa{\E \qa{\Delta y(\bfx)}} \leq \absa{\E \qa{\mathbf{1}_\mathcal{E}\Delta y(\bfx)}} + \absa{\E \qa{(1-\mathbf{1}_{\mathcal{E}})\Delta y(\bfx)}} ,
}
for Equation \eqref{eq_E2_to_bound}, and 
\al{\label{eq_E3_to_bound_2}
    \absa{\E \qa{\Delta y(\bfx)}} &\leq \absa{\E \qa{\mathbf{1}_\mathcal{E} \Delta(K_\bfx^\top K^{-1}Y^\top+\bar{K}_\bfx^\top \bar{K}^{-1}Y^\top) }} \\
    &\quad + \absa{\E \qa{(1-\mathbf{1}_{\mathcal{E}}) \pa{(K_\bfx^\top K^{-1}Y^\top)^2-(\bar{K}_\bfx^\top \bar{K}^{-1}Y^\top)^2} }} ,
}
for Equation \eqref{eq_E3_to_bound}. We demonstrate Equations \eqref{eq_E2_to_bound_2} and \eqref{eq_E3_to_bound_2} are $o(1)$ in two steps. The first terms on the right-hand sides of Equations \eqref{eq_E2_to_bound_2} and \eqref{eq_E3_to_bound_2} represent the typical behavior of the random variables. To bound them, we show that, given $\mathcal{E}$, $\Delta\to0$ in \cref{lem_typical_event}. The second terms on the right-hand sides of Equations \eqref{eq_E2_to_bound_2} and \eqref{eq_E3_to_bound_2} represent the atypical behavior of the random variables. To bound them, we use the Cauchy-Schwarz inequality and the fact that $\P[\mathcal{E}^\mathsf{c}] \to 0$ in \cref{lem_aypical_event}.

To briefly outline the structure of this section: \cref{sec_pre_lin} proves concentration of some quadratic forms of the underlying random matrices and bounds their operator norms with high probability; \cref{sec_atypical} controls the atypical behavior mentioned above; \cref{sec_schur} applies the Schur complement formula to derive an expression for $\Delta$ that we can bound; \cref{sec_typical_event} defines an event where the expression for $\Delta$ from the Schur complement formula can be easily bounded; and \cref{sec_typical} completes the argument by bounding $\Delta$ on this typical event.

\subsubsection{Prerequisites for concentration and linearization}\label{sec_pre_lin}
Just as in \citet{adlam2019random}, the concentration of the linear data kernel about its expectation is a key ingredient to the analysis. Define the random variables
\eq{
    \Upsilon \deq \frac{1}{n_0}X^\top X \quad\text{and}\quad \upsilon_*\deq \norm{\bfx}_2^2 / n_0.
}

\begin{remark}
    Throughout this section and the next we will use $C$ to denote an arbitrarily large, $n_0$-independent, positive constant, which can increase from line to line. For example, if $X\leq C$ and $Y\leq C$, we will simply write $XY\leq C$, since $C^2$ is still some $n_0$-independent constant. This approach is valid as such replacements only occur a finite number of times. Similarly, $c$ denotes some arbitrarily small, $n_0$-independent, positive constant that can decrease from line to line.
\end{remark}

Then we note
\eq{\label{eq_ex_upsilons}
    \E \Upsilon = \bar{\tr}(\Sigmatr) I_{n_0} \quad\text{and}\quad \E \upsilon_* = \bar{\tr}(\Sigmate).
}
For the variance, we see by \cref{assump:esd} that
\eq{\label{eq_var_upsilons}
    \V \Upsilon_{ij} = \frac{1+\mathbb{I}(i=j)}{n_0} \bar{\tr}(\Sigmatr) \leq C / n_0 \quad\text{and}\quad \V \upsilon_* = \frac{2}{n_0} \bar{\tr}((\Sigmate)^2) \leq C/n_0.
}
This motivates the following lemma.

\begin{lemma}
\label{lem:data}
The event
\eq{
    \mathcal{E}_\text{data} \deq \ha{\norm{\Upsilon}_\infty \leq C} \cap \ha{\absa{\upsilon_*\!-\!\ntr(\Sigmate)} \leq Cn_0^{c-1/2}} \cap \bigcap_{i,j} \ha{ \absa{ \Upsilon_{ij} \!-\! \delta_{ij} \ntr(\Sigmatr)} \leq Cn_0^{c-1/2}}
}
occurs with high-probability, that is, for some positive, $n_0$-independent constant $C$
\eq{
    \P[\mathcal{E}_\text{data}^\mathsf{c}] \leq Cn_0^{-c}
}
for any positive, $n_0$-independent constant $c$.
\end{lemma}

\begin{proof}
Tighter concentration than is obtained directly from the variance can be shown by observing that $\Upsilon$ and $\upsilon_*$ are equal in distribution to
\eq{
    \frac{1}{n_0}Z^T\Sigmatr Z \quad\text{and} \quad \frac{1}{n_0} z^\top\Sigmate z
}
for $Z$ an $(n_0,m)$-dimensional matrix and $z$ an $n_0$-dimensional vector both containing i.i.d. standard Gaussian random variables. Then using the assumption $\norm{\Sigmatr}_\infty,\norm{\Sigmate}_\infty\leq C$ and applying the results of \citet[Sec. B]{erdHos2012bulk} (or similar concentration results), we find that for some positive constant $C$ that
\eq{
    \P[\absa{\Upsilon_{ij} - \delta_{ij} \ntr(\Sigmatr) } \geq Cn_0^{c-1/2} ] \leq C \exp\p{-n_0^c},
}
where $c$ is any positive constant. An identical concentration result holds for $\upsilon_*-\ntr(\Sigmate)$. Then, by the union bound, we see
\al{
    &\P\qa{ \ha{\absa{\upsilon_*-\ntr(\Sigmate)} \geq Cn_0^{-c}} \cup \bigcup_{i,j} \ha{ \absa{ \Upsilon_{ij} - \delta_{ij} \ntr(\Sigmatr)} \geq Cn_0^{c-1/2}} } \\
    &\qquad\leq Cn_0^2 C\exp\pa{-n_0^{c}} \\
    &\qquad\leq C \exp(-n_0^c).
}

The operator norm of $\Upsilon$ can also be bounded probabilistically. For some $n_0$-independent constant $C$,
\eq{
    \norm{\Upsilon}_\infty \leq \norm{\Sigma}_\infty \norm{Z/\sqrt{n_0}}_\infty^2 \leq C\norm{\Sigma}_\infty \pa{1+\sqrt{\frac{m}{n_0}} + \sqrt{\frac{\log(2/\delta)}{n_0}}} \leq C
}
with probability at least $1-\delta$, where $Z$ is i.i.d. Gaussian \citep[see for example][Theorem 4.4.5]{vershynin2018high}. Thus, setting $\delta=n_0^{-c}$, we see $ \P[\norm{\Upsilon}_\infty \geq C] \leq n_0^{-c}$. Applying the union bound again completes the proof.
\end{proof}

\begin{lemma}\label{sec_expected_op_norm}
    The (squared) expected operator norm of $\Upsilon$ is also bounded:
    \eq{
        \E \norm{\Upsilon}_\infty^2 \leq C
    }
    for some positive, $n_0$-independent constant $C$.
\end{lemma}

\begin{proof}
    This result follows from a tail integration argument. That is using \citet[Theorem 4.4.5]{vershynin2018high}) we have that \eq{
    \norm{\Upsilon}_\infty \leq \norm{\Sigma}_\infty \norm{Z/\sqrt{n_0}}_\infty^2 \leq C\norm{\Sigma}_\infty \pa{1+\sqrt{\frac{m}{n_0}} + u} 
}
with probability at least $1-2 \exp(-n_0 u^2)$, where $Z$ is i.i.d. Gaussian. Hence we have that $\P[\norm{\Upsilon}_\infty \geq C\norm{\Sigma}_\infty \pa{1+\sqrt{\frac{m}{n_0}} + u}] \leq 2 \exp(-n_0 u^2).$
Now by the tail integration identity,
\al{
    \E[\norm{\Upsilon}_\infty^2] &= 2\int_{0}^{\infty} u \P[\norm{\Upsilon}_\infty \geq u] du \\
    &\leq C^2 \norm{\Sigma}_\infty^2\! \pa{1\!+\!\sqrt{\frac{m}{n_0}}}^2 \!+\! C^2 \norm{\Sigma}_{\infty}^2 \int_{0}^{\infty}\! u \P\qa{\norm{\Upsilon}_\infty \!\geq\! C\norm{\Sigma}_\infty \!\pa{1\!+\!\sqrt{\frac{m}{n_0}} \!+\! u}} du  \\
    & \leq C^2 \norm{\Sigma}_\infty^2 \pa{1+\sqrt{\frac{m}{n_0}}}^2 + C^2 \norm{\Sigma}_{\infty}^2 \cdot \frac{1}{n_0} \\
    & \leq C.
}
by integrating the Gaussian tail.
\end{proof}

Next, we state similar results for the matrix $\tilde{\Upsilon}\deq W\Sigmate W^\top /n_0$ as were obtained for $\Upsilon$ above.

\begin{lemma}
\label{lem:Wlem}
The event
\eq{
    \mathcal{E}_W \deq \ha{\norma{\tilde{\Upsilon}}_\infty \leq C} \cap \bigcap_{i,j} \ha{ \absa{ \tilde{\Upsilon}_{ij} \!-\! \delta_{ij} \ntr(\Sigmate)} \leq Cn_0^{c-1/2}}
}
occurs with high-probability, that is, for some positive, $n_0$-independent constant $C$
\eq{
    \P[\mathcal{E}_W^\mathsf{c}] \leq Cn_0^{-c}
}
for any positive, $n_0$-independent constant $c$.
\end{lemma}

\begin{proof}
The claims follow identically to \cref{lem:data}.
\end{proof}

\begin{lemma}\label{lem:Wlem_taylor}
Moreover on the event $\mathcal{E}_W$, we have 
\eq{
    \E_\bfx \bar{\sigma}(W\bfx/\sqrt{n_0})_i \leq  C n_0^{c-1/2} 
}
and
\eq{
    \E_\bfx \bar{\sigma}(W\bfx/\sqrt{n_0}) \bar{\sigma}(W\bfx/\sqrt{n_0})^\top = \zetas\tilde{\Upsilon} + (\etas - \zetas) I_m + \Delta,
}
where $\norm{\Delta}_{\infty} \leq C n_0^c$.
\end{lemma}

\begin{proof}
Note that conditional on $W$, $W \bfx/\sqrt{n_0} \sim \mathcal{N}(0, \tilde{\Upsilon})$. Using a Taylor expansion in ${\epsilon\deq  \tilde{\Upsilon}_{ii} - \ntr(\Sigma*)}$ below with \cref{assump:sigma} shows that elementwise
\al{
    \absa{\E_\bfx \bar{\sigma}(W\bfx/\sqrt{n_0})_i} &= \absa{ \E_Z \bar{\sigma}\pa{\sqrt{\tilde{\Upsilon}_{ii}} Z }} \\
    &= \absa{\E_Z {\sigma}\pa{\sqrt{\epsilon + \bar{\tr}(\Sigmate)} Z } - \E_Z {\sigma}\pa{\sqrt{\bar{\tr}(\Sigmate)} Z }} \\
    & \leq C \epsilon \\
    &\leq C n_0^{c-1/2},\label{eq_cen_wrt_x2}
}
    where $Z$ is a standard Gaussian. See \citet{adlam2019random} for additional details. This allows us to assume that $\sigma(W\bfx)$ is centered as a conditional expectation over $\bfx$ at the expense of an elementwise $C n_0^{c-1/2}$ error. Denoting $\bar{f}\deq \bar{\sigma}(W\bfx/\sqrt{n_0})$, we note this can be extended to the matrix $\E \bar{f}\bar{f}^\top$. We see
\eq{\label{eq_cen_wrt_x}
    \E_\bfx \bar{f} \bar{f}^\top = \E_\bfx[ (\bar{f} - \E_\bfx \bar{f}) (\bar{f} - \E_\bfx \bar{f})^\top] + (\E_\bfx \bar{f})(\E_\bfx \bar{f})^\top.
}
However, the second term on the right-hand side of Equation \eqref{eq_cen_wrt_x} is small is operator norm, since it is rank-1 and
\eq{\label{eq_rank1_part}
    \norma{(\E_\bfx \bar{f})(\E_\bfx \bar{f})^\top }_\infty = \norma{\E_\bfx \bar{f} }_2^2 \leq C n_0^c
}
by Equation \eqref{eq_cen_wrt_x2}. Next, we see for $i\neq j$ that
\eq{
    \E_\bfx \bar{f}_i \bar{f}_j = \E_{Z_1, Z_2} \bar{\sigma}(c_1 Z_1) \bar{\sigma}(c_{12} Z_1 + c_2Z_2)
}
for constants $c_1$, $c_2$, and $c_{12}$ depending on $W$ and $\Sigmate$, where $Z_1$ and $Z_2$ are independent, standard Gaussians. Again, see \citet{adlam2019random} for details. The $i=j$ terms can be handled similarly. Moreover, ${c_{12}=\tilde{\Upsilon}_{ij} / \sqrt{\tilde{\Upsilon}_{ii}} \leq C n_0^{c-1/2}}$, so Taylor expanding in $c_{12} Z_1$, we find
\eq{
    \E_\bfx \bar{f}\bar{f}^\top = \zetas\Upsilon + (\etas - \zetas) I + \tilde{\Delta} + (\E_\bfx \bar{f})(\E_\bfx \bar{f})^\top.
}
As a consequence of the Taylor expansion and \cref{assump:sigma}, $\tilde{\Delta}$ has entries that are bounded by $C n_0^{c-1}$ in absolute value \citep{adlam2019random}. The final conclusion for ${\Delta =\tilde{\Delta} + (\E_\bfx \bar{f})(\E_\bfx \bar{f})^\top} $ follows by upper bounding $\norm{\tilde{\Delta}}_{\infty} \leq \norm{\tilde{\Delta}}_{F} \leq Cn_0^c$ and using the previous bound in Equation \eqref{eq_rank1_part}.
\end{proof}

Finally we include a bound on the operator norm of the random feature matrix. The argument follows \citet[Lemma C.3]{mei2019generalization} closely.

\begin{lemma}
\label{lem:op_norm}
Under \cref{assump:sigma}, the event
\eq{
    \mathcal{E}_F \deq \ha{\norma{\bar{F} / \sqrt{n_0}}_\infty \leq Cn_0^c}
}
occurs with high-probability, that is, for some positive, $n_0$-independent constant $C$
\eq{
    \P[\mathcal{E}_F^\mathsf{c}] \leq Cn_0^{-c}
}
for any positive, $n_0$-independent constant $c<10$.
 
\end{lemma}
\begin{proof}
Consider the matrix
\begin{align}
    \bar{R}_{ij} = 1_{i \neq j} \bar{\sigma}(z_i^\top\Sigmatr^{1/2} z_j/\sqrt{n_1})/\sqrt{n_1} 
\end{align}
for $z_i = \Sigmatr^{-1/2}X_i$ (i.e. columns of $\Sigmatr^{-1/2}X$) for $1\leq i\leq m$ and $z_{m+i} = W_i$ (i.e. rows of $W$) for $1\leq i\leq n_0$. By construction, we note that $\bar{F}/\sqrt{n_1}$ is a minor of $\bar{R}$. Thus, bounding the operator norm of $\bar{R}$ suffices to bound the operator norm of $\bar{F}$. Moreover, $z_i$ are independent Gaussians distributed as $\mathcal{N}(0, I_{n_0})$.

Next we show that the entries to the activation function cannot be too large with high probability. For convenience throughout we define $M=m+n_0$ and let $\Omega$ be a symmetric matrix with $\norm{\Omega}_{\infty} \leq C$. Note that, for $i \neq j$, the random variable $z_i^\top \Omega z_j/\sqrt{n_0} = \sum_{k=1}^{n_0} \lambda_k(\Omega) a_k b_k/\sqrt{n_0}$ where $a_k$ and $b_k$ are i.i.d. from $\sim \mathcal{N}(0,1)$. So the moment generating function can be bounded as,
\al{
    \mE\qa{\exp\pa{\sum_{i=1}^{n_0} t \lambda_i(\Omega) a_i b_i/\sqrt{n_0}}} &= \prod_{i=1}^{n_0} \E[\exp(a_i \cdot t \lambda_i(\Omega) b_i/\sqrt{n_0})]\\
    &= \prod_{i=1}^{n_0} \E[\exp(t^2 b_i^2 \lambda_i(\Omega)^2/(2n_0)] \\
    &\leq \prod_{i=1}^{n_0} \E[\exp(t^2 b_i^2 C^2/(2n_0)] \\
    &\leq \E[\exp(t^2 b_i^2 C^2/2)] \\
    &= \frac{1}{\sqrt{1-C^2t^2}}
}
for $\abs{t} \leq \frac{1}{C}$. Further $\frac{1}{\sqrt{1-C^2t^2}} \leq \exp(C^2 t^2)$ for $\abs{t} \leq \frac{1}{2 C}$. This establishes the original random variable is subexponential. Thus this random variables satisfies
\begin{align}
    \P[\abs{z_i^\top \Omega z_j/\sqrt{n_0}} \geq u] \leq 2 \max\{\exp(-u^2/2C^2), \exp(-u/2C)\}
\end{align}
\citep[see for example][Proposition 2.9]{wainwright2019high}. Define the event 
\eq{
    \mathcal{G} \deq  \bigcap_{1 \leq i < j \leq M} \ha{\absa{ \frac{z_i^\top \Sigma^{1/2} z_j}{\sqrt{n_0}}} \geq C \sqrt{\log M}} \cap \bigcap_{1 \leq i < j \leq M}\ha{ \absa{\frac{z_i^\top \Sigma z_j}{\sqrt{n_0}}} \geq C^2 \sqrt{\log M} }.
} 
Then by a union bound and Markov's inequality, we have that
\eq{
    \P[\mathcal{G}^\mathsf{c}] \leq {4}/{M^{20}}
}
for large enough $C$.

We now define a modified version of $\bar{\sigma}$, that is the same up to a constant factor on $\mathcal{G}$ but is truncated outside of $\mathcal{G}$. Let $\bar{u} = C \sqrt{\log M}$ and taking the constants from \cref{assump:sigma}, we see
\eq{
    \tilde{\sigma}(u) \deq \begin{cases}
        \bar{\sigma}(u) \exp(-c_1 \abs{\bar{u}})/c_0 &\text{for } \abs{u} \leq \bar{u} \\
        \bar{\sigma}(\bar{u}) \exp(-c_1 \abs{\bar{u}})/c_0 & \text{for } u > \bar{u} \\
        \bar{\sigma}(-\bar{u}) \exp(-c_1 \abs{\bar{u}})/c_0 & \text{for } u < -\bar{u}
    \end{cases}.
}
Just as we did with $\sigma$ in Equation \eqref{eq_centering_def}, we center $\tilde{\sigma}$ with its mean $\tilde{a}$. Note $\tilde{\sigma}$ is a 1-bounded and 1-Lipschitz
function so $\abs{\tilde{a}} \leq 1$. 

Now consider the matrix
\begin{align}
    \tilde{R}_{ij} \deq  1_{i \neq j} (\tilde{\sigma}(z_i^\top \Sigma^{1/2} z_j/\sqrt{n_0})-\tilde{a})/\sqrt{n_1}.
\end{align}
By controlling the operator norm of $\tilde{R}$, we can control the operator norm of $\bar{R}$ at the end of the proof. By a covering argument it suffices to control
\begin{align}
    \norm{\tilde{R}}_{\infty} \leq \max_{v \in S} 10 \underbrace{\abs{v^\top \tilde{R} v}}_{F_v(Z)},
\end{align}
where $S$ is a $1/4$-covering of the $M$-dimensional sphere with cardinality $\exp(c M)$ and the matrix $Z$ is of all the variables $z_i$. 

We now seek to apply \citet[Lemma 9, Lemma 20]{10.5555/2946645.3007094}. To this end
we wish to show that $F_{v}(Z)$ is Lipschitz in $Z$, that is, if we define $\mathcal{Z} \deq \sqrt{t} Z + \sqrt{1-t} Z'$ for ($Z, Z') \in \mathcal{G} \times \mathcal{G}$, we need to show that
\begin{align}
    \max_{v \in S} \max_{t \in [0,1]} \norma{\nabla F_{v}(\sqrt{t} Z + \sqrt{1-t} Z')}_F \leq L
\end{align}
for suitable $L$. Consider the gradient with respect to a column of $\mathcal{Z}$ (denoted by $\zeta_l$):
\begin{align}
    \nabla_{\zeta_{l}} F_{v}(\mathcal{Z}) 
    &= 2\frac{v_{l}}{\sqrt{n_0 n_1}} \sum_{i \neq l} \Sigma^{1/2} \zeta_i \underbrace{v_i \tilde{\sigma}'(\zeta_i^\top \Sigma^{1/2} \zeta_l^\top/\sqrt{n_0})}_{\xi_i} 
    \\
    & =2\frac{v_l}{\sqrt{n_0 n_1}} \Sigma^{1/2} \mathcal{Z} \mathbf{\xi},
\end{align}
where $\mathbf{\xi}$ is the vector with coordinates $\xi_i$ except at $l$ where it is zero. Continuing, on the set $\mathcal{G}$, we find 
\begin{align}
    \norm{\nabla_{w_l} F_{v}(W)}_2^2 &\leq C^2 v_l^2/n_0^2 \sum_{i \neq l, j \neq l} \absa{\xi_i \xi_j w_i^\top \Sigma w_j}  \\
    &\leq C v_l^2 \sqrt{\log M}/n_0^2 \sum_i\sum_{j \neq i} \abs{\xi_i \xi_j} \\
    &\leq C v_l^2 \sqrt{\log M}/n_0^2 \sum_i\sum_{j \neq i} (\xi_i^2+\xi_j^2) \\
    &\leq C v_l^2 \sqrt{\log M}/n_0
\end{align}

where the last line follows since  $\abs{\tilde{\sigma}'(\cdot)}_2 \leq 1$ implies that $\norm{\xi} \leq 1$. So finally, we obtain
\begin{align}
    \norm{\nabla_{W} F_{v}(W)}_F^2 \leq C \sqrt{\log M}/n_0 =: L^2.
\end{align}
Now \citet[Lemma 9]{10.5555/2946645.3007094} shows for a constant $D$ that
\begin{align}
    \P[{\abs{v^\top \tilde{R} v}} > D] \leq C \exp \left(C n_0-\frac{D^2}{L^2} \right) + \frac{C}{D^2} \mE[\max_{v \in S} (F_v(Z)-F_v(Z'))^2 \cdot \mathbf{1}_{\mathcal{G}^\mathsf{c}} ].
\end{align}
We use a crude bound on the complement. First note that
\eq{
    \max_{v \in S} (F_v(Z)-F_v(Z'))^2 \leq C \norm{\tilde{\sigma}(Z^\top \Sigma^{1/2} Z/\sqrt{n_0})}_{F}^2 \leq C\norm{Z^\top Z/\sqrt{n_0}}_{F}^2 + C n_0^2,
}
since removing the 1-bounded-Lipschitz $\tilde{\sigma}(\cdot)$ can be done by centering each activation with $\tilde{\sigma}(0)$. Then,
\begin{align}
    \mE\qa{\max_{v \in S} (F_v(Z)-F_v(Z'))^2 \cdot \mathbf{1}_{\mathcal{G}^\mathsf{c}}} \leq \sqrt{\frac{C}{n_1} (\E[ \norm{Z^\top \Sigma^{1/2} Z/\sqrt{n_0}}_F^4]+n_0^4) \P[\mathcal{G}^\mathsf{c}]}.
\end{align}
A short computation shows that $ \E[ \norm{Z^\top \Sigma^{1/2} Z/\sqrt{n_0}}_F^4] \leq C n_0^8$. Combining all terms then shows that
\begin{align}
    \P[\abs{v^\top \tilde{R} v} > D] \leq C \exp (C n_0) \cdot \exp\pa{-\frac{n_0 D^2}{\log M} } + \frac{C}{D^2} \cdot\frac{C}{M^{5}}.
\end{align}
Choosing $D = c_3 \sqrt{\log M}$ for sufficiently large $c_3$ shows that
\begin{align}
    \P[\abs{v^\top \tilde{R} v} > D] \leq C \exp(-c n_0) + \frac{C}{n_0^{10}} \leq \frac{C}{n_0^{10}},
\end{align}
where $c>0$. 

Recalling the original $\epsilon$-net covering, this implies that
\begin{align}
    \norm{\tilde{R}}_{\infty} \leq C \sqrt{\log n_0} 
\end{align}
with probability at least $C/n^{10}$. We can now finish the argument by relating the operator norm of $\bar{R}$ and $\tilde{R}$. Recalling the rescaling between the modified and unmodified activations, we see
\begin{align}
    \P[\norm{\bar{R}}_{\infty} \geq C \sqrt{\log n_0} \cdot c_0 \exp(c_1 \abs{\bar{u}})] \leq \P[\norm{\tilde{R}}_{\infty} \geq C \sqrt{\log (n_0)}, \mathcal{G}] + \P[\mathcal{G}^\mathsf{c}] \leq C/n_0^{10}.
\end{align}
\end{proof}

\subsubsection{Controlling the atypical behavior}\label{sec_atypical}

Since the argument for the atypical event is the most straightforward, we provide that first.

\begin{lemma}\label{lem_aypical_event}
    Suppose $\P[\mathcal{E}^\mathsf{c}] = C n_0^{-c}$ for some $n_0$-independent constants $C>0$ and $c>0$. Then,
    \eq{\label{eq_atypical_11}
        \absa{\E \qa{(1-\mathbf{1}_{\mathcal{E}})\Delta y(\bfx)}} \to0 
    }
    and
    \eq{\label{eq_atypical_21}
        \absa{\E \qa{(1-\mathbf{1}_{\mathcal{E}}) \pa{(K_\bfx^\top K^{-1}Y^\top)^2-(\bar{K}_\bfx^\top \bar{K}^{-1}Y^\top)^2} }} \to0
    }
    as $n_0\to\infty$.
\end{lemma}

\begin{proof}
     Applying the Cauchy-Schwarz inequality to the left-hand side of Equation \eqref{eq_atypical_11}, we may bound it by
    \eq{\label{eq_atypical_1_1}
        \E \absa{ \qa{(1-\mathbf{1}_\mathcal{E})\Delta y(\bfx)}} \leq \pa{ \E \qa{\mathbf{1}_{\mathcal{E}^\mathsf{c}} } }^{1/2} \pa{ \E\absa{\Delta y}^2 }^{1/2}.
    }
    Since the first term $\P[\mathcal{E}_c] \to 0$ it also follows that $(\P[\mathcal{E}_c])^{1/2} \to 0$. So it suffices to show ${\E[\abs{\Delta y}^2] = O(1)}$.
    
    We now apply the Cauchy-Schwarz inequality again to bound the second term by
    \begin{align}
        \E\absa{ (\Delta y(\bfx))^2} \leq \sqrt{\E[\Delta^4] \E[y(\bfx)^4]}. 
    \end{align}
    A simple argument shows that $\mE[y(\bfx)^4] \leq C(\mE[(\beta^\top \x)^4/n_0^2] + \mE[\e^4]) \leq C\mE[\norm{\Sigmatr^{1/2} \z}^4/n_0^2] + C \sigma_{\e}^4 \leq C$ for $\z \sim \cN(0, I_d)$ since $\E[\norm{\z}_2^4]/n_0^2] \leq C$. Hence it suffices to show $\E[\Delta^4] \leq C$. To this end note that
    \begin{align}
        \mE[\Delta^4] \leq C( \E[(K_\bfx^\top K^{-1}Y^\top)^4] + \E[(\bar{K}_\bfx^\top \bar{K}^{-1}Y^\top)^4])
    \end{align}
    An identical argument can be applied to bound both terms. We outline the argument for the first term. 
    We see
    \begin{align}
        \abs{\E[(K_\bfx^\top K^{-1}Y^\top)^4]} \leq \E \abs{ \E_{\beta, \e}[(K_\bfx^\top K^{-1}Y^\top)^4 | W, X, x]}.
    \end{align}
    Now note that by the definition of $Y=\beta^\top X/\sqrt{n_0} + \e$ we have that,
    \begin{align}
        (K_\bfx^\top K^{-1}Y^\top)^4 \leq C \left((K_\bfx^\top K^{-1} \frac{X}{\sqrt{n_0}} \beta )^4 + (K_\bfx^\top K^{-1} \e)^4 \right)
    \end{align}
    Recalling that $\beta \sim \mathcal{N}(0, I_{n_0})$ and $\e \sim \mathcal{N}(0, I_m)$, we can compute the Gaussian moments (marginally in $\beta, \e$) as,
    \begin{align}
        &\mE_{\beta, \e}[(K_\bfx^\top K^{-1} \frac{X}{\sqrt{n_0}} \beta )^4 + (K_\bfx^\top K^{-1} \e)^4 | W, X, x] \\
        &\qquad \leq C \left( (K_\bfx^\top K^{-1} \left(\frac{X^\top X}{n_0} \right) K^{-1}K_\bfx^\top )^2 + \sigma_{\e}^2 (K_\bfx^\top K^{-2}K_\bfx^\top )^2\right)
    \end{align}
    Continuing, using the definition of $K_{\bfx}$,
    \al{
       \E[(K_\bfx^\top K^{-1} \left(\frac{X^\top X}{n_0} \right) K^{-1}K_\bfx^\top)^2] &\leq \E \left [\left \Vert \frac{f^\top}{\sqrt{n_0}} \frac{FK^{-1}}{\sqrt{n_0}} (\frac{X^\top X}{n_0}) \frac{K^{-1} F^\top}{\sqrt{n_0}} \frac{f}{\sqrt{n_0}} \right \Vert_{\infty}^2 \right]  \\
        &\leq \E\qa{\norma{ \frac{f}{\sqrt{n_0}}  }_{2}^4 \cdot \norma{ \frac{F K^{-1}}{\sqrt{n_0}} }_{\infty}^4 \cdot \norma{ \frac{X^\top X}{n_0} }_{\infty}^2} \label{eq_atypical_intermediate}
    }
    Noting that $K = FF^\top/n_0 + \gamma I_m$ an application of the SVD to $\frac{F}{\sqrt{n_0}}$ (denoting the corresponding singular values as $s_i$) shows that $\norm{\frac{FK^{-1}}{\sqrt{n_0}}}_{\infty} = \max_{i} \frac{s_i}{s_i^2+\gamma}  \leq \max_{s \geq 0} \frac{s}{s^2+\gamma} = 1/\p{2\sqrt{\gamma}}$. Hence, we can continue bounding Equation \eqref{eq_atypical_intermediate} by 
    \al{
          C  \mE[\norma{ {f}/{\sqrt{n_0}}  }_{\infty}^4]  \cdot\E[\Vert X^\top X/n_0 \Vert_{\infty}^2],
    }
    where we exploited the independence of $f$ and $X$. \cref{sec_expected_op_norm} ensures that $ \E[\norm{X^\top X/n_0}_{\infty}^2] = \E[\norm{\Upsilon}_{\infty}^2] \leq C$. The former term becomes $\mE[\Vert {f}/{\sqrt{n_0}} \Vert_{2}^4] = \E[\sigma(\sqrt{\ntr\p{\Sigmate}} z)^4] \leq C$ for $z \sim \cN(0,1)$ by \cref{assump:sigma}. Thus, the  previous displays are bounded by some constant $C$. An entirely analogous argument shows that $\E[(K_\bfx^\top K^{-2}K_\bfx^\top )^2] \leq C$. Together these two results along with the previous computations establish that, 
    \begin{align}
        \abs{\E[(K_\bfx^\top K^{-1}Y^\top)^4]} \leq C
    \end{align}
    Note that our argument did not exploit any explicit properties of the centered vs uncentered activation function $\sigma$ vs. $\bar{\sigma}$. Hence an identical argument shows that, $\abs{\E[(\bar{K}_\bfx^\top \bar{K}^{-1}Y^\top)^4]} \leq C$.
    These two results imply ${\E[\Delta^4]} \leq C$ as desired.
    
    Now we turn to the second term. Using an identical application of the Cauchy-Schwarz inequality as used for the first term, we can bound Equation \eqref{eq_atypical_21} by
    \eq{
        \pa{ \E \qa{\mathbf{1}_{\mathcal{E}^\mathsf{c}} } }^{1/2} \pa{ \E\absa{(K_\bfx^\top K^{-1}Y^\top)^2-(\bar{K}_\bfx^\top \bar{K}^{-1}Y^\top)^2}^2 }^{1/2}.
    }
 Hence we can show Equation \eqref{eq_atypical_21} is $o(1)$ if 
   $\E \absa{(K_\bfx^\top K^{-1}Y^\top)^2-(\bar{K}_\bfx^\top \bar{K}^{-1}Y^\top)^2}^2 = O(1)$. This result follows immediately from the computations shown for the previous term since we have already established that $\E [(K_\bfx^\top K^{-1}Y^\top)^4]\leq C$ and $\E[(\bar{K}_\bfx^\top \bar{K}^{-1}Y^\top)^4]\leq C$. These two previously shown results imply
   \begin{align}
       \E \absa{(K_\bfx^\top K^{-1}Y^\top)^2-(\bar{K}_\bfx^\top \bar{K}^{-1}Y^\top)^2}^2 \leq C \left(\E[(K_\bfx^\top K^{-1}Y^\top)^4]+\E[(\bar{K}_\bfx^\top \bar{K}^{-1}Y^\top)^4] \right) \leq C.
   \end{align}
\end{proof}

\subsubsection{Schur complement formula for bounding $\Delta$}\label{sec_schur}
The centering procedure is a rank-1 change to $F$ and $f$, so applying the Schur complement formula to understand its effect is natural. Write ${a\deq \E_{Z} \sigma\pa{ \sqrt{\ntr(\Sigmatr)}Z }}$, ${a_*\deq \E_{Z} \sigma\pa{ \sqrt{\ntr(\Sigmate)}Z }}$, ${\bfv \deq 1/n_1 \bar{F}^\top 1_{n_1}}$, ${\bfu \deq a 1_m}$, ${U \deq [\bfu, \bfv]^\top}$, and ${C\deq \left(\begin{smallmatrix} 1 & 1\\1 & 0\end{smallmatrix}\right)}$. Then,
\al{
K^{-1} &= (\bar{K}+ \bfu \bfv^\top + \bfv \bfu^\top + \bfu \bfu^\top)^{-1}\\
&= (\bar{K}+ U^\top C U)^{-1}\\
&= \bar{K}^{-1} -\bar{K}^{-1} U^\top (C^{-1} + U \bar{K}^{-1} U^\top)^{-1}U \bar{K}^{-1}\,,
}
and, for $\delta \deq 1/n_1 \bar{f}^\top 1_{n_1}$ and $P \deq (\delta + a_*, a_*)^\top$,
\al{
K_\bfx &= \bar{K}_\bfx + \frac{1}{n_1}(F-\bar{F})^\top \bar{f} + \frac{1}{n_1}\bar{F}^\top (f-\bar{f}) + \frac{1}{n_1}(F-\bar{F})^\top (f-\bar{f}) \\
&= \bar{K}_\bfx + (\delta +a_*)\bfu + a_* \bfv\\
&= \bar{K}_\bfx + U^\top P\,.
}
Combining these expressions,
\al{
K^{-1} K_\bfx &= K^{-1}(\bar{K}_\bfx + U^\top P)\\
&= \bar{K}^{-1}(\bar{K}_\bfx + U^\top P) - \bar{K}^{-1} U^\top (C^{-1} + U \bar{K}^{-1} U^\top)^{-1}U \bar{K}^{-1}(\bar{K}_\bfx + U^\top P)\\
&= \bar{K}^{-1}\bar{K}_\bfx + T_1 + T_2\,,
}
with
\al{
T_1 &= \bar{K}^{-1}U^\top \left(I_2 - (C^{-1} + U \bar{K}^{-1} U^\top)^{-1}U \bar{K}^{-1}U^\top\right)P\\
&= \bar{K}^{-1}U^\top \left(I_2+C U\bar{K}^{-1}U^\top\right)^{-1}P\\
T_2 &= -\bar{K}^{-1} U^\top (C^{-1} + U \bar{K}^{-1} U^\top)^{-1}U \bar{K}^{-1}\bar{K}_\bfx\\
&= -\bar{K}^{-1} U^\top \left(I_2+C U\bar{K}^{-1}U^\top\right)^{-1}C U \bar{K}^{-1}\bar{K}_\bfx\,.
}
Furthermore, writing ${c_0 \deq \tfrac{1}{m} \bfu^\top\bar{K}^{-1}\bfu}$, ${c_1 \deq 1 + \bfu^\top \bar{K}^{-1}\bfv}$,
${c_2 \deq 1 - \bfv^\top\bar{K}^{-1}\bfv}$, 
\al{
\left(I_2+C U\bar{K}^{-1}U^\top\right)^{-1} &= \frac{1}{c_1^2+m c_0c_2}\begin{pmatrix}c_1 & c_2-c_1  \\ -m c_0 & m c_0 + c_1\end{pmatrix}\,,
}
so that,
\al{
\left(I_2+C U\bar{K}^{-1}U^\top\right)^{-1}P &= \frac{1}{c_1^2+m c_0c_2}\begin{pmatrix}c_1\delta + c_2a_* \\ -m c_0 \delta + c_1a_*  \end{pmatrix}\,,
}
and,
\al{
\left(I_2+C U\bar{K}^{-1}U^\top\right)^{-1}C &= \frac{1}{c_1^2+m c_0c_2}\begin{pmatrix}c_2& c_1\\ c_1  & -m c_0\end{pmatrix}\,.
}

\subsubsection{Concentration with high-probability}\label{sec_typical_event}

\begin{definition}\label{def_E}
Recall the definitions from \cref{sec_schur}. To characterize the typical behavior of the random variables, we define the event
\al{
    &\mathcal{E}\deq \mathcal{E}_\text{data}\cap \mathcal{E}_W \cap\mathcal{E}_F \cap\ha{|\delta| \leq Cn_0^{c-1/2}} \cap\ha{cn_0^{-c}\leq c_0 \leq Cn_0^{c}}  \cap\ha{c_1 \leq Cn_0^{c+1/2}} \label{eq_E_def1}\\
    &\quad \cap\ha{cn_0^{-c}\leq c_2 \leq Cn_0^{c}} \cap \ha{Y\bar{K}^{-1}\bfv^\top \leq Cn_0^{c} } \cap\ha{Y\bar{K}^{-1}\bfu^\top \leq Cn_0^{c+1/2}}. \label{eq_E_def2} \\
    & \cap \ha{\bar{K}_\bfx\bar{K}^{-1}\bfu^\top \leq C n_0^{c}}
    \cap \ha{\bar{K}_\bfx\bar{K}^{-1}\bfv^\top \leq C n_0^{c-1/2}} \cap \ha{y(\bfx) \leq C n_0^{c}} \label{eq_E_def3}
}

\end{definition}

\begin{lemma}\label{lem_E_is_hp}
    The event $\mathcal{E}$ is high-probability, that is, $\P[\mathcal{E}^\mathsf{c}]\leq Cn_0^{-c}$ for some constant $C>0$ and any constant $0<c<10$.
\end{lemma}
\begin{proof}
    First, we already know $\mathcal{E}_\text{data}$, $\mathcal{E}_W$, and $\mathcal{E}_F$ all occur with high-probability (see \cref{lem:data,lem:Wlem,lem:op_norm}). When bounding the other events in Equations \eqref{eq_E_def1}-\eqref{eq_E_def3}, it will be useful to introduce conditioning on one of $\mathcal{E}_\text{data}$, $\mathcal{E}_W$, or $\mathcal{E}_F$. This will suffice since for any event $\mathcal{A}$, we have
    \eq{
        \P[\mathcal{A}^\mathsf{c}] \leq \P[\mathcal{A}^\mathsf{c}|\mathcal{E}_\text{data}] + \P[\mathcal{E}_\text{data}^\mathsf{c}] \leq \P[\mathcal{A}^\mathsf{c}|\mathcal{E}_\text{data}] + Cn_0^{-c}
    }
    for example. We will explicitly denote this conditioning, but when taking expectation over some random variables (e.g. $x$ or $W$) we will not explicitly denote that we are conditioning on the remaining independent random variables (e.g. $X$). Once we have bounded the probability of the complement of each event in Equations \eqref{eq_E_def1}-\eqref{eq_E_def3}, we can apply the union bound to complete the proof.
    
    For future reference recall ${\bfv \deq 1/n_1 \bar{F}^\top 1_{n_1}}$, ${\bfu \deq a 1_m}$. We deal with each event in Equations \eqref{eq_E_def1}-\eqref{eq_E_def3} in turn.

    \paragraph{Controlling $\delta$:} Here, we introduce conditioning on the event $\mathcal{E}_\text{data}$. To control $\delta$, we want to calculate its mean and variance to apply Chebyshev's inequality. Note that $\mathcal{E}_\text{data}$ is independent of $W$, so the expectations over $W$ below are unchanged by this conditioning. Recall
\eq{
    \delta = \frac{1}{n_1}\sum_{k=1}^{n_1} {\sigma}\pa{\sum_{j=1}^{n_0} W_{kj}x_j/\sqrt{n_0}} - a_*
}
and note that conditional on $x$ the sum $\sum_j W_{kj}x_j/\sqrt{n_0}$ is distributed as $\mathcal{N}(0, \upsilon_*)$ for all $k$. 

Expanding in $\upsilon_* - \ntr(\Sigmate)$, we see
\al{
    \E_W[ \delta |\mathcal{E}_\text{data} ]&= \E[\E_{Z\sim\mathcal{N}(0,1)} [{\sigma}(\sqrt{\bar{\tr}(\Sigmate)+(\upsilon_*-\bar{\tr}(\Sigmate))}Z)]  - a_*|\mathcal{E}_\text{data}] \\
    &=\E[\E_Z[ {\sigma}(\sqrt{\bar{\tr}(\Sigmate)}Z)] - a_*|\mathcal{E}_\text{data}]\\
    &\qquad+\E\qa{\frac{1}{2}\frac{(\upsilon_* -\bar{\tr}(\Sigmate))}{\bar{\tr}(\Sigmate)} \E_Z [\sqrt{\bar{\tr}(\Sigmate)}Z {\sigma}(\sqrt{\bar{\tr}(\Sigmate)}Z) ] + \Delta\Big|\mathcal{E}_\text{data}}\\
    &= \E\qa{\frac{1}{2}\frac{(\upsilon_* -\bar{\tr}(\Sigmate))}{\bar{\tr}(\Sigmate)} \E_Z [\sqrt{\bar{\tr}(\Sigmate)}Z {\sigma}(\sqrt{\bar{\tr}(\Sigmate)}Z) ] + \Delta|\mathcal{E}_\text{data}},\label{eq_ex_delta}
}
where $Z$ is a standard Gaussian. The first term of Equation \eqref{eq_ex_delta} is less than $Cn_0^{c-1/2}$ and the remainder term $\Delta$ is less than $Cn_0^{c-1}$ on $\mathcal{E}_\text{data}$. More detail on this argument can be found in \citet{adlam2019random}. Taking expectation over the remaining randomness, we see $\E[\delta|\mathcal{E}_\text{data}] \leq Cn_0^{c-1/2}$.

Similarly,
\eq{
    \V_W[\delta |\mathcal{E}_\text{data}] = \frac{1}{n_1^2}\sum_{k=1}^{n_1} \V_W\qa{{\sigma}\pa{\sum_{j=1}^{n_0} W_{kj}x_j / \sqrt{n_0}}\Big|\mathcal{E}_\text{data}} = \frac{1}{n_1} \V_{Z\sim\mathcal{N}(0,\upsilon_* )}{\sigma}(Z),
}
where we can again Taylor expand in $\upsilon_* - \ntr(\Sigmate)$ to get the bound $\V_W[\delta |\mathcal{E}_\text{data} ] \leq C/n_0$ on $\mathcal{E}_\text{data}$. Using the law of total variance, we can then bound $\V [\delta|\mathcal{E}_\text{data}] \leq C / n_0$.

Finally applying Chebyshev's inequality implies for all $c>0$ that
\eq{
    \P[ \abs{\delta-\E[\delta|\mathcal{E}_\text{data}]} > Cn_0^{c-1/2} |\mathcal{E}_\text{data}] \leq C n_0^{-c},
}
and so
\eq{
    \P[ \abs{\delta} > Cn_0^{c-1/2} |\mathcal{E}_\text{data}] \leq C n_0^{-c}.
}

    \paragraph{Controlling $c_0$, $c_1$, and $c_2$:}  Here, we introduce conditioning on the event $\mathcal{E}_F$. The results for $c_0$, $c_1$, and $c_2$ can be found using a similar argument to terms in \citet[Lemma 9.6/Step 3]{mei2019generalization}. The identifications $c_0 \leftrightarrow K_{11}$, $c_1 \leftrightarrow K_{12}$, $c_2 \leftrightarrow 1-K_{22}$ hold.
    
    For $c_0=\frac{1}{m} \bfu^\top \bar{K}^{-1} \bfu$ we have that
    \eq{
        c_0 \leq \frac{a^2 \norm{1_m}^2}{m} \norm{\bar{K}^{-1}}_{\infty} \leq \frac{C}{\gamma}
    }
    and
    \eq{
        c_0 \geq \frac{a^2 \norm{1_m}^2}{m} \lambda_{\min}(\bar{K}^{-1}) \geq \frac{C}{(\gamma + \norm{\bar{F}}_{\infty}^2/n_0)} \geq c n_0^{-c}
    }
    using the operator norm bound on the event $\mathcal{E}_F$. 
    
    For $c_1 = 1+\bfu^\top \bar{K}^{-1} \bfv$ we have
    \eq{
        \abs{c_1} = \absa{1 + a 1_m^\top \bar{K}^{-1} \frac{\bar{F}^\top}{\sqrt{n_1}} \frac{1_{n_1}}{\sqrt{n_1}}} \leq 1 + a \norm{1_m}_2 \frac{\norm{1_{n_1}}_2}{\sqrt{n_1}} \norm{\bar{K}^{-1} \frac{\bar{F}}{\sqrt{n_1}}} \leq C n_0^{1/2},
    }
    where the bound $\norm{\bar{K}^{-1} \frac{\bar{F}}{\sqrt{n_1}}} \leq C$ follows by considering the SVD of $\bar{F}$.
    
    For $c_2=1-\bfv^\top \bar{K}^{-1} \bfv$ we obviously have that $c_2 \leq 1$ since $\bfv^\top \bar{K}^{-1} \bfv \geq 0$. Additionally, we have that
    \al{
        c_2 &= 1-\frac{1}{n_1} 1_{n_1}^\top \frac{\bar{F}}{\sqrt{n_1}}\bar{K}^{-1} \frac{\bar{F}^\top}{\sqrt{n_1}} 1_{n_1} \\
        &= \frac{1_{n_1}^\top}{\sqrt{n_1}} \pa{I_{n_1} - \frac{\bar{F}}{\sqrt{n_1}}\bar{K}^{-1} \frac{\bar{F}^\top}{\sqrt{n_1}}} \frac{1_{n_1}}{\sqrt{n_1}} \\
        &\geq 1-\norm{\frac{\bar{F}}{\sqrt{n_1}}\bar{K}^{-1} \frac{\bar{F}^\top}{\sqrt{n_1}}}_{\infty}\\ 
        &= 1-\frac{\norm{\bar{F}/\sqrt{n_1}}_{\infty}^2}{\gamma + \norm{\bar{F}/\sqrt{n_1}}_{\infty}^2} \geq 1-\frac{c n_0^c}{\gamma+c n_0^c} \\
        &\geq \frac{C\gamma}{c n_0^c} \\
        &\geq c n_0^{-c}
    }
    as desired once again using the conditioning on the event $\mathcal{E}_F$.
    
    Finally,  
    \eq{
        \bfv^\top \bar{K}^{-1} \bfv  = \frac{1}{n_1} 1_{n_1}^\top \frac{\bar{F}}{\sqrt{n_1}}\bar{K}^{-1} \frac{\bar{F}^\top}{\sqrt{n_1}} 1_{n_1} \leq \frac{\norm{1_{n_1}}_2^2}{n_1} \norma{ \frac{\bar{F}}{\sqrt{n_1}}\bar{K}^{-1} \frac{\bar{F}^\top}{\sqrt{n_1}}}_{\infty} \leq C
    }
    since $\frac{\bar{F}}{\sqrt{n_1}}\bar{K}^{-1} \frac{\bar{F}^\top}{\sqrt{n_1}} \leq C$ follows once again by considering the SVD of $\bar{F}$. 
    
    \paragraph{Controlling $Y\bar{K}^{-1}\bfv^\top$:} Here, we introduce conditioning on the event $\mathcal{E}_\text{data}$ and note its independence from $\beta$ and $\e$. Recalling that $Y = \beta^\top X/\sqrt{n_0} + \e$, we note that $Y$ has expectation zero since $\mE_{\beta, \e}[Y\bar{K}^{-1}\bfv^\top | \mathcal{E}_\text{data}]= 0$. Similarly the conditional variance can be bounded as
    \al{
        \V_{\beta, \e} [Y\bar{K}^{-1}\bfv^\top|\mathcal{E}_\text{data}] &= \bfv^\top \bar{K}^{-1}(\Upsilon +\sigma_{\e}^2 I_m) \bar{K}^{-1}\bfv\\
        & \leq \norma{ \Upsilon + \sigma_{\e}^2 I_{m} }_{\infty} \norma{ \bar{K}^{-1} \frac{\bar{F}^\top}{\sqrt{n_1}} }_{\infty}^2 \norma{ \frac{1_{n_1}}{\sqrt{n_1}}  }_{2}^2 \\
        &\leq C.
    }
    The bound $\| \bar{K}^{-1} \frac{\bar{F}^\top}{\sqrt{n_1}} \|_{\infty} \leq C$ follows by considering the SVD of $F$ as in the proof of \cref{lem_aypical_event}, while the bound $\| \Upsilon \| \leq C$ holds on $\mathcal{E}_\text{data}$. The law of total variance then shows ${\V[Y\bar{K}^{-1}\bfv^\top|\mathcal{E}_\text{data}]\leq C}$ also. Hence an application of Chebyshev's inequality shows that
    \begin{align}
        & \P[\abs{Y\bar{K}^{-1}\bfv^\top} \geq C n_0^c |\mathcal{E}_\text{data}] \leq Cn_0^{-c}
    \end{align}
    for some $C>0$ and any $c>0$.

    \paragraph{Controlling $Y\bar{K}^{-1}\bfu^\top$:} Again, we introduce conditioning on the event $\mathcal{E}_\text{data}$ and note its independence from $\beta$ and $\e$. A nearly identical argument to the previous one suffices, so we only outline the argument. Again $\E[Y\bar{K}^{-1}\bfv^\top|\mathcal{E}_\text{data}]= 0$ and the conditional variance can be bounded as
    \begin{align}
        \V_{\beta, \e} [Y\bar{K}^{-1}\bfu^\top|\mathcal{E}_\text{data}] &= \bfu^\top \bar{K}^{-1}(\Upsilon +\sigma_{\e}^2 I_m) \bar{K}^{-1}\bfu\\
        & \leq \norm{\bar{K}^{-1}}_{\infty}^2 \| \Upsilon + \sigma_{\e}^2 I_{m} \|_{\infty} \| \bfu \|_2^2 \\
        & \leq \frac{1}{\gamma^2} \cdot a^2 m \cdot \| \Upsilon + \sigma_{\e}^2 I_{m} \|_{\infty} \\
        &\leq C n_0.
    \end{align}
    Finally, Chebyshev's inequality suffices to show
    \begin{align}
      \P[\abs{Y\bar{K}^{-1}\bfu^\top} \geq C n_0^{c+1/2} |\mathcal{E}_{data}] \leq C n_0^{-c}
    \end{align}
    for some $C>0$ and any $c>0$.

    \paragraph{Controlling $\bar{K}_\bfx^\top \bar{K}^{-1}\bfu^\top$:} Here, we introduce conditioning on the event $\mathcal{E}_W$ and note its independence from $\bfx$. The overall argument then uses Chebyshev's inequality exploiting the randomness in $\bfx$. First,
    \begin{align}
        \absa{\E_{\bfx}[\bar{K}_{\bfx}^\top \bar{K}^{-1} \bfu^\top | \mathcal{E}_W] } &= \absa{\E_{\bfx}[\bar{f}^\top |\mathcal{E}_W] \frac{\bar{F}}{\sqrt{n_1}} \bar{K}^{-1} \bfu^\top/\sqrt{n_1} } \\
        &\leq \norma{\E_{\bfx}[\bar{f} |\mathcal{E}_W]}_2 \norma{\frac{\bar{F}}{\sqrt{n_1}} \bar{K}^{-1}}_{\infty} \norma{\bfu/\sqrt{n_1}}_2 \\
        &\leq C n_0^{c}.\label{eq_KKu_ex}
    \end{align}
    The bound $\norm{\frac{\bar{F}}{\sqrt{n_1}} \bar{K}^{-1}}_{\infty} \leq C$ follows by an SVD as before, while the first bound ${\norm{\E_{\bfx}[\bar{f}^\top |\mathcal{E}_W]}_2 \leq C n_0^c}$ follows from the fact $\E_\bfx \bar{\sigma}(W\bfx)_i \leq C n_0^{c-1/2}$ on $\mathcal{E}_W$ (see \cref{lem:Wlem_taylor}). Next we turn to the conditional variance. We see
    \al{
        \V_\bfx [\bar{K}_\bfx \bar{K}^{-1}\bfu|\mathcal{E}_W] &= \bfu^\top \bar{K}^{-1} \mathbb{E}_\bfx [\bar{K}_\bfx \bar{K}_\bfx^\top | \mathcal{E}_W] \bar{K}^{-1}\bfu\\
        &= \frac{1}{n_1^2}\bfu^\top \bar{K}^{-1} \bar{F}^\top\big(\frac{\rhos}{n_0} W \Sigmate W^\top + (\eta_* - \zeta_*) I_{n_1} + \Delta \big)\bar{F} \bar{K}^{-1}\bfu\\
        & \le  \norma{\frac{\bfu}{\sqrt{n_1}}}_2^2 \norma{\bar{K}^{-1} \frac{\bar{F}^\top}{\sqrt{n_1}}}_2^2 \cdot \pa{ \norma{\rhos\tilde{\Upsilon}}_{\infty} \!+ \norma{(\eta_* - \zeta_*) I_{n_1}}_{\infty} \!+ \norma{\Delta}_{\infty}}  \\
        & \leq C \cdot (C+C+n_0^{c}) \\
        &\leq C n_0^c.\label{eq_KKu_var}
    }
    Once again the bound $\norm{\bar{K}^{-1} \frac{\bar{F}^\top}{\sqrt{n_1}}}_2^2$ follows by considering the SVD, while the nontrivial operator norm bounds follow from \cref{lem:Wlem_taylor}. Using the law of total variance, we see
    \eq{
        \V[ \bar{K}_\bfx \bar{K}^{-1}\bfu |\mathcal{E}_W]= \E[ \V_\bfx[ \bar{K}_\bfx \bar{K}^{-1}\bfu |\mathcal{E}_W]|\mathcal{E}_W]+ \V[ \E_\bfx[ \bar{K}_\bfx \bar{K}^{-1}\bfu|\mathcal{E}_W]|\mathcal{E}_W] \leq Cn_0^c
    }
    by Equations \eqref{eq_KKu_ex} and \eqref{eq_KKu_var}. Applying Chebyshev's inequality yields
    \eq{
         \P[\abs{\bar{K}_\bfx^\top \bar{K}^{-1}\bfu-\mE_{\bfx}[\bar{K}_{\bfx}^\top \bar{K}^{-1} \bfu]} \geq Cn_0^{c} | \mathcal{E}_W]  \leq C n_0^{-c},
    }
    and so
    \eq{
        \P[\abs{\bar{K}_\bfx^\top \bar{K}^{-1}\bfu} \geq Cn_0^{c} | \mathcal{E}_W]  \leq C n_0^{-c} 
    }
    by the triangle inequality since $\abs{\mE_{\bfx}[\bar{K}_{\bfx}^\top \bar{K}^{-1} \bfu]} \leq C n_0^c$. 
    
    \paragraph{Controlling $\bar{K}_\bfx^\top \bar{K}^{-1}\bfv^\top$:} Again, we introduce conditioning on the event $\mathcal{E}_W$ and note its independence from $\bfx$. A nearly identical argument to the previous one shows the result for this term. Hence we only outline the argument. First, we see
    \begin{align}
        \abs{\mE_{\bfx}[\bar{K}_{\bfx}^\top \bar{K}^{-1} \bfv^\top | \mathcal{E}_W]} &= \absa{\E_{\bfx}[\bar{f}^\top |\mathcal{E}_W] \frac{\bar{F}}{\sqrt{n_1}} \bar{K}^{-1} \frac{\bar{F}^\top}{\sqrt{n_1}} \frac{1_{n_1}}{n_1}}\\ 
        &\leq \norma{\E_{\bfx}[\bar{f}^\top|\mathcal{E}_W]}_2 \norma{\frac{\bar{F}}{\sqrt{n_1}} \bar{K}^{-1} \frac{\bar{F}^\top}{\sqrt{n_1}}}_{\infty} \norma{\frac{1_{n_1}}{n_1}}_2 \\
        &\leq  C n_0^{c-1/2}.
    \end{align}
    The bound $\norm{\frac{\bar{F}}{\sqrt{n_1}} \bar{K}^{-1} \frac{\bar{F}^\top}{\sqrt{n_1}}}_{\infty} \leq 1$ follows by an SVD as before, while the first bound again follows from \cref{lem:Wlem_taylor}. Next we compute the conditional variance and bound as in Equation \eqref{eq_KKu_var} to find
    \al{
        \V_\bfx[ \bar{K}_\bfx \bar{K}^{-1}\bfv^\top |\mathcal{E}_W]&= \bfv^\top \bar{K}^{-1} \mathbb{E}_\bfx [\bar{K}_\bfx \bar{K}_\bfx^\top | \mathcal{E}_W] \bar{K}^{-1}\bfv\\
        &= \frac{1}{n_1^4}1_{n_1}^\top \bar{F} \bar{K}^{-1} \bar{F}^\top\pa{\frac{\rhos}{n_0} W \Sigmate W^\top + (\eta_* - \zeta_*) I_{n_1} + \Delta }\bar{F} \bar{K}^{-1} \bar{F}^\top 1_{n_1}\\
        & \leq \frac{1}{n_1} \norma{\frac{1_{n_1}}{\sqrt{n_1}}}_2^2 \!\cdot \norma{\frac{\bar{F}}{\sqrt{n_1}} \bar{K}^{-1} \frac{\bar{F}^\top}{\sqrt{n_1}}}_2^2 \!\cdot \pa{ \norma{{\rhos} \tilde{\Upsilon}}_{\infty} \!+\! \norma{(\eta_* \!-\! \zeta_*) I_{n_1}}_{\infty} \!+\! \norma{\Delta}_{\infty}} \\
        & \leq \frac{C \cdot (C+C+n_0^{c})}{n_1}\\
        &\leq C n_0^{c-1}.
    }
Once again the bound $\norm{\frac{\bar{F}}{\sqrt{n_1}} \bar{K}^{-1} \frac{\bar{F}^\top}{\sqrt{n_1}}}_2 \leq 1$ follows by considering the SVD while the nontrivial operator norm bounds follow on $\mathcal{E}_W$. Now applying Chebyshev's inequality as before shows
\begin{align}
     \P[\abs{\bar{K}_\bfx^\top \bar{K}^{-1}\bfv^\top} \geq Cn_0^{c-1/2} |\mathcal{E}_W]\leq Cn_0^{-c}.
\end{align}

\paragraph{Controlling $y(\bfx)$:} Finally we show $y(\bfx) \leq C n_0^{c}$ with probability at least $C n_0^{-c}$. Note that $\E[y(\bfx)]=0$ and a short computation shows that $\V[y(\bfx)] = \ntr(\Sigmate)+\sigma_{\e}^2 \leq C$. So a direct application of Chebyshev's inequality shows that
\begin{align}
    \P[\abs{y(\bfx)} \geq Cn_0^c]  \leq C n_0^{-c}.
\end{align}

\end{proof}

\subsubsection{Completing the argument for the typical behavior}\label{sec_typical}

\begin{lemma}\label{lem_typical_event}
     For some $n_0$-independent constants $c>0$ and $C>0$, we have
     \eq{\label{eq_atypical_1}
        \absa{\E \qa{\mathbf{1}_\mathcal{E}\Delta y(\bfx)}} \leq Cn_0^{-c}
     }
     and
     \eq{\label{eq_atypical_2}
        \absa{\E \qa{\mathbf{1}_\mathcal{E} \Delta(K_\bfx^\top K^{-1}Y^\top+\bar{K}_\bfx^\top \bar{K}^{-1}Y^\top) }} \leq Cn_0^{-c}
     }
\end{lemma}

\begin{proof}
    Obviously, $\absa{\E \qa{\mathbf{1}_\mathcal{E}\Delta y(\bfx)}}$ is bounded by $\absa{\E \qa{\Delta y(\bfx)} |\mathcal{E} }$, so we can remove the indicator $\mathbf{1}_\mathcal{E}$ from the expectations in Equations \eqref{eq_atypical_1} and \eqref{eq_atypical_2} by conditioning on $\mathcal{E}$.

    We want to show that conditional on $\mathcal{E}$ and for some $n_0$-independent constants $c>0$ and $C>0$, the bounds $\absa{\Delta}\leq Cn_0^{-2c}$, $\absa{y(\bfx)} \leq Cn_0^{c}$, and $\absa{(K_\bfx^\top K^{-1}Y^\top+\bar{K}_\bfx^\top \bar{K}^{-1}Y^\top)} \leq Cn_0^{c}$ hold. 
    
    Recall, ${a\deq \E_{Z} \sigma\pa{ \sqrt{\ntr(\Sigmatr)}Z }}$, ${a_*\deq \E_{Z} \sigma\pa{ \sqrt{\ntr(\Sigmate)}Z }}$, ${\bfv \deq 1/n_1 \bar{F}^\top 1_{n_1}}$, ${\bfu \deq a 1_m}$, ${U \deq [\bfu, \bfv]^\top}$, and ${C\deq \left(\begin{smallmatrix} 1 & 1\\1 & 0\end{smallmatrix}\right)}$.

    First, consider the term $\Delta = YT_1+YT_2$.
    To show that $\absa{\Delta}\leq Cn_0^{-2c}$ on $\mathcal{E}$, we write out $YT_1$ and $YT_2$ more explicitly:
    \al{\label{eq_YT1_explicit}
        YT_1 &= Y\bar{K}^{-1}U^\top \left(I_2+C U\bar{K}^{-1}U^\top\right)^{-1}P \\
        &= \frac{1}{c_1^2+m c_0c_2} \begin{pmatrix} Y\bar{K}^{-1} \bfu \\ Y\bar{K}^{-1} \bfv \end{pmatrix}^\top \begin{pmatrix}c_1\delta + c_2a_* \\ -m c_0 \delta + c_1a_*  \end{pmatrix} \\
        &=\frac{1}{c_1^2+mc_0 c_2}\left(Y\bar{K}^{-1} \bfu \cdot (c_1\delta + c_2a_*) +  Y\bar{K}^{-1} \bfv \cdot(-m c_0 \delta + c_1a_*) \right)\\
    }
    and
    \al{\label{eq_YT2_explicit}
        YT_2 &= -Y\bar{K}^{-1} U^\top \left(I_2+C U\bar{K}^{-1}U^\top\right)^{-1}C U \bar{K}^{-1}\bar{K}_\bfx \\
        &= \frac{1}{c_1^2+m c_0c_2} \begin{pmatrix} Y\bar{K}^{-1} \bfu \\ Y\bar{K}^{-1} \bfv  \end{pmatrix}^\top \begin{pmatrix}c_2& c_1\\ c_1  & -m c_0\end{pmatrix} \begin{pmatrix} \bfu^\top \bar{K}^{-1} \bar{K}_{\x} \\ \bfv^\top \bar{K}^{-1} \bar{K}_{\x} \end{pmatrix} \\
        &\leq \frac{1}{c_1^2+m c_0c_2} \Big( Y\bar{K}^{-1} \bfu \cdot ( c_2 \bfu^\top \bar{K}^{-1} \bar{K}_{\x} + c_1  \bfv^\top \bar{K}^{-1} \bar{K}_{\x}) \\
        &\qquad+ Y\bar{K}^{-1} \bfv \cdot(c_1 \bfu^\top \bar{K}^{-1} \bar{K}_{\x} - mc_0 \bfv^\top \bar{K}^{-1} \bar{K}_{\x} ) \Big).
    }
    Now using Equation \eqref{eq_YT1_explicit} and the definition of $\mathcal{E}$, we see that on $\mathcal{E}$
    \al{
        \absa{YT_1} &\leq \absa{\frac{1}{c_1^2+mc_0 c_2}} \pa{ \absa{Y\bar{K}^{-1} \bfu}  (\absa{c_1}\absa{\delta} + \absa{c_2} \absa{a_*}) +  \absa{Y\bar{K}^{-1} \bfv} (m\absa{c_0} \absa{\delta} + \absa{c_1}\absa{a_*}) } \\
        &\leq C n_0^{2c-1}\pa{ Cn_0^{c+1/2} ( Cn_0^{c} Cn_0^{c-1/2} + Cn_0^{c}) + Cn_0^{c}( Cn_0 Cn_0^{c} Cn_0^{c-1/2} + Cn_0^{c} ) } \\
        &\leq Cn_0^{5c-1/2}.
    }
    Since the $c$s used in the bounds above can be arbitrarily small, we may replace write $\absa{YT_1} \leq Cn_0^{c-1/2}$ for arbitrarily small $c>0$. Finally, a similar argument for Equation \eqref{eq_YT2_explicit} yields the bound
    \begin{align}
        \absa{YT_2} &\leq  \absa{\frac{1}{c_1^2+m c_0c_2}} \Big( \abs{Y\bar{K}^{-1} \bfu} \cdot ( \abs{c_2} \abs{\bfu^\top \bar{K}^{-1} \bar{K}_{\x}} + \abs{c_1} \abs{\bfv^\top \bar{K}^{-1} \bar{K}_{\x}})  \\
        & \qquad+\abs{Y\bar{K}^{-1} \bfv} \cdot(\abs{c_1} \abs{\bfu^\top \bar{K}^{-1} \bar{K}_{\x}} - m \abs{c_0} \abs{\bfv^\top \bar{K}^{-1} \bar{K}_{\x}} ) \Big) \\
        & \leq C n_0^{2c-1} \cdot\Big(Cn_0^{c+1/2} \cdot(C n_0^c Cn_0^c+Cn_0^{c+1/2} Cn_0^{c-1/2})  \\
        &\qquad+Cn_0^c \cdot( Cn_0^{c+1/2} Cn_0^c+Cn_0 Cn_0^c Cn_0^{c-1/2})\Big) \\
        & \leq Cn_0^{2c-1} \cdot( C n_0^{3c+1/2} + Cn_0^{3c+1/2}) \\
        &\leq C n_0^{5c-1/2}.
    \end{align}
    and as before since $c$ can be chosen arbitrarily small  we can write $\absa{YT_2} \leq C n_0^{c-1/2}$.
\end{proof}

\subsection{Gaussian equivalents}
\label{sec_gaussian}
Before describing how (in the high-dimensional limit) Gaussian equivalents can be used to compute the test error we make the following observation: the limiting test error is invariant to the centering of the activation function across the training and test distributions. That is, by \cref{sec:mean_zero} the limiting test error is unchanged by replacing,
\eq{
    \bar{F}_{ij} \deq F_{ij} - \E_{z \sim \mathcal{N}(0, s)} \sigma\pa{z}  \quad\text{and}\quad \bar{f}_{i} \deq f_{i} - \E_{z \sim \mathcal{N}(0, s_*)} \sigma\pa{ z }
}
for all $i$ and $j$.

With this simplification in mind, the proof of \cref{thm:main_b_v} relies on the concept of Gaussian equivalents and the linearization analysis developed in \citet{adlam2020neural, adlam2020understanding, adlam2019random}, which we briefly review here, though we refer the reader to these works for a more detailed description. To proceed, we first define the moment-matched Gaussian linearizations,
\begin{align}
\bar{F} \to & \sqrt{\frac{\bar{\rho}}{n_0}} W X + \sqrt{\bar{\eta} - \bar{\zeta}} \Theta \label{eq_barFlin}\\
\bar{f} \to &\sqrt{\frac{\bar{\rho}_*}{n_0}} W \bfx + \sqrt{\bar{\eta}_* - \bar{\zeta}_*} \theta \label{eq_barflin_mean_zero}
\end{align}
where $\bar{f} \deq \bar{\fs}(W \bfx/\sqrt{n_0})$ is the random feature representation of the test point $\bfx$ and $y \deq \sqrt{\frac{1}{n_0}} \beta^\top \x$ is its corresponding label, which we assume has no additive noise (including noise on the test labels merely shifts by the total error by an irreducible additive constant). The centered parameters\footnote{With some abuse of notation, we use $\bar{\sigma}(\cdot)$ to denote the centering with respect to the training scale $s$ and $s_*$ for unasterisked and asterisked quantities respectively.} are defined as, $\bar{\zeta}$ and $\bar{\zeta}_*$ as $\bar{\zeta} = s \bar{\rho}$ and $\bar{\zeta}_* = s_* \bar{\rho}_*$ with,
\begin{alignat}{4}
\bar{\eta} &\deq \mathbb{E}_{z\sim \cN(0,s)}[\bar{\sigma}(z)^2]\,,\;\; & \bar{\rho} &\deq (\frac{1}{s}\mE_{z\sim \cN(0,s)} [z \bar{\sigma}(z)])^2\,,\\
\bar{\eta}_* &\deq \mathbb{E}_{z\sim \cN(0,s_*)}[\bar{\sigma}(z)^2]\,,\;\; & \bar{\rho}_* &\deq (\frac{1}{s_*}\mE_{z\sim \cN(0,s_*)} [z \bar{\sigma}(z)])^2\,.
\end{alignat}
The new objects $\Theta$, $\theta$ are matrices of the appropriate shapes with i.i.d. standard Gaussian entries independent of the other random variables under consideration. The constants $\bar{\eta}$, $\bar{\zeta}$, $\bar{\eta}_*$, $\bar{\zeta}_*$ are chosen so that the mixed moments up to second order are the same for the original and linearized matrices. The Gaussian equivalents defined are essentially constructed via a Taylor expansion of the nonlinearity $\bar{\sigma}(\cdot)$. The explicit calculations defining these expressions can be found via the Gaussian moment matching technique in \citet{adlam2019random}. 

Having appropriately linearized the random feature matrices, we can map back to the definition of the original activation functions
by recalling that the variables $\zeta$ and $\zetas$ are defined in \cref{sec:main_thms} as $\zeta = s \rho$ and $\zetas = s_* \rhos$, and
\begin{alignat}{4}
\eta &\deq \bar{\eta} = \E_{z\sim \cN(0,s)}[(\sigma(z)-\mE_{z \sim \cN(0,s)}[\sigma(z)])^2] =  \mathbb{V}_{z\sim \cN(0,s)}[\sigma(z)]\\
\rho &\deq \bar{\rho} = (\frac{1}{s} \mE_{z\sim \cN(0,s)} [z (\sigma(z)-\mE_{z \sim \cN(0,s)}[\sigma(z)])])^2 = (\frac{1}{s} \mE_{z\sim \cN(0,s)} [z (\sigma(z)])^2 \\
\etas &\deq \bar{\etas} = \E_{z\sim \cN(0,s_*)}[(\sigma(z)-\mE_{z \sim \cN(0,s_*)}[\sigma(z)])^2] = \mathbb{V}_{z\sim \cN(0,s_*)}[\sigma(z)] \\
\rhos &\deq \bar{\rhos} = (\frac{1}{s_*} \mE_{z\sim \cN(0,s_*)} [z (\sigma(z)-\mE_{z \sim \cN(0,s)}[\sigma(z)])])^2 = (\frac{1}{s_*} \mE_{z\sim \cN(0,s_*)} [z (\sigma(z)])^2\,.
\end{alignat}
by using the definition of the centering. Effectively, this is equivalent to using the linearizations,
\begin{align}
F \to \; \bar{F} & \to \sqrt{\frac{\rho}{n_0}} W X + \sqrt{\eta - \zeta} \Theta \label{eq_Flin}\\
f \to \; \bar{f} & \to \sqrt{\frac{\rhos}{n_0}} W \bfx + \sqrt{\eta_* - \zeta_*} \theta \label{eq_flin}
\end{align}
directly to compute the test error.

Simply by definition, the teacher function is exactly linear,
\begin{align}
    Y = & \sqrt{\frac{1}{n_0}} \beta^\top X + \e \label{eq_Ylin} \\
    y = & \sqrt{\frac{1}{n_0}} \beta^\top \bfx \label{eq_ylin}\,.
\end{align}
We emphasize that the restriction to linear teacher functions is simply a matter of convenience and simplicity---nonlinear teacher neural networks can be studied by means of an analogous linearization process, as discussed in \citet{adlam2020neural}, which merely requires introducing additional moment-matching constants and additional i.i.d. standard Gaussian terms $\theta$, and $\theta_y$.

In the high-dimensional limit the bulk statistics we compute defining the error, bias, and variance are invariant to the above replacements by linearized Gaussian equivalents. We remark that further intuition for these linearized information-plus-noise replacements can be gathered from the universality results of \citet{merlevede2013limiting, banna2020clt}. To summarize briefly, the final expressions we compute are tracial (nonlinear) functions of several random matrices. A large body of universality results stemming from \citet{banna2020clt, erdos2019matrix} show that in the limit it is sufficient to match them to an equivalent tracial functional of a (linearized) rational function of random matrices, whose moments match their nonlinear counterparts. 

More specifically, two basic trace objects arise in the calculations, which take the form
\begin{equation}
    \tr(AB) \label{eq_con_1}\,,\quad\text{and}\quad \tr\pa{A\frac{1}{B-zI}}\,.
\end{equation}
For the first case, $\tr(AB)$, if the random matrices $A$ and $B$ are independent, its asymptotics can be understood via concentration of measure arguments. Conditionally on $A$, it is sufficient to replace $B$ with an equivalent matrix designed to match low-order moments, at the expense of an error that vanishes asymptotically. The second case, $\tr\pa{A\frac{1}{B-zI}}$, is a bit more involved. Our arguments proceed by 1) using the linearized matrices in \cref{eq_Flin,,eq_flin} to express the trace object as a rational function of i.i.d. Gaussian matrices, and then 2) using the linear pencil method\footnote{This is essentially an iterative application of the Schur complement formula.} to express it as the trace of a large (inverted) block matrix \citep{helton2018applications,mingo2017free}. The reason the linearization over $B$ preserves the asymptotic statistics even for general $A$ with correlated entries stems from the matrix Dyson equation~\citep{erdos2019matrix}. 

\subsection{Decomposition of the test loss}
\label{sec:sec_exact_asymptotics}
Recall the expression for the test loss in \cref{eq:Etest},
\eq{
\Etest =  E_1 + E_2 + E_3
}
with
\begin{align}
E_1 &= \mathbb{E}_{(\bfx,\beta,\e)}y(\bfx)^2 =\mathbb{E}_{(\bfx,\e)}\tr(y(\bfx)y(\bfx)^\top)\\
E_2 &= -2\mathbb{E}_{(\bfx,\beta,\e)}(K_\bfx^\top K^{-1}Y) y(\bfx) = -2\mathbb{E}_{(\bfx,\e)}\tr(K_\bfx^\top K^{-1}Y^\top y(\bfx))\\
E_3 &= \mathbb{E}_{(\bfx,\beta,\e)}(K_\bfx^\top K^{-1}Y^\top )^2= \mathbb{E}_{(\bfx,\e)}\tr(K_\bfx^\top K^{-1}Y^\top Y K^{-1} K_\bfx)\,,
\end{align}
where the kernels $K = K(X,X)$ and $K_\bfx = K(X,\bfx)$ are given by,
\eq{\label{eq_K_and_Kx}
    K = \frac{F^\top \!F}{n_1} + \gamma I_m\qquad\text{and}\qquad K_\bfx = \frac{1}{n_1} F^\top f\,.
}
Using the cyclicity and linearity of the trace, the expectation over $\bfx$ requires the computation of
\eq{
\mathbb{E}_{\bfx}K_\bfx K_\bfx^\top \,,\qquad \mathbb{E}_{\bfx} y(\bfx)K_\bfx^\top\,,\qquad \mathbb{E}_{\bfx} y(\bfx) y(\bfx)^\top\,.
}
As described in \cref{sec_gaussian}, without loss of generality we consider the case of a linear teacher, and ~\cref{eq_ylin,,eq_flin} read
\eq{\label{eq_sub_y_f}
y = \frac{
1}{\sqrt{n_0}} \beta^\top \bfx\,\qquad\text{and}\qquad f \to f^{\text{lin}} = \frac{\sqrt{\rhos}}{\sqrt{n_0}} W \bfx + \sqrt{\eta_* - \zeta_*} \theta\, .
}
The expectations over $\bfx$ are now trivial and we readily find,
\begin{align}
   \mathbb{E}_\bfx K_\bfx K_\bfx^\top & = \frac{1}{n_1^2}F^\top\pa{\frac{\rhos}{n_0} W \Sigmate W^\top + (\eta_* - \zeta_*) I_{n_1}}F\\
\mathbb{E}_{\bfx} y(\bfx)K_\bfx^\top & = \frac{\sqrt{\rhos}}{n_0n_1  } \beta^\top \Sigmate W^\top F\\
\mathbb{E}_{\bfx} y(\bfx)y(\bfx)^\top &= \frac{1}{n_0} \beta \Sigmate \beta^\top
\end{align}
One may interpret the substitution in \cref{eq_sub_y_f} as a tool to calculate the expectations above to leading order, which generates terms like \cref{eq_con_1}. Next, we recall the definition, $Y = \beta^\top X / \sqrt{n_0} + \e $, and, as above, we consider the leading order behavior with respect to the random variables $\beta$ to find
\begin{align}
    \E_{\beta,\e}\qa{Y^\top Y} &= \frac{1}{n_0}X^\top X + \sigma_{\e}^2 I_m\\
    \E_{\beta,\e}\qa{Y^\top \mathbb{E}_{\bfx} y(\bfx)K_\bfx^\top} &= \frac{\sqrt{\rhos}}{n_0^{3/2}n_1} X^\top \Sigmate  W^\top F\,.
\end{align}

Putting these pieces together, we have
\begin{align}
    E_1 & = \frac{\tr(\Sigmate)}{n_0} \label{eqn:E1}\\
    E_2 &= E_{21}\label{eqn:E2}\\
    E_3 &= E_{31} + E_{32} \label{eqn:E3}\,,
\end{align}
where,
\begin{align}
E_{21} & = -2 \frac{\sqrt{\rhos}}{n_0^{3/2}n_1} \E \tr \left(X^\top \Sigmate  W^\top FK^{-1}\right)\label{eq_E21}\\
E_{31} &= \sigma_{\e}^2 \E\tr\left(K^{-1} \Sigma_3 K^{-1}\right)\label{eq_E31}\\
E_{32} &= \frac{1}{n_0} \E \tr\left(K^{-1} \Sigma_3 K^{-1} X^\top X\right)\label{eq_E32}
\end{align}
and,
\eq{
\Sigma_3 = \frac{\rhos}{n_0 n_1^2}F^\top W \Sigmate  W^\top F + \frac{\eta_*-\zeta_*}{n_1^2} F^\top F\,.
}

\subsection{Decomposition of the bias and variance}
Note that it is sufficient to calculate the bias term since, given the total test loss, the total variance can be obtained as $\var{\Sigmate} = E_{\Sigmate}-\bias{\Sigmate}$. Following the total multivariate bias-variance decomposition of \citet{adlam2020understanding}, for each random variable in question we introduce an i.i.d. copy of it denoted by either the subscript $1$ or $2$. We can then write,
\begin{align}
\bias{\Sigmate} &= \mathbb{E}_{(\bfx,y)} (y - \mathbb{E}_{(W, X,\e)}\hat{y}(\bfx; W,X,\e))^2\\
&= \mathbb{E}_{(\bfx,y)}\mathbb{E}_{(W_1,X_1,\e_1)}\mathbb{E}_{(W_2,X_2,\e_2)}(y - \hat{y}(\bfx; W_1,X_1,\e_1))(y - \hat{y}(\bfx; W_2,X_2,\e_2))\\
&= \frac{\tr(\Sigmate)}{n_0} + E_{21} + H_{000}\,,
\end{align}
where an expression for $E_{21}$ was given previously and $H_{000}$ satisfies

\begin{align}
H_{000} &= \mathbb{E}\hat{y}(\bfx; W_1,X_1,\e_1)\hat{y}(\bfx; W_2,X_2,\e_2)\,,
\end{align}
where the expectations are over $\bfx, W_1, X_1, \e_1, W_2,X_2$, and $\e_2$. Recalling the definition of $\hat{y}$,
\eq{
\hat{y}(\bfx; W, X, \e) \deq Y(X,\e) K(X,X;W)^{-1}K(X,\bfx;W)\,
}
and the techniques described in the previous section, it is straightforward to analyze the above term. First note we can write,
\begin{align}
\mathbb{E}_\bfx K(X_1,\bfx;W_1)K(\bfx, X_2;W_2) & =  \frac{\rhos}{n_0 n_1^2}F_{11}^\top W_1 \Sigmate W_2^\top F_{22}
\end{align}
since the $f$ linearizations use different sources of auxiliary randomness. Here we have defined $F_{11} \equiv F(W_1, X_1)$ and $F_{22} \equiv F(W_2,X_2)$. Now we proceed to calculate $H_{000}$ as
\begin{align}
H_{000} &= \mathbb{E}\hat{y}(\bfx; W_1,X_1,\e)\hat{y}(\bfx; W_2,X_2,\e_2)\\
&= \mathbb{E} K(\bfx,X_2;W_2)K(X_2,X_2;W_2)^{-1} Y(X_2,\e_2)^\top Y(X_1,\e_1)K(X_1,X_1;W_1)^{-1}K(X_1,\bfx;W)\\
&= \mathbb{E} \tr \big(K(X_2,X_2;W_2)^{-1} X_2^\top X_1 K(X_1,X_1;W_1)^{-1}K(X_1,\bfx;W)K(\bfx,X_2;W_2)\big)\\
&= \frac{\rhos}{n_0^2 n_1^2} \E \tr \left(K_{22}^{-1} X_2^\top X_1 K_{11}^{-1} F_{11}^\top W_1 \Sigmate W_2^\top F_{22} \right) \\
&\equiv E_4\,,
\end{align}
where in the second-to-last line we have defined $K_{11} \equiv K(X_1,X_1;W_1)$ and $K_{22} \equiv K(X_2,X_2;W_2)$.
\subsection{Summary of linearized trace terms}
\label{app:summary_lin}
We now summarize the requisite terms needed to compute the total test error, bias, and variance after using cyclicity of the trace to rearrange several of them. In the following, we slightly change notation in order to make explicit the dependence on the population covariance matrices $\Sigmatr$ and $\Sigmate$. To be specific, whereas above we assumed that the columns of $X_1$ and $X_2$ were drawn from multivariate Gaussians with covariance $\Sigma$, below we assume that they are drawn from multivariate Gausssians with identity covariance. This change is equivalent to replaceing $X_1 \to \Sigmatr^{1/2} X_1$ and $X_2 \to \Sigmatr^{1/2} X_2$ in the above expressions. We use this definition so that $X_1$, $X_2$, $W_1$, $W_2$, and $\Theta$ all have i.i.d. standard Gaussian entries. From the previous computations, we can now write the requisite terms as,
\begin{eqnarray}
    \Sigma_3 &=& \frac{\rhos}{n_0 n_1^2}F_{11}^\top W_1 \Sigmate  W_1^\top F_{11} + \frac{\eta_*-\zeta_*}{n_1^2} F_{11}^\top F_{11}\\
    E_{21} &=& -2 \frac{\sqrt{\rhos}}{n_0^{3/2}n_1} \tr \left(X_1^\top \Sigmatr ^{1/2} \Sigmate  W_1^\top F_{11}K_{11}^{-1}\right)\\
    E_{31} &=& \sigma_{\e}^2 \tr\left(K_{11}^{-1}\Sigma_3 K_{11}^{-1}\right)\\
    E_{32} &=& \frac{1}{n_0} \tr\left(K_{11}^{-1} \Sigma_3 K_{11}^{-1}X_1^\top \Sigmatr X_1\right)\\
    E_4 &=& \frac{\rhos}{n_0^2 n_1^2} \tr\left(F_{22}K_{22}^{-1}X_2^\top\Sigmatr X_1 K_{11}^{-1}F_{11}^\top W_1 \Sigmate W_2^\top\right)\\
    \err{\Sigmate} &=& \frac{\tr(\Sigmate)}{n_0} + E_{21} + E_{31} + E_{32} \label{eqn:err_expr}\\
    \bias{\Sigmate} &=& \frac{\tr(\Sigmate)}{n_0} + E_{21} + E_{4} \label{eqn:bias_expr} \\
    \var{\Sigmate} &=& \err{\Sigmate}-\bias{\Sigmate} \label{eqn:var_expr}
 \end{eqnarray}
 
 In the remainder of this section we use the machinery of operator-valued free probability \citep{mingo2017free} and a series of lengthy algebraic computations to compute the limiting tracial expressions in $E_{21}, E_{31}, E_{32}, E_{4}$, from which the total test error, bias, and variance can be reconstructed.

\subsection{Calculation of error terms}
\label{app:calculation}
To compute the test error, bias, and total variance, we need to evaluate the asymptotic trace objects appearing in the expressions for $E_{21}$, $E_{31}$, $E_{32}$, and $E_4$, defined in the previous section. As these expressions are essentially rational functions of the random matrices $X$, $W$, $\Theta$, $\Sigmatr$, and $\Sigmate$, these computations can be accomplished by constructing a linear pencil~\citep{far2006spectra} and using the theory of operator-valued free probability~\citep{mingo2017free}. These techniques and their application to problems of this type have been well-established elsewhere~\citep{adlam2019random,adlam2020neural,adlam2020understanding}, we only sketch the mathematical details, referring the reader to the literature for a more pedagogical overview. Instead, we focus on presenting the details of the requisite calculations.

Relative to the prior work of \citet{adlam2020neural,adlam2020understanding}, the main challenge in the current setting is generalizing the calculations to include an arbitrary training covariance matrix $\Sigmatr$. This generalization is facilitated by the general theory of operator-valued free probability, and in particular through the subordinated form of the operator-valued self-consistent equations that we first present in \cref{eqn:GK_SCE}. The form of this equation enables the simple computation of the operator-valued R-transform of the remaining random matrices, $W$, $X$, and $\Theta$, which are all i.i.d. Gaussian and can therefore be obtained simply by using the methods of \citet{far2006spectra}. The remaining complication amounts to performing the trace in \cref{eqn:GK_SCE}, which becomes an integral over the LJSD $\mu$ in the limit. While this might in general lead to a complicated coupling of many transcendental equations, it turns out that the trascendentality can be entirely factored into a single scalar fixed-point equation, whose solution we denote by $x$ (see \cref{eqn:x}), and the remaining equations are purely algebraic given $x$. To facilitate this particular simplification, it is necessary to first compute all of the entries in the operator-valued Stieltjes transform of the kernel matrix $K$, which we do in \cref{sec:Kinv}. Using these results, we compute the remaining error terms in the subsequent sections.

As a matter of notation, note that throughout this entire section whenever a matrix $X$, $X_1$, or $X_2$ appears it is composed of i.i.d. standard Gaussian entries as in \cref{app:summary_lin}. This differs from the notation of the main paper, but we follow this prescription to ease the already cumbersome presentation. This definition of $X$ allows us to explicitly extract and represent the training covariance $\Sigmatr$ in our calculations.
\allowdisplaybreaks
\subsubsection{$K^{-1}$}
\label{sec:Kinv}
Define the block matrix $\QK$ as,
\begin{equation}
 \QK = \left(
\begin{array}{cccccc}
 I_m & \frac{\sqrt{\eta -\zeta } \Theta ^\top}{\gamma  \sqrt{n_1}} & \frac{\sqrt{\rho } X^\top}{\gamma  \sqrt{n_0}} & 0 & 0 & 0 \\
 -\frac{\Theta  \sqrt{\eta -\zeta }}{\sqrt{n_1}} & I_{n_1} & 0 & 0 & -\frac{\sqrt{\rho } W}{\sqrt{n_1}} & 0 \\
 0 & 0 & I_{n_0} & -\Sigmatr^{1/2} & 0 & 0 \\
 0 & -\frac{W^\top}{\sqrt{n_1}} & 0 & I_{n_0} & 0 & 0 \\
 0 & 0 & 0 & 0 & I_{n_0} & -\Sigmatr^{1/2} \\
 -\frac{X}{\sqrt{n_0}} & 0 & 0 & 0 & 0 & I_{n_0} \\
\end{array}
\right)
\,.\end{equation}
Then block matrix inversion (i.e. repeated applications of the Schur complement formula) shows that,
\begin{eqnarray}
\GK_{1,1} &=& \gamma  \,\ntr(K^{-1})\\
\GK_{2,2} &=& \gamma  \,\ntr(\hK^{-1})\\
\GK_{3,3} &=& \GK_{6,6} = 1-\frac{\sqrt{\rho } \,\ntr\left(\Sigmatr^{1/2} W^\top F K^{-1} X^\top\right)}{\sqrt{n_0} n_1}\\
\GK_{4,3} &=& \GK_{6,5} = \,\ntr(\Sigmatr^{1/2})-\frac{\sqrt{\rho } \,\ntr\left(\Sigmatr W^\top F K^{-1} X^\top\right)}{\sqrt{n_0} n_1}\\
\GK_{5,3} &=& \GK_{6,4} = \frac{\gamma  \sqrt{\rho } \,\ntr\left(\Sigmatr^{1/2} W^\top \hK^{-1} W\right)}{n_1}\\
\GK_{6,3} &=& \frac{\gamma  \sqrt{\rho } \,\ntr\left(\Sigmatr W^\top \hK^{-1} W\right)}{n_1}\\
\GK_{3,4} &=& \GK_{5,6} = -\frac{\sqrt{\rho } \,\ntr\left(F K^{-1} X^\top W^\top\right)}{\sqrt{n_0} n_1 \psi }\\
\GK_{4,4} &=& \GK_{5,5} = 1-\frac{\sqrt{\rho } \,\ntr\left(\Sigmatr^{1/2} W^\top F K^{-1} X^\top\right)}{\sqrt{n_0} n_1}\\
\GK_{5,4} &=& \frac{\gamma  \sqrt{\rho } \,\ntr\left(\hK^{-1} W W^\top\right)}{n_1 \psi }\\
\GK_{3,5} &=& \GK_{4,6} = -\frac{\sqrt{\rho } \,\ntr\left(\Sigmatr^{1/2} X K^{-1} X^\top\right)}{n_0}\\
\GK_{4,5} &=& -\frac{\sqrt{\rho } \,\ntr\left(\Sigmatr X K^{-1} X^\top\right)}{n_0}\\
\GK_{3,6} &=& -\frac{\sqrt{\rho } \,\ntr\left(K^{-1} X^\top X\right)}{n_0 \phi }\,,
\end{eqnarray}
where $\GK_{i,j}$ denotes the normalized trace of the $(i,j)$-block of the inverse of $\QK$, and we have defined $\hK = \frac{1}{n_1}FF^\top + \gamma I_{n_1}$.

We aim to compute the limiting values of these trace terms as $n_0,n_1,m \to \infty$, as they will be related to the error terms of interest. Both here and in the sequel, to ease the already cumbersome presentation, we use $G$ to denote the limiting values as well as the non-limiting values and we refrain from explicitly denoting the limit operation itself, noting its existing can be inferred from context.

To proceed, recall that the asymptotic block-wise traces of the inverse of $\QK$ can be determined from its operator-valued Stieltjes transform~\citep{mingo2017free}. The simplest way to apply the results of \citet{far2006spectra,mingo2017free} is to augment $\QK$ to form the the self-adjoint matrix $\bQK$,
\begin{equation}
\label{eqn:Qbar}
    \bQK = \left(\begin{array}{cc} 0 & {[\QK]}^\top\\
\QK & 0 \end{array}\right)\,,\\
\end{equation}
and observe that we can write $\bQK$ as,
\begin{equation}
\begin{split}
\bQK &= \bar{Z} + \bQK_{W,X,\Theta} + \bQK_\Sigma\\
&= \left(\begin{array}{cc} 0 & Z^\top\\
Z & 0 \end{array}\right) + \left(\begin{array}{cc} 0 & {[\QK_{W,X,\Theta}]}^\top\\
\QK_{W,X,\Theta} & 0 \end{array}\right) + \left(\begin{array}{cc} 0 & {[\QK_\Sigma]}^\top\\
\QK_{\Sigma} & 0 \end{array}\right)\,,
\end{split}
\end{equation}
where $Z = I_{m+4 n_0+n_1}$, and,
\begin{eqnarray}
 \QK_{W,X,\Theta} &=&\left(
\begin{array}{cccccc}
 0 & \frac{\sqrt{\eta -\zeta } \Theta ^\top}{\gamma  \sqrt{n_1}} & \frac{\sqrt{\rho } X^\top}{\gamma  \sqrt{n_0}} & 0 & 0 & 0 \\
 -\frac{\Theta  \sqrt{\eta -\zeta }}{\sqrt{n_1}} & 0 & 0 & 0 & -\frac{\sqrt{\rho } W}{\sqrt{n_1}} & 0 \\
 0 & 0 & 0 & 0 & 0 & 0 \\
 0 & -\frac{W^\top}{\sqrt{n_1}} & 0 & 0 & 0 & 0 \\
 0 & 0 & 0 & 0 & 0 & 0 \\
 -\frac{X}{\sqrt{n_0}} & 0 & 0 & 0 & 0 & 0 \\
\end{array}
\right)
\\
\QK_{\Sigmatr} &=&\left(
\begin{array}{cccccc}
 0 & 0 & 0 & 0 & 0 & 0 \\
 0 & 0 & 0 & 0 & 0 & 0 \\
 0 & 0 & 0 & -\Sigmatr^{1/2} & 0 & 0 \\
 0 & 0 & 0 & 0 & 0 & 0 \\
 0 & 0 & 0 & 0 & 0 & -\Sigmatr^{1/2} \\
 0 & 0 & 0 & 0 & 0 & 0 \\
\end{array}
\right)
\,.\end{eqnarray}

Note that we have separated the i.i.d. Gaussian matrices $W,X,\Theta$ from the constant terms and from the $\Sigma$-dependent terms. Denote by $\bGK$ the block matrix
\begin{equation}
     \bGK = \left(\begin{array}{cc} 0 & {[\GK]}^\top\\
\GK & 0 \end{array}\right)\,,\\
\end{equation}
and by $\bGK_{\Sigma}$ the operator-valued Stieltjes transform of $\bQK_\Sigma$. Using the subordinated form~\citep{mingo2017free} of the self-consistent equation for $\bGK$ and the defining equation for $\bGK_{\Sigma}$, the operator-valued theory of free probability shows that the limiting asymptotic Stieltjes transforms satisfy,
\begin{equation}
\label{eqn:GK_SCE}
\begin{split}
    \bGK &= \bGK_{\Sigma}(\bar{Z} - \bRK_{W,X,\Theta}(\bGK))\\
    &= \id\otimes\ntr\left(\bar{Z} - \bRK_{W,X,\Theta}(\bGK)-\bQK_\Sigma\right)^{-1}\,,
\end{split}
\end{equation}
where $\bRK_{W,X,\Theta}(\bGK)$ is the operator-valued R-transform of $\bQK_{W,X,\Theta}$. Here the normalized trace $\ntr$ acts on the constituent blocks, and the identity operator $\id$ acts on the space of $12\times 12$ matrices. As described in \citet{adlam2020neural,adlam2020understanding}, since $\bQK_{W,X,\Theta}$ is a block matrix whose blocks are i.i.d. Gaussian matrices (and their transposes), an explicit expression for $\bRK_{W,X,\Theta}(\bGK)$ can be obtained through a covariance map, denoted by $\eta$~\citep{far2006spectra}. In particular, $\eta: M_d(\mathbb{C})\to M_d(\mathbb{C})$ is defined by,
\eq{\label{eqn_Deqn}
[\eta(D)]_{ij} = \sum_{kl} \sigma(i,k;l,j) \alpha_k D_{kl} \,,
}
where $\alpha_k$ is dimensionality of the $k$th block and $\sigma(i,k;l,k)$ denotes the covariance between the entries of the blocks $ij$ block of $\bQK_{W,X,\Theta}$ and entries of the $kl$ block of $\bQK_{W,X,\Theta}$. Here $d=12$ is the number of blocks. When the constituent blocks are i.i.d. Gaussian matrices and their transposes, as is the case here, then $\bRK_{W,X,\Theta} = \eta$~\citep{mingo2017free}, and therefore the entries of $\bRK_{W,X,\Theta}$ can be read off from \cref{eqn:Qbar}. To simplify the presentation, we only report the entries of $\bRK_{W,X,\Theta}(\GK)$ that are nonzero, given the specific sparsity pattern of $\GK$. The latter follows from \cref{eqn:GK_SCE} in the manner described in \citet{mingo2017free,far2006spectra}. Practically speaking, the sparsity pattern can be obtained by iterating an \cref{eqn:GK_SCE}, starting with an ansatz sparsity pattern determined by $\bar{Z}$, and stopping when the iteration converges to a fixed sparsity pattern. In this case (and all cases that follow in the subsequent sections), the number of necessary iterations is small and can be done explicitly. We omit the details and instead simply report the following results for the nonzero entries:
\begin{equation}
    \bRK_{W,X,\Theta}(\bGK) = \left(\begin{array}{cc} 0 & {\RK_{W,X,\Theta}}(\GK)^\top\\
\RK_{W,X,\Theta}(\GK) & 0 \end{array}\right)\,,
\end{equation}
where,
\begin{eqnarray}
{[\RK_{W,X,\Theta}(\GK)]}_{1,1} &=&\frac{\GK_{2,2} (\zeta -\eta )-\sqrt{\rho } \GK_{6,3}}{\gamma }\\
{[\RK_{W,X,\Theta}(\GK)]}_{2,2} &=&\frac{\psi  \GK_{1,1} (\zeta -\eta )}{\gamma  \phi }+\sqrt{\rho } \psi  \GK_{4,5}\\
{[\RK_{W,X,\Theta}(\GK)]}_{4,5} &=&\sqrt{\rho } \GK_{2,2}\\
{[\RK_{W,X,\Theta}(\GK)]}_{6,3} &=&-\frac{\sqrt{\rho } \GK_{1,1}}{\gamma  \phi }\,,
\end{eqnarray}and the remaining entries of $\RK_{W,X,\Theta}(\GK)$ are zero. Owing to the large degree of sparsity, the matrix inverse in \cref{eqn:GK_SCE} can be performed explicitly and yields relatively simple expressions that depend on the entries of $\GK$ and the matrix $\Sigmatr$. For example, the $(9,6)$ entry of the self-consistent equation reads,
\al{
\GK_{3,6} &\stackrel{\phantom{n_0\to \infty}}{=}\left[ \id\otimes\ntr\left(\bar{Z} - \bRK_{W,X,\Theta}(\bGK)-\bQK_\Sigma\right)^{-1}\right]_{9,6}\\
&\stackrel{\phantom{n_0\to \infty}}{=} \ntr\Big[\sqrt{\rho } \GK_{1,1}\big(-\Sigmatr  \rho  \GK_{1,1} \GK_{2,2}-\gamma  \phi I_{n_0}\big)^{-1}\Big]\\
&\stackrel{n_0\to \infty}{=} \mathbb{E}_\mu\Big[\frac{\sqrt{\rho } \GK_{1,1}}{-\lambda  \rho  \GK_{1,1} \GK_{2,2}-\gamma  \phi }\Big]\,,
}
where to compute the asymptotic normalized trace we moved to an eigenbasis of $\Sigmatr$ and recalled the definition of the LJSD $\mu$. The remaining entries of the \cref{eqn:GK_SCE} can be obtained in a similar manner and together yield the following set of coupled equations for the entries of $\GK$,
\begin{eqnarray}
\GK_{1,1} &=&-\frac{\gamma }{-\GK_{2,2} (-\zeta +\eta +\rho )+\rho  \GK_{2,2}-\sqrt{\rho } \GK_{6,3}-\gamma }\\
\GK_{2,2} &=&\frac{\gamma  \phi }{\psi  \GK_{1,1} (\eta -\zeta )-\gamma  \phi  \left(\sqrt{\rho } \psi  \GK_{4,5}-1\right)}\\
\GK_{3,6} &=&\mathbb{E}_\mu\Big[\frac{\sqrt{\rho } \GK_{1,1}}{-\lambda  \rho  \GK_{1,1} \GK_{2,2}-\gamma  \phi }\Big]\\
\GK_{4,5} &=&\mathbb{E}_\mu\Big[\frac{\lambda  \sqrt{\rho } \GK_{1,1}}{-\lambda  \rho  \GK_{1,1} \GK_{2,2}-\gamma  \phi }\Big]\\
\GK_{5,4} &=&\mathbb{E}_\mu\Big[-\frac{\gamma  \sqrt{\rho } \phi  \GK_{2,2}}{-\lambda  \rho  \GK_{1,1} \GK_{2,2}-\gamma  \phi }\Big]\\
\GK_{6,3} &=&\mathbb{E}_\mu\Big[-\frac{\gamma  \lambda  \sqrt{\rho } \phi  \GK_{2,2}}{-\lambda  \rho  \GK_{1,1} \GK_{2,2}-\gamma  \phi }\Big]\\
\GK_{3,4} &=&\GK_{5,6} = \mathbb{E}_\mu\Big[\frac{\sqrt{\lambda } \rho  \GK_{1,1} \GK_{2,2}}{-\lambda  \rho  \GK_{1,1} \GK_{2,2}-\gamma  \phi }\Big]\\
\GK_{3,5} &=&\GK_{4,6} = \mathbb{E}_\mu\Big[-\frac{\GK_{1,1} \sqrt{\lambda  \rho }}{\lambda  \rho  \GK_{1,1} \GK_{2,2}+\gamma  \phi }\Big]\\
\GK_{4,3} &=&\GK_{6,5} = \mathbb{E}_\mu\Big[-\frac{\gamma  \sqrt{\lambda } \phi }{-\lambda  \rho  \GK_{1,1} \GK_{2,2}-\gamma  \phi }\Big]\\
\GK_{5,3} &=&\GK_{6,4} = \mathbb{E}_\mu\Big[\frac{\gamma  \phi  \GK_{2,2} \sqrt{\lambda  \rho }}{\lambda  \rho  \GK_{1,1} \GK_{2,2}+\gamma  \phi }\Big]\\
\GK_{3,3} &=&\GK_{4,4} = \GK_{5,5} = \GK_{6,6} = \mathbb{E}_\mu\Big[-\frac{\gamma  \phi }{-\lambda  \rho  \GK_{1,1} \GK_{2,2}-\gamma  \phi }\Big]\,,
\end{eqnarray}
where we have used the fact that, asymptotically, the normalized trace becomes equivalent to an expectation over $\mu$. After eliminating $\GK_{6,3}$ and $\GK_{4,5}$ from the first two equations, it is straightforward to show that
\al{
    \tauone &\equiv \ntr(K^{-1}) = \frac{1}{\gamma}\GK_{1,1} = \frac{\sqrt{(\psi -\phi )^2+ 4 x \psi\phi \gamma/\rho}+\psi -\phi}{2 \psi  \gamma}\label{eqn:tau1}\\
    \taub &\equiv \ntr(\hK^{-1}) = \frac{1}{\gamma} \GK_{2,2} = \frac{1}{\gamma} + \frac{\psi}{\phi}\big(\tauone - \frac{1}{\gamma}\big)\label{eqn:tau1b}
}
where $\taub$ is the companion transform of $\tauone$, and where $x$ satisfies the self-consistent equation,
\begin{equation}
\label{eqn:x}
  x = \frac{1 - \gamma \tauone}{\omega + I_{1,1}} = \frac{1 - \frac{\sqrt{(\psi -\phi )^2+ 4 x \psi\phi \gamma/\rho}+\psi -\phi}{2 \psi}}{\omega + \I_{1,1}}\,.
\end{equation}
Here we used the two-index set of functionals of $\mu$, $\I_{a,b}$ defined in \cref{eqn:Idefs}.

Note that the product $\tauone \taub$ is simply related to $x$,
\begin{equation}
\label{eqn:xtau}
    x = \gamma \rho \tauone \taub\,,
\end{equation}
so that, given $x$, the equations for the remaining entries of $\GK$ completely decouple. In particular,
\begin{eqnarray}
\label{eqn:K_SCE_I}
\GK_{3,6} &=&-\frac{\sqrt{\rho } \tauone \I_{0,1}}{\phi }\\
\GK_{4,5} &=&-\frac{\sqrt{\rho } \tauone \I_{1,1}}{\phi }\\
\GK_{5,4} &=&\gamma  \sqrt{\rho } \taub \I_{0,1}\\
\GK_{6,3} &=&\gamma  \sqrt{\rho } \taub \I_{1,1}\\
\GK_{3,4} &=&\GK_{5,6} = -\frac{x \I_{\frac{1}{2},1}}{\phi }\\
\GK_{3,5} &=&\GK_{4,6} = -\frac{\sqrt{\rho } \tauone \I_{\frac{1}{2},1}}{\phi }\\
\GK_{4,3} &=&\GK_{6,5} = \I_{\frac{1}{2},1}\\
\GK_{5,3} &=&\GK_{6,4} = \gamma  \sqrt{\rho } \taub \I_{\frac{1}{2},1}\\
\GK_{3,3} &=&\GK_{4,4} = \GK_{5,5} = \GK_{6,6} = \I_{0,1}\,,
\end{eqnarray}
which will be important intermediate results for the subsequent sections.
\subsubsection{$E_{21}$}
Define the block matrix $\QE{21}$ as,
\begin{equation}
 \QE{21} = \scriptsize \left(
\begin{array}{ccccccccc}
 I_{n_0} & 0 & -\Sigmatr^{1/2} & 0 & 0 & 0 & 0 & 0 & 0 \\
 -\frac{X^\top}{\sqrt{n_0}} & I_m & 0 & 0 & 0 & 0 & 0 & 0 & 0 \\
 0 & 0 & I_{n_0} & -\Sigmate & 0 & 0 & 0 & 0 & 0 \\
 0 & 0 & 0 & I_{n_0} & -\frac{W^\top}{\sqrt{n_1}} & 0 & 0 & 0 & 0 \\
 0 & 0 & 0 & 0 & I_{n_1} & -\frac{\Theta  \sqrt{\eta -\zeta }}{\sqrt{n_1}} & -\frac{\sqrt{\rho } W}{\sqrt{n_1}} & 0 & 0 \\
 0 & 0 & 0 & 0 & \frac{\sqrt{\eta -\zeta } \Theta ^\top}{\gamma  \sqrt{n_1}} & I_m & 0 & 0 & \frac{\sqrt{\rho } X^\top}{\gamma  \sqrt{n_0}} \\
 0 & 0 & 0 & 0 & 0 & 0 & I_{n_0} & -\Sigmatr^{1/2} & 0 \\
 0 & 0 & 0 & 0 & 0 & -\frac{X}{\sqrt{n_0}} & 0 & I_{n_0} & 0 \\
 0 & 0 & 0 & -\Sigmatr^{1/2} & 0 & 0 & 0 & 0 & I_{n_0} \\
\end{array}
\right)
\,.\end{equation}
Then block matrix inversion (i.e. repeated applications of the Schur complement formula) shows that,
\begin{eqnarray}
\label{eqn:QinvE21_first}
\GE{21}_{1,1} &=& \GE{21}_{2,2} = \GE{21}_{3,3} = 1\\
\GE{21}_{6,6} &=& \GK_{1,1} \label{eq:GE2166}\\
\GE{21}_{5,5} &=& \GK_{2,2} \label{eq:GE2155}\\
\GE{21}_{4,4} &=& \GE{21}_{7,7} = \GE{21}_{8,8} = \GE{21}_{9,9} = \GK_{3,3}\\
\GE{21}_{7,8} &=& \GE{21}_{9,4} = \GK_{3,4}\\
\GE{21}_{4,8} &=& \GE{21}_{9,7} = \GK_{3,5}\\
\GE{21}_{9,8} &=& \GK_{3,6}\\
\GE{21}_{4,9} &=& \GE{21}_{8,7} = \GK_{4,3}\\
\GE{21}_{4,7} &=& \GK_{4,5}\\
\GE{21}_{7,9} &=& \GE{21}_{8,4} = \GK_{5,3}\\
\GE{21}_{7,4} &=& \GK_{5,4}\\
\GE{21}_{8,9} &=& \GK_{6,3}\\
\GE{21}_{3,1} &=& \ \ntr(\Sigmatr^{1/2})\\
\GE{21}_{7,3} &=& \frac{\gamma  \sqrt{\rho } \ \ntr\left(\hK^{-1} W \Sigmate W^\top\right)}{n_1 \psi }\\
\GE{21}_{7,1} &=& \GE{21}_{8,3} = \frac{\gamma  \sqrt{\rho } \ \ntr\left(\Sigmatr^{1/2} W^\top \hK^{-1} W \Sigmate\right)}{n_1}\\
\GE{21}_{9,3} &=& -\frac{\sqrt{\rho } \ \ntr\left(F K^{-1} X^\top \Sigmate W^\top\right)}{\sqrt{n_0} n_1 \psi }\\
\GE{21}_{8,1} &=& \frac{\gamma  \sqrt{\rho } \ \ntr\left(\Sigmatr W^\top \hK^{-1} W \Sigmate\right)}{n_1}\\
\GE{21}_{9,1} &=& -\frac{\sqrt{\rho } \ \ntr\left(\Sigmatr^{1/2} X F^\top \hK^{-1} W \Sigmate\right)}{\sqrt{n_0} n_1}\\
\GE{21}_{6,2} &=& \frac{\gamma  \phi  \ \ntr\left(\Sigmatr^{1/2} X F^\top \hK^{-1} W \Sigmate\right)}{\sqrt{n_0} n_1}\\
\GE{21}_{4,3} &=& \ \ntr(\Sigmate)-\frac{\sqrt{\rho } \ \ntr\left(\Sigmatr^{1/2} X F^\top \hK^{-1} W \Sigmate\right)}{\sqrt{n_0} n_1}\\
\GE{21}_{4,1} &=& \ \ntr(\Sigmatr^{1/2} \Sigmate)-\frac{\sqrt{\rho } \ \ntr\left(\Sigmatr X F^\top \hK^{-1} W \Sigmate\right)}{\sqrt{n_0} n_1}\label{eqn:QinvE21_last}\,,
\end{eqnarray}
where $\GE{21}_{i,j}$ denotes the normalized trace of the $(i,j)$-block of the inverse of $\QE{21}$.
Comparing to \cref{eq_E21}, we see that the error term $E_{21}$ is related to $\GE{21}_{6,2}$ by
\begin{equation}
    E_{21} = -\frac{2  \sqrt{\rhos}}{\gamma \phi}\GE{21}_{6,2}\,.
\end{equation}
To compute the limiting values of these traces, we require the asymptotic block-wise traces of $\QE{21}$, which may be determined from the operator-valued Stieltjes transform. Proceeding as above, we first augment $\QE{21}$ to form the the self-adjoint matrix $\bQE{21}$,
\begin{equation}
\label{eqn:QbarE21}
    \bQE{21} = \left(\begin{array}{cc} 0 & {[\QE{21}]}^\top\\
\QE{21} & 0 \end{array}\right)\,.\\
\end{equation}
and observe that we can write $\bQE{21}$ as
\begin{equation}
\begin{split}
\bQE{21} &= \bar{Z} + \bQE{21}_{W,X,\Theta} + \bQE{21}_\Sigma\\
&= \left(\begin{array}{cc} 0 & Z^\top\\
Z & 0 \end{array}\right) + \left(\begin{array}{cc} 0 & {[\QE{21}_{W,X,\Theta}]}^\top\\
\QE{21}_{W,X,\Theta} & 0 \end{array}\right) + \left(\begin{array}{cc} 0 & {[\QE{21}_\Sigma]}^\top\\
\QE{21}_{\Sigma} & 0 \end{array}\right)\,,
\end{split}
\end{equation}
where $Z = I_{2 m+6 n_0+n_1}$, and,
\begin{eqnarray}
 \QE{21}_{W,X,\Theta} &=&\left(
\begin{array}{ccccccccc}
 0 & 0 & 0 & 0 & 0 & 0 & 0 & 0 & 0 \\
 -\frac{X^\top}{\sqrt{n_0}} & 0 & 0 & 0 & 0 & 0 & 0 & 0 & 0 \\
 0 & 0 & 0 & 0 & 0 & 0 & 0 & 0 & 0 \\
 0 & 0 & 0 & 0 & -\frac{W^\top}{\sqrt{n_1}} & 0 & 0 & 0 & 0 \\
 0 & 0 & 0 & 0 & 0 & -\frac{\Theta  \sqrt{\eta -\zeta }}{\sqrt{n_1}} & -\frac{\sqrt{\rho } W}{\sqrt{n_1}} & 0 & 0 \\
 0 & 0 & 0 & 0 & \frac{\sqrt{\eta -\zeta } \Theta ^\top}{\gamma  \sqrt{n_1}} & 0 & 0 & 0 & \frac{\sqrt{\rho } X^\top}{\gamma  \sqrt{n_0}} \\
 0 & 0 & 0 & 0 & 0 & 0 & 0 & 0 & 0 \\
 0 & 0 & 0 & 0 & 0 & -\frac{X}{\sqrt{n_0}} & 0 & 0 & 0 \\
 0 & 0 & 0 & 0 & 0 & 0 & 0 & 0 & 0 \\
\end{array}
\right)
\\
\QE{21}_{\Sigmatr} &=&\left(
\begin{array}{ccccccccc}
 0 & 0 & -\Sigmatr^{1/2} & 0 & 0 & 0 & 0 & 0 & 0 \\
 0 & 0 & 0 & 0 & 0 & 0 & 0 & 0 & 0 \\
 0 & 0 & 0 & -\Sigmate & 0 & 0 & 0 & 0 & 0 \\
 0 & 0 & 0 & 0 & 0 & 0 & 0 & 0 & 0 \\
 0 & 0 & 0 & 0 & 0 & 0 & 0 & 0 & 0 \\
 0 & 0 & 0 & 0 & 0 & 0 & 0 & 0 & 0 \\
 0 & 0 & 0 & 0 & 0 & 0 & 0 & -\Sigmatr^{1/2} & 0 \\
 0 & 0 & 0 & 0 & 0 & 0 & 0 & 0 & 0 \\
 0 & 0 & 0 & -\Sigmatr^{1/2} & 0 & 0 & 0 & 0 & 0 \\
\end{array}
\right)
\,.\end{eqnarray}
The operator-valued Stieltjes transforms satisfy,
\begin{equation}
\label{eqn:SCE_E21}
\begin{split}
    \bGE{21} &= \bGE{21}_{\Sigma}(\bar{Z} - \bRE{21}_{W,X,\Theta}(\bGE{21}))\\
    &= \id\otimes\ntr\left(\bar{Z} - \bRE{21}_{W,X,\Theta}(\bGE{21})-\bQE{21}_\Sigma\right)^{-1}\,,
\end{split}
\end{equation}
where $\bRE{21}_{W,X,\Theta}(\bGE{21})$ is the operator-valued R-transform of $\bQE{21}_{W,X,\Theta}$. As discussed above, since $\bQE{21}_{W,X,\Theta}$ is a block matrix whose blocks are i.i.d. Gaussian matrices (and their transposes), an explicit expression for $\bRE{21}_{W,X,\Theta}(\bGE{21})$ can be obtained from the covariance map $\eta$, which can be read off from \cref{eqn:QbarE21}. As above, we use the specific sparsity pattern for $\GE{21}$ that is induced by \cref{eqn:SCE_E21}, to obtain,
\begin{equation}
    \bRE{21}_{W,X,\Theta}(\bGE{21}) = \left(\begin{array}{cc} 0 & {\RE{21}_{W,X,\Theta}}(\GE{21})^\top\\
\RE{21}_{W,X,\Theta}(\GE{21}) & 0 \end{array}\right)\,,
\end{equation}
where,
\begin{eqnarray}
{[\RE{21}_{W,X,\Theta}(\GE{21})]}_{2,6} &=&\GE{21}_{8,1}\\
{[\RE{21}_{W,X,\Theta}(\GE{21})]}_{4,7} &=&\sqrt{\rho } \GE{21}_{5,5}\\
{[\RE{21}_{W,X,\Theta}(\GE{21})]}_{5,5} &=&\frac{\psi  \GE{21}_{6,6} (\zeta -\eta )}{\gamma  \phi }+\sqrt{\rho } \psi  \GE{21}_{4,7}\\
{[\RE{21}_{W,X,\Theta}(\GE{21})]}_{6,6} &=&\frac{\GE{21}_{5,5} (\zeta -\eta )-\sqrt{\rho } \GE{21}_{8,9}}{\gamma }\\
{[\RE{21}_{W,X,\Theta}(\GE{21})]}_{8,1} &=&\frac{\GE{21}_{2,6}}{\phi }\\
{[\RE{21}_{W,X,\Theta}(\GE{21})]}_{8,9} &=&-\frac{\sqrt{\rho } \GE{21}_{6,6}}{\gamma  \phi }\,,
\end{eqnarray}and the remaining entries of $\RE{21}_{W,X,\Theta}(\GE{21})$ are zero. 

Owing to the large degree of sparsity, the matrix inverse in \cref{eqn:SCE_E21} can be performed explicitly and yields relatively simple expressions that depend on the entries of $\GE{21}$ and the matrices $\Sigmatr$ and $\Sigmate$. For example, the $(13,3)$ entry of the self-consistent equation reads,
\al{
\GE{21}_{4,3} &\stackrel{\phantom{n_0\to \infty}}{=}\left[ \id\otimes\ntr\left(\bar{Z} - \bRE{21}_{W,X,\Theta}(\bGE{21})-\bQE{21}_\Sigma\right)^{-1}\right]_{13,3}\\
&\stackrel{\phantom{n_0\to \infty}}{=} \ntr\Big[\Sigmate\left(I_{n_0}+\frac{\rho}{\gamma \phi}\GE{21}_{5,5}\GE{21}_{6,6} \Sigmatr\right)^{-1} \Big]\\
&\stackrel{n_0\to \infty}{=} \mathbb{E}_\mu\Big[\frac{r}{1+\frac{\rho}{\gamma \phi}\lambda \GE{21}_{5,5}\GE{21}_{6,6}} \Big]\label{eq:GE2143_3} \\
&\stackrel{\phantom{n_0\to \infty}}{=} \mathbb{E}_\mu\Big[\frac{r}{1+\frac{\rho}{\gamma \phi}\lambda \GK_{1,1}\GK_{2,2}} \Big]\label{eq:GE2143_4} \\
&\stackrel{\phantom{n_0\to \infty}}{=} \phi \mathbb{E}_\mu \Big[\frac{r}{\phi+x \lambda} \Big] \label{eq:GE2143_5}\\
&\stackrel{\phantom{n_0\to \infty}}{=} \Ishift_{1,1}\,.
}
To obtain \cref{eq:GE2143_3}, we computed the asymptotic normalized trace by moving to an eigenbasis of $\Sigmatr$ and recalling the definition of the LJSD $\mu$. We also used Eqs.~(\ref{eq:GE2155}) and (\ref{eq:GE2166}) to obtain \cref{eq:GE2143_4} and  Eqs.~(\ref{eqn:tau1}), (\ref{eqn:tau1b}), and (\ref{eqn:xtau}) to obtain \cref{eq:GE2143_5}. The final line follows from the definition of $\Ishift$ in \cref{eqn:Idefs}. The remaining nonzero entries of \cref{eqn:SCE_E21} can be obtained in a similar manner and together yield the following set of coupled equations for the entries of $\GE{21}$,
\begin{eqnarray}
\GE{21}_{3,1} &=&\frac{\I_{\frac{1}{2},0}}{\phi }\\
\GE{21}_{4,1} &=&\Ishift_{\frac{3}{2},1}\\
\GE{21}_{4,3} &=&\Ishift_{1,1}\\
\GE{21}_{6,2} &=&\gamma  \tauone \GE{21}_{8,1}\\
\GE{21}_{7,3} &=&\gamma  \sqrt{\rho } \taub \Ishift_{1,1}\\
\GE{21}_{8,1} &=&\gamma  \sqrt{\rho } \taub \Ishift_{2,1}\\
\GE{21}_{9,1} &=&-\frac{x \Ishift_{2,1}}{\phi }\\
\GE{21}_{9,3} &=&-\frac{x \Ishift_{\frac{3}{2},1}}{\phi }\\
\GE{21}_{7,1} &=&\GE{21}_{8,3} = \gamma  \sqrt{\rho } \taub \Ishift_{\frac{3}{2},1}\\
\GE{21}_{1,1} &=&\GE{21}_{2,2} = \GE{21}_{3,3} = 1\,,
\end{eqnarray}
where we have again used the definition of the LJSD $\mu$, the relations in Eqs.~(\ref{eqn:QinvE21_first})-(\ref{eqn:QinvE21_last}), as well as the results in \cref{sec:Kinv} to simplify the expressions. Note that in these equations and the above example for the $(13,3)$ entry, we have leveraged the simple manner in which $\Sigmate$ enters in Eqs.~(\ref{eqn:QinvE21_first})-(\ref{eqn:QinvE21_last}), namely linearly in the numerator, to simplify the dependence on $\Sigmatr$ and $\Sigmate$. In particular, by rewriting the arguments of the trace terms in an eigenbasis of $\Sigma$, the only dependence on $\Sigmatr$ and $\Sigmate$ that remains is through the training eigenvalues $\lambda$ and the overlap coefficients $r$. As such, the $\ntr$ in (\cref{eqn:SCE_E21}) can be written as an expectation over the LJSD $\mu$ in the limit, which leads to significant simplification through the introduction of the two-index set of functions of $\mu$, $\Ishift_{a,b}$, defined in \cref{eqn:Idefs}. 

It is straightforward algebra to solve these equations for the undetermined entries of $\GE{21}$ and thereby obtain the following expression for $E_{21}$,
\begin{equation}
\label{eqn:e21_result}
      E_{21} = -2\xi\frac{x}{\phi}\Ishift_{2,1}\,.
\end{equation}
\subsubsection{$E_{31}$}
Define the block matrix $\QE{31}\equiv [\QE{31}_1\; \QE{31}_2]$ by,
\begin{equation}
 \QE{31}_1 = \left(
\begin{array}{cccccc}
 I_m & \frac{\sqrt{\eta -\zeta } \Theta ^\top}{\gamma  \sqrt{n_1}} & \frac{\sqrt{\rho } X^\top}{\gamma  \sqrt{n_0}} & 0 & 0 & 0 \\
 -\frac{\Theta  \sqrt{\eta -\zeta }}{\sqrt{n_1}} & I_{n_1} & 0 & 0 & -\frac{\sqrt{\rho } W}{\sqrt{n_1}} & 0 \\
 0 & 0 & I_{n_0} & -\Sigmatr^{1/2} & 0 & 0 \\
 0 & -\frac{W^\top}{\sqrt{n_1}} & 0 & I_{n_0} & 0 & 0 \\
 0 & 0 & 0 & 0 & I_{n_0} & -\Sigmatr^{1/2} \\
 -\frac{X}{\sqrt{n_0}} & 0 & 0 & 0 & 0 & I_{n_0} \\
 0 & 0 & 0 & 0 & 0 & 0 \\
 0 & 0 & 0 & 0 & 0 & 0 \\
 0 & 0 & 0 & 0 & 0 & 0 \\
 0 & 0 & 0 & 0 & 0 & 0 \\
 0 & 0 & 0 & 0 & 0 & 0 \\
 0 & 0 & 0 & 0 & 0 & 0 \\
\end{array}
\right)
\,,\end{equation}and,
\begin{equation}
 \QE{31}_2 = \left(
\begin{array}{cccccc}
 \frac{\sqrt{\eta -\zeta } \Theta ^\top \left(\zetas-\etas\right)}{\gamma  \sqrt{n_1}} & \frac{\sqrt{\rho } X^\top \left(\zetas-\etas\right)}{\gamma  \sqrt{n_0}} & 0 & 0 & 0 & 0 \\
 0 & 0 & 0 & 0 & 0 & 0 \\
 0 & 0 & 0 & 0 & 0 & 0 \\
 0 & 0 & 0 & 0 & 0 & 0 \\
 0 & 0 & \frac{n_1 \Sigmate \rhos}{n_0 \sqrt{\rho }} & 0 & 0 & 0 \\
 0 & 0 & 0 & 0 & 0 & 0 \\
 I_{n_1} & 0 & 0 & -\frac{\Theta  \sqrt{\eta -\zeta }}{\sqrt{n_1}} & -\frac{\sqrt{\rho } W}{\sqrt{n_1}} & 0 \\
 0 & I_{n_0} & -\Sigmatr^{1/2} & 0 & 0 & 0 \\
 -\frac{W^\top}{\sqrt{n_1}} & 0 & I_{n_0} & 0 & 0 & 0 \\
 \frac{\sqrt{\eta -\zeta } \Theta ^\top}{\gamma  \sqrt{n_1}} & \frac{\sqrt{\rho } X^\top}{\gamma  \sqrt{n_0}} & 0 & I_m & 0 & 0 \\
 0 & 0 & 0 & 0 & I_{n_0} & -\Sigmatr^{1/2} \\
 0 & 0 & 0 & -\frac{X}{\sqrt{n_0}} & 0 & I_{n_0} \\
\end{array}
\right)
\,.\end{equation}
Then block matrix inversion (i.e. repeated applications of the Schur complement formula) shows that,
\begin{eqnarray}
\label{eqn:QinvE31_first}
\GE{31}_{1,1} &=& \GE{31}_{10,10} = \GK_{1,1}\\
\GE{31}_{2,2} &=& \GE{31}_{7,7} = \GK_{2,2}\\
\GE{31}_{3,3} &=& \GE{31}_{6,6} = \GE{31}_{8,8} = \GE{31}_{12,12} = \GE{31}_{4,4} = \GE{31}_{5,5} = \GE{31}_{9,9} = \GE{31}_{11,11} = \GK_{3,3}\\
\GE{31}_{3,4} &=& \GE{31}_{5,6} = \GE{31}_{8,9} = \GE{31}_{11,12} = \GK_{3,4}\\
\GE{31}_{3,5} &=& \GE{31}_{4,6} = \GE{31}_{8,11} = \GE{31}_{9,12} = \GK_{3,5}\\
\GE{31}_{3,6} &=& \GE{31}_{8,12} = \GK_{3,6}\\
\GE{31}_{4,3} &=& \GE{31}_{6,5} = \GE{31}_{9,8} = \GE{31}_{12,11} = \GK_{4,3}\\
\GE{31}_{4,5} &=& \GE{31}_{9,11} = \GK_{4,5}\\
\GE{31}_{5,3} &=& \GE{31}_{6,4} = \GE{31}_{11,8} = \GE{31}_{12,9} = \GK_{5,3}\\
\GE{31}_{5,4} &=& \GE{31}_{11,9} = \GK_{5,4}\\
\GE{31}_{6,3} &=& \GE{31}_{12,8} = \GK_{6,3}\\
\GE{31}_{10,1} &=& \frac{\gamma \phi}{\psi \sigma_{\e }^2}E_{31}\label{eqn:QinvE31_last}\,,
\end{eqnarray}
where $\GE{31}_{i,j}$ denotes the normalized trace of the $(i,j)$-block of the inverse of $\QE{31}$.
For brevity, we have suppressed the expressions for the other non-zero blocks.

To compute the limiting values of these traces, we require the asymptotic block-wise traces of $\QE{31}$, which may be determined from the operator-valued Stieltjes transform. Proceeding as above, we first augment $\QE{31}$ to form the the self-adjoint matrix $\bQE{31}$,
\begin{equation}
\label{eqn:QbarE31}
    \bQE{31} = \left(\begin{array}{cc} 0 & {[\QE{31}]}^\top\\
\QE{31} & 0 \end{array}\right)\,.\\
\end{equation}
and observe that we can write $\bQE{31}$ as,
\begin{equation}
\begin{split}
\bQE{31} &= \bar{Z} + \bQE{31}_{W,X,\Theta} + \bQE{31}_\Sigma\\
&= \left(\begin{array}{cc} 0 & Z^\top\\
Z & 0 \end{array}\right) + \left(\begin{array}{cc} 0 & {[\QE{31}_{W,X,\Theta}]}^\top\\
\QE{31}_{W,X,\Theta} & 0 \end{array}\right) + \left(\begin{array}{cc} 0 & {[\QE{31}_\Sigma]}^\top\\
\QE{31}_{\Sigma} & 0 \end{array}\right)\,,
\end{split}
\end{equation}
where $Z = I_{2 m+8 n_0+2 n_1}$, $\QE{31}_{W,X,\Theta} \equiv [[\QE{31}_{W,X,\Theta}]_1\;[\QE{31}_{W,X,\Theta}]_2]$ and,
\begin{eqnarray}
{[\QE{31}_{W,X,\Theta}]}_1 &=&\scriptsize\left(
\begin{array}{cccccc}
 0 & \frac{\sqrt{\eta -\zeta } \Theta ^\top}{\gamma  \sqrt{n_1}} & \frac{\sqrt{\rho } X^\top}{\gamma  \sqrt{n_0}} & 0 & 0 & 0 \\
 -\frac{\Theta  \sqrt{\eta -\zeta }}{\sqrt{n_1}} & 0 & 0 & 0 & -\frac{\sqrt{\rho } W}{\sqrt{n_1}} & 0 \\
 0 & 0 & 0 & 0 & 0 & 0 \\
 0 & -\frac{W^\top}{\sqrt{n_1}} & 0 & 0 & 0 & 0 \\
 0 & 0 & 0 & 0 & 0 & 0 \\
 -\frac{X}{\sqrt{n_0}} & 0 & 0 & 0 & 0 & 0 \\
 0 & 0 & 0 & 0 & 0 & 0 \\
 0 & 0 & 0 & 0 & 0 & 0 \\
 0 & 0 & 0 & 0 & 0 & 0 \\
 0 & 0 & 0 & 0 & 0 & 0 \\
 0 & 0 & 0 & 0 & 0 & 0 \\
 0 & 0 & 0 & 0 & 0 & 0 \\
\end{array}
\right)
\\
{[\QE{31}_{W,X,\Theta}]}_2 &=&\scriptsize\left(
\begin{array}{cccccc}
 \frac{\sqrt{\eta -\zeta } \Theta ^\top \left(\zetas-\etas\right)}{\gamma  \sqrt{n_1}} & \frac{\sqrt{\rho } X^\top \left(\zetas-\etas\right)}{\gamma  \sqrt{n_0}} & 0 & 0 & 0 & 0 \\
 0 & 0 & 0 & 0 & 0 & 0 \\
 0 & 0 & 0 & 0 & 0 & 0 \\
 0 & 0 & 0 & 0 & 0 & 0 \\
 0 & 0 & 0 & 0 & 0 & 0 \\
 0 & 0 & 0 & 0 & 0 & 0 \\
 0 & 0 & 0 & -\frac{\Theta  \sqrt{\eta -\zeta }}{\sqrt{n_1}} & -\frac{\sqrt{\rho } W}{\sqrt{n_1}} & 0 \\
 0 & 0 & 0 & 0 & 0 & 0 \\
 -\frac{W^\top}{\sqrt{n_1}} & 0 & 0 & 0 & 0 & 0 \\
 \frac{\sqrt{\eta -\zeta } \Theta ^\top}{\gamma  \sqrt{n_1}} & \frac{\sqrt{\rho } X^\top}{\gamma  \sqrt{n_0}} & 0 & 0 & 0 & 0 \\
 0 & 0 & 0 & 0 & 0 & 0 \\
 0 & 0 & 0 & -\frac{X}{\sqrt{n_0}} & 0 & 0 \\
\end{array}
\right)
\\
\QE{31}_{\Sigmatr} &=&\scriptsize\left(
\begin{array}{cccccccccccc}
 0 & 0 & 0 & 0 & 0 & 0 & 0 & 0 & 0 & 0 & 0 & 0 \\
 0 & 0 & 0 & 0 & 0 & 0 & 0 & 0 & 0 & 0 & 0 & 0 \\
 0 & 0 & 0 & -\Sigmatr^{1/2} & 0 & 0 & 0 & 0 & 0 & 0 & 0 & 0 \\
 0 & 0 & 0 & 0 & 0 & 0 & 0 & 0 & 0 & 0 & 0 & 0 \\
 0 & 0 & 0 & 0 & 0 & -\Sigmatr^{1/2} & 0 & 0 & \frac{n_1 \Sigmate \rhos}{n_0 \sqrt{\rho }} & 0 & 0 & 0 \\
 0 & 0 & 0 & 0 & 0 & 0 & 0 & 0 & 0 & 0 & 0 & 0 \\
 0 & 0 & 0 & 0 & 0 & 0 & 0 & 0 & 0 & 0 & 0 & 0 \\
 0 & 0 & 0 & 0 & 0 & 0 & 0 & 0 & -\Sigmatr^{1/2} & 0 & 0 & 0 \\
 0 & 0 & 0 & 0 & 0 & 0 & 0 & 0 & 0 & 0 & 0 & 0 \\
 0 & 0 & 0 & 0 & 0 & 0 & 0 & 0 & 0 & 0 & 0 & 0 \\
 0 & 0 & 0 & 0 & 0 & 0 & 0 & 0 & 0 & 0 & 0 & -\Sigmatr^{1/2} \\
 0 & 0 & 0 & 0 & 0 & 0 & 0 & 0 & 0 & 0 & 0 & 0 \\
\end{array}
\right)
\,.\end{eqnarray}
The operator-valued Stieltjes transforms satisfy,
\begin{equation}
\label{eqn:SCE_E31}
\begin{split}
    \bGE{31} &= \bGE{31}_{\Sigma}(\bar{Z} - \bRE{31}_{W,X,\Theta}(\bGE{31}))\\
    &= \id\otimes\ntr\left(\bar{Z} - \bRE{31}_{W,X,\Theta}(\bGE{31})-\bQE{31}_\Sigma\right)^{-1}\,,
\end{split}
\end{equation}
where $\bRE{31}_{W,X,\Theta}(\bGE{31})$ is the operator-valued R-transform of $\bQE{31}_{W,X,\Theta}$. As discussed above, since $\bQE{31}_{W,X,\Theta}$ is a block matrix whose blocks are i.i.d. Gaussian matrices (and their transposes), an explicit expression for $\bRE{31}_{W,X,\Theta}(\bGE{31})$ can be obtained from the covariance map $\eta$, which can be read off from \cref{eqn:QbarE31}. As above, we use the specific sparsity pattern for $\GE{31}$ that is induced by \cref{eqn:SCE_E31}, to obtain,
\begin{equation}
    \bRE{31}_{W,X,\Theta}(\bGE{31}) = \left(\begin{array}{cc} 0 & {\RE{31}_{W,X,\Theta}}(\GE{31})^\top\\
\RE{31}_{W,X,\Theta}(\GE{31}) & 0 \end{array}\right)\,,
\end{equation}
where,
\begin{align}
{[\RE{31}_{W,X,\Theta}(\GE{31})]}_{1,1} &=\frac{\GE{31}_{2,2} (\zeta -\eta )}{\gamma }-\frac{\sqrt{\rho } \GE{31}_{6,3}}{\gamma }+\frac{\sqrt{\rho } \GE{31}_{6,8} \left(\etas-\zetas\right)}{\gamma }\nonumber\\
&\quad+\frac{\GE{31}_{2,7} (\zeta -\eta ) \left(\zetas-\etas\right)}{\gamma }\\
{[\RE{31}_{W,X,\Theta}(\GE{31})]}_{1,10} &=\frac{\GE{31}_{7,2} (\zeta -\eta )}{\gamma }-\frac{\sqrt{\rho } \GE{31}_{12,3}}{\gamma }+\frac{\sqrt{\rho } \GE{31}_{12,8} \left(\etas-\zetas\right)}{\gamma }\nonumber\\
&\quad+\frac{\GE{31}_{7,7} (\zeta -\eta ) \left(\zetas-\etas\right)}{\gamma }\\
{[\RE{31}_{W,X,\Theta}(\GE{31})]}_{2,2} &=\frac{\psi  \GE{31}_{1,1} (\zeta -\eta )}{\gamma  \phi }+\sqrt{\rho } \psi  \GE{31}_{4,5}\\
{[\RE{31}_{W,X,\Theta}(\GE{31})]}_{2,7} &=\frac{\psi  \GE{31}_{10,1} (\zeta -\eta )}{\gamma  \phi }+\sqrt{\rho } \psi  \GE{31}_{9,5}+\frac{\psi  \GE{31}_{1,1} (\zeta -\eta ) \left(\zetas-\etas\right)}{\gamma  \phi }\\
{[\RE{31}_{W,X,\Theta}(\GE{31})]}_{4,5} &=\sqrt{\rho } \GE{31}_{2,2}\\
{[\RE{31}_{W,X,\Theta}(\GE{31})]}_{4,11} &=\sqrt{\rho } \GE{31}_{7,2}\\
{[\RE{31}_{W,X,\Theta}(\GE{31})]}_{6,3} &=-\frac{\sqrt{\rho } \GE{31}_{1,1}}{\gamma  \phi }\\
{[\RE{31}_{W,X,\Theta}(\GE{31})]}_{6,8} &=\frac{\sqrt{\rho } \GE{31}_{1,1} \left(\etas-\zetas\right)}{\gamma  \phi }-\frac{\sqrt{\rho } \GE{31}_{10,1}}{\gamma  \phi }\\
{[\RE{31}_{W,X,\Theta}(\GE{31})]}_{7,2} &=\frac{\psi  \GE{31}_{1,10} (\zeta -\eta )}{\gamma  \phi }+\sqrt{\rho } \psi  \GE{31}_{4,11}\\
{[\RE{31}_{W,X,\Theta}(\GE{31})]}_{7,7} &=\frac{\psi  \GE{31}_{10,10} (\zeta -\eta )}{\gamma  \phi }+\sqrt{\rho } \psi  \GE{31}_{9,11}+\frac{\psi  \GE{31}_{1,10} (\zeta -\eta ) \left(\zetas-\etas\right)}{\gamma  \phi }\\
{[\RE{31}_{W,X,\Theta}(\GE{31})]}_{9,5} &=\sqrt{\rho } \GE{31}_{2,7}\\
{[\RE{31}_{W,X,\Theta}(\GE{31})]}_{9,11} &=\sqrt{\rho } \GE{31}_{7,7}\\
{[\RE{31}_{W,X,\Theta}(\GE{31})]}_{10,1} &=\frac{\GE{31}_{2,7} (\zeta -\eta )}{\gamma }-\frac{\sqrt{\rho } \GE{31}_{6,8}}{\gamma }\\
{[\RE{31}_{W,X,\Theta}(\GE{31})]}_{10,10} &=\frac{\GE{31}_{7,7} (\zeta -\eta )}{\gamma }-\frac{\sqrt{\rho } \GE{31}_{12,8}}{\gamma }\\
{[\RE{31}_{W,X,\Theta}(\GE{31})]}_{12,3} &=-\frac{\sqrt{\rho } \GE{31}_{1,10}}{\gamma  \phi }\\
{[\RE{31}_{W,X,\Theta}(\GE{31})]}_{12,8} &=\frac{\sqrt{\rho } \GE{31}_{1,10} \left(\etas-\zetas\right)}{\gamma  \phi }-\frac{\sqrt{\rho } \GE{31}_{10,10}}{\gamma  \phi }\,,
\end{align}and the remaining entries of $\RE{31}_{W,X,\Theta}(\GE{31})$ are zero. Similarly, following the example from $\GE{21}$ above, plugging these expressions into \cref{eqn:SCE_E31} and explicitly performing the block-matrix inverse yields the following set of coupled equations,
\begin{align}
\GE{31}_{7,2} &=\gamma ^2 \sqrt{\rho } \taub^2 \psi  \GE{31}_{9,5}+\frac{\gamma  \taub^2 \psi  \GE{31}_{10,1} (\zeta -\eta )}{\phi }+\frac{\gamma ^2 \tauone \taub^2 \psi  (\zeta -\eta ) \left(\zetas-\etas\right)}{\phi }\\
\GE{31}_{8,6} &=\frac{\rho  \tauone \psi  \phi  \I_{0,2} \left(\etas-\zetas\right)-x^2 \Ishift_{2,2} \rhos}{\sqrt{\rho } \psi  \phi }-\frac{\sqrt{\rho } \GE{31}_{10,1} \I_{0,2}}{\gamma }+\frac{\rho ^{3/2} \tauone^2 \GE{31}_{7,2} \I_{1,2}}{\phi }\\
\GE{31}_{9,5} &=\frac{\rho  \tauone \I_{1,2} \left(\etas-\zetas\right)-\frac{\phi  \Ishift_{1,2} \rhos}{\psi }}{\sqrt{\rho }}-\frac{\sqrt{\rho } \GE{31}_{10,1} \I_{1,2}}{\gamma }+\frac{\rho ^{3/2} \tauone^2 \GE{31}_{7,2} \I_{2,2}}{\phi }\\
\GE{31}_{10,1} &=\gamma  \tauone^2 \GE{31}_{7,2} (\zeta -\eta )-\gamma  \sqrt{\rho } \tauone^2 \GE{31}_{12,3}+\gamma  \tauone \left(\gamma  \tauone-1\right) \left(\zetas-\etas\right)\\
\GE{31}_{11,4} &=-\frac{\gamma ^2 \sqrt{\rho } \taub^2 \left(\phi  \Ishift_{1,2} \rhos+\rho  \tauone \psi  \I_{1,2} \left(\zetas-\etas\right)\right)}{\psi }+\sqrt{\rho } \phi  \GE{31}_{7,2} \I_{0,2}-\gamma  \rho ^{3/2} \taub^2 \GE{31}_{10,1} \I_{1,2}\\
\GE{31}_{12,3} &=-\frac{\gamma ^2 \sqrt{\rho } \taub^2 \left(\phi  \Ishift_{2,2} \rhos+\rho  \tauone \psi  \I_{2,2} \left(\zetas-\etas\right)\right)}{\psi }+\sqrt{\rho } \phi  \GE{31}_{7,2} \I_{1,2}-\gamma  \rho ^{3/2} \taub^2 \GE{31}_{10,1} \I_{2,2}\\
\GE{31}_{8,3} &=\GE{31}_{12,6} = \frac{\gamma ^2 \rho  \tauone \taub^2 \Ishift_{2,2} \rhos}{\psi }-\rho  \tauone \GE{31}_{7,2} \I_{1,2}-\rho  \taub \GE{31}_{10,1} \I_{1,2}+x \I_{1,2} \left(\etas-\zetas\right)\\
\GE{31}_{8,4} &=\GE{31}_{11,6} = \frac{\gamma ^2 \rho  \tauone \taub^2 \Ishift_{\frac{3}{2},2} \rhos}{\psi }-\rho  \tauone \GE{31}_{7,2} \I_{\frac{1}{2},2}-\rho  \taub \GE{31}_{10,1} \I_{\frac{1}{2},2}+x \I_{\frac{1}{2},2} \left(\etas-\zetas\right)\\
\GE{31}_{8,5} &=\GE{31}_{9,6} = \frac{\frac{x}{\psi} \Ishift_{\frac{3}{2},2} \rhos +\rho  \tauone \I_{\frac{1}{2},2} \left(\etas-\zetas\right)}{\sqrt{\rho }}-\frac{\sqrt{\rho } \GE{31}_{10,1} \I_{\frac{1}{2},2}}{\gamma }+\frac{\rho ^{3/2} \tauone^2 \GE{31}_{7,2} \I_{\frac{3}{2},2}}{\phi }\\
\GE{31}_{9,3} &=\GE{31}_{12,5} = -\frac{\gamma  \taub \phi  \Ishift_{\frac{3}{2},2} \rhos}{\psi }-\rho  \tauone \GE{31}_{7,2} \I_{\frac{3}{2},2}-\rho  \taub \GE{31}_{10,1} \I_{\frac{3}{2},2}+x \I_{\frac{3}{2},2} \left(\etas-\zetas\right)\\
\GE{31}_{9,4} &=\GE{31}_{11,5} = -\frac{\gamma  \taub \phi  \Ishift_{1,2} \rhos}{\psi }-\rho  \tauone \GE{31}_{7,2} \I_{1,2}-\rho  \taub \GE{31}_{10,1} \I_{1,2}+x \I_{1,2} \left(\etas-\zetas\right)\\
\GE{31}_{11,3} &=\GE{31}_{12,4} = -\frac{\gamma ^2 \sqrt{\rho } \taub^2 \left(\phi  \Ishift_{\frac{3}{2},2} \rhos+\rho  \tauone \psi  \I_{\frac{3}{2},2} \left(\zetas-\etas\right)\right)}{\psi }+\sqrt{\rho } \phi  \GE{31}_{7,2} \I_{\frac{1}{2},2}\nonumber\\
&\qquad\qquad\quad-\gamma  \rho ^{3/2} \taub^2 \GE{31}_{10,1} \I_{\frac{3}{2},2}\,,
\end{align}
Here we have used the definition of the LJSD $\mu$, the relations in Eqs.~(\ref{eqn:QinvE31_first})-(\ref{eqn:QinvE31_last}), the definition of $\Ishift_{a,b}$, and the results in \cref{sec:Kinv} to simplify the expressions. It is straightforward algebra to solve these equations for the undetermined entries of $\GE{31}$ and thereby obtain the following expression for $E_{31}$,
\begin{equation}
E_{31} = \sigma_{\e}^2 \frac{(\etas-\zetas)A_{31} + \rhos B_{31}}{D_{31}}\,,
\end{equation}
where,
\begin{eqnarray}
A_{31} &=& \rho ^2 \tauone \psi  x^2 \I_{2,2} \left(-\gamma  \tauone \psi +\psi +\phi \right)+2 \rho  \tauone \psi ^2 x^2 \phi  (\eta -\zeta )\I_{1,2}+\rho ^2 \tauone \psi ^2 x^2 \phi ^2 \I_{1,2}^2\qquad\qquad\\
&&\qquad+\tauone \psi  \left(x^2 \psi  (\zeta -\eta )^2-\rho ^2 \left(\phi -\gamma  \tauone \phi \right)\right)-\rho ^2 \tauone \psi ^2 x^4 \I_{2,2}^2\nonumber\\
B_{31} &=&\rho  \psi  x^4 \phi  \Ishift_{2,2} \I_{2,2}-\rho  \psi  x^2 \phi ^3 \Ishift_{1,2} \I_{1,2}+\psi  x^2 \phi ^2 \Ishift_{1,2} (\zeta -\eta )-\rho  x^2 \phi ^2 \Ishift_{2,2}\\
D_{31} &=& -\rho ^2 \psi  x^4 \phi  \I_{2,2}^2+2 \rho  \psi  x^2 \phi ^2 \I_{1,2} (\eta -\zeta )+\rho ^2 \psi  x^2 \phi ^3 \I_{1,2}^2+\rho ^2 x^2 \phi  \I_{2,2} (\psi +\phi )\nonumber\\
&&\qquad+\phi  \left(x^2 \psi  (\zeta -\eta )^2-\rho ^2 \phi \right)\,.
\end{eqnarray}
Further simplifications are possible using the raising and lowering identities in \cref{eq:raise_lower}, as well as the results in  \cref{sec:Kinv}, to obtain,
\al{
\label{eqn:e31_result}
E_{31} &=\sigma_\e^2\xi^2 \frac{ \tauone\taub x \big(\rho\frac{\psi}{\phi}(\phi \I_{1,2}+\omega)(\omegas+\Ishift_{1,1}) + \frac{x}{\tauone}\Ishift_{2,2}\big)}{\tauone + \taub x (\omega + \phi \I_{1,2})-\tauone x^2 \frac{\psi}{\phi}\I_{2,2}}\\
&=-\sigma_\e^2\xi^2 \frac{\partial x}{\partial \gamma} \Big(\rho\frac{\psi}{\phi}(\phi \I_{1,2}+\omega)(\omegas+\Ishift_{1,1}) + \frac{x}{\tauone}\Ishift_{2,2}\Big)\,,
}
where we have used,
\begin{equation}
\label{eq:dxdgamma}
    \frac{\partial x}{\partial \gamma} = -\frac{x}{\gamma+\rho \gamma (\tfrac{\psi}{\phi} \tauone + \taub)(\omega + \phi \I_{1,2})}\,,
\end{equation}
which follows from \cref{eqn:x} via implicit differentiation.
\subsubsection{$E_{32}$}
Define the block matrix $\QE{32}\equiv [\QE{32}_1\; \QE{32}_2]$ by,
\begin{equation}
 \QE{32}_1 = \scriptsize\left(
\begin{array}{ccccccc}
 I_m & \frac{\sqrt{\eta -\zeta } \Theta ^\top}{\gamma  \sqrt{n_1}} & \frac{\sqrt{\rho } X^\top}{\gamma  \sqrt{n_0}} & 0 & 0 & 0 & \frac{\sqrt{\eta -\zeta } \Theta ^\top \left(\zetas-\etas\right)}{\gamma  \sqrt{n_1}} \\
 -\frac{\Theta  \sqrt{\eta -\zeta }}{\sqrt{n_1}} & I_{n_1} & 0 & 0 & -\frac{\sqrt{\rho } W}{\sqrt{n_1}} & 0 & 0 \\
 0 & 0 & I_{n_0} & -\Sigmatr^{1/2} & 0 & 0 & 0 \\
 0 & -\frac{W^\top}{\sqrt{n_1}} & 0 & I_{n_0} & 0 & 0 & 0 \\
 0 & 0 & 0 & 0 & I_{n_0} & -\Sigmatr^{1/2} & 0 \\
 -\frac{X}{\sqrt{n_0}} & 0 & 0 & 0 & 0 & I_{n_0} & 0 \\
 0 & 0 & 0 & 0 & 0 & 0 & I_{n_1} \\
 0 & 0 & 0 & 0 & 0 & 0 & -\frac{W^\top}{\sqrt{n_1}} \\
 0 & 0 & 0 & 0 & 0 & 0 & \frac{\sqrt{\eta -\zeta } \Theta ^\top}{\gamma  \sqrt{n_1}} \\
 0 & 0 & 0 & 0 & 0 & 0 & 0 \\
 0 & 0 & 0 & 0 & 0 & 0 & 0 \\
 0 & 0 & 0 & 0 & 0 & 0 & 0 \\
 0 & 0 & 0 & 0 & 0 & 0 & -\frac{W^\top}{\sqrt{n_1}} \\
 0 & 0 & 0 & 0 & 0 & 0 & 0 \\
 0 & 0 & 0 & 0 & 0 & 0 & 0 \\
\end{array}
\right)
\,,\end{equation}and,
\begin{equation}
 \QE{32}_2 =\scriptsize \left(
\begin{array}{cccccccc}
 0 & 0 & 0 & 0 & 0 & 0 & 0 & 0 \\
 0 & 0 & 0 & 0 & 0 & 0 & 0 & 0 \\
 \Sigmatr^{1/2} \left(\etas-\zetas\right) & 0 & 0 & 0 & 0 & 0 & 0 & 0 \\
 0 & 0 & 0 & 0 & 0 & 0 & 0 & 0 \\
 \frac{n_1 \Sigmate \rhos}{n_0 \sqrt{\rho }} & 0 & 0 & 0 & 0 & 0 & 0 & 0 \\
 0 & 0 & 0 & 0 & 0 & 0 & 0 & 0 \\
 0 & -\frac{\Theta  \sqrt{\eta -\zeta }}{\sqrt{n_1}} & -\frac{\sqrt{\rho } W}{\sqrt{n_1}} & 0 & 0 & 0 & 0 & 0 \\
 I_{n_0} & 0 & 0 & 0 & 0 & 0 & 0 & 0 \\
 0 & I_m & 0 & 0 & \frac{\sqrt{\rho } X^\top}{\gamma  \sqrt{n_0}} & 0 & 0 & 0 \\
 0 & 0 & I_{n_0} & -\Sigmatr^{1/2} & 0 & 0 & 0 & 0 \\
 0 & -\frac{X}{\sqrt{n_0}} & 0 & I_{n_0} & 0 & 0 & 0 & 0 \\
 0 & 0 & 0 & 0 & I_{n_0} & -\Sigmatr^{1/2} & 0 & 0 \\
 0 & 0 & 0 & 0 & 0 & I_{n_0} & \frac{\Sigmatr^{1/2}}{\sqrt{\rho }} & 0 \\
 0 & 0 & 0 & 0 & 0 & 0 & I_{n_0} & -\frac{X}{\sqrt{n_0}} \\
 0 & 0 & 0 & 0 & 0 & 0 & 0 & I_m \\
\end{array}
\right)
\,.\end{equation}
Then block matrix inversion (i.e. repeated applications of the Schur complement formula) shows that,
\begin{eqnarray}
\label{eqn:QinvE32_first}
\GE{32}_{8,8} &=& \GE{32}_{14,14} = \GE{32}_{15,15} = 1\\
\GE{32}_{1,1} &=& \GE{32}_{9,9} = \GK_{1,1}\\
\GE{32}_{2,2} &=& \GE{32}_{7,7} = \GK_{2,2}\\
\GE{32}_{3,3} &=& \GE{32}_{6,6} = \GE{32}_{11,11} = \GE{32}_{12,12} = \GE{32}_{4,4} = \GE{32}_{5,5} = \GE{32}_{10,10} = \GE{32}_{13,13} = \GK_{3,3}\\
\GE{32}_{3,4} &=& \GE{32}_{5,6} = \GE{32}_{10,11} = \GE{32}_{12,8} = \GE{32}_{12,13} = \GK_{3,4}\\
\GE{32}_{3,5} &=& \GE{32}_{4,6} = \GE{32}_{12,10} = \GE{32}_{13,11} = \GK_{3,5}\\
\GE{32}_{3,6} &=& \GE{32}_{12,11} = \GK_{3,6}\\
\GE{32}_{4,3} &=& \GE{32}_{6,5} = \GE{32}_{11,10} = \GE{32}_{13,12} = \GK_{4,3}\\
\GE{32}_{4,5} &=& \GE{32}_{13,10} = \GK_{4,5}\\
\GE{32}_{5,3} &=& \GE{32}_{6,4} = \GE{32}_{10,12} = \GE{32}_{11,8} = \GE{32}_{11,13} = \GK_{5,3}\\
\GE{32}_{5,4} &=& \GE{32}_{10,8} = \GE{32}_{10,13} = \GK_{5,4}\\
\GE{32}_{6,3} &=& \GE{32}_{11,12} = \GK_{6,3}\\
\GE{32}_{15,1} &=& \frac{\phi}{\psi} E_{32}\label{eqn:QinvE32_last}\,,
\end{eqnarray}
where $\GE{32}_{i,j}$ denotes the normalized trace of the $(i,j)$-block of the inverse of $\QE{32}$.
For brevity, we have suppressed the expressions for the other non-zero blocks.

To compute the limiting values of these traces, we require the asymptotic block-wise traces of $\QE{32}$, which may be determined from the operator-valued Stieltjes transform. To proceed, we first augment $\QE{32}$ to form the the self-adjoint matrix $\bQE{32}$,
\begin{equation}
\label{eqn:Qbar32}
    \bQE{32} = \left(\begin{array}{cc} 0 & {[\QE{32}]}^\top\\
\QE{32} & 0 \end{array}\right)\,.\\
\end{equation}
and observe that we can write $\bQE{32}$ as,
\begin{equation}
\begin{split}
\bQE{32} &= \bar{Z} + \bQE{32}_{W,X,\Theta} + \bQE{32}_\Sigma\\
&= \left(\begin{array}{cc} 0 & Z^\top\\
Z & 0 \end{array}\right) + \left(\begin{array}{cc} 0 & {[\QE{32}_{W,X,\Theta}]}^\top\\
\QE{32}_{W,X,\Theta} & 0 \end{array}\right) + \left(\begin{array}{cc} 0 & {[\QE{32}_\Sigma]}^\top\\
\QE{32}_{\Sigma} & 0 \end{array}\right)\,,
\end{split}
\end{equation}
where $Z = I_{3 m+10 n_0+2 n_1}$, $\QE{32}_{W,X,\Theta} \equiv [[\QE{32}_{W,X,\Theta}]_1\;[\QE{32}_{W,X,\Theta}]_2]$ and,
\begin{align}
{[\QE{32}_{W,X,\Theta}]}_1 &= \scriptsize\left(
\begin{array}{ccccccc}
 0 & \frac{\sqrt{\eta -\zeta } \Theta ^\top}{\gamma  \sqrt{n_1}} & \frac{\sqrt{\rho } X^\top}{\gamma  \sqrt{n_0}} & 0 & 0 & 0 & \frac{\sqrt{\eta -\zeta } \Theta ^\top \left(\zetas-\etas\right)}{\gamma  \sqrt{n_1}} \\
 -\frac{\Theta  \sqrt{\eta -\zeta }}{\sqrt{n_1}} & 0 & 0 & 0 & -\frac{\sqrt{\rho } W}{\sqrt{n_1}} & 0 & 0 \\
 0 & 0 & 0 & 0 & 0 & 0 & 0 \\
 0 & -\frac{W^\top}{\sqrt{n_1}} & 0 & 0 & 0 & 0 & 0 \\
 0 & 0 & 0 & 0 & 0 & 0 & 0 \\
 -\frac{X}{\sqrt{n_0}} & 0 & 0 & 0 & 0 & 0 & 0 \\
 0 & 0 & 0 & 0 & 0 & 0 & 0 \\
 0 & 0 & 0 & 0 & 0 & 0 & -\frac{W^\top}{\sqrt{n_1}} \\
 0 & 0 & 0 & 0 & 0 & 0 & \frac{\sqrt{\eta -\zeta } \Theta ^\top}{\gamma  \sqrt{n_1}} \\
 0 & 0 & 0 & 0 & 0 & 0 & 0 \\
 0 & 0 & 0 & 0 & 0 & 0 & 0 \\
 0 & 0 & 0 & 0 & 0 & 0 & 0 \\
 0 & 0 & 0 & 0 & 0 & 0 & -\frac{W^\top}{\sqrt{n_1}} \\
 0 & 0 & 0 & 0 & 0 & 0 & 0 \\
 0 & 0 & 0 & 0 & 0 & 0 & 0 \\
\end{array}
\right)
\\
{[\QE{32}_{W,X,\Theta}]}_2 &= \scriptsize\left(
\begin{array}{cccccccc}
 0 & 0 & 0 & 0 & 0 & 0 & 0 & 0 \\
 0 & 0 & 0 & 0 & 0 & 0 & 0 & 0 \\
 0 & 0 & 0 & 0 & 0 & 0 & 0 & 0 \\
 0 & 0 & 0 & 0 & 0 & 0 & 0 & 0 \\
 0 & 0 & 0 & 0 & 0 & 0 & 0 & 0 \\
 0 & 0 & 0 & 0 & 0 & 0 & 0 & 0 \\
 0 & -\frac{\Theta  \sqrt{\eta -\zeta }}{\sqrt{n_1}} & -\frac{\sqrt{\rho } W}{\sqrt{n_1}} & 0 & 0 & 0 & 0 & 0 \\
 0 & 0 & 0 & 0 & 0 & 0 & 0 & 0 \\
 0 & 0 & 0 & 0 & \frac{\sqrt{\rho } X^\top}{\gamma  \sqrt{n_0}} & 0 & 0 & 0 \\
 0 & 0 & 0 & 0 & 0 & 0 & 0 & 0 \\
 0 & -\frac{X}{\sqrt{n_0}} & 0 & 0 & 0 & 0 & 0 & 0 \\
 0 & 0 & 0 & 0 & 0 & 0 & 0 & 0 \\
 0 & 0 & 0 & 0 & 0 & 0 & 0 & 0 \\
 0 & 0 & 0 & 0 & 0 & 0 & 0 & -\frac{X}{\sqrt{n_0}} \\
 0 & 0 & 0 & 0 & 0 & 0 & 0 & 0 \\
\end{array}
\right)
\\
\QE{32}_{\Sigmatr} &= \scriptsize\left(
\begin{array}{ccccccccccccccc}
 0 & 0 & 0 & 0 & 0 & 0 & 0 & 0 & 0 & 0 & 0 & 0 & 0 & 0 & 0 \\
 0 & 0 & 0 & 0 & 0 & 0 & 0 & 0 & 0 & 0 & 0 & 0 & 0 & 0 & 0 \\
 0 & 0 & 0 & -\Sigmatr^{1/2} & 0 & 0 & 0 & \Sigmatr^{1/2} \left(\etas-\zetas\right) & 0 & 0 & 0 & 0 & 0 & 0 & 0 \\
 0 & 0 & 0 & 0 & 0 & 0 & 0 & 0 & 0 & 0 & 0 & 0 & 0 & 0 & 0 \\
 0 & 0 & 0 & 0 & 0 & -\Sigmatr^{1/2} & 0 & \frac{n_1 \Sigmate \rhos}{n_0 \sqrt{\rho }} & 0 & 0 & 0 & 0 & 0 & 0 & 0 \\
 0 & 0 & 0 & 0 & 0 & 0 & 0 & 0 & 0 & 0 & 0 & 0 & 0 & 0 & 0 \\
 0 & 0 & 0 & 0 & 0 & 0 & 0 & 0 & 0 & 0 & 0 & 0 & 0 & 0 & 0 \\
 0 & 0 & 0 & 0 & 0 & 0 & 0 & 0 & 0 & 0 & 0 & 0 & 0 & 0 & 0 \\
 0 & 0 & 0 & 0 & 0 & 0 & 0 & 0 & 0 & 0 & 0 & 0 & 0 & 0 & 0 \\
 0 & 0 & 0 & 0 & 0 & 0 & 0 & 0 & 0 & 0 & -\Sigmatr^{1/2} & 0 & 0 & 0 & 0 \\
 0 & 0 & 0 & 0 & 0 & 0 & 0 & 0 & 0 & 0 & 0 & 0 & 0 & 0 & 0 \\
 0 & 0 & 0 & 0 & 0 & 0 & 0 & 0 & 0 & 0 & 0 & 0 & -\Sigmatr^{1/2} & 0 & 0 \\
 0 & 0 & 0 & 0 & 0 & 0 & 0 & 0 & 0 & 0 & 0 & 0 & 0 & \frac{\Sigmatr^{1/2}}{\sqrt{\rho }} & 0 \\
 0 & 0 & 0 & 0 & 0 & 0 & 0 & 0 & 0 & 0 & 0 & 0 & 0 & 0 & 0 \\
 0 & 0 & 0 & 0 & 0 & 0 & 0 & 0 & 0 & 0 & 0 & 0 & 0 & 0 & 0 \\
\end{array}
\right)
\,.\end{align}
The operator-valued Stieltjes transforms satisfy,
\begin{equation}
\label{eqn:SCE_E32}
\begin{split}
    \bGE{32} &= \bGE{32}_{\Sigma}(\bar{Z} - \bRE{32}_{W,X,\Theta}(\bGE{32}))\\
    &= \id\otimes\ntr\left(\bar{Z} - \bRE{32}_{W,X,\Theta}(\bGE{32})-\bQE{32}_\Sigma\right)^{-1}\,,
\end{split}
\end{equation}
where $\bRE{32}_{W,X,\Theta}(\bGE{32})$ is the operator-valued R-transform of $\bQE{32}_{W,X,\Theta}$. As discussed above, since $\bQE{32}_{W,X,\Theta}$ is a block matrix whose blocks are i.i.d. Gaussian matrices (and their transposes), an explicit expression for $\bRE{32}_{W,X,\Theta}(\bGE{32})$ can be obtained from the covariance map $\eta$, which can be read off from \cref{eqn:Qbar32}. As above, we use the specific sparsity pattern for $\GE{32}$ that is induced by \cref{eqn:SCE_E32}, to obtain, 
\begin{equation}
    \bRE{32}_{W,X,\Theta}(\bGE{32}) = \left(\begin{array}{cc} 0 & {\RE{32}_{W,X,\Theta}}(\GE{32})^\top\\
\RE{32}_{W,X,\Theta}(\GE{32}) & 0 \end{array}\right)\,,
\end{equation}
where,
\begin{align}
{[\RE{32}_{W,X,\Theta}(\GE{32})]}_{1,1} &=\frac{\GE{32}_{2,2} (\zeta -\eta )}{\gamma }-\frac{\sqrt{\rho } \GE{32}_{6,3}}{\gamma }+\frac{\GE{32}_{2,7} (\zeta -\eta ) \left(\zetas-\etas\right)}{\gamma }\\
{[\RE{32}_{W,X,\Theta}(\GE{32})]}_{1,9} &=\frac{\GE{32}_{7,2} (\zeta -\eta )}{\gamma }-\frac{\sqrt{\rho } \GE{32}_{11,3}}{\gamma }+\frac{\GE{32}_{7,7} (\zeta -\eta ) \left(\zetas-\etas\right)}{\gamma }\\
{[\RE{32}_{W,X,\Theta}(\GE{32})]}_{1,15} &=-\frac{\sqrt{\rho } \GE{32}_{14,3}}{\gamma }\\
{[\RE{32}_{W,X,\Theta}(\GE{32})]}_{2,2} &=\frac{\psi  \GE{32}_{1,1} (\zeta -\eta )}{\gamma  \phi }+\sqrt{\rho } \psi  \GE{32}_{4,5}\\
{[\RE{32}_{W,X,\Theta}(\GE{32})]}_{2,7} &=\frac{\psi  \GE{32}_{9,1} (\zeta -\eta )}{\gamma  \phi }+\sqrt{\rho } \psi  \GE{32}_{8,5}+\sqrt{\rho } \psi  \GE{32}_{13,5}\nonumber\\
&\quad+\frac{\psi  \GE{32}_{1,1} (\zeta -\eta ) \left(\zetas-\etas\right)}{\gamma  \phi }\\
{[\RE{32}_{W,X,\Theta}(\GE{32})]}_{4,5} &=\sqrt{\rho } \GE{32}_{2,2}\\
{[\RE{32}_{W,X,\Theta}(\GE{32})]}_{4,10} &=\sqrt{\rho } \GE{32}_{7,2}\\
{[\RE{32}_{W,X,\Theta}(\GE{32})]}_{6,3} &=-\frac{\sqrt{\rho } \GE{32}_{1,1}}{\gamma  \phi }\\
{[\RE{32}_{W,X,\Theta}(\GE{32})]}_{6,12} &=-\frac{\sqrt{\rho } \GE{32}_{9,1}}{\gamma  \phi }\\
{[\RE{32}_{W,X,\Theta}(\GE{32})]}_{7,2} &=\frac{\psi  \GE{32}_{1,9} (\zeta -\eta )}{\gamma  \phi }+\sqrt{\rho } \psi  \GE{32}_{4,10}\\
{[\RE{32}_{W,X,\Theta}(\GE{32})]}_{7,7} &=\frac{\psi  \GE{32}_{9,9} (\zeta -\eta )}{\gamma  \phi }+\sqrt{\rho } \psi  \GE{32}_{8,10}+\sqrt{\rho } \psi  \GE{32}_{13,10}\nonumber\\
&\quad+\frac{\psi  \GE{32}_{1,9} (\zeta -\eta ) \left(\zetas-\etas\right)}{\gamma  \phi }\\
{[\RE{32}_{W,X,\Theta}(\GE{32})]}_{8,5} &=\sqrt{\rho } \GE{32}_{2,7}\\
{[\RE{32}_{W,X,\Theta}(\GE{32})]}_{8,10} &=\sqrt{\rho } \GE{32}_{7,7}\\
{[\RE{32}_{W,X,\Theta}(\GE{32})]}_{9,1} &=\frac{\GE{32}_{2,7} (\zeta -\eta )}{\gamma }-\frac{\sqrt{\rho } \GE{32}_{6,12}}{\gamma }\\
{[\RE{32}_{W,X,\Theta}(\GE{32})]}_{9,9} &=\frac{\GE{32}_{7,7} (\zeta -\eta )}{\gamma }-\frac{\sqrt{\rho } \GE{32}_{11,12}}{\gamma }\\
{[\RE{32}_{W,X,\Theta}(\GE{32})]}_{9,15} &=-\frac{\sqrt{\rho } \GE{32}_{14,12}}{\gamma }\\
{[\RE{32}_{W,X,\Theta}(\GE{32})]}_{11,3} &=-\frac{\sqrt{\rho } \GE{32}_{1,9}}{\gamma  \phi }\\
{[\RE{32}_{W,X,\Theta}(\GE{32})]}_{11,12} &=-\frac{\sqrt{\rho } \GE{32}_{9,9}}{\gamma  \phi }\\
{[\RE{32}_{W,X,\Theta}(\GE{32})]}_{13,5} &=\sqrt{\rho } \GE{32}_{2,7}\\
{[\RE{32}_{W,X,\Theta}(\GE{32})]}_{13,10} &=\sqrt{\rho } \GE{32}_{7,7}\\
{[\RE{32}_{W,X,\Theta}(\GE{32})]}_{14,3} &=-\frac{\sqrt{\rho } \GE{32}_{1,15}}{\gamma  \phi }\\
{[\RE{32}_{W,X,\Theta}(\GE{32})]}_{14,12} &=-\frac{\sqrt{\rho } \GE{32}_{9,15}}{\gamma  \phi }\,,
\end{align}and the remaining entries of $\RE{32}_{W,X,\Theta}(\GE{32})$ are zero. Similarly, following the example from $\GE{21}$ above, plugging these expressions into \cref{eqn:SCE_E32} and explicitly performing the block-matrix inverse yields the following set of coupled equations,
\begin{align}
\GE{32}_{7,2} &=\gamma ^2 \sqrt{\rho } \taub^2 \psi  \GE{32}_{8,5}+\gamma ^2 \sqrt{\rho } \taub^2 \psi  \GE{32}_{13,5}+\frac{\gamma  \taub^2 \psi  \GE{32}_{9,1} (\zeta -\eta )}{\phi }\nonumber\\
&\quad+\frac{\gamma ^2 \tauone \taub^2 \psi  (\zeta -\eta ) \left(\zetas-\etas\right)}{\phi }\\
\GE{32}_{8,3} &=\I_{\frac{1}{2},1} \left(\zetas-\etas\right)-\frac{\gamma  \taub \Ishift_{\frac{3}{2},1} \rhos}{\psi }\\
\GE{32}_{8,4} &=\gamma  (-\taub) \left(\frac{\Ishift_{1,1} \rhos}{\psi }+\frac{\rho  \tauone \I_{1,1} \left(\zetas-\etas\right)}{\phi }\right)\\
\GE{32}_{8,5} &=\frac{\rho  \tauone \psi  \I_{1,1} \left(\etas-\zetas\right)-\phi  \Ishift_{1,1} \rhos}{\sqrt{\rho } \psi  \phi }\\
\GE{32}_{8,6} &=\frac{\sqrt{\rho } \tauone \left(\gamma  \taub \Ishift_{\frac{3}{2},1} \rhos+\psi  \I_{\frac{1}{2},1} \left(\etas-\zetas\right)\right)}{\psi  \phi }\\
\GE{32}_{9,1} &=\gamma  \tauone^2 \GE{32}_{7,2} (\zeta -\eta )-\gamma  \sqrt{\rho } \tauone^2 \GE{32}_{11,3}+\gamma ^2 \tauone^2 \taub (\zeta -\eta ) \left(\zetas-\etas\right)\\
\GE{32}_{10,3} &=-\frac{\gamma  \sqrt{\rho } \taub \phi  \left(\gamma  \taub \Ishift_{\frac{3}{2},2} \rhos+\psi  \I_{\frac{1}{2},2} \left(\etas-\zetas\right)\right)}{\psi }+\sqrt{\rho } \phi  \GE{32}_{7,2} \I_{\frac{1}{2},2}\nonumber\\
&\quad-\gamma  \rho ^{3/2} \taub^2 \GE{32}_{9,1} \I_{\frac{3}{2},2}\\
\GE{32}_{10,4} &=-\frac{\gamma ^2 \sqrt{\rho } \taub^2 \left(\phi  \Ishift_{1,2} \rhos+\rho  \tauone \psi  \I_{1,2} \left(\zetas-\etas\right)\right)}{\psi }+\sqrt{\rho } \phi  \GE{32}_{7,2} \I_{0,2}\nonumber\\
&\quad-\gamma  \rho ^{3/2} \taub^2 \GE{32}_{9,1} \I_{1,2}\\
\GE{32}_{10,5} &=-\frac{\gamma  \taub \phi  \Ishift_{1,2} \rhos}{\psi }-\rho  \tauone \GE{32}_{7,2} \I_{1,2}-\rho  \taub \GE{32}_{9,1} \I_{1,2}+x \I_{1,2} \left(\etas-\zetas\right)\\
\GE{32}_{10,6} &=\frac{\gamma ^2 \rho  \tauone \taub^2 \Ishift_{\frac{3}{2},2} \rhos}{\psi }-\rho  \tauone \GE{32}_{7,2} \I_{\frac{1}{2},2}-\rho  \taub \GE{32}_{9,1} \I_{\frac{1}{2},2}+x \I_{\frac{1}{2},2} \left(\etas-\zetas\right)\\
\GE{32}_{11,3} &=-\frac{\gamma  \sqrt{\rho } \taub \phi  \left(\gamma  \taub \Ishift_{2,2} \rhos+\psi  \I_{1,2} \left(\etas-\zetas\right)\right)}{\psi }+\sqrt{\rho } \phi  \GE{32}_{7,2} \I_{1,2}-\gamma  \rho ^{3/2} \taub^2 \GE{32}_{9,1} \I_{2,2}\\
\GE{32}_{11,4} &=-\frac{\gamma ^2 \sqrt{\rho } \taub^2 \left(\phi  \Ishift_{\frac{3}{2},2} \rhos+\rho  \tauone \psi  \I_{\frac{3}{2},2} \left(\zetas-\etas\right)\right)}{\psi }+\sqrt{\rho } \phi  \GE{32}_{7,2} \I_{\frac{1}{2},2}\nonumber\\
&\quad-\gamma  \rho ^{3/2} \taub^2 \GE{32}_{9,1} \I_{\frac{3}{2},2}\\
\GE{32}_{11,5} &=-\frac{\gamma  \taub \phi  \Ishift_{\frac{3}{2},2} \rhos}{\psi }-\rho  \tauone \GE{32}_{7,2} \I_{\frac{3}{2},2}-\rho  \taub \GE{32}_{9,1} \I_{\frac{3}{2},2}+x \I_{\frac{3}{2},2} \left(\etas-\zetas\right)\\
\GE{32}_{12,4} &=\frac{\gamma ^2 \rho  \tauone \taub^2 \Ishift_{\frac{3}{2},2} \rhos}{\psi }-\rho  \tauone \GE{32}_{7,2} \I_{\frac{1}{2},2}-\rho  \taub \GE{32}_{9,1} \I_{\frac{1}{2},2}+\frac{x^2 \I_{\frac{3}{2},2} \left(\zetas-\etas\right)}{\phi }\\
\GE{32}_{12,5} &=\frac{\frac{x \Ishift_{\frac{3}{2},2} \rhos}{\psi }+\frac{\gamma  \rho ^2 \tauone^2 \taub \I_{\frac{3}{2},2} \left(\zetas-\etas\right)}{\phi }}{\sqrt{\rho }}-\frac{\sqrt{\rho } \GE{32}_{9,1} \I_{\frac{1}{2},2}}{\gamma }+\frac{\rho ^{3/2} \tauone^2 \GE{32}_{7,2} \I_{\frac{3}{2},2}}{\phi }\\
\GE{32}_{12,6} &=\frac{\gamma  \rho ^2 \tauone^2 \taub \psi  \I_{1,2} \left(\zetas-\etas\right)-x^2 \Ishift_{2,2} \rhos}{\sqrt{\rho } \psi  \phi }-\frac{\sqrt{\rho } \GE{32}_{9,1} \I_{0,2}}{\gamma }+\frac{\rho ^{3/2} \tauone^2 \GE{32}_{7,2} \I_{1,2}}{\phi }\\
\GE{32}_{13,3} &=\frac{\gamma ^2 \rho  \tauone \taub^2 \Ishift_{\frac{5}{2},2} \rhos}{\psi }-\rho  \tauone \GE{32}_{7,2} \I_{\frac{3}{2},2}-\rho  \taub \GE{32}_{9,1} \I_{\frac{3}{2},2}+x \I_{\frac{3}{2},2} \left(\etas-\zetas\right)\\
\GE{32}_{13,4} &=\frac{\gamma ^2 \rho  \tauone \taub^2 \Ishift_{2,2} \rhos}{\psi }-\rho  \tauone \GE{32}_{7,2} \I_{1,2}-\rho  \taub \GE{32}_{9,1} \I_{1,2}+\frac{x^2 \I_{2,2} \left(\zetas-\etas\right)}{\phi }\\
\GE{32}_{13,5} &=\frac{\frac{x \Ishift_{2,2} \rhos}{\psi }+\frac{\gamma  \rho ^2 \tauone^2 \taub \I_{2,2} \left(\zetas-\etas\right)}{\phi }}{\sqrt{\rho }}-\frac{\sqrt{\rho } \GE{32}_{9,1} \I_{1,2}}{\gamma }+\frac{\rho ^{3/2} \tauone^2 \GE{32}_{7,2} \I_{2,2}}{\phi }\\
\GE{32}_{13,6} &=\frac{\gamma  \rho ^2 \tauone^2 \taub \psi  \I_{\frac{3}{2},2} \left(\zetas-\etas\right)-x^2 \Ishift_{\frac{5}{2},2} \rhos}{\sqrt{\rho } \psi  \phi }-\frac{\sqrt{\rho } \GE{32}_{9,1} \I_{\frac{1}{2},2}}{\gamma }+\frac{\rho ^{3/2} \tauone^2 \GE{32}_{7,2} \I_{\frac{3}{2},2}}{\phi }\\
\GE{32}_{13,8} &=-\frac{x \I_{1,1}}{\phi }\\
\GE{32}_{14,3} &=\frac{x \psi  \I_{2,2} \left(\zetas-\etas\right)-\gamma ^2 \rho  \tauone \taub^2 \Ishift_{3,2} \rhos}{\sqrt{\rho } \psi }+\sqrt{\rho } \tauone \GE{32}_{7,2} \I_{2,2}+\sqrt{\rho } \taub \GE{32}_{9,1} \I_{2,2}\\
\GE{32}_{14,4} &=\frac{x^2 \psi  \I_{\frac{5}{2},2} \left(\etas-\zetas\right)-\gamma ^2 \rho  \tauone \taub^2 \phi  \Ishift_{\frac{5}{2},2} \rhos}{\sqrt{\rho } \psi  \phi }+\sqrt{\rho } \tauone \GE{32}_{7,2} \I_{\frac{3}{2},2}+\sqrt{\rho } \taub \GE{32}_{9,1} \I_{\frac{3}{2},2}\\
\GE{32}_{14,5} &=-\frac{x \Ishift_{\frac{5}{2},2} \rhos}{\rho  \psi }+\frac{\GE{32}_{9,1} \I_{\frac{3}{2},2}}{\gamma }-\frac{\rho  \tauone^2 \GE{32}_{7,2} \I_{\frac{5}{2},2}}{\phi }+\frac{\gamma  \rho  \tauone^2 \taub \I_{\frac{5}{2},2} \left(\etas-\zetas\right)}{\phi }\\
\GE{32}_{14,6} &=\frac{\frac{x^2 \Ishift_{3,2} \rhos}{\psi }+\gamma  \rho ^2 \tauone^2 \taub \I_{2,2} \left(\etas-\zetas\right)}{\rho  \phi }+\frac{\GE{32}_{9,1} \I_{1,2}}{\gamma }-\frac{\rho  \tauone^2 \GE{32}_{7,2} \I_{2,2}}{\phi }\\
\GE{32}_{14,8} &=\frac{x \I_{\frac{3}{2},1}}{\sqrt{\rho } \phi }\\
\GE{32}_{14,10} &=\frac{\tauone \I_{\frac{3}{2},1}}{\phi }\\
\GE{32}_{14,11} &=\frac{\tauone \I_{1,1}}{\phi }\\
\GE{32}_{14,12} &=-\frac{\I_{1,1}}{\sqrt{\rho }}\\
\GE{32}_{14,13} &=-\frac{\I_{\frac{1}{2},1}}{\sqrt{\rho }}\\
\GE{32}_{15,1} &=\GE{32}_{14,12} \left(\sqrt{\rho } \tauone^2 \GE{32}_{7,2} (\eta -\zeta )+\rho  \tauone^2 \GE{32}_{11,3}+\gamma  \sqrt{\rho } \tauone^2 (-\taub) (\zeta -\eta ) \left(\zetas-\etas\right)\right)\nonumber\\
&\quad-\sqrt{\rho } \tauone \GE{32}_{14,3}\\
\GE{32}_{15,9} &=-\sqrt{\rho } \tauone \GE{32}_{14,12}\\
\GE{32}_{11,6} &=\GE{32}_{12,3} = \frac{\gamma ^2 \rho  \tauone \taub^2 \Ishift_{2,2} \rhos}{\psi }-\rho  \tauone \GE{32}_{7,2} \I_{1,2}-\rho  \taub \GE{32}_{9,1} \I_{1,2}+x \I_{1,2} \left(\etas-\zetas\right)\\
\GE{32}_{8,8} &=\GE{32}_{14,14} = \GE{32}_{15,15} = 1\,,
\end{align}
Here we have used the definition of the LJSD $\mu$, the relations in Eqs.~(\ref{eqn:QinvE32_first})-(\ref{eqn:QinvE32_last}), the definition of $\Ishift_{a,b}$, and the results in \cref{sec:Kinv} to simplify the expressions. It is straightforward algebra to solve these equations for the undetermined entries of $\GE{32}$ and thereby obtain the following expression for $E_{32}$,
\begin{equation}
E_{32} = \frac{(\etas-\zetas)A_{32} + \rhos B_{32}}{D_{32}}\,,
\end{equation}
where,
\begin{align}
A_{32} &= -\rho ^3 \tauone \psi ^2 x^4 \I_{1,1} \I_{2,2}^2+\rho ^2 \tauone \psi  x^3 \I_{2,2}^2 (\rho  \phi +x \psi  (\zeta -\eta ))-\rho ^3 \tauone \psi ^2 x^3 \phi  \I_{1,1} \I_{1,2} \I_{2,2}\nonumber\\
    &\qquad+\rho ^2 \tauone \psi ^2 x^2 \I_{1,1}^2 (\eta -\zeta )+\rho ^2 \tauone \psi ^2 x^2 \I_{1,1} \I_{2,2} (\rho +x (\zeta -\eta ))\nonumber\\
    &\qquad+\rho ^2 \tauone \psi  x^2 \phi  \I_{1,2} \I_{2,2} (\rho  \phi +x \psi  (\zeta -\eta ))+\rho ^3 \tauone \psi ^2 x^2 \phi  \I_{1,1}^2 \I_{1,2}\nonumber\\
    &\qquad-\rho ^2 \tauone \psi  x \phi  \I_{1,1} \I_{1,2} (\rho  \phi +x \psi  (\zeta -\eta ))+\rho  \tauone \psi  x \I_{1,1} (\zeta -\eta ) (\rho  \phi +x \psi  (\zeta -\eta ))\nonumber\\
    &\qquad-\rho  \tauone \psi  x \I_{2,2} (\rho +x (\zeta -\eta )) (\rho  \phi +x \psi  (\zeta -\eta ))\\
B_{32} &= -\rho ^2 \psi  x^6 \Ishift_{3,2} \I_{2,2}^2-2 \rho ^2 \psi  x^5 \phi  \Ishift_{2,2} \I_{2,2}^2+2 \rho  \psi  x^4 \phi  \Ishift_{3,2} \I_{1,2} (\eta -\zeta )+\rho ^2 \psi  x^4 \phi ^2 \Ishift_{3,2} \I_{1,2}^2\nonumber\\
    &\qquad-2 \rho ^2 \psi  x^4 \phi ^2 \Ishift_{2,2} \I_{1,2} \I_{2,2}+\rho ^2 \psi  x^4 \phi  \Ishift_{1,1} \I_{2,2}^2+\rho ^2 x^4 \Ishift_{3,2} \I_{2,2} (\psi +\phi )\nonumber\\
    &\qquad+\rho ^2 \psi  x^4 \phi  \Ishift_{3,2} \I_{1,1} \I_{2,2}+\rho  x^3 \phi  \Ishift_{2,2} \I_{2,2} (\rho  (\psi +\phi )+2 x \psi  (\zeta -\eta ))\nonumber\\
    &\qquad+\rho ^2 \psi  x^3 \phi ^2 \Ishift_{2,2} \I_{1,1} \I_{1,2}+\rho ^2 \psi  x^3 \phi ^2 \Ishift_{1,1} \I_{1,2} \I_{2,2}+\rho  \psi  x^2 \phi  \Ishift_{1,1} \I_{1,1} (\zeta -\eta )\nonumber\\
    &\qquad-\rho  x^2 \phi  \Ishift_{2,2} \I_{1,1} (\rho  \phi +x \psi  (\zeta -\eta ))-\rho  \psi  x^2 \phi  \Ishift_{1,1} \I_{2,2} (\rho +x (\zeta -\eta ))\nonumber\\
    &\qquad-\rho ^2 \psi  x^2 \phi ^2 \Ishift_{1,1} \I_{1,1} \I_{1,2}+\Ishift_{3,2} \left(x^4 \psi  (\zeta -\eta )^2-\rho ^2 x^2 \phi \right)\\
D_{32} &= -\rho ^3 \psi  x^4 \phi  \I_{2,2}^2+2 \rho ^2 \psi  x^2 \phi ^2 \I_{1,2} (\eta -\zeta )+\rho ^3 \psi  x^2 \phi ^3 \I_{1,2}^2+\rho ^3 x^2 \phi  \I_{2,2} (\psi +\phi )\nonumber\\
    &\qquad+\rho  \phi  \left(x^2 \psi  (\zeta -\eta )^2-\rho ^2 \phi \right)\,.
\end{align}
Further simplifications are possible using the raising and lowering identities in \cref{eq:raise_lower}, as well as the results in  \cref{sec:Kinv}, to obtain,
\al{
\label{eqn:e32_result}
E_{32} &= -\xi^2 \frac{\partial x}{\partial \gamma} \bigg((\phi \I_{1,2}+\omega)\Big(\frac{\Ishift_{1,1}}{\tauone}+\frac{\psi\rho}{\phi} \I_{1,1}(\omegas+\Ishift_{1,1})\Big)+ \frac{\omegas\psi x}{\phi\taub}\I_{2,2} + \frac{\phi}{x\taub}\Ishift_{1,2} -\frac{\omega x}{\tauone}\Ishift_{2,2}\bigg)\,,
}
where
\begin{equation}
    \frac{\partial x}{\partial \gamma} = -\frac{x}{\gamma+\rho \gamma (\tauone \psi/\phi + \taub)(\omega + \phi \I_{1,2})}\,.
\end{equation}
\subsubsection{$E_{4}$}
Define the block matrix $\QE{4}\equiv [\QE{4}_1\; \QE{4}_2]$ by,
\begin{equation}
 \QE{4}_1 = \left(
\begin{array}{ccccccc}
 I_m & \frac{\sqrt{\eta -\zeta } \Theta_{22}^\top}{\gamma  \sqrt{n_1}} & 0 & 0 & \frac{\sqrt{\rho } X_2^\top}{\gamma  \sqrt{n_0}} & 0 & 0 \\
 -\frac{\Theta_{22} \sqrt{\eta -\zeta }}{\sqrt{n_1}} & I_{n_1} & -\frac{\sqrt{\rho } W_2}{\sqrt{n_1}} & 0 & 0 & 0 & 0 \\
 0 & 0 & I_{n_0} & -\Sigmatr^{1/2} & 0 & 0 & 0 \\
 -\frac{X_2}{\sqrt{n_0}} & 0 & 0 & I_{n_0} & 0 & 0 & 0 \\
 0 & 0 & 0 & 0 & I_{n_0} & -\Sigmatr^{1/2} & 0 \\
 0 & -\frac{W_2^\top}{\sqrt{n_1}} & 0 & 0 & 0 & I_{n_0} & \frac{\Sigmatr^{1/2}}{\sqrt{\rho }} \\
 0 & 0 & 0 & 0 & 0 & 0 & I_{n_0} \\
 0 & 0 & 0 & 0 & 0 & 0 & 0 \\
 0 & 0 & 0 & 0 & 0 & 0 & 0 \\
 0 & 0 & 0 & 0 & 0 & 0 & 0 \\
 0 & 0 & 0 & 0 & 0 & 0 & 0 \\
 0 & 0 & 0 & 0 & 0 & 0 & -\Sigmatr^{1/2} \\
 0 & 0 & 0 & 0 & 0 & 0 & 0 \\
 0 & 0 & 0 & 0 & 0 & 0 & 0 \\
\end{array}
\right)
\,,\end{equation}and,
\begin{equation}
 \QE{4}_2 = \left(
\begin{array}{ccccccc}
 0 & 0 & 0 & 0 & 0 & 0 & 0 \\
 0 & 0 & 0 & 0 & 0 & 0 & 0 \\
 0 & 0 & 0 & 0 & 0 & 0 & 0 \\
 0 & 0 & 0 & 0 & 0 & 0 & 0 \\
 0 & 0 & 0 & 0 & 0 & 0 & 0 \\
 0 & 0 & 0 & 0 & 0 & 0 & 0 \\
 -\frac{X_1}{\sqrt{n_0}} & 0 & 0 & 0 & 0 & 0 & 0 \\
 I_m & \frac{\sqrt{\eta -\zeta } \Theta_{11}^\top}{\gamma  \sqrt{n_1}} & \frac{\sqrt{\rho } X_1^\top}{\gamma  \sqrt{n_0}} & 0 & 0 & 0 & 0 \\
 -\frac{\Theta_{11} \sqrt{\eta -\zeta }}{\sqrt{n_1}} & I_{n_1} & 0 & 0 & -\frac{\sqrt{\rho } W_1}{\sqrt{n_1}} & 0 & 0 \\
 0 & 0 & I_{n_0} & -\Sigmatr^{1/2} & 0 & 0 & 0 \\
 0 & -\frac{W_1^\top}{\sqrt{n_1}} & 0 & I_{n_0} & 0 & 0 & 0 \\
 0 & 0 & 0 & 0 & I_{n_0} & \frac{\Sigmate}{\sqrt{\rho }} & 0 \\
 0 & 0 & 0 & 0 & 0 & I_{n_0} & -\frac{W_2^\top}{\sqrt{n_1}} \\
 0 & 0 & 0 & 0 & 0 & 0 & I_{n_1} \\
\end{array}
\right)
\,.\end{equation}
Then block matrix inversion (i.e. repeated applications of the Schur complement formula) shows that,
\begin{eqnarray}
\label{eqn:QinvE4_first}
\GE{4}_{13,13} &=& \GE{4}_{14,14} = 1\\
\GE{4}_{1,1} &=& \GE{4}_{8,8} = \GK_{1,1}\\
\GE{4}_{2,2} &=& \GE{4}_{9,9} = \GK_{2,2}\\
\GE{4}_{3,3} &=& \GE{4}_{6,6} = \GE{4}_{11,11} = \GE{4}_{12,12} = \GE{4}_{4,4} = \GE{4}_{5,5} = \GE{4}_{7,7} = \GE{4}_{10,10} = \GK_{3,3}\\
\GE{4}_{3,4} &=& \GE{4}_{5,6} = \GE{4}_{10,11} = \GE{4}_{12,7} = \GK_{3,4}\\
\GE{4}_{5,3} &=& \GE{4}_{6,4} = \GE{4}_{10,12} = \GE{4}_{11,7} = \GK_{3,5}\\
\GE{4}_{5,4} &=& \GE{4}_{10,7} = \GK_{3,6}\\
\GE{4}_{4,3} &=& \GE{4}_{6,5} = \GE{4}_{7,12} = \GE{4}_{11,10} = \GK_{4,3}\\
\GE{4}_{6,3} &=& \GE{4}_{11,12} = \GK_{4,5}\\
\GE{4}_{3,5} &=& \GE{4}_{4,6} = \GE{4}_{7,11} = \GE{4}_{12,10} = \GK_{5,3}\\
\GE{4}_{3,6} &=& \GE{4}_{12,11} = \GK_{5,4}\\
\GE{4}_{4,5} &=& \GE{4}_{7,10} = \GK_{6,3}\\
\GE{4}_{14,2} &=& \frac{\psi}{\rhos}E_4\label{eqn:QinvE4_last}\,,
\end{eqnarray}
where $\GE{4}_{i,j}$ denotes the normalized trace of the $(i,j)$-block of the inverse of $\QE{4}$.
For brevity, we have suppressed the expressions for the other non-zero blocks.

To compute the limiting values of these traces, we require the asymptotic block-wise traces of $\QE{4}$, which may be determined from the operator-valued Stieltjes transform. To proceed, we first augment $\QE{4}$ to form the the self-adjoint matrix $\bQE{4}$,
\begin{equation}
\label{eqn:Qbar4}
    \bQE{4} = \left(\begin{array}{cc} 0 & {[\QE{4}]}^\top\\
\QE{4} & 0 \end{array}\right)\,.\\
\end{equation}
and observe that we can write $\bQE{4}$ as,
\begin{equation}
\begin{split}
\bQE{4} &= \bar{Z} + \bQE{4}_{W,X,\Theta} + \bQE{4}_\Sigma\\
&= \left(\begin{array}{cc} 0 & Z^\top\\
Z & 0 \end{array}\right) + \left(\begin{array}{cc} 0 & {[\QE{4}_{W,X,\Theta}]}^\top\\
\QE{4}_{W,X,\Theta} & 0 \end{array}\right) + \left(\begin{array}{cc} 0 & {[\QE{4}_\Sigma]}^\top\\
\QE{4}_{\Sigma} & 0 \end{array}\right)\,,
\end{split}
\end{equation}
where $Z = I_{2 m+9 n_0+3 n_1}$, $\QE{4}_{W,X,\Theta} \equiv [[\QE{4}_{W,X,\Theta}]_1\;[\QE{4}_{W,X,\Theta}]_2]$ and,
\begin{align}
{[\QE{4}_{W,X,\Theta}]}_1 &= \scriptsize\left(
\begin{array}{ccccccc}
 I_m & \frac{\sqrt{\eta -\zeta } \Theta_{22}^\top}{\gamma  \sqrt{n_1}} & 0 & 0 & \frac{\sqrt{\rho } X_2^\top}{\gamma  \sqrt{n_0}} & 0 & 0 \\
 -\frac{\Theta_{22} \sqrt{\eta -\zeta }}{\sqrt{n_1}} & I_{n_1} & -\frac{\sqrt{\rho } W_2}{\sqrt{n_1}} & 0 & 0 & 0 & 0 \\
 0 & 0 & I_{n_0} & 0 & 0 & 0 & 0 \\
 -\frac{X_2}{\sqrt{n_0}} & 0 & 0 & I_{n_0} & 0 & 0 & 0 \\
 0 & 0 & 0 & 0 & I_{n_0} & 0 & 0 \\
 0 & -\frac{W_2^\top}{\sqrt{n_1}} & 0 & 0 & 0 & I_{n_0} & 0 \\
 0 & 0 & 0 & 0 & 0 & 0 & I_{n_0} \\
 0 & 0 & 0 & 0 & 0 & 0 & 0 \\
 0 & 0 & 0 & 0 & 0 & 0 & 0 \\
 0 & 0 & 0 & 0 & 0 & 0 & 0 \\
 0 & 0 & 0 & 0 & 0 & 0 & 0 \\
 0 & 0 & 0 & 0 & 0 & 0 & 0 \\
 0 & 0 & 0 & 0 & 0 & 0 & 0 \\
 0 & 0 & 0 & 0 & 0 & 0 & 0 \\
\end{array}
\right)
\\
{[\QE{4}_{W,X,\Theta}]}_2 &= \scriptsize\left(
\begin{array}{ccccccc}
 0 & 0 & 0 & 0 & 0 & 0 & 0 \\
 0 & 0 & 0 & 0 & 0 & 0 & 0 \\
 0 & 0 & 0 & 0 & 0 & 0 & 0 \\
 0 & 0 & 0 & 0 & 0 & 0 & 0 \\
 0 & 0 & 0 & 0 & 0 & 0 & 0 \\
 0 & 0 & 0 & 0 & 0 & 0 & 0 \\
 -\frac{X_1}{\sqrt{n_0}} & 0 & 0 & 0 & 0 & 0 & 0 \\
 I_m & \frac{\sqrt{\eta -\zeta } \Theta_{11}^\top}{\gamma  \sqrt{n_1}} & \frac{\sqrt{\rho } X_1^\top}{\gamma  \sqrt{n_0}} & 0 & 0 & 0 & 0 \\
 -\frac{\Theta_{11} \sqrt{\eta -\zeta }}{\sqrt{n_1}} & I_{n_1} & 0 & 0 & -\frac{\sqrt{\rho } W_1}{\sqrt{n_1}} & 0 & 0 \\
 0 & 0 & I_{n_0} & 0 & 0 & 0 & 0 \\
 0 & -\frac{W_1^\top}{\sqrt{n_1}} & 0 & I_{n_0} & 0 & 0 & 0 \\
 0 & 0 & 0 & 0 & I_{n_0} & 0 & 0 \\
 0 & 0 & 0 & 0 & 0 & I_{n_0} & -\frac{W_2^\top}{\sqrt{n_1}} \\
 0 & 0 & 0 & 0 & 0 & 0 & I_{n_1} \\
\end{array}
\right)
\\
\QE{4}_{\Sigmatr} &= \scriptsize\left(
\begin{array}{cccccccccccccc}
 0 & 0 & 0 & 0 & 0 & 0 & 0 & 0 & 0 & 0 & 0 & 0 & 0 & 0 \\
 0 & 0 & 0 & 0 & 0 & 0 & 0 & 0 & 0 & 0 & 0 & 0 & 0 & 0 \\
 0 & 0 & 0 & -\Sigmatr^{1/2} & 0 & 0 & 0 & 0 & 0 & 0 & 0 & 0 & 0 & 0 \\
 0 & 0 & 0 & 0 & 0 & 0 & 0 & 0 & 0 & 0 & 0 & 0 & 0 & 0 \\
 0 & 0 & 0 & 0 & 0 & -\Sigmatr^{1/2} & 0 & 0 & 0 & 0 & 0 & 0 & 0 & 0 \\
 0 & 0 & 0 & 0 & 0 & 0 & \frac{\Sigmatr^{1/2}}{\sqrt{\rho }} & 0 & 0 & 0 & 0 & 0 & 0 & 0 \\
 0 & 0 & 0 & 0 & 0 & 0 & 0 & 0 & 0 & 0 & 0 & 0 & 0 & 0 \\
 0 & 0 & 0 & 0 & 0 & 0 & 0 & 0 & 0 & 0 & 0 & 0 & 0 & 0 \\
 0 & 0 & 0 & 0 & 0 & 0 & 0 & 0 & 0 & 0 & 0 & 0 & 0 & 0 \\
 0 & 0 & 0 & 0 & 0 & 0 & 0 & 0 & 0 & 0 & -\Sigmatr^{1/2} & 0 & 0 & 0 \\
 0 & 0 & 0 & 0 & 0 & 0 & 0 & 0 & 0 & 0 & 0 & 0 & 0 & 0 \\
 0 & 0 & 0 & 0 & 0 & 0 & -\Sigmatr^{1/2} & 0 & 0 & 0 & 0 & 0 & \frac{\Sigmate}{\sqrt{\rho }} & 0 \\
 0 & 0 & 0 & 0 & 0 & 0 & 0 & 0 & 0 & 0 & 0 & 0 & 0 & 0 \\
 0 & 0 & 0 & 0 & 0 & 0 & 0 & 0 & 0 & 0 & 0 & 0 & 0 & 0 \\
\end{array}
\right)
\,.\end{align}
The operator-valued Stieltjes transforms satisfy,
\begin{equation}
\label{eqn:SCE_E4}
\begin{split}
    \bGE{4} &= \bGE{4}_{\Sigma}(\bar{Z} - \bRE{4}_{W,X,\Theta}(\bGE{4}))\\
    &= \id\otimes\ntr\left(\bar{Z} - \bRE{4}_{W,X,\Theta}(\bGE{4})-\bQE{4}_\Sigma\right)^{-1}\,,
\end{split}
\end{equation}
where $\bRE{4}_{W,X,\Theta}(\bGE{4})$ is the operator-valued R-transform of $\bQE{4}_{W,X,\Theta}$. As discussed above, since $\bQE{4}_{W,X,\Theta}$ is a block matrix whose blocks are i.i.d. Gaussian matrices (and their transposes), an explicit expression for $\bRE{4}_{W,X,\Theta}(\bGE{4})$ can be obtained from the covariance map $\eta$, which can be read off from \cref{eqn:Qbar4}. As above, we use the specific sparsity pattern for $\GE{4}$ that is induced by \cref{eqn:SCE_E4}, to obtain,
\begin{equation}
    \bRE{4}_{W,X,\Theta}(\bGE{4}) = \left(\begin{array}{cc} 0 & {\RE{4}_{W,X,\Theta}}(\GE{4})^\top\\
\RE{4}_{W,X,\Theta}(\GE{4}) & 0 \end{array}\right)\,,
\end{equation}
where,
\begin{align}
{[\RE{4}_{W,X,\Theta}(\GE{4})]}_{1,1} &=\frac{\GE{4}_{2,2} (\zeta -\eta )}{\gamma }-\frac{\sqrt{\rho } \GE{4}_{4,5}}{\gamma }\\
{[\RE{4}_{W,X,\Theta}(\GE{4})]}_{2,2} &=\frac{\psi  \GE{4}_{1,1} (\zeta -\eta )}{\gamma  \phi }+\sqrt{\rho } \psi  \GE{4}_{6,3}\\
{[\RE{4}_{W,X,\Theta}(\GE{4})]}_{2,14} &=\sqrt{\rho } \psi  \GE{4}_{13,3}\\
{[\RE{4}_{W,X,\Theta}(\GE{4})]}_{4,5} &=-\frac{\sqrt{\rho } \GE{4}_{1,1}}{\gamma  \phi }\\
{[\RE{4}_{W,X,\Theta}(\GE{4})]}_{6,3} &=\sqrt{\rho } \GE{4}_{2,2}\\
{[\RE{4}_{W,X,\Theta}(\GE{4})]}_{7,10} &=-\frac{\sqrt{\rho } \GE{4}_{8,8}}{\gamma  \phi }\\
{[\RE{4}_{W,X,\Theta}(\GE{4})]}_{8,8} &=\frac{\GE{4}_{9,9} (\zeta -\eta )}{\gamma }-\frac{\sqrt{\rho } \GE{4}_{7,10}}{\gamma }\\
{[\RE{4}_{W,X,\Theta}(\GE{4})]}_{9,9} &=\frac{\psi  \GE{4}_{8,8} (\zeta -\eta )}{\gamma  \phi }+\sqrt{\rho } \psi  \GE{4}_{11,12}\\
{[\RE{4}_{W,X,\Theta}(\GE{4})]}_{11,12} &=\sqrt{\rho } \GE{4}_{9,9}\\
{[\RE{4}_{W,X,\Theta}(\GE{4})]}_{13,3} &=\sqrt{\rho } \GE{4}_{2,14}\,,
\end{align}and the remaining entries of $\RE{4}_{W,X,\Theta}(\GE{4})$ are zero. Similarly, following the example from $\GE{21}$ above, plugging these expressions into \cref{eqn:SCE_E4} and explicitly performing the block-matrix inverse yields the following set of coupled equations,
\begin{align}
\GE{4}_{7,5} &=-\frac{\phi  \I_{1,2}}{\sqrt{\rho }}\\
\GE{4}_{7,6} &=-\frac{\phi  \I_{\frac{1}{2},2}}{\sqrt{\rho }}\\
\GE{4}_{10,4} &=-\frac{\sqrt{\rho } \tauone^2 \I_{1,2}}{\phi }\\
\GE{4}_{10,6} &=\tauone \I_{\frac{1}{2},2}\\
\GE{4}_{11,3} &=-\frac{\sqrt{\rho } \tauone^2 \I_{2,2}}{\phi }\\
\GE{4}_{12,3} &=-\frac{\gamma  \rho  \tauone^2 \taub \I_{2,2}}{\phi }\\
\GE{4}_{12,4} &=-\frac{\gamma  \rho  \tauone^2 \taub \I_{\frac{3}{2},2}}{\phi }\\
\GE{4}_{12,5} &=\frac{x \I_{\frac{3}{2},2}}{\sqrt{\rho }}\\
\GE{4}_{12,6} &=\frac{x \I_{1,2}}{\sqrt{\rho }}\\
\GE{4}_{13,3} &=\frac{\gamma  \sqrt{\rho } \tauone^2 \taub \Ishift_{3,2}}{\phi }\\
\GE{4}_{13,4} &=\frac{\gamma  \sqrt{\rho } \tauone^2 \taub \Ishift_{\frac{5}{2},2}}{\phi }\\
\GE{4}_{13,5} &=-\frac{x \Ishift_{\frac{5}{2},2}}{\rho }\\
\GE{4}_{13,6} &=-\frac{x \Ishift_{2,2}}{\rho }\\
\GE{4}_{13,7} &=\frac{x \Ishift_{\frac{3}{2},1}}{\sqrt{\rho } \phi }\\
\GE{4}_{13,10} &=\gamma  (-\taub) \Ishift_{\frac{3}{2},1}\\
\GE{4}_{13,11} &=\gamma  (-\taub) \Ishift_{1,1}\\
\GE{4}_{13,12} &=-\frac{\Ishift_{1,1}}{\sqrt{\rho }}\\
\GE{4}_{14,2} &=\gamma  \sqrt{\rho } \taub \psi  \GE{4}_{13,3}\\
\GE{4}_{7,3} &=\GE{4}_{11,5} = \tauone \I_{\frac{3}{2},2}\\
\GE{4}_{10,3} &=\GE{4}_{11,4} = -\frac{\sqrt{\rho } \tauone^2 \I_{\frac{3}{2},2}}{\phi }\\
\GE{4}_{13,13} &=\GE{4}_{14,14} = 1\\
\GE{4}_{7,4} &=\GE{4}_{10,5} = \GE{4}_{11,6} = \tauone \I_{1,2}\,,
\end{align}
Here we have used the definition of the LJSD $\mu$, the relations in Eqs.~(\ref{eqn:QinvE4_first})-(\ref{eqn:QinvE4_last}), the definition of $\Ishift_{a,b}$, and the results in \cref{sec:Kinv}, to simplify the expressions. It is straightforward algebra to solve these equations for the undetermined entries of $\GE{4}$ and thereby obtain the following expression for $E_{4}$,
\begin{equation}
\label{eqn:e4_result}
E_{4} = \xi^2 \frac{x^2}{\phi}\Ishift_{3,2}\,.
\end{equation}
\subsubsection{Results for total error, bias, and variance}
General expressions for the total error, bias and variance in terms of  $E_{21}, E_{31}, E_{32}, E_{4}$ can be found in \cref{app:summary_lin}.
Combining the results from \cref{eqn:err_expr,eqn:e21_result,eqn:e31_result,eqn:e32_result} yields expressions for the bias and total error. Using $\err{\mu} = \bias{\mu} + \var{\mu}$, these results also determine the variance. Putting all the pieces together we find,
\al{
    \bias{\mu} &= (1-\xi)^2 s_{*} + 2(1-\xi)\xi \Ishift_{1,1} + \xi^2 \phi \Ishift_{1,2}\\
    \var{\mu} &= -\rho\xi^2 \frac{\psi}{\phi}\frac{\partial x}{\partial \gamma} \bigg(\I_{1,1}(\omega + \phi \I_{1,2})(\omegas+\Ishift_{1,1}) +\frac{\phi^2}{\psi}\gamma\taub \I_{1,2}\Ishift_{2,2}+\gamma\tauone\I_{2,2}(\omegas+\phi\Ishift_{1,2})\nonumber\\
&\qquad\qquad\quad\quad + \sigma_\e^2\Big((\omega + \phi \I_{1,2})(\omegas+\Ishift_{1,1}) +\frac{\phi}{\psi}\gamma\taub \Ishift_{2,2}\Big)\bigg)\\
}
where $x$ is the unique positive real root of $x = \tfrac{1-\gamma\tauone}{\omega + \I_{1,1}}$, as in \cref{eqn:x}. The derivative $\tfrac{\partial x}{\partial \gamma}$ follows from implicit differentiation and was given in \cref{eq:dxdgamma},
\begin{equation}
   \frac{\partial x}{\partial \gamma} = -\frac{x}{\gamma+\rho \gamma (\tauone \psi/\phi + \taub)(\omega + \phi \I_{1,2})}\,.
\end{equation}
The asymptotic trace objects $\tauone$ and $\taub$ were defined in \cref{eqn:tau1} and \cref{eqn:tau1b} and are given by
\eq{
    \tauone = \frac{\sqrt{(\psi -\phi )^2+ 4 x \psi\phi  \gamma/\rho}+\psi -\phi }{2 \psi \gamma}\quad \text{and}\quad \taub = \frac{1}{\gamma} + \frac{\psi}{\phi}\big(\tauone - \frac{1}{\gamma}\big)\,\label{eq:tau_taub}.
}
All together, these results prove \cref{thm:main_b_v}.
\subsection{Proofs of \texorpdfstring{\cref{cor:ridgeless_lim}}{} and \texorpdfstring{\cref{cor:infinite_overparameterization}}{}}
\label{sec:cor_proofs}
\cref{cor:ridgeless_lim} follows immediately from ~\cref{thm:main_b_v} by taking $\gamma\to 0$ after substituting the following expressions, which are direct consequences of \cref{eq:dxdgamma} and \cref{eq:tau_taub}:
\al{
\tauone & \,\stackrel{\gamma \to 0}{\longrightarrow}\, \frac{x \phi}{\rho |\phi-\psi|} + \begin{cases}\frac{|\phi-\psi|}{\psi \gamma} & \text{if } \phi < \psi \\ 0 & \text{otherwise}\end{cases}\,,\\
\taub &\,\stackrel{\gamma \to 0}{\longrightarrow}\, \frac{x \psi}{\rho |\phi-\psi|} + \begin{cases} 0 & \text{if } \phi < \psi \\ \frac{|\phi-\psi|}{\phi \gamma} & \text{otherwise}\end{cases}\,,
}
so that,
\al{
\frac{\partial x}{\partial \gamma}  &\,\stackrel{\gamma \to 0}{\longrightarrow}\, - \frac{x \phi}{\rho |\phi - \psi| (\omega + \phi \I_{1,2})}\,,
}
where $x$ is the unique positive real root of
\al{
x = \frac{\min(1,\phi/\psi)}{\omega+\I_{1,1}}\,.
}
Similarly, \cref{cor:infinite_overparameterization} follows immediately from ~\cref{thm:main_b_v} by taking $\psi\to 0$ after substituting the following expressions, which are also direct consequences of \cref{eq:dxdgamma} and \cref{eq:tau_taub}:
\al{
\tauone \,\stackrel{\psi \to 0}{\longrightarrow}\, \frac{x}{\rho}\quad\text{and}\quad \taub \,\stackrel{\psi \to 0}{\longrightarrow}\, \frac{1}{\gamma}\,,
}
so that,
\al{
\frac{\partial x}{\partial \gamma}  &\,\stackrel{\psi \to 0}{\longrightarrow}\, - \frac{x/\rho}{\gammaeff + \phi \I_{1,2}}\label{eq:xprime_inf_overparam}\,,
}
where $\gammaeff = \gamma/\rho + \omega$ and $x$ is the unique positive real root of
\al{
x = \frac{1}{\gammaeff+\I_{1,1}}\,.
}
\bibliography{main}

\begin{thebibliography}{76}
\providecommand{\natexlab}[1]{#1}
\providecommand{\url}[1]{\texttt{#1}}
\expandafter\ifx\csname urlstyle\endcsname\relax
  \providecommand{\doi}[1]{doi: #1}\else
  \providecommand{\doi}{doi: \begingroup \urlstyle{rm}\Url}\fi

\bibitem[Adlam and Pennington(2020{\natexlab{a}})]{adlam2020neural}
B.~Adlam and J.~Pennington.
\newblock The neural tangent kernel in high dimensions: Triple descent and a
  multi-scale theory of generalization.
\newblock In \emph{Proceedings of the 37th International Conference on Machine
  Learning, {ICML} 2020, 13-18 July 2020, Virtual Event}, volume 119 of
  \emph{Proceedings of Machine Learning Research}, pages 74--84. {PMLR},
  2020{\natexlab{a}}.
\newblock URL \url{http://proceedings.mlr.press/v119/adlam20a.html}.

\bibitem[Adlam and Pennington(2020{\natexlab{b}})]{adlam2020understanding}
B.~Adlam and J.~Pennington.
\newblock Understanding double descent requires a fine-grained bias-variance
  decomposition.
\newblock In H.~Larochelle, M.~Ranzato, R.~Hadsell, M.~F. Balcan, and H.~Lin,
  editors, \emph{Advances in Neural Information Processing Systems}, volume~33,
  pages 11022--11032. Curran Associates, Inc., 2020{\natexlab{b}}.
\newblock URL
  \url{https://proceedings.neurips.cc/paper/2020/file/7d420e2b2939762031eed0447a9be19f-Paper.pdf}.

\bibitem[Adlam et~al.(2019)Adlam, Levinson, and Pennington]{adlam2019random}
B.~Adlam, J.~Levinson, and J.~Pennington.
\newblock A random matrix perspective on mixtures of nonlinearities for deep
  learning.
\newblock \emph{arXiv preprint arXiv:1912.00827}, 2019.

\bibitem[Advani et~al.(2020)Advani, Saxe, and Sompolinsky]{advani2020high}
M.~S. Advani, A.~M. Saxe, and H.~Sompolinsky.
\newblock High-dimensional dynamics of generalization error in neural networks.
\newblock \emph{Neural Networks}, 132:\penalty0 428--446, 2020.
\newblock \doi{10.1016/j.neunet.2020.08.022}.
\newblock URL \url{https://doi.org/10.1016/j.neunet.2020.08.022}.

\bibitem[Banna et~al.(2015)Banna, Merlevède, and
  Peligrad]{merlevede2013limiting}
M.~Banna, F.~Merlevède, and M.~Peligrad.
\newblock On the limiting spectral distribution for a large class of symmetric
  random matrices with correlated entries.
\newblock \emph{Stochastic Processes and their Applications}, 125\penalty0
  (7):\penalty0 2700--2726, 2015.
\newblock ISSN 0304-4149.
\newblock \doi{https://doi.org/10.1016/j.spa.2015.01.010}.
\newblock URL
  \url{https://www.sciencedirect.com/science/article/pii/S0304414915000290}.

\bibitem[Banna et~al.(2020)Banna, Najim, and Yao]{banna2020clt}
M.~Banna, J.~Najim, and J.~Yao.
\newblock A {CLT} for linear spectral statistics of large random
  information-plus-noise matrices.
\newblock \emph{Stochastic Processes and their Applications}, 130\penalty0
  (4):\penalty0 2250--2281, 2020.

\bibitem[Bartlett et~al.(2021)Bartlett, Montanari, and
  Rakhlin]{bartlett2021deep}
P.~L. Bartlett, A.~Montanari, and A.~Rakhlin.
\newblock Deep learning: a statistical viewpoint, 2021.

\bibitem[Becker et~al.(2013)Becker, Christoudias, and Fua]{becker2013non}
C.~J. Becker, C.~M. Christoudias, and P.~Fua.
\newblock Non-linear domain adaptation with boosting.
\newblock In C.~J.~C. Burges, L.~Bottou, M.~Welling, Z.~Ghahramani, and K.~Q.
  Weinberger, editors, \emph{Advances in Neural Information Processing
  Systems}, volume~26. Curran Associates, Inc., 2013.
\newblock URL
  \url{https://proceedings.neurips.cc/paper/2013/file/c042f4db68f23406c6cecf84a7ebb0fe-Paper.pdf}.

\bibitem[Belkin et~al.(2019)Belkin, Hsu, Ma, and Mandal]{belkin2019reconciling}
M.~Belkin, D.~Hsu, S.~Ma, and S.~Mandal.
\newblock Reconciling modern machine-learning practice and the classical
  bias--variance trade-off.
\newblock \emph{Proceedings of the National Academy of Sciences}, 116\penalty0
  (32):\penalty0 15849--15854, 2019.

\bibitem[Ben-David et~al.(2007)Ben-David, Blitzer, Crammer, and
  Pereira]{ben2007analysis}
S.~Ben-David, J.~Blitzer, K.~Crammer, and F.~Pereira.
\newblock Analysis of representations for domain adaptation.
\newblock In B.~Sch\"{o}lkopf, J.~Platt, and T.~Hoffman, editors,
  \emph{Advances in Neural Information Processing Systems}, volume~19. MIT
  Press, 2007.
\newblock URL
  \url{https://proceedings.neurips.cc/paper/2006/file/b1b0432ceafb0ce714426e9114852ac7-Paper.pdf}.

\bibitem[Chang et~al.(2021)Chang, Li, Oymak, and
  Thrampoulidis]{Chang2021ProvableBO}
X.~Chang, Y.~Li, S.~Oymak, and C.~Thrampoulidis.
\newblock Provable benefits of overparameterization in model compression: From
  double descent to pruning neural networks.
\newblock In \emph{AAAI}, 2021.

\bibitem[Clevert et~al.(2016)Clevert, Unterthiner, and
  Hochreiter]{clevert2015fast}
D.~Clevert, T.~Unterthiner, and S.~Hochreiter.
\newblock Fast and accurate deep network learning by exponential linear units
  (elus).
\newblock In Y.~Bengio and Y.~LeCun, editors, \emph{4th International
  Conference on Learning Representations, {ICLR} 2016, San Juan, Puerto Rico,
  May 2-4, 2016, Conference Track Proceedings}, 2016.
\newblock URL \url{http://arxiv.org/abs/1511.07289}.

\bibitem[Cortes and Mohri(2014)]{cortes2014domain}
C.~Cortes and M.~Mohri.
\newblock Domain adaptation and sample bias correction theory and algorithm for
  regression.
\newblock \emph{Theor. Comput. Sci.}, 519:\penalty0 103--126, 2014.
\newblock \doi{10.1016/j.tcs.2013.09.027}.
\newblock URL \url{https://doi.org/10.1016/j.tcs.2013.09.027}.

\bibitem[D'Amour et~al.(2020)D'Amour, Heller, Moldovan, Adlam, Alipanahi,
  Beutel, Chen, Deaton, Eisenstein, Hoffman, et~al.]{d2020underspecification}
A.~D'Amour, K.~Heller, D.~Moldovan, B.~Adlam, B.~Alipanahi, A.~Beutel, C.~Chen,
  J.~Deaton, J.~Eisenstein, M.~D. Hoffman, et~al.
\newblock Underspecification presents challenges for credibility in modern
  machine learning.
\newblock \emph{Journal of Machine Learning Research}, 2020.
\newblock forthcoming.

\bibitem[Deshpande and Montanari(2016)]{10.5555/2946645.3007094}
Y.~Deshpande and A.~Montanari.
\newblock Sparse pca via covariance thresholding.
\newblock \emph{Journal of Machine Learning Research}, 17\penalty0
  (141):\penalty0 1--41, 2016.
\newblock URL \url{http://jmlr.org/papers/v17/15-160.html}.

\bibitem[Dobriban and Wager(2018)]{dobriban2018high}
E.~Dobriban and S.~Wager.
\newblock {High-dimensional asymptotics of prediction: Ridge regression and
  classification}.
\newblock \emph{The Annals of Statistics}, 46\penalty0 (1):\penalty0 247 --
  279, 2018.
\newblock \doi{10.1214/17-AOS1549}.
\newblock URL \url{https://doi.org/10.1214/17-AOS1549}.

\bibitem[Duchi et~al.(2020)Duchi, Hashimoto, and
  Namkoong]{duchi2020distributionally}
J.~Duchi, T.~Hashimoto, and H.~Namkoong.
\newblock Distributionally robust losses for latent covariate mixtures, 2020.

\bibitem[Duchi and Namkoong(2021)]{duchi2020learning}
J.~C. Duchi and H.~Namkoong.
\newblock {Learning models with uniform performance via distributionally robust
  optimization}.
\newblock \emph{The Annals of Statistics}, 49\penalty0 (3):\penalty0 1378 --
  1406, 2021.
\newblock \doi{10.1214/20-AOS2004}.
\newblock URL \url{https://doi.org/10.1214/20-AOS2004}.

\bibitem[Erdos(2019)]{erdos2019matrix}
L.~Erdos.
\newblock The matrix dyson equation and its applications for random matrices.
\newblock \emph{arXiv preprint arXiv:1903.10060}, 2019.

\bibitem[Erd{\H{o}}s et~al.(2012)Erd{\H{o}}s, Yau, and Yin]{erdHos2012bulk}
L.~Erd{\H{o}}s, H.-T. Yau, and J.~Yin.
\newblock Bulk universality for generalized wigner matrices.
\newblock \emph{Probability Theory and Related Fields}, 154\penalty0
  (1-2):\penalty0 341--407, 2012.

\bibitem[Far et~al.(2006)Far, Oraby, Bryc, and Speicher]{far2006spectra}
R.~R. Far, T.~Oraby, W.~Bryc, and R.~Speicher.
\newblock Spectra of large block matrices.
\newblock \emph{arXiv preprint cs/0610045}, 2006.

\bibitem[Glorot et~al.(2011)Glorot, Bordes, and Bengio]{glorot2011domain}
X.~Glorot, A.~Bordes, and Y.~Bengio.
\newblock Domain adaptation for large-scale sentiment classification: A deep
  learning approach.
\newblock In \emph{Proceedings of the 28th International Conference on
  International Conference on Machine Learning}, ICML'11, page 513–520,
  Madison, WI, USA, 2011. Omnipress.
\newblock ISBN 9781450306195.

\bibitem[Goldt et~al.(2020)Goldt, Reeves, M{\'e}zard, Krzakala, and
  Zdeborov{\'a}]{goldt2020gaussian}
S.~Goldt, G.~Reeves, M.~M{\'e}zard, F.~Krzakala, and L.~Zdeborov{\'a}.
\newblock The gaussian equivalence of generative models for learning with
  two-layer neural networks.
\newblock \emph{arXiv preprint 2006.14709}, 2020.

\bibitem[Goodfellow et~al.(2016)Goodfellow, Bengio, and
  Courville]{goodfellow2016deep}
I.~Goodfellow, Y.~Bengio, and A.~Courville.
\newblock \emph{Deep learning}.
\newblock MIT press, 2016.

\bibitem[Goodfellow et~al.(2015)Goodfellow, Shlens, and
  Szegedy]{goodfellow2014explaining}
I.~J. Goodfellow, J.~Shlens, and C.~Szegedy.
\newblock Explaining and harnessing adversarial examples.
\newblock In Y.~Bengio and Y.~LeCun, editors, \emph{3rd International
  Conference on Learning Representations, {ICLR} 2015, San Diego, CA, USA, May
  7-9, 2015, Conference Track Proceedings}, 2015.
\newblock URL \url{http://arxiv.org/abs/1412.6572}.

\bibitem[Grimmett(1999)]{grimmett1999percolation}
G.~Grimmett.
\newblock Percolation.
\newblock In \emph{Percolation}, pages 1--31. Springer, 1999.

\bibitem[Hastie et~al.(2019)Hastie, Montanari, Rosset, and
  Tibshirani]{hastie2019surprises}
T.~Hastie, A.~Montanari, S.~Rosset, and R.~J. Tibshirani.
\newblock Surprises in high-dimensional ridgeless least squares interpolation.
\newblock \emph{arXiv preprint arXiv:1903.08560}, 2019.

\bibitem[Helton et~al.(2018)Helton, Mai, and Speicher]{helton2018applications}
J.~W. Helton, T.~Mai, and R.~Speicher.
\newblock Applications of realizations (aka linearizations) to free
  probability.
\newblock \emph{Journal of Functional Analysis}, 274\penalty0 (1):\penalty0
  1--79, 2018.

\bibitem[Hendrycks et~al.(2021)Hendrycks, Basart, Mu, Kadavath, Wang, Dorundo,
  Desai, Zhu, Parajuli, Guo, Song, Steinhardt, and Gilmer]{hendrycks2020faces}
D.~Hendrycks, S.~Basart, N.~Mu, S.~Kadavath, F.~Wang, E.~Dorundo, R.~Desai,
  T.~Zhu, S.~Parajuli, M.~Guo, D.~Song, J.~Steinhardt, and J.~Gilmer.
\newblock The many faces of robustness: A critical analysis of
  out-of-distribution generalization.
\newblock In \emph{Proceedings of the IEEE/CVF International Conference on
  Computer Vision (ICCV)}, pages 8340--8349, October 2021.

\bibitem[Ioffe and Szegedy(2015)]{ioffe2015batch}
S.~Ioffe and C.~Szegedy.
\newblock Batch normalization: Accelerating deep network training by reducing
  internal covariate shift.
\newblock In \emph{Proceedings of the 32nd International Conference on
  International Conference on Machine Learning - Volume 37}, ICML'15, page
  448–456. JMLR.org, 2015.

\bibitem[Jacot et~al.(2018)Jacot, Gabriel, and Hongler]{jacot2018neural}
A.~Jacot, F.~Gabriel, and C.~Hongler.
\newblock Neural tangent kernel: Convergence and generalization in neural
  networks.
\newblock In \emph{Proceedings of the 32nd International Conference on Neural
  Information Processing Systems}, NIPS'18, page 8580–8589, Red Hook, NY,
  USA, 2018. Curran Associates Inc.

\bibitem[Karoui(2010)]{Karoui2010TheSO}
N.~E. Karoui.
\newblock The spectrum of kernel random matrices.
\newblock \emph{Annals of Statistics}, 38:\penalty0 1--50, 2010.

\bibitem[Kobak et~al.(2020)Kobak, Lomond, and Sanchez]{kobak2020optimal}
D.~Kobak, J.~Lomond, and B.~Sanchez.
\newblock The optimal ridge penalty for real-world high-dimensional data can be
  zero or negative due to the implicit ridge regularization.
\newblock \emph{Journal of Machine Learning Research}, 21\penalty0
  (169):\penalty0 1--16, 2020.
\newblock URL \url{http://jmlr.org/papers/v21/19-844.html}.

\bibitem[Koh et~al.(2021)Koh, Sagawa, Marklund, Xie, Zhang, Balsubramani, Hu,
  Yasunaga, Phillips, Gao, Lee, David, Stavness, Guo, Earnshaw, Haque, Beery,
  Leskovec, Kundaje, Pierson, Levine, Finn, and Liang]{koh2020wilds}
P.~W. Koh, S.~Sagawa, H.~Marklund, S.~M. Xie, M.~Zhang, A.~Balsubramani, W.~Hu,
  M.~Yasunaga, R.~L. Phillips, I.~Gao, T.~Lee, E.~David, I.~Stavness, W.~Guo,
  B.~Earnshaw, I.~Haque, S.~M. Beery, J.~Leskovec, A.~Kundaje, E.~Pierson,
  S.~Levine, C.~Finn, and P.~Liang.
\newblock Wilds: A benchmark of in-the-wild distribution shifts.
\newblock In M.~Meila and T.~Zhang, editors, \emph{Proceedings of the 38th
  International Conference on Machine Learning}, volume 139 of
  \emph{Proceedings of Machine Learning Research}, pages 5637--5664. PMLR,
  18--24 Jul 2021.
\newblock URL \url{https://proceedings.mlr.press/v139/koh21a.html}.

\bibitem[Kumar et~al.(2020)Kumar, Ma, and Liang]{kumar2020understanding}
A.~Kumar, T.~Ma, and P.~Liang.
\newblock Understanding self-training for gradual domain adaptation.
\newblock In H.~D. III and A.~Singh, editors, \emph{Proceedings of the 37th
  International Conference on Machine Learning}, volume 119 of
  \emph{Proceedings of Machine Learning Research}, pages 5468--5479. PMLR,
  13--18 Jul 2020.
\newblock URL \url{https://proceedings.mlr.press/v119/kumar20c.html}.

\bibitem[Lee et~al.(2018)Lee, Bahri, Novak, Schoenholz, Pennington, and
  Sohl{-}Dickstein]{lee2017deep}
J.~Lee, Y.~Bahri, R.~Novak, S.~S. Schoenholz, J.~Pennington, and
  J.~Sohl{-}Dickstein.
\newblock Deep neural networks as gaussian processes.
\newblock In \emph{6th International Conference on Learning Representations,
  {ICLR} 2018, Vancouver, BC, Canada, April 30 - May 3, 2018, Conference Track
  Proceedings}. OpenReview.net, 2018.
\newblock URL \url{https://openreview.net/forum?id=B1EA-M-0Z}.

\bibitem[Lee et~al.(2020)Lee, Schoenholz, Pennington, Adlam, Xiao, Novak, and
  Sohl-Dickstein]{lee2020finite}
J.~Lee, S.~Schoenholz, J.~Pennington, B.~Adlam, L.~Xiao, R.~Novak, and
  J.~Sohl-Dickstein.
\newblock Finite versus infinite neural networks: an empirical study.
\newblock In H.~Larochelle, M.~Ranzato, R.~Hadsell, M.~F. Balcan, and H.~Lin,
  editors, \emph{Advances in Neural Information Processing Systems}, volume~33,
  pages 15156--15172. Curran Associates, Inc., 2020.
\newblock URL
  \url{https://proceedings.neurips.cc/paper/2020/file/ad086f59924fffe0773f8d0ca22ea712-Paper.pdf}.

\bibitem[Lei et~al.(2021)Lei, Hu, and Lee]{lei21a}
Q.~Lei, W.~Hu, and J.~Lee.
\newblock Near-optimal linear regression under distribution shift.
\newblock In M.~Meila and T.~Zhang, editors, \emph{Proceedings of the 38th
  International Conference on Machine Learning}, volume 139 of
  \emph{Proceedings of Machine Learning Research}, pages 6164--6174. PMLR,
  18--24 Jul 2021.
\newblock URL \url{https://proceedings.mlr.press/v139/lei21a.html}.

\bibitem[Liao et~al.(2020)Liao, Couillet, and Mahoney]{liao2020random}
Z.~Liao, R.~Couillet, and M.~W. Mahoney.
\newblock A random matrix analysis of random fourier features: beyond the
  gaussian kernel, a precise phase transition, and the corresponding double
  descent.
\newblock In H.~Larochelle, M.~Ranzato, R.~Hadsell, M.~F. Balcan, and H.~Lin,
  editors, \emph{Advances in Neural Information Processing Systems}, volume~33,
  pages 13939--13950. Curran Associates, Inc., 2020.
\newblock URL
  \url{https://proceedings.neurips.cc/paper/2020/file/a03fa30821986dff10fc66647c84c9c3-Paper.pdf}.

\bibitem[Lin and Dobriban(2020)]{Lin2020WhatCT}
L.~Lin and E.~Dobriban.
\newblock What causes the test error? going beyond bias-variance via anova.
\newblock \emph{arXiv preprint 2010.05170}, 2020.

\bibitem[Long et~al.(2016)Long, Zhu, Wang, and Jordan]{long2016unsupervised}
M.~Long, H.~Zhu, J.~Wang, and M.~I. Jordan.
\newblock Unsupervised domain adaptation with residual transfer networks.
\newblock In \emph{Proceedings of the 30th International Conference on Neural
  Information Processing Systems}, NIPS'16, page 136–144, Red Hook, NY, USA,
  2016. Curran Associates Inc.
\newblock ISBN 9781510838819.

\bibitem[Long et~al.(2017)Long, Zhu, Wang, and Jordan]{pmlr-v70-long17a}
M.~Long, H.~Zhu, J.~Wang, and M.~I. Jordan.
\newblock Deep transfer learning with joint adaptation networks.
\newblock In D.~Precup and Y.~W. Teh, editors, \emph{Proceedings of the 34th
  International Conference on Machine Learning}, volume~70 of \emph{Proceedings
  of Machine Learning Research}, pages 2208--2217, International Convention
  Centre, Sydney, Australia, 06--11 Aug 2017. PMLR.
\newblock URL \url{http://proceedings.mlr.press/v70/long17a.html}.

\bibitem[Louart et~al.(2018)Louart, Liao, Couillet, et~al.]{louart2018random}
C.~Louart, Z.~Liao, R.~Couillet, et~al.
\newblock A random matrix approach to neural networks.
\newblock \emph{The Annals of Applied Probability}, 28\penalty0 (2):\penalty0
  1190--1248, 2018.

\bibitem[Madry et~al.(2018)Madry, Makelov, Schmidt, Tsipras, and
  Vladu]{madry2017towards}
A.~Madry, A.~Makelov, L.~Schmidt, D.~Tsipras, and A.~Vladu.
\newblock Towards deep learning models resistant to adversarial attacks.
\newblock In \emph{6th International Conference on Learning Representations,
  {ICLR} 2018, Vancouver, BC, Canada, April 30 - May 3, 2018, Conference Track
  Proceedings}. OpenReview.net, 2018.
\newblock URL \url{https://openreview.net/forum?id=rJzIBfZAb}.

\bibitem[Mania and Sra(2021)]{mania2021classifier}
H.~Mania and S.~Sra.
\newblock Why do classifier accuracies show linear trends under distribution
  shift?, 2021.

\bibitem[Mei and Montanari(2021)]{mei2019generalization}
S.~Mei and A.~Montanari.
\newblock The generalization error of random features regression: Precise
  asymptotics and the double descent curve.
\newblock \emph{Communications on Pure and Applied Mathematics}, 06 2021.
\newblock \doi{10.1002/cpa.22008}.

\bibitem[Mel and Ganguli(2021)]{mel2021theory}
G.~Mel and S.~Ganguli.
\newblock A theory of high dimensional regression with arbitrary correlations
  between input features and target functions: sample complexity, multiple
  descent curves and a hierarchy of phase transitions.
\newblock In M.~Meila and T.~Zhang, editors, \emph{Proceedings of the 38th
  International Conference on Machine Learning}, volume 139 of
  \emph{Proceedings of Machine Learning Research}, pages 7578--7587. PMLR,
  18--24 Jul 2021.
\newblock URL \url{https://proceedings.mlr.press/v139/mel21a.html}.

\bibitem[Miller et~al.(2020)Miller, Krauth, Recht, and
  Schmidt]{pmlr-v119-miller20a}
J.~Miller, K.~Krauth, B.~Recht, and L.~Schmidt.
\newblock The effect of natural distribution shift on question answering
  models.
\newblock In H.~D. III and A.~Singh, editors, \emph{Proceedings of the 37th
  International Conference on Machine Learning}, volume 119 of
  \emph{Proceedings of Machine Learning Research}, pages 6905--6916. PMLR,
  13--18 Jul 2020.
\newblock URL \url{http://proceedings.mlr.press/v119/miller20a.html}.

\bibitem[Miller et~al.(2021)Miller, Taori, Raghunathan, Sagawa, Koh, Shankar,
  Liang, Carmon, and Schmidt]{pmlr-v139-miller21b}
J.~P. Miller, R.~Taori, A.~Raghunathan, S.~Sagawa, P.~W. Koh, V.~Shankar,
  P.~Liang, Y.~Carmon, and L.~Schmidt.
\newblock Accuracy on the line: on the strong correlation between
  out-of-distribution and in-distribution generalization.
\newblock In M.~Meila and T.~Zhang, editors, \emph{Proceedings of the 38th
  International Conference on Machine Learning}, volume 139 of
  \emph{Proceedings of Machine Learning Research}, pages 7721--7735. PMLR,
  18--24 Jul 2021.
\newblock URL \url{https://proceedings.mlr.press/v139/miller21b.html}.

\bibitem[Mingo and Speicher(2017)]{mingo2017free}
J.~A. Mingo and R.~Speicher.
\newblock \emph{Free probability and random matrices}, volume~35.
\newblock Springer, 2017.

\bibitem[Nado et~al.(2020)Nado, Padhy, Sculley, D'Amour, Lakshminarayanan, and
  Snoek]{nado2020evaluating}
Z.~Nado, S.~Padhy, D.~Sculley, A.~D'Amour, B.~Lakshminarayanan, and J.~Snoek.
\newblock Evaluating prediction-time batch normalization for robustness under
  covariate shift.
\newblock \emph{arXiv preprint arXiv:2006.10963}, 2020.

\bibitem[Nagarajan and Kolter(2019)]{nagarajan2019uniform}
V.~Nagarajan and J.~Z. Kolter.
\newblock Uniform convergence may be unable to explain generalization in deep
  learning.
\newblock In H.~Wallach, H.~Larochelle, A.~Beygelzimer, F.~d~Alch\'{e}-Buc,
  E.~Fox, and R.~Garnett, editors, \emph{Advances in Neural Information
  Processing Systems}, volume~32. Curran Associates, Inc., 2019.
\newblock URL
  \url{https://proceedings.neurips.cc/paper/2019/file/05e97c207235d63ceb1db43c60db7bbb-Paper.pdf}.

\bibitem[Neal(1996)]{neal1996priors}
R.~M. Neal.
\newblock Priors for infinite networks.
\newblock In \emph{Bayesian Learning for Neural Networks}, pages 29--53.
  Springer, 1996.

\bibitem[Neyshabur et~al.(2017)Neyshabur, Bhojanapalli, Mcallester, and
  Srebro]{neyshabur2017exploring}
B.~Neyshabur, S.~Bhojanapalli, D.~Mcallester, and N.~Srebro.
\newblock Exploring generalization in deep learning.
\newblock In I.~Guyon, U.~V. Luxburg, S.~Bengio, H.~Wallach, R.~Fergus,
  S.~Vishwanathan, and R.~Garnett, editors, \emph{Advances in Neural
  Information Processing Systems}, volume~30. Curran Associates, Inc., 2017.
\newblock URL
  \url{https://proceedings.neurips.cc/paper/2017/file/10ce03a1ed01077e3e289f3e53c72813-Paper.pdf}.

\bibitem[Ovadia et~al.(2019)Ovadia, Fertig, Ren, Nado, Sculley, Nowozin,
  Dillon, Lakshminarayanan, and Snoek]{ovadia2019can}
Y.~Ovadia, E.~Fertig, J.~Ren, Z.~Nado, D.~Sculley, S.~Nowozin, J.~Dillon,
  B.~Lakshminarayanan, and J.~Snoek.
\newblock Can you trust your models uncertainty? evaluating predictive
  uncertainty under dataset shift.
\newblock In H.~Wallach, H.~Larochelle, A.~Beygelzimer, F.~d~Alch\'{e}-Buc,
  E.~Fox, and R.~Garnett, editors, \emph{Advances in Neural Information
  Processing Systems}, volume~32. Curran Associates, Inc., 2019.
\newblock URL
  \url{https://proceedings.neurips.cc/paper/2019/file/8558cb408c1d76621371888657d2eb1d-Paper.pdf}.

\bibitem[P{\'e}ch{\'e} et~al.(2019)]{peche2019note}
S.~P{\'e}ch{\'e} et~al.
\newblock A note on the pennington-worah distribution.
\newblock \emph{Electronic Communications in Probability}, 24, 2019.

\bibitem[Pennington and Worah(2017)]{pennington2019nonlinear}
J.~Pennington and P.~Worah.
\newblock Nonlinear random matrix theory for deep learning.
\newblock In I.~Guyon, U.~V. Luxburg, S.~Bengio, H.~Wallach, R.~Fergus,
  S.~Vishwanathan, and R.~Garnett, editors, \emph{Advances in Neural
  Information Processing Systems}, volume~30. Curran Associates, Inc., 2017.
\newblock URL
  \url{https://proceedings.neurips.cc/paper/2017/file/0f3d014eead934bbdbacb62a01dc4831-Paper.pdf}.

\bibitem[Pennington and Worah(2018)]{pennington2018spectrum}
J.~Pennington and P.~Worah.
\newblock The spectrum of the fisher information matrix of a
  single-hidden-layer neural network.
\newblock In S.~Bengio, H.~Wallach, H.~Larochelle, K.~Grauman, N.~Cesa-Bianchi,
  and R.~Garnett, editors, \emph{Advances in Neural Information Processing
  Systems}, volume~31. Curran Associates, Inc., 2018.
\newblock URL
  \url{https://proceedings.neurips.cc/paper/2018/file/18bb68e2b38e4a8ce7cf4f6b2625768c-Paper.pdf}.

\bibitem[Rahimi and Recht(2008)]{rahimi2007random}
A.~Rahimi and B.~Recht.
\newblock Random features for large-scale kernel machines.
\newblock In J.~Platt, D.~Koller, Y.~Singer, and S.~Roweis, editors,
  \emph{Advances in Neural Information Processing Systems}, volume~20. Curran
  Associates, Inc., 2008.
\newblock URL
  \url{https://proceedings.neurips.cc/paper/2007/file/013a006f03dbc5392effeb8f18fda755-Paper.pdf}.

\bibitem[Recht et~al.(2018)Recht, Roelofs, Schmidt, and
  Shankar]{recht2018cifar}
B.~Recht, R.~Roelofs, L.~Schmidt, and V.~Shankar.
\newblock Do cifar-10 classifiers generalize to cifar-10?
\newblock \emph{arXiv preprint arXiv:1806.00451}, 2018.

\bibitem[Recht et~al.(2019)Recht, Roelofs, Schmidt, and
  Shankar]{recht2019imagenet}
B.~Recht, R.~Roelofs, L.~Schmidt, and V.~Shankar.
\newblock Do imagenet classifiers generalize to imagenet?
\newblock In K.~Chaudhuri and R.~Salakhutdinov, editors, \emph{Proceedings of
  the 36th International Conference on Machine Learning, {ICML} 2019, 9-15 June
  2019, Long Beach, California, {USA}}, volume~97 of \emph{Proceedings of
  Machine Learning Research}, pages 5389--5400. {PMLR}, 2019.
\newblock URL \url{http://proceedings.mlr.press/v97/recht19a.html}.

\bibitem[Richards et~al.(2021)Richards, Mourtada, and
  Rosasco]{pmlr-v130-richards21b}
D.~Richards, J.~Mourtada, and L.~Rosasco.
\newblock Asymptotics of ridge(less) regression under general source condition.
\newblock In A.~Banerjee and K.~Fukumizu, editors, \emph{Proceedings of The
  24th International Conference on Artificial Intelligence and Statistics},
  volume 130 of \emph{Proceedings of Machine Learning Research}, pages
  3889--3897. PMLR, 13--15 Apr 2021.
\newblock URL \url{http://proceedings.mlr.press/v130/richards21b.html}.

\bibitem[Rosen(1997)]{von1988moments}
D.~V. Rosen.
\newblock On moments of the inverted wishart distribution.
\newblock \emph{Statistics}, 30\penalty0 (3):\penalty0 259--278, 1997.
\newblock \doi{10.1080/02331889708802613}.
\newblock URL \url{https://doi.org/10.1080/02331889708802613}.

\bibitem[Sagawa et~al.(2020)Sagawa, Koh, Hashimoto, and
  Liang]{sagawa2020distributionally}
S.~Sagawa, P.~W. Koh, T.~B. Hashimoto, and P.~Liang.
\newblock Distributionally robust neural networks.
\newblock In \emph{8th International Conference on Learning Representations,
  {ICLR} 2020, Addis Ababa, Ethiopia, April 26-30, 2020}. OpenReview.net, 2020.
\newblock URL \url{https://openreview.net/forum?id=ryxGuJrFvS}.

\bibitem[Schmidt et~al.(2018)Schmidt, Santurkar, Tsipras, Talwar, and
  Madry]{schmidt2018adversarially}
L.~Schmidt, S.~Santurkar, D.~Tsipras, K.~Talwar, and A.~Madry.
\newblock Adversarially robust generalization requires more data.
\newblock In S.~Bengio, H.~Wallach, H.~Larochelle, K.~Grauman, N.~Cesa-Bianchi,
  and R.~Garnett, editors, \emph{Advances in Neural Information Processing
  Systems}, volume~31. Curran Associates, Inc., 2018.
\newblock URL
  \url{https://proceedings.neurips.cc/paper/2018/file/f708f064faaf32a43e4d3c784e6af9ea-Paper.pdf}.

\bibitem[Shankar et~al.(2020)Shankar, Fang, Guo, Fridovich-Keil, Ragan-Kelley,
  Schmidt, and Recht]{pmlr-v119-shankar20a}
V.~Shankar, A.~Fang, W.~Guo, S.~Fridovich-Keil, J.~Ragan-Kelley, L.~Schmidt,
  and B.~Recht.
\newblock Neural kernels without tangents.
\newblock In H.~D. III and A.~Singh, editors, \emph{Proceedings of the 37th
  International Conference on Machine Learning}, volume 119 of
  \emph{Proceedings of Machine Learning Research}, pages 8614--8623. PMLR,
  13--18 Jul 2020.
\newblock URL \url{http://proceedings.mlr.press/v119/shankar20a.html}.

\bibitem[Sugiyama et~al.(2007)Sugiyama, Krauledat, and
  M{{\"u}}ller]{sugiyama2007covariate}
M.~Sugiyama, M.~Krauledat, and K.-R. M{{\"u}}ller.
\newblock Covariate shift adaptation by importance weighted cross validation.
\newblock \emph{Journal of Machine Learning Research}, 8\penalty0
  (35):\penalty0 985--1005, 2007.
\newblock URL \url{http://jmlr.org/papers/v8/sugiyama07a.html}.

\bibitem[Taori et~al.(2020)Taori, Dave, Shankar, Carlini, Recht, and
  Schmidt]{taori2020measuring}
R.~Taori, A.~Dave, V.~Shankar, N.~Carlini, B.~Recht, and L.~Schmidt.
\newblock Measuring robustness to natural distribution shifts in image
  classification.
\newblock In H.~Larochelle, M.~Ranzato, R.~Hadsell, M.~F. Balcan, and H.~Lin,
  editors, \emph{Advances in Neural Information Processing Systems}, volume~33,
  pages 18583--18599. Curran Associates, Inc., 2020.
\newblock URL
  \url{https://proceedings.neurips.cc/paper/2020/file/d8330f857a17c53d217014ee776bfd50-Paper.pdf}.

\bibitem[Tripuraneni et~al.(2021)Tripuraneni, Adlam, and
  Pennington]{tripuraneni2021over}
N.~Tripuraneni, B.~Adlam, and J.~Pennington.
\newblock Overparameterization improves robustness to covariate shift in high
  dimensions.
\newblock In \emph{Advances in Neural Information Processing Systems},
  volume~35, 2021.
\newblock forthcoming.

\bibitem[van Wieringen(2015)]{Wieringen2015LectureNO}
W.~N. van Wieringen.
\newblock Lecture notes on ridge regression.
\newblock \emph{arXiv: Methodology}, 2015.

\bibitem[Vershynin(2018)]{vershynin2018high}
R.~Vershynin.
\newblock \emph{High-Dimensional Probability: An Introduction with Applications
  in Data Science}.
\newblock Cambridge Series in Statistical and Probabilistic Mathematics.
  Cambridge University Press, 2018.
\newblock \doi{10.1017/9781108231596}.

\bibitem[Wainwright(2019)]{wainwright2019high}
M.~J. Wainwright.
\newblock \emph{High-Dimensional Statistics: A Non-Asymptotic Viewpoint}.
\newblock Cambridge Series in Statistical and Probabilistic Mathematics.
  Cambridge University Press, 2019.
\newblock \doi{10.1017/9781108627771}.

\bibitem[Wu and Xu(2020)]{wu2020optimal}
D.~Wu and J.~Xu.
\newblock On the optimal weighted $\ell_2$ regularization in overparameterized
  linear regression.
\newblock In H.~Larochelle, M.~Ranzato, R.~Hadsell, M.~F. Balcan, and H.~Lin,
  editors, \emph{Advances in Neural Information Processing Systems}, volume~33,
  pages 10112--10123. Curran Associates, Inc., 2020.
\newblock URL
  \url{https://proceedings.neurips.cc/paper/2020/file/72e6d3238361fe70f22fb0ac624a7072-Paper.pdf}.

\bibitem[Yang et~al.(2021)Yang, Bai, and Mei]{yang2021exact}
Z.~Yang, Y.~Bai, and S.~Mei.
\newblock Exact gap between generalization error and uniform convergence in
  random feature models.
\newblock In M.~Meila and T.~Zhang, editors, \emph{Proceedings of the 38th
  International Conference on Machine Learning, {ICML} 2021, 18-24 July 2021,
  Virtual Event}, volume 139 of \emph{Proceedings of Machine Learning
  Research}, pages 11704--11715. {PMLR}, 2021.
\newblock URL \url{http://proceedings.mlr.press/v139/yang21a.html}.

\bibitem[Zhao et~al.(2018)Zhao, Zhang, Wu, Moura, Costeira, and
  Gordon]{zhao2018adversarial}
H.~Zhao, S.~Zhang, G.~Wu, J.~M.~F. Moura, J.~P. Costeira, and G.~J. Gordon.
\newblock Adversarial multiple source domain adaptation.
\newblock In S.~Bengio, H.~Wallach, H.~Larochelle, K.~Grauman, N.~Cesa-Bianchi,
  and R.~Garnett, editors, \emph{Advances in Neural Information Processing
  Systems}, volume~31. Curran Associates, Inc., 2018.
\newblock URL
  \url{https://proceedings.neurips.cc/paper/2018/file/717d8b3d60d9eea997b35b02b6a4e867-Paper.pdf}.

\bibitem[Zhao et~al.(2019)Zhao, Combes, Zhang, and Gordon]{zhao2019learning}
H.~Zhao, R.~T.~D. Combes, K.~Zhang, and G.~Gordon.
\newblock On learning invariant representations for domain adaptation.
\newblock In K.~Chaudhuri and R.~Salakhutdinov, editors, \emph{Proceedings of
  the 36th International Conference on Machine Learning}, volume~97 of
  \emph{Proceedings of Machine Learning Research}, pages 7523--7532. PMLR,
  09--15 Jun 2019.
\newblock URL \url{https://proceedings.mlr.press/v97/zhao19a.html}.

\end{thebibliography}
\end{document}